\documentclass[10pt,journal,compsoc]{IEEEtran}

\usepackage[nocompress]{cite}
\usepackage[pdftex]{graphicx}
\DeclareGraphicsExtensions{.pdf,.jpeg,.png}
\usepackage[cmex10]{amsmath}
\usepackage{amssymb}
\usepackage{amsthm}
\usepackage{amsfonts}
\usepackage[ruled,linesnumbered,titlenumbered]{algorithm2e}
\usepackage{multirow}
\usepackage{xspace}
\usepackage{xcolor}
\usepackage{array}
\usepackage{diagbox}
\usepackage{captdef}
\usepackage{url}

\newcommand*{\eg}{\emph{e.g.}\@\xspace}
\newcommand*{\ie}{\emph{i.e.}\@\xspace}
\newcommand*{\etal}{\emph{et al.}\@\xspace}

\newcommand{\argmin}{\operatornamewithlimits{argmin}}

\def\myP{{\Omega}}
\def\myN{\mathcal{N}}
\def\myS{\mathcal{L}}

\newcommand{\eref}[1]{Eq.~(\ref{eq:#1})}
\newcommand{\tref}[1]{Table~\ref{tab:#1}}
\newcommand{\wrt}{\textit{w.r.t.}\@\xspace}

\newtheorem{lemma}{Lemma}

\definecolor{mossgreen}{HTML}{308014}

\def\capskip{\vskip -1.1mm}

\begin{document}
\title{Continuous 3D Label Stereo Matching using \\Local Expansion Moves}

\author{Tatsunori~Taniai,
	Yasuyuki~Matsushita,
	Yoichi~Sato,
	and~Takeshi~Naemura,
	\IEEEcompsocitemizethanks{\IEEEcompsocthanksitem T. Taniai, T. Naemura, and Y. Sato are with the University of Tokyo, Japan.
		\IEEEcompsocthanksitem Y. Matsushita is with Osaka University, Japan.}% <-this % stops a space
	\thanks{Report submitted Mar. 28, 2016; revised Jul. 5 and Oct. 17, 2017.}}

% The paper headers
\markboth{Technical Report 2017. Taniai \MakeLowercase{\etal}: Continuous 3D Label Stereo Matching using Local Expansion Moves.}{}

\IEEEtitleabstractindextext{%
	\begin{abstract}
		We present an accurate stereo matching method using \emph{local expansion moves} based on graph cuts. %, a new move making scheme using graph cuts.
		This new move-making scheme is used to efficiently infer per-pixel 3D plane labels on a pairwise Markov random field (MRF) that effectively combines recently proposed slanted patch matching and curvature regularization terms.
		The local expansion moves are presented as many $\alpha$-expansions defined for small grid regions.
		The local expansion moves extend traditional expansion moves by two ways: localization and spatial propagation. By localization, we use different candidate $\alpha$-labels according to the locations of local $\alpha$-expansions. By spatial propagation, we design our local $\alpha$-expansions to propagate currently assigned labels for nearby regions.
		With this localization and spatial propagation, our method can efficiently infer MRF models with a continuous label space using  randomized search.
		Our method has several advantages over previous approaches that are based on fusion moves or belief propagation; it produces \emph{submodular moves} deriving a \emph{subproblem optimality}; %it enables powerful randomized search; 
		it helps find good, smooth, piecewise linear disparity maps; it is suitable for parallelization; it can use cost-volume filtering techniques for accelerating the matching cost computations.
		Even using a simple pairwise MRF, our method is shown to have best performance in  the Middlebury stereo benchmark V2 and V3.
	\end{abstract}
	\begin{IEEEkeywords}
		Stereo Vision, 3D Reconstruction, Graph Cuts, Markov Random Fields, Discrete-Continuous Optimization.
\end{IEEEkeywords}}

% make the title area
\maketitle
\IEEEdisplaynontitleabstractindextext
\IEEEpeerreviewmaketitle

\IEEEraisesectionheading{\section{Introduction}\label{sec:introduction}}
\IEEEPARstart{S}{tereo} vision often struggles with a bias toward reconstructing fronto-parallel surfaces, which can stem from matching cost, smoothness regularization, or even inference~\cite{Bleyer11,Woodford09}.

Segment-based stereo~\cite{Birchfield99} that represents disparity maps by disparity plane segments, could ease the bias issue but recovered surfaces are constrained to be piecewise planar.
%can only recover 3D scenes up to piecewise planar approximations.
%Birchfield and Tomasi~\cite{Birchfield99} propose to use disparity plane segments to represent disparity maps.
%This segment-based stereo could ease the issue but can only reconstruct 3D scenes up to piecewise planar approximations.
Recently, two breakthroughs have been independently made that overcome the fronto-parallel bias while surpassing the limitation of segment-based stereo.

%More recently, slanted patch matching~\cite{Bleyer11,Besse12,Lu13,Olsson13,Heise13} has made a breakthrough overcoming the fronto-parallel bias in matching cost.

One breakthrough is matching cost using slanted patch matching~\cite{Bleyer11}.
In this approach, the disparity $d_p$ of each pixel $p$ is over-parameterized by a local disparity plane
\begin{equation}
d_p = a_pu + b_pv + c_p \label{eq:disparity_plane}
\end{equation}
defined on the image domain $(u, v)$, and the triplet $(a_p, b_p, c_p)$ is estimated for each pixel $p$ instead of directly estimating $d_p$.
The matching window is then slanted according to Eq.~(\ref{eq:disparity_plane}),  which produces linearly-varying disparities within the window and thereby  measures patch dissimilarity accurately even with large matching windows.
%achieves accurate photo-consistency measures between matching pixels even with large matching windows.

Another key invention is curvature regularization~\cite{Olsson13,Li10} by tangent plane labels.
Unlike second-order smoothness~\cite{Woodford09} that forms a higher-order term, this curvature regularization is nicely represented by pairwise terms like conventional (fronto-parallel) linear~\cite{Ishikawa03} and truncated linear models, and can handle smooth surfaces beyond planes.

%\\ Apart from slanted patch matching, Olsson~\etal~\cite{Olsson13} propose a curvature regularization term that can overcome a front-parallel bias in the form of a pairwise term using disparity plane labels.
Given slanted patch matching~\cite{Bleyer11} and {tangent-based} curvature regularization~\cite{Olsson13}, the use of 3D plane labels allows us to establish a new stereo model using a pairwise Markov random field (MRF)~\cite{Geman84} that is free from the fronto-parallel bias. 
%we now arrive at a new stereo model using a pairwise Markov random field (MRF)~\cite{Geman84} with dense 3D plane labels that is free from the fronto-parallel bias.
However, while stereo with standard 1D discrete disparity labels~\cite{LiangWang11,Kolmogorov01,Boykov01}
can be directly solved by discrete optimizers such as graph cuts (GC)~\cite{Kolmogorov04,Boykov04} and belief propagation (BP)~\cite{Yedidia00,Felzenszwalb04},
such approaches cannot be directly used for continuous 3D labels due to the huge or infinite label space $(a, b, c) \in \mathbb{R}^3$.

\begin{figure}[!t]
	\centering
	\begin{minipage}{0.23\linewidth}
		\begin{center}
			\includegraphics[width=\linewidth]{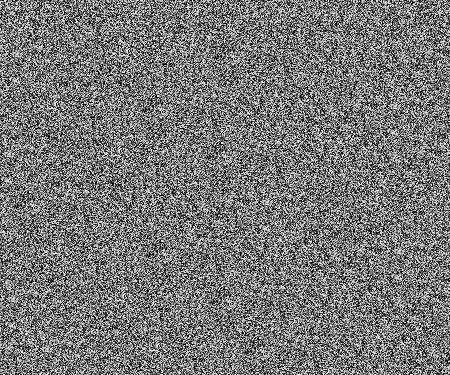}
		\end{center}
	\end{minipage}
	\hfill
	\begin{minipage}{0.23\linewidth}
		\begin{center}
			\includegraphics[width=\linewidth]{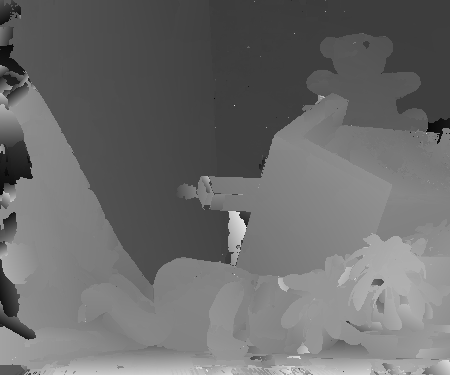}
		\end{center}
	\end{minipage}
	\hfill
	\begin{minipage}{0.23\linewidth}
		\begin{center}
			\includegraphics[width=\linewidth]{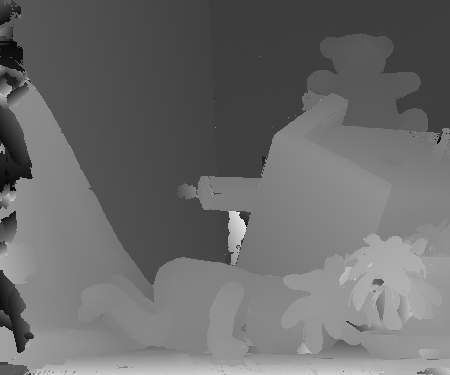}
		\end{center}
	\end{minipage}
	\hfill
	\begin{minipage}{0.23\linewidth}
		\begin{center}
			\includegraphics[width=\linewidth]{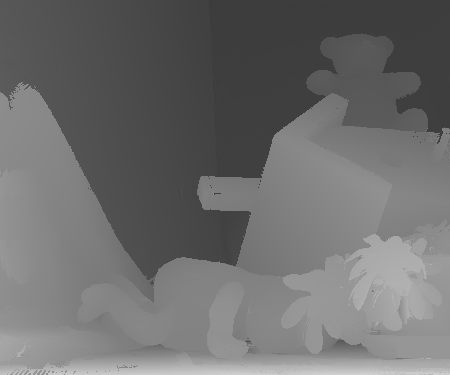}
		\end{center}
	\end{minipage}
	\vskip 0.5mm
	\begin{minipage}{0.23\linewidth}
		\begin{center}
			\includegraphics[width=\linewidth]{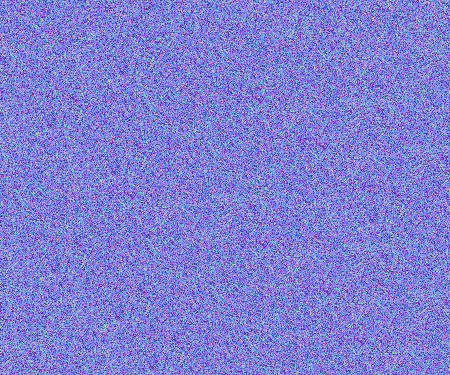}
		\end{center}
	\end{minipage}
	\hfill
	\begin{minipage}{0.23\linewidth}
		\begin{center}
			\includegraphics[width=\linewidth]{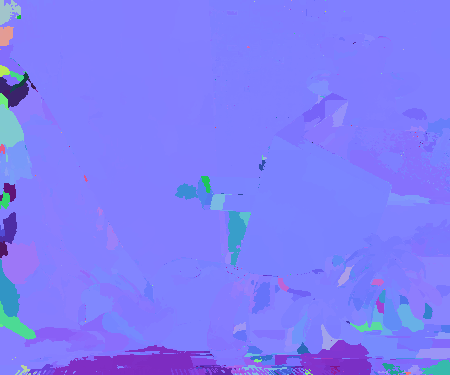}
		\end{center}
	\end{minipage}
	\hfill
	\begin{minipage}{0.23\linewidth}
		\begin{center}
			\includegraphics[width=\linewidth]{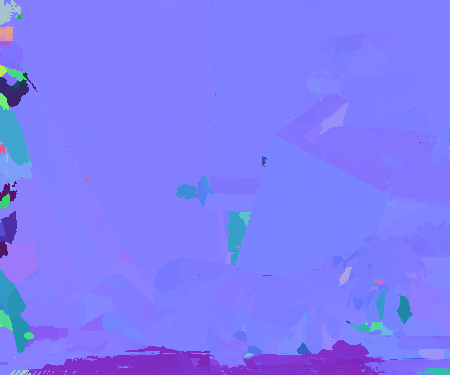}
		\end{center}
	\end{minipage}
	\hfill
	\begin{minipage}{0.23\linewidth}
		\begin{center}
			\includegraphics[width=\linewidth]{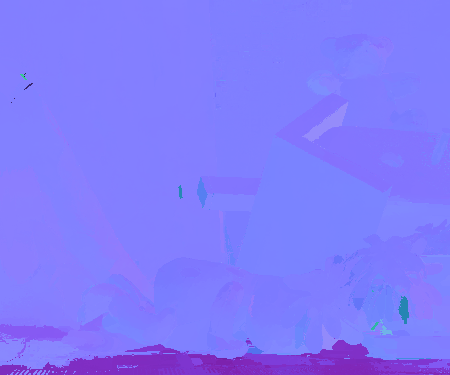}
		\end{center}
	\end{minipage}
	\vskip 0.5mm
	\begin{minipage}{0.23\linewidth}
		\begin{center}
			\includegraphics[width=\linewidth]{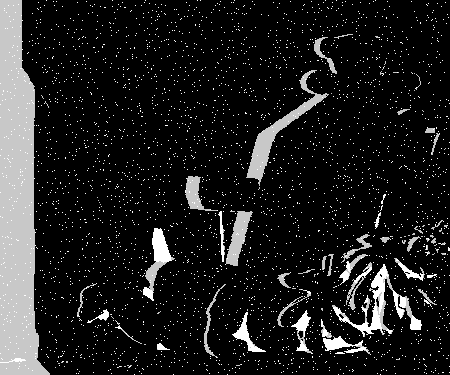}
		\end{center}
	\end{minipage}
	\hfill
	\begin{minipage}{0.23\linewidth}
		\begin{center}
			\includegraphics[width=\hsize]{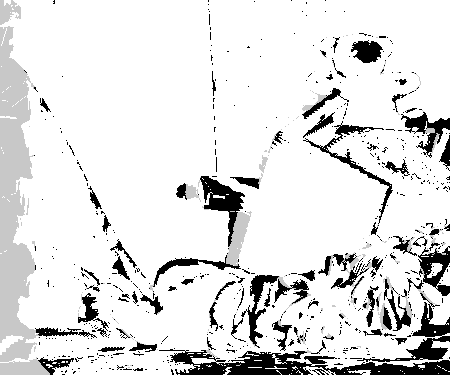}
		\end{center}
	\end{minipage}
	\hfill
	\begin{minipage}{0.23\linewidth}
		\begin{center}
			\includegraphics[width=\hsize]{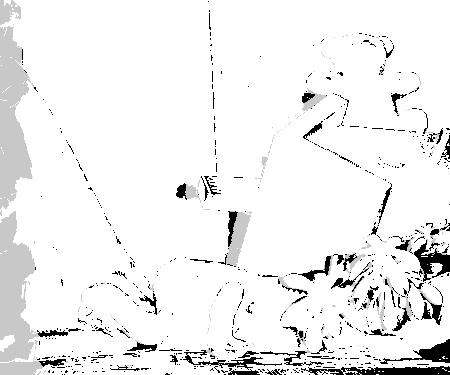}
		\end{center}
	\end{minipage}
	\hfill
	\begin{minipage}{0.23\linewidth}
		\begin{center}
			\includegraphics[width=\hsize]{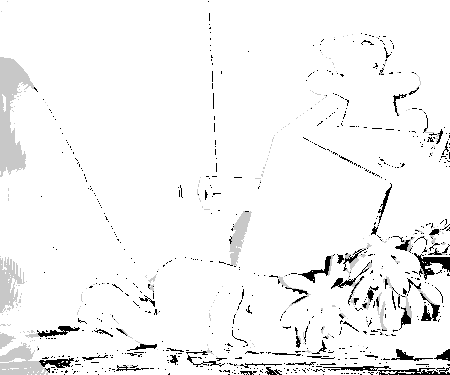}
		\end{center}
	\end{minipage}
	\vskip 0.5mm
	\begin{minipage}{0.23\linewidth}
		\begin{center}
			{\footnotesize Random init.}
		\end{center}
	\end{minipage}
	\hfill
	\begin{minipage}{0.23\linewidth}
		\begin{center}
			{\footnotesize $1$ iteration}
		\end{center}
	\end{minipage}
	\hfill
	\begin{minipage}{0.23\linewidth}
		\begin{center}
			{\footnotesize $2$ iterations}
		\end{center}
	\end{minipage}
	\hfill
	\begin{minipage}{0.23\linewidth}
		\begin{center}
			{\footnotesize Post-processing}
		\end{center}
	\end{minipage}
	\caption{{Evolution of our stereo matching estimates.} From top to bottom, we show disparity maps, normal maps of disparity planes, and error maps with $0.5$ pixel threshold where ground truth is given.
		In our framework, we start with random disparities that are represented by per-pixel 3D planes, \ie, disparities (top) and normals (middle).
		We then iteratively apply our local expansion moves using GC (middles) to update and propagate local disparity planes.
		Finally, the resulting disparity map is further refined at a post-processing stage using left-right consistency check and weighted median filtering (rightmost).}
	\label{fig:process}
\end{figure}

{To efficiently infer 3D labels, recent successful methods~\cite{Bleyer11,Besse12,Lu13,Heise13} use \emph{PatchMatch}~\cite{Barnes09,Barnes10}, an inference algorithm using spatial label propagation. In PatchMatch, each pixel is updated in raster-scan order and its refined label is propagated to next pixels as their candidate labels.}
%\revb{To efficiently infer 3D labels, recent successful methods~\cite{Bleyer11,Besse12,Lu13,Heise13} use \emph{PatchMatch inference}~\cite{Barnes09,Barnes10}, which is an inference algorithm using spatial propagation where each pixel's candidate label is, in raster-scan order, refined and then propagated to next pixels.}
%Recent successful methods~\cite{Bleyer11,Besse12,Lu13,Heise13} use PatchMatch inference~\cite{Barnes09,Barnes10} to efficiently infer correct 3D planes using spatial propagation; each pixel's candidate label is, in raster-scan order, refined and then propagated to next pixels.
Further in~\cite{Besse12}, this sequential algorithm is combined with BP yielding an efficient optimizer PMBP for pairwise MRFs.
In terms of MRF optimization, however, BP is considered a \emph{sequential optimizer}, which improves each variable individually keeping others conditioned at the current state.
In contrast, GC improves all variables simultaneously by accounting for interactions across variables, and this global property helps optimization avoid bad local minimums~\cite{Szeliski08,Woodford09}. 
In order to take advantage of this and efficiently infer 3D planes by GC, it is important to use spatial propagation.
Nevertheless, incorporating such spatial propagation into GC-based optimization is not straightforward, because 
{inference using GC rather processes \emph{all nodes simultaneously}}, not \emph{one-by-one sequentially} like PatchMatch and BP.

In this paper, we introduce a new move making scheme, \emph{local expansion moves}, that enables spatial propagation in GC optimization. 
The local expansion moves are presented as many $\alpha$-expansions~\cite{Boykov01} defined for a small extent of regions at different locations. 
Each of this small or local $\alpha$-expansion tries to improve the current labels in its local region in an energy minimization manner using GC.
Here, those current labels are allowed to move to a candidate label $\alpha$, which is given uniquely to each local $\alpha$-expansion in order to achieve efficient label searching. At the same time, this procedure is designed to propagate a current label in a local region for nearby pixels.
% With this localization and spatial propagation, we can use a powerful randomized search strategy (see \fref{process}), where we no longer need to produce plausible initial solutions as usually done in fusion-based approaches~\cite{Lempitsky10,Woodford08,Woodford09,Olsson13}.
For natural scenes that often exhibit locally planar structures, the joint use of local expansion moves and GC has a useful property. It allows multiple pixels in a local region to be assigned the same disparity plane \emph{by a single min-cut} in order to find a smooth solution. Being able to simultaneously update multiple variables also helps to avoid being trapped at a bad local minimum.

For continuous MRF inference, fusion-based approaches using GC~\cite{Lempitsky10} are often employed with some heuristics to generate disparity map {proposals~\cite{Woodford09,Olsson13,Olsson14}}.
Our work bridges between apparently different inference approaches of PatchMatch and GC, and can take advantage of those heuristics used in both PatchMatch and fusion based approaches leading to higher efficiency and accuracy.

The advantages of our method are as follows.
1) Our local expansion move method produces \emph{submodular moves} that guarantee the optimal labeling at each min-cut (subproblem optimal), which in contrast is not guaranteed in general fusion moves~\cite{Lempitsky10}.
2) This optimality and spatial propagation allow randomized search, rather than employ external methods to generate plausible initial proposals as done in fusion approaches~\cite{Lempitsky10,Woodford09}, which may limit the possible solutions.
3) Our method achieves greater accuracy than BP~\cite{Besse12} thanks to the good properties of GC and local expansion moves.
4) Unlike other PatchMatch based methods~\cite{Bleyer11,Besse12,Heise13}, our method can incorporate the fast cost filtering technique of \cite{Lu13}. In this manner, we can efficiently reduce the computation complexity of unary terms from $O(|W|)$ to approximately $O(1)$, removing dependency from support window size $|W|$. 
5) Unlike PMBP~\cite{Besse12}, our method is well suited for parallelization using both CPU and GPU.\footnote{Although BP is usually parallelizable, PMBP differs from BP's standard settings in that it defines label space \emph{dynamically} for each pixel and \emph{propagates} it; both make parallelization indeed non-trivial.} With multiple CPU cores, each of our local $\alpha$-expansions (\ie, min-cut computations) can be individually performed in parallel. With a GPU implementation,  computations of unary terms can be more efficiently performed in a parallel manner for further acceleration. 

This paper is an extended version of our conference paper~\cite{Taniai14}.
The extensions are summarized as follows.
We add theoretical verifications on the preferability of our method for piecewise planar scenes in Sec.~\ref{sec:slanted_patch_matching}, and also on the subproblem optimality of our algorithm  in Sec.~\ref{sec:localexp}.
Furthermore, the efficiency of our algorithm is improved by two ways;
In Sec.~\ref{sec:localexp} we show the parallelizablity of our local expansion move algorithm;
In Sec.~\ref{sec:fastimp} we show that the fast cost filtering technique of \cite{Lu13} can be used in our method.
The effectiveness of both extensions is thoroughly evaluated in the experiments, and we show that even a CPU implementation of the proposed method achieves about 2.1x faster running times than our previous GPU implementation~\cite{Taniai14}, with comparable or even greater accuracy.
Finally, we add experimental evaluations in Sec.~\ref{sec:experiments} showing that our method outperforms a fusion-based approach~\cite{Olsson13} and all the state-of-the-art methods registered in the latest Middlebury benchmark V3. 
%\rev{Our code is publicly available. \footnote{\rev{\url{https://github.com/t-taniai/LocalExpStereo}}}}
{Our code is publicly available online (\url{https://github.com/t-taniai/LocalExpStereo}).}

\section{Related work}
\subsection{MRF stereo methods}
%MRF stereo methods can be categorized into three approaches: discrete stereo, segment-based stereo, and continuous stereo.
MRF stereo methods can be categorized into thee approaches: discrete stereo, segment-based stereo, and continuous stereo.

Discrete stereo~\cite{LiangWang11,Kolmogorov01,Boykov01} formulates stereo matching as a discrete multi-labeling problem, where each pixel is individually assigned one of pre-defined discrete disparity values. For this problem, many powerful discrete optimizers, such as BP~\cite{Yedidia00,Felzenszwalb04}, TRW~\cite{Kolmogorov06}, and GC~\cite{Kolmogorov04,Boykov04}, can be directly used. Successful results are shown using GC with expansion moves~\cite{Boykov01,Szeliski08}. In expansion moves, the multi-labeling problem is reduced to a sequence of binary-labeling problems, each of which can be exactly solved by GC, if only pairwise potentials $\psi$ meet the following submodularity of expansion moves~\cite{Kolmogorov07,Boykov01}:
\begin{equation}
\psi(\alpha, \alpha) + \psi(\beta, \gamma) \le \psi(\beta, \alpha) + \psi(\alpha, \gamma). \label{eq:axregularity}
\end{equation}

Segment-based stereo~\cite{Birchfield99,Tao01,Hong04,Wang08} assigns a 3D disparity plane for each of over-segmented image regions.
%The candidate planes are generated by fitting planes to a roughly estimated disparity map, and then the optimal assignment of the planes is estimated using combinatorial optimization~\cite{Boykov01,Felzenszwalb04}.
Although this approach yields continuous-valued disparities, it strictly limits the reconstruction to a piecewise planar representation and  is subject to the quality of initial segmentation.
More recent methods alleviate these limitations by using a complex layered MRF~\cite{Bleyer10}, multi-scale segmentation~\cite{Li16}, or jointly estimating segmentation~\cite{Yamaguchi12}.  %Also, results are subject to the quality of the segmentation.~

The last group, to which our method belongs, is continuous stereo~\cite{Woodford09,Bleyer11,Besse12,Olsson13,Lu13,Heise13}, where each pixel is assigned a distinct continuous disparity value.
Some methods~\cite{Woodford09,Olsson13} use fusion moves~\cite{Lempitsky10}, an operation that combines two solutions to make better one (binary fusion) by solving a non-submodular binary-labeling problem using QPBO-GC~\cite{Kolmogorov07}.
In this approach, a number of continuous-valued disparity maps (or so-called \emph{proposals} in the literature~\cite{Lempitsky10}) are first generated by other external methods (\eg, segment-based stereo~\cite{Woodford09}), which are then combined as a sequence of binary fusions.
Our method differs from this fusion-based approach in that we use spatial propagation and randomization search during inference, by which we only require a randomized initial solution instead of those generated by external methods. Also, binary energies produced in our method are always submodular allowing exact inference via GC (subproblem optimal). More importantly, we rather provide an efficient inference mechanism that is conceptually orthogonal to proposal generation. Thus, conventional proposal generation schemes used in fusion {approaches~\cite{Woodford09,Olsson13,Olsson14}} can be incorporated in our method for further improvements.
A stereo method by Bleyer~\etal~\cite{Bleyer11} proposes accurate photo-consistency measures using 3D disparity planes
that are inferred by PatchMatch~\cite{Barnes09,Barnes10}. 
Heise~\etal~\cite{Heise13} incorporate Huber regularization into \cite{Bleyer11} using convex optimization.
Besse~\etal~\cite{Besse12} point out a close relationship between PatchMatch and BP and present a unified method called PatchMatch BP (PMBP) for pairwise continuous MRFs.
%Borrowing the unary term from~\cite{Bleyer11}, their stereo method has shown successful improvements.
PMBP is probably the closest approach to ours in spirit, but we use GC instead of BP for the inference. Therefore, our method is able to take advantage of better convergence of GC~\cite{Szeliski08} for achieving greater accuracy. In addition, our method allows efficient parallel computation of unary matching costs and even min-cut operations. %We show in the experiments that how these two differences make our method advantageous.

\subsection{Cost-volume filtering}
Patch-based stereo methods often use cost-volume filtering for fast implementations.
Generally, computing a matching cost $C$ for a patch requires $O(|W|)$ of computation, where $|W|$ is the size of the patch. However, given a cost-volume slice $\rho_d(p)$ that represents pixelwise raw matching costs \hbox{$\|I(p) - I'(p-d)\|$} for a certain disparity label~$d$, the patch-based matching costs can be efficiently computed by applying a filtering to the cost map as \hbox{$C_d(p)=\sum_{q}  \omega_{pq} \rho_d(q)$}. 
Here, the filter kernel $\omega_{pq}$ represents the matching window at $p$. If we use a constant-time filtering, each matching cost $C_d(p)$ is efficiently computed in $O(1)$.

The box filtering can achieve $O(1)$ by using integral image {but such simple filtering flattens} object boundaries. For this boundary issue, Yoon and Kweon~\cite{Yoon05} propose an adaptive support-window technique that uses the joint bilateral filtering~\cite{Petschnigg04} for cost filtering. Although this adaptive window technique successfully deals with the boundary issue~\cite{Hosni12}, it involves $O(|W|)$ of computation because of the complexity of the bilateral filtering.
Recently, He~\etal~\cite{He13} propose a constant-time edge-aware filtering named the \emph{guided image filtering}. This filtering is employed in a cost-volume filtering method of \cite{Rhemann11,Hosni13}, achieving both edge-awareness and $O(1)$ of matching-cost computation.

In principle, stereo methods using PatchMatch inference~\cite{Bleyer11,Besse12,Heise13} cannot take advantage of the cost filtering acceleration, since  in those methods the candidate disparity labels are given {dynamically} to each pixel and we cannot make a consistent-label cost-volume slice $\rho_d(p)$. 
To this issue, Lu~\etal~\cite{Lu13} extend PatchMatch~\cite{Bleyer11} to use superpixels as a unit of cost calculations. In their method, called PatchMatch filter (PMF), fast cost-volume filtering of \cite{Rhemann11,Hosni13} is applied in small subregions and that approximately achieves $O(1)$ of complexity. 
We will show in Sec.~\ref{sec:fastimp} that their subregion filtering technique can be effectively incorporated into our method, and we achieve greater accuracy than their local stereo method~\cite{Lu13}.

\section{Proposed method}
\label{sec:proposed}
This section describes our proposed method. Given two input images $I_L$ and $I_R$, our purpose is to estimate the disparities of both images.

In Sec.~\ref{sec:slanted_patch_matching}, we first discuss a geometric property of our model. % analyze the formulation of disparity planes used in the slanted patch matching approach~\cite{Bleyer11}, and show its relationship to homographies.
We then define our energy function in Sec.~\ref{sec:formulation}, and describe the fundamental idea of our optimizing strategy and its properties in Sec.~\ref{sec:localexp}.
The whole optimization procedure is presented in Sec.~\ref{sec:optimization}, and we further discuss a fast implementation in Sec.~\ref{sec:fastimp}.

%This section describes the proposed stereo matching method. 

\subsection{Geometric interpretation to slanted patch matching}
\label{sec:slanted_patch_matching}
Slanted patch matching~\cite{Bleyer11} by Eq.~(\ref{eq:disparity_plane})  %is a key element for avoiding a fronto-parallel bias in patch-based stereo matching.
%Specifically, each patch is warped by the following form of affine transformations
%\begin{equation}
%\mathbf{u}' =  \begin{bmatrix}
%1-a & -b & -c\\
%0 & 1 & 0\\
%0 & 0 & 1
%\end{bmatrix} \mathbf{u},
%\end{equation}
%where $\mathbf{u}$ and $\mathbf{u}'$ are homogeneous pixel coordinates in the left and right view images, respectively.
implicitly assumes that the true disparity maps are approximately piecewise linear.
While it is not discussed in \cite{Bleyer11},   linear disparity is exactly related to motion by \emph{planar surfaces}~\cite{Birchfield99,Olsson13}.
Formally, if there exists a plane in the 3D world coordinates $(x, y, z) \in \mathbb{R}^3$ 
\begin{equation}
a'_p x + b'_p y + c'_p z = h'_p \label{eq:geometry_plane}
\end{equation}
parameterized by $(a'_p,b'_p,c'_p,h'_p)$, then motion due to this geometric plane is represented by a disparity plane
\begin{equation}
d(u,v) = \frac{B}{h'_p}\left({a'_p}u + {b'_p}v + fc'_p \right), \label{eq:disp_conversion}
\end{equation}
where $B$ and $f$ are the baseline and focal length of the stereo cameras.\footnote{Derived easily by using the perspective projection $x=uz/f$ and $y=vz/f$, and the depth-disparity relation $z=Bf/d$ in rectified stereo.}
This piecewise planar assumption is reasonable as many natural scenes approximately exhibit locally planar surfaces. 
As we will discuss later, this fact also supports our method design as our model and inference both softly prefer piecewise linear disparity maps while having the ability to express arbitrary curved surfaces.
%Note that enforcing linearity on \emph{depth maps} has an unnatural bias toward distorted surfaces.

\subsection{Formulation}
\label{sec:formulation}
%\rev{We follow the slanted patch matching approach of \cite{Bleyer11}}. 
We use the 3D plane label formulation to take advantage of the powerful slanted patch matching~\cite{Bleyer11} and curvature regularization~\cite{Olsson13}. 
Here, each pixel $p$'s disparity $d_p$ is over-parameterized by a 3D plane \hbox{$d_p = a_p u + b_p v + c_p$}.
Therefore, we seek a mapping \hbox{$f_p = f(p) : \myP \to \myS$} that assigns a disparity plane \hbox{$f_p = (a_p, b_p, c_p) \in \myS$} for every pixel $p$ in the left and right images.
%To estimate $f$, we use a pairwise MRF formulation by following conventional stereo matching methods~\cite{Olsson13,LiangWang11,Kolmogorov01,Boykov01}. In the MRF framework, we seek $f$ such that minimizes 
Following conventional MRF stereo methods~\cite{Olsson13,LiangWang11,Kolmogorov01,Boykov01},
we estimate $f$ by minimizing the following energy function  based on a pairwise MRF.
\begin{equation}
E(f) = \sum_{p \in \myP} \phi_p(f_p) + \lambda \sum_{(p, q) \in \myN} \psi_{pq}(f_p, f_q). \label{eq:mrf}
\end{equation}
The first term, called the \emph{data term} or \emph{unary term}, measures the photo-consistency between matching pixels. The disparity plane $f_p$ defines a warp from a pixel $p$ in one image to its correspondence in the other image. 
%To increase the reliability of photo-consistency measures, nearby pixels  in a support window (or a \emph{patch}) often participate in this measurement~\cite{Hosni12}.
The second term is called the \emph{smoothness term} or \emph{pairwise term}, which penalizes discontinuity of disparities between neighboring pixel pairs $(p, q) \in \myN$. We define these terms {below}.

\subsubsection*{Data term}
To measure photo-consistencies, we use the slanted patch matching term that has been recently proposed by~\cite{Bleyer11}. 
%Here, each pixel $p$'s disparity $d_p$ is over-parameterized by a 3D plane \hbox{$d_p = a_p u + b_p v + c_p$} to avoid the frontal-parallel bias. Therefore, the objective becomes to seek a disparity plane \hbox{$f_p = (a_p, b_p, c_p)^T \in \myS$} for every pixel $p$ in the left and right images such that disparity map $f$ minimizes the energy function $E(f)$ of \eref{mrf}.
%Using this $p$'s disparity plane $f_p$, a pixel \hbox{$s = (s_u, s_v)^T$} in the left image is warped to a new location in the right image by a warping $w_{f_p}$ as
%\begin{equation}
%w_{f_p}(s) = s  -  (a_p s_u + b_p s_v + c_p ,0)^T. 
%\label{eq:warp}
%\end{equation}
The data term of $p$ in the left image is  defined as
\begin{eqnarray}
\phi_p (f_p) = \sum_{s \in W_p} \omega_{ps} \, \rho ( s |f_p). \label{eq:unary_term}
\end{eqnarray}
Here, $W_p$ is a square window centered at $p$. The weight  $\omega_{ps}$ implements the adaptive support window proposed in~\cite{Yoon05}. For this we replace the bilateral filter weight used in \cite{Yoon05,Bleyer11} with more advanced guided image filtering~\cite{He13} as below.
%\begin{equation}
%\omega_{ps} = e^{- \| I_L(p) - I_L(s) \|_1 / \gamma}, 
%\label{eq:weight}
%\end{equation}
%where $\gamma$ is a user-defined parameter, and $\| \cdot \|_1$ represents the ${\ell}_1$-norm.\footnote{We have removed the spatial range weight of \cite{Yoon05} that compares pixel positions. As mentioned in \cite{Bleyer11,Hosni12}, improvement due to this term is minor.} Here, we assume RGB color intensities $I \in [0, 255]^3$.
% Note that this weight $\omega_{ps}$ will be re-defined in Sec.~\ref{sec:fastimp} for a fast implementation.
\begin{equation}
\omega_{ps} = \frac{1}{|W'|^2} \sum_{k:(p,s)\in W'_k} \!\! \left(1 + (I_p - \mu_k)^T(\Sigma_k + e)^{-1}(I_s - \mu_k) \right) \label{eq:gf}
\end{equation}
Here, $I_p = I_L(p)/255$ is a normalized color vector, $\mu_k$ and $\Sigma_k$ are the mean and co-variance matrix of $I_p$ in a local regression window $W'_k$, and $e$ is an identity matrix with a small positive coefficient for avoiding over-fitting.
We will discuss how to efficiently compute this filtering in Sec.~\ref{sec:fastimp}.

Given a disparity plane $f_p=(a_p,b_p,c_p)$,
the function $\rho(s|f_p)$ in Eq.~(\ref{eq:unary_term}) measures the pixel dissimilarity between a support pixel $s=(s_u, s_v)$ in the window $W_p$ and its matching point in the right image
\begin{equation}
s' = s  -  (a_p s_u + b_p s_v + c_p ,0)
\label{eq:warp}
\end{equation}
as
\begin{eqnarray}
\rho(s|f_p) = (1-\alpha) \, \min(\| I_L(s) -I_R(s')  \|_1 , \tau_\text{col}) \notag \\
+ \, \alpha \, \min(| \nabla_x I_L(s) - \nabla_x  I_R(s')  |, \tau_\text{grad}). 
\label{eq:photofunc}
\end{eqnarray}
Here, $\nabla_x I$ represents the $x$-component of the gray-value gradient of image $I$, and $\alpha$ is a factor that balances the weights of color and gradient terms. 
The two terms are truncated by $\tau_\text{col}$ and $\tau_\text{grad}$ to increase the robustness for occluded regions. 
We use linear interpolation for $I_R(s')$, and a sobel filter kernel of $[-0.5 \;\; 0 \;\; 0.5]$ for $\nabla_x$.
When the data term is defined on the right image, we swap $I_L$ and $I_R$ in Eqs.~(\ref{eq:gf}) and (\ref{eq:photofunc}), and add the disparity value in \eref{warp}.

\subsubsection*{Smoothness term}
For the smoothness term, we use a curvature-based, second-order smoothness regularization term~\cite{Olsson13} defined as
\begin{align}
\psi_{pq}(f_p,f_q)=\max\left( w_{pq}, \epsilon \right) \, \min ( \bar{\psi}_{pq}(f_p, f_q),\tau_\text{dis} ).  \label{eq:smooth_phi}
\end{align}
Here, $w_{pq}$ is a contrast-sensitive weight defined as
\begin{equation}
w_{pq} = e^{- \| I_L(p) - I_L(q) \|_1 / \gamma},
\label{eq:smweight}
\end{equation}
where $\gamma$ is a user-defined parameter. The $\epsilon$ is a small constant value that gives a lower bound to the weight $\omega_{pq}$ to increase the robustness to image noise. The function $\bar{\psi}_{pq}(f_p, f_q)$ penalizes the discontinuity between $f_p$ and $f_q$ in terms of disparity as
\begin{align}
\bar{\psi}_{pq}(f_p, f_q) = | d_p(f_p)\!-\!d_p(f_q)| + |d_q(f_q)\!-\!d_q(f_p)|, \label{eq:smooth}
\end{align}
where \hbox{$d_p(f_q) = a_q p_u + b_q p_v + c_q$}.
The first term in Eq.~(\ref{eq:smooth}) measures the difference between $f_p$ and $f_q$ by their disparity values at $p$, and the second term is defined similarly at $q$. 
We visualize $\bar{\psi}_{pq}(f_p, f_q)$ as red arrows in Fig.~\ref{fig:smoothness}~(a).
The $\bar{\psi}_{pq}(f_p, f_q)$ is truncated at $\tau_\text{dis}$ to allow sharp jumps in disparity at depth edges.

\begin{figure}[h]
	\centering
	\includegraphics[width=0.45\linewidth]{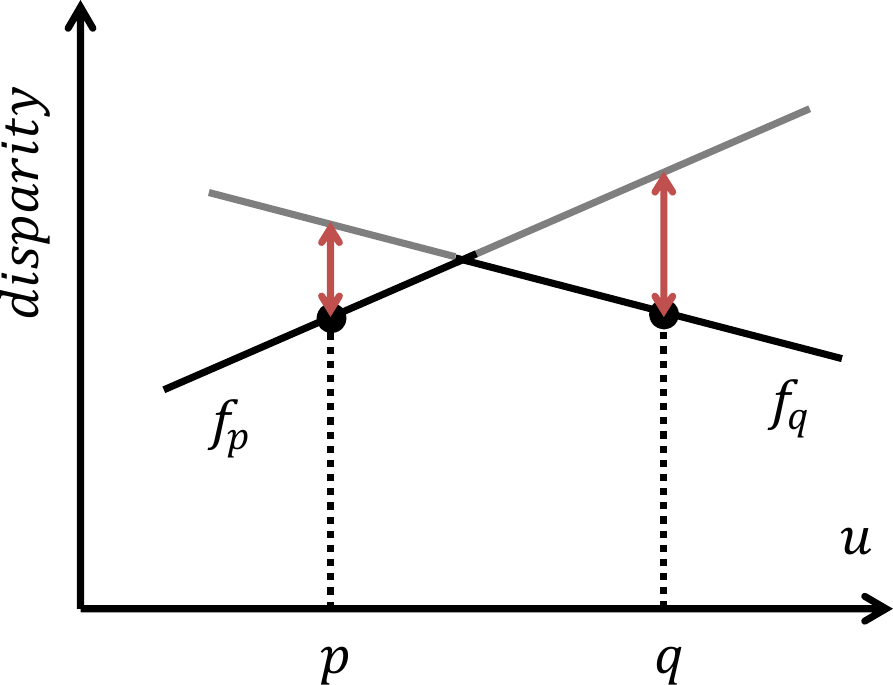}\hfil
	\includegraphics[width=0.45\linewidth]{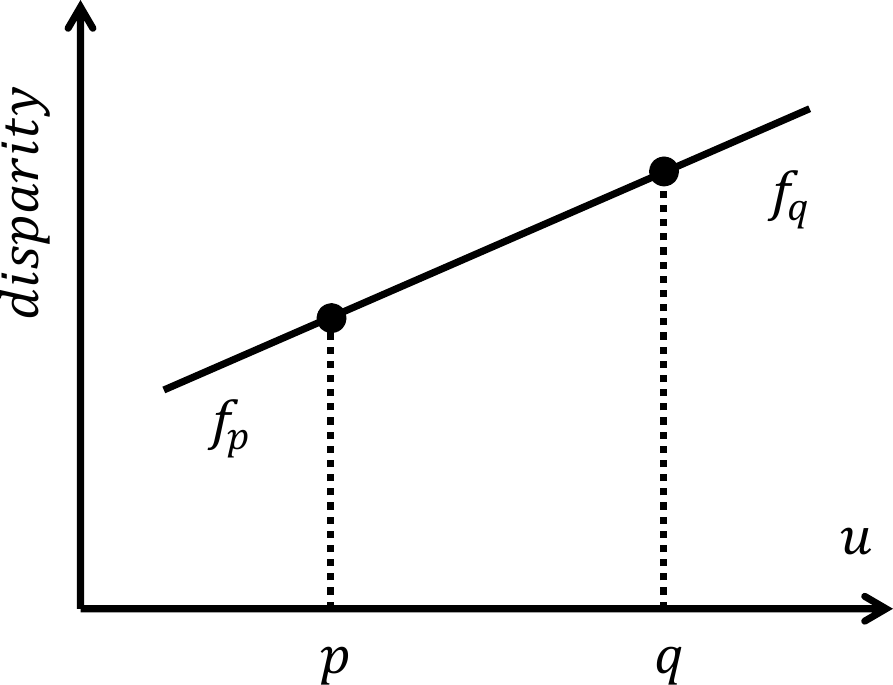}
	\\
	\begin{minipage}{0.45\linewidth}\begin{center}{(a) $f_p \ne f_q$}\end{center}\end{minipage}\hfil
	\begin{minipage}{0.45\linewidth}\begin{center}{(b) $f_p = f_q$}\end{center}\end{minipage}
	%\begin{minipage}{0.3\linewidth}\begin{center}{(c) $a = b = 0$}\end{center}\end{minipage}\hfill
	\caption{Illustration of the smoothness term proposed in~\cite{Olsson13}. (a) The smoothness term penalizes the deviations of neighboring disparity planes shown as red arrows. (b) When neighboring pixels are assigned the same disparity plane, it gives no penalty; thus, it enforces second order smoothness for the disparity maps.}
	\label{fig:smoothness}
\end{figure}

Notice that $\bar{\psi}_{pq}(f_p, f_q) = 2|c_p-c_q|$ when $a=b=0$ is forced;
therefore, the smoothness term ${\psi}_{pq}(f_p, f_q)$ naturally extends the traditional truncated linear model~\cite{Boykov01}, although the latter has a fronto-parallel bias and should be avoided~\cite{Woodford09,Olsson13}.
Also, as shown in Fig.~\ref{fig:smoothness}~(b), this term becomes zero when $f_p=f_q$. This enforces piecewise linear disparity maps and so piecewise planar object surfaces. Furthermore, this term satisfies the following property for taking advantage of GC.

\begin{figure*}[!t]
	\centering
	\includegraphics[width=\textwidth]{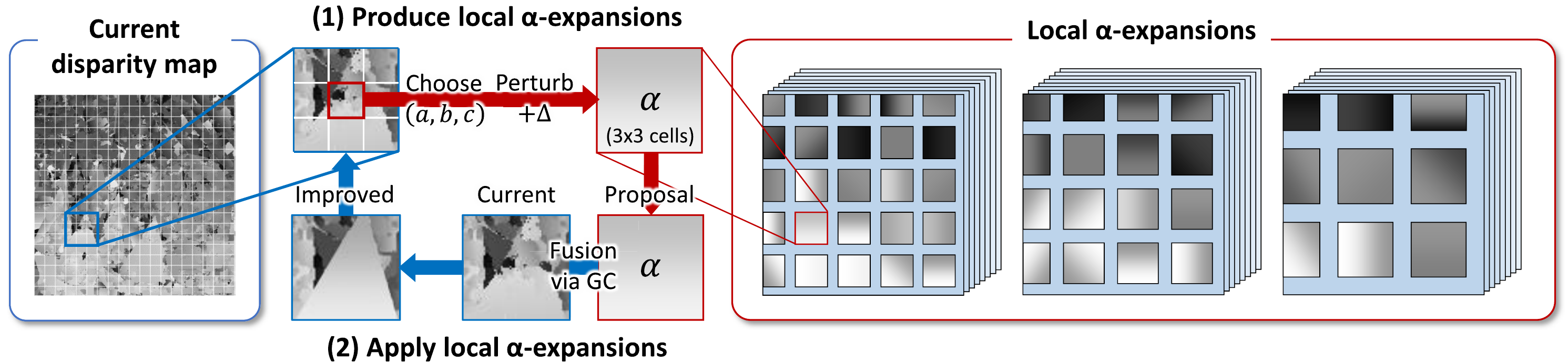}
	\vskip -2mm
	\caption{Illustration of the proposed local expansion moves.
		The local expansion moves consist of many small $\alpha$-expansions (or \textit{local $\alpha$-expansions}), which are defined using grid structures such shown in the left figure.
		These local $\alpha$-expansions are defined at each grid-cell and applied for $3\times 3$ neighborhood cells (or \textit{expansion regions}).
		In the middle part, we illustrate how each of local $\alpha$-expansions works.
		(1) The candidate label $\alpha$ (\ie, $\alpha=(a,b,c)$ representing a disparity plane $d=au+bv+c$) is produced by randomly choosing and perturbing one of the currently assigned labels in its center cell. (2) The current labels in the expansion region are updated by $\alpha$ in an energy minimization manner using GC. Consequently, a current label in the center cell can be propagated for its surrounding cells.
		In the right part, local $\alpha$-expansions are visualized as small patches on stacked layers with three different sizes of grid structures. As shown here, using local $\alpha$-expansions we can localize the scopes of label searching by their locations.
		Each layer represents a group of mutually-disjoint local $\alpha$-expansions, which are  performed individually in a parallel manner.
	}
	\label{fig:locallabel}
\end{figure*}

\begin{lemma}
	\label{lem:subm_expansion}
	The term ${\psi}_{pq}(f_p, f_q)$ in Eq.~(\ref{eq:smooth_phi}) satisfies the submodularity of expansion moves in \eref{axregularity}.
\end{lemma}

\begin{proof}
	See \cite{Olsson13} and also Appendix~A.
\end{proof}

\subsection{Local expansion moves}
\label{sec:localexp}
In this section, we describe the fundamental idea of our method, local expansion moves, as the main contribution of this paper.
We first briefly review the original expansion move algorithm~\cite{Boykov01},
and then describe how we extend it for efficiently optimizing continuous MRFs.

The expansion move  algorithm~\cite{Boykov01} is a discrete optimization method for pairwise MRFs of Eq.~(\ref{eq:mrf}), which iteratively solves a sequence of the following binary labeling problems
\begin{equation}
f^{(t+1)} = \argmin_{f'} E(f' \, | \, f'_p\in \{f^{(t)}_p, \alpha\}) \label{eq:expansion}
\end{equation}
for all possible candidate labels ${}^\forall \alpha \in \myS$.
Here, the binary variable $f'_p$ for each pixel $p$ is assigned either its current label $f^{(t)}_p$ or a candidate label $\alpha$.
If all the pairwise terms in  $E(f)$ meet the condition of Eq.~(\ref{eq:axregularity}),
then the binary energies $E(f')$ in Eq.~(\ref{eq:expansion}) are submodular and this minimization can thus be exactly solved via GC~\cite{Boykov01} (subproblem optimal).
Here, it is guaranteed that the energy does not increase: $E(f^{(t+1)}) \le E(f^{(t)})$.
However, the label space $\myS$ in our setting is a three dimensional continuous space $(a,b,c)$; therefore, such an exhaustive approach cannot be employed.

Our local expansion moves extend traditional expansion moves by two ways; \textit{localization} and \textit{spatial propagation}. By localization, we use different candidate labels $\alpha$ depending on the locations of pixels $p$ in Eq.~(\ref{eq:expansion}), rather than using the same $\alpha$ label for all pixels. This is reasonable because the distributions of disparities should be different from location to location, therefore the selection of candidate labels $\alpha$ should be accordingly different.
By spatial propagation, we incorporate label propagation similar to the PatchMatch inference~\cite{Barnes09,Barnes10,Bleyer11} into GC optimization, and propagate currently assigned labels to nearby pixels via GC.
The assumption behind this propagation is that, if a good label is assigned to a pixel, this label is likely a good estimate for nearby pixels as well.
The localization and spatial propagation together make it possible for us to use a powerful randomized search scheme, where we no longer need to produce initial solution proposals as usually done in the fusion based approach~\cite{Lempitsky10}.
Below we provide the detailed descriptions of our algorithm.

\subsubsection*{Local $\alpha$-expansions for spatial propagation}
We first define a grid structure that divides the image domain $\myP$ into grid regions ${C_{ij}  \subset \myP}$, which are indexed by 2D integer coordinates $(i, j) \in \mathbb{Z}^2$.
We refer to each of these grid regions as a \textit{cell}. 
We in this paper simply assume regular square cells, {but this could be extended to use superpixels~\cite{Hur17,Li17pmsc,Taniai16}}.
At a high level, the size of cells balances between the level of localization and the range of spatial propagation.
Smaller sizes of cells can achieve finer localization but result in shorter ranges of spatial propagation.
Later in Sec~\ref{sec:optimization} we introduce different sizes of multiple grid structures for balancing these two factors well, but for now let us focus on using one grid structure.

Given a grid structure, we define a \textit{local $\alpha$-expansion} at each cell $(i, j)$,  which we specifically denote as \textit{$\alpha_{ij}$-expansion}.
We further define two types of regions for each $\alpha_{ij}$-expansion: its \textit{center region} $C_{ij}$ and \textit{expansion region}
\begin{equation}
R_{ij} = C_{ij} \cup \left\lbrace  \bigcup_{(m,n) \in \myN(i, j)} C_{mn} \right\rbrace,
\end{equation}
\ie, $3 \times 3$ cells consisting of the center region $C_{ij}$ and its eight neighbor cells.

In the middle part of Fig.~\ref{fig:locallabel}, we focus on an expansion region and illustrate how an $\alpha_{ij}$-expansion works.
We first randomly select a pixel $r$ from the center region $C_{ij}$, and take its currently assigned label as $(a, b, c) = f_r$.
We then make a candidate label $\alpha_{ij}$ by perturbing this current label as $\alpha_{ij} = (a, b, c) + \Delta$. Finally, we update the current labels of pixels $p$ in the expansion region $R_{ij}$, by choosing either  their current labels $f_p$ or the candidate label $\alpha_{ij}$.
Here, similarly to Eq.~(\ref{eq:expansion}), we update the partial labeling by minimizing  $E(f')$ with binary variables: $f'_p\in \{f_p, \alpha_{ij}\}$ for \hbox{$p \in R_{ij}$}, and $f'_p = f_p$ for \hbox{$p \notin R_{ij}$}. Consequently, we obtain an improved solution as its minimizer with a lower or equal energy.
%We do this operation for all cells $(i,j)$ and repeat until convergence.
%We iterate this step for 

Notice that making the expansion region $R_{ij}$ larger than the label selection region $C_{ij}$ is the key idea for achieving spatial propagation. We can see this since, in an $\alpha_{ij}$-expansion without perturbation ($\Delta=\textbf{0}$), a current label $f_p$ in the center region $C_{ij}$ can be propagated for its nearby pixels in $R_{ij}$ as the candidate label $\alpha_{ij}$.
%This label propagation is slightly different from that of the original PatchMatch algorithms~\cite{Barnes09,Barnes10} in that we \textit{push} current labels to neighbors rather than \textit{pull} from neighbors.\tani{}

\begin{algorithm}[t]
	\small
	\SetKwInOut{Input}{input}
	\SetKwInOut{Output}{output}
	\caption[Iterative $\alpha_{ij}$-expansion]{\textsc{Iterative $\alpha_{ij}$-expansion}}         
	\label{alg2}   
	\Input{current $f$, target cell $(i, j)$, perturbation size $|\Delta'|$}
	\Output{updated $f$ (only labels $f_p$ at $p\in R_{ij}$ are updated)}
	\Repeat(\textcolor{mossgreen}{\textbf{ expansion proposer (propagation):}}){$K_\text{prop}$ times}{ 
		$\alpha_{ij} \gets f_r$ with randomly chosen $r \in C_{ij}$ \;
		$f \gets \argmin E(f' \, | \, f'_p\in \{f_p, \alpha_{ij}\}, p\in R_{ij})$  \;
	}
	\Repeat(\textcolor{mossgreen}{\textbf{ RANSAC proposer (optional):}}){$K_\text{RANS}$ times}{ 
		$\alpha_{ij} \gets \text{RANSAC}(f, C_{ij})$\;
		$f \gets \argmin E(f' \, | \, f'_p\in \{f_p, \alpha_{ij}\}, p\in R_{ij})$  \;
	}
	\Repeat(\textcolor{mossgreen}{\textbf{ randomization proposer (refinement):}}){$K_\text{rand}$ times (or $|\Delta'|$ is sufficiently small)}{
		$\alpha_{ij} \gets f_r$ with randomly chosen $r \in C_{ij}$ \;
		$\alpha_{ij} \gets \alpha_{ij} + \Delta'$ \;
		$f \gets \argmin E(f' \, | \, f'_p\in \{f_p, \alpha_{ij}\}, p\in R_{ij})$ \;
		$|\Delta'| \gets |\Delta'|/2$ \;
	}
\end{algorithm}

We use this $\alpha_{ij}$-expansion iteratively as shown in Algorithm~\ref{alg2}. Similarly to the PatchMatch algorithm~\cite{Barnes09}, this iterative algorithm has propagation (lines 1--4) and refinement (lines 9--14) steps. In the propagation step, we apply $\alpha_{ij}$-expansions without perturbation to propagate labels from $C_{ij}$ for $R_{ij}$. In the refinement step, we apply $\alpha_{ij}$-expansions with an exponentially decreasing perturbation-size to refine the labels. Here, propagation and refinement can be seen as heuristic proposers of a candidate label $\alpha_{ij}$ that is expectedly a good estimate for the region $C_{ij}$ or even its neighborhood $R_{ij}$. Our local expansion move method is conceptually orthogonal to any such heuristics for generating $\alpha_{ij}$. As such an example, we in lines 5--8 use a RANSAC proposer by following conventional proposal generation schemes used in fusion-based methods~\cite{Woodford09,Olsson13}. Here,  we generate $\alpha_{ij}$ by fitting a plane to the current disparity map $f$ in the region $C_{ij}$ using LO-RANSAC~\cite{ChumMK03}. 
We perform this iterative $\alpha_{ij}$-expansion at every cell $(i,j)$. 

This local $\alpha$-expansion method has the following useful properties.
\textbf{Piecewise linearity}: It helps to obtain smooth solutions.
In each $\alpha_{ij}$-expansion, multiple pixels in the expansion region $R_{ij}$ are allowed to move-at-once to the same candidate label $\alpha_{ij}$ at one binary-energy minimization, which contrasts to BP that updates only one pixel at once.
Since a label represents a disparity plane here, our method helps to obtain piecewise linear disparity maps and thus piecewise planar surfaces as we have discussed in Sec~\ref{sec:slanted_patch_matching}.
\textbf{Cost filtering acceleration}: We can accelerate the computation of matching costs $\phi_p(f_p)$ in Eq.~(\ref{eq:unary_term}) by using cost-volume filtering techniques.
%This is thanks to the property that the candidate label $\alpha_{ij}$ is given constant throughout the expansion region $R_{ij}$, and 
We discuss this more in Sec~\ref{sec:fastimp}.
\textbf{Optimality and parallelizability}: With our energy formulation, it is guaranteed that each binary-energy minimization in Algorithm~\ref{alg2} can be optimally solved via GC. In addition, we can efficiently perform many $\alpha_{ij}$-expansions in a parallel manner. We discuss these matters in the following sections.

\begin{figure*}
	\centering
	\includegraphics[height=5cm]{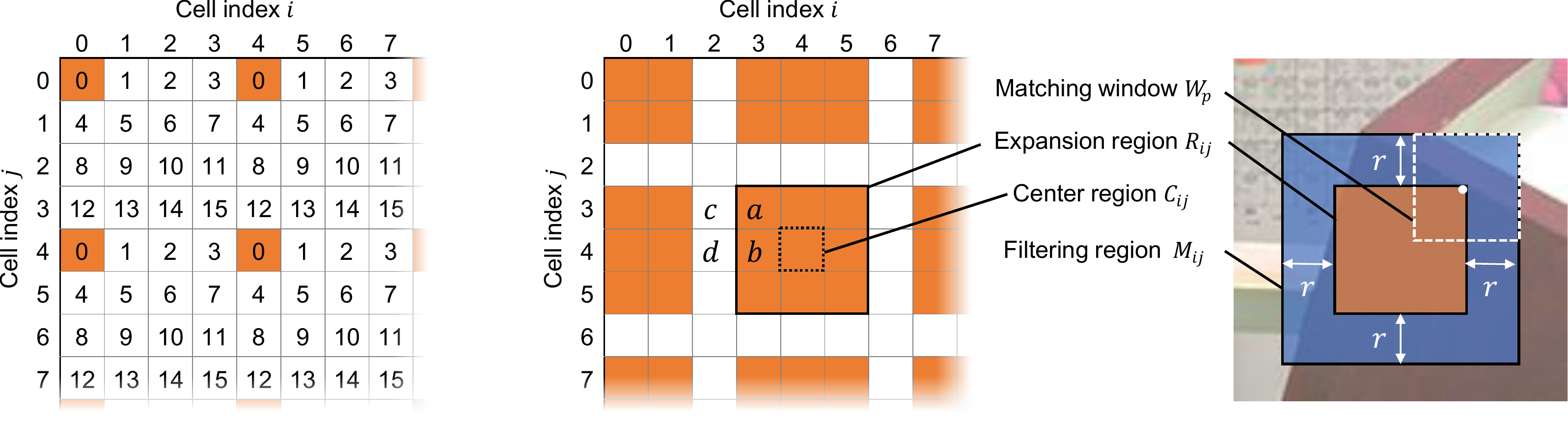}\\
	\vskip -2mm
	\begin{minipage}{\linewidth}
		\begin{minipage}{0.01\linewidth}\mbox{}\end{minipage}
		\begin{minipage}{0.28\linewidth}
			\caption{Group index for $\alpha_{ij}$-expansions. We perform $\alpha_{ij}$-expansions in the same group in parallel.}
			\label{fig:grouping}
		\end{minipage}
		\begin{minipage}{0.05\linewidth}\mbox{}\end{minipage}
		\begin{minipage}{0.28\linewidth}
			\caption{Expansion regions of mutually disjoint $\alpha_{ij}$-expansions (group index $k = 0$). We leave white gaps between neighbors.}
			\label{fig:disjoint}
		\end{minipage}
		\begin{minipage}{0.1\linewidth}\mbox{}\end{minipage}
		\begin{minipage}{0.25\linewidth}
			\caption{Filtering region $M_{ij}$. The margin width $r$ corresponds with the radius of the matching window $W_p$. }
			\label{fig:filtering}
		\end{minipage}
	\end{minipage}
\end{figure*}

\subsubsection*{Mutually-disjoint local $\alpha$-expansions}
\label{sec:disjoint}
While the previous section shows how each local $\alpha$-expansion behaves, here we discuss the scheduling of local $\alpha$-expansions.
We need a proper scheduling, because local $\alpha$-expansions cannot be simultaneously performed due to the overlapping expansion regions.

To efficiently perform local $\alpha$-expansions, we divide them into groups such that the local $\alpha$-expansions in each group are mutually disjoint.
Specifically, we assign each $\alpha_{ij}$-expansion a group index $k$ given by
\begin{equation}
k = 4 \, (j \text{ mod } 4) + (i \text{ mod } 4), \label{eq:groupindex}
\end{equation}
and perform the iterative $\alpha_{ij}$-expansions in one group and another.
As illustrated in Fig.~\ref{fig:grouping}, this grouping rule picks $\alpha_{ij}$-expansions at every four vertical and horizontal cells into the same group, and it amounts to 16 groups of mutually-disjoint local $\alpha$-expansions.
We visualize a group of disjoint local $\alpha$-expansions as orange regions in Fig.~\ref{fig:disjoint},  and also as a single layer of stacks in the right part of Fig.~\ref{fig:locallabel} with three different grid structures.

Notice that in each group, we leave \textit{gaps} of one cell width between neighboring local $\alpha$-expansions.
These gaps are to guarantee submodularity and independence for local $\alpha$-expansions.
By the submodularity, we can show that our local $\alpha$-expansions always produce submodular energies and can thus be optimally solved via GC. Because of this submodularity, we can use a standard GC algorithm~\cite{Boykov04} instead of an expensive QPBO-GC algorithm~\cite{Kolmogorov07}, which is usually required in the fusion based approach.
By the independence, we can show that the local $\alpha$-expansions in the same group do not interfere with each other.
Hence, we can perform them simultaneously in a parallel manner.
The parallelization of GC algorithms are of interest in computer vision~\cite{Liu10,Strandmark10}. Our scheme is simple and can directly use existing GC implementations.
The proof of this submodularity and independence is presented in the next section.

\subsubsection*{Submodularity and independence}
\label{sec:submodularity}
To formally address the  submodularity and independence of local $\alpha$-expansions, we discuss it using the form of fusion moves~\cite{Lempitsky10}.
%By this form, we can also discuss a relationship between our method and fusion moves.
Let us assume a current solution $f$ and a group of mutually-disjoint $\alpha_{ij}$-expansions to be applied.
Simultaneously applying these $\alpha_{ij}$-expansions is equivalent to the following fusion-energy minimization:
\begin{equation}
f^{*} = \argmin E(f' \, | \, f'_p\in \{f_p, g_p\}), \label{eq:fusion}
\end{equation}
where the proposal solution $g$ is set to $g_p = \alpha_{ij}$ if $p$ belongs to any of expansion regions $R_{ij}$ in this group; otherwise, $p$ is in gaps so we assign $\phi_p(g_p)$ an infinite unary cost for forcing $f'_p = f_p$. This proposal disparity map $g$ is visualized as a single layer of stacks in the right part of Fig.~\ref{fig:locallabel}.
We prove the following lemmas:

\begin{lemma}
	{Submodularity}: the binary energies in Eq.~(\ref{eq:fusion}) are submodular, \ie, all the pairwise interactions in Eq.~(\ref{eq:fusion}) meet the following submodularity of fusion moves~\cite{Lempitsky10,Kolmogorov07}:
	\begin{equation}
	\psi_{pq}(g_p, g_q) + \psi_{pq}(f_p, f_q) \le \psi_{pq}(f_p, g_q) + \psi_{pq}(g_p, f_q). \label{eq:sub_fusion}
	\end{equation}
\end{lemma}

\begin{proof}
	Omitted.\footnote{We outline proof. For $(p,q)$ inside an expansion region such as $(a,b)$ in Fig.~\ref{fig:disjoint},  Eq.~(\ref{eq:sub_fusion}) is relaxed to Eq.~(\ref{eq:axregularity}). For the other cases, \eg,  $(a,c)$ or $(c,d)$, pairwise terms  $\psi_{ac}(f'_a,f'_c)$ and  $\psi_{cd}(f'_c,f'_d)$ become unary or constant terms. Therefore, when implementing an $\alpha_{ij}$-expansion using min-cut in $R_{ij}$, we create nodes for ${}^\forall p \in R_{ij}$ and add $\psi_{ac}(f'_a,f_c)$ as unary potentials of nodes $a$ at the inner edge of $R_{ij}$. See also \cite{Boykov01} for the conversion of  $\psi_{ab}(f'_a,f'_b)$ into edge capacities under expansion moves.
	}
\end{proof}

%Note that if the energy is not submodular, the QPBO-GC algorithm~\cite{Kolmogorov07} is usually used for solving Eq.~(\ref{eq:fusion}). However, it only finds a \textit{partial solution} with some variables left \textit{unlabeled}.

\begin{lemma}
	{Independence}: the assignments to $f'_p$ and $f'_q$ do not influence each other, if $p$ and $q$ are in different expansion regions.
\end{lemma}
\begin{proof}
	The $f'_p$ and $f'_q$ have interactions if and only if there exists a chain of valid pairwise interactions $C = \{\psi(f'_{s_0},f'_{s_1}), \cdots, \psi(f'_{s_{n-1}},f'_{s_n})\}$ connecting \mbox{$p = s_0$} and {$q = s_n$}.
	%However, there is no such chain because pairwise terms in or across a gap become constant or unary terms.
	But there is no such chain because $C$ inevitably contains constant or unary terms at a gap ($\psi_{cd}$ and $\psi_{ac}$ in Fig.~\ref{fig:disjoint}). 
	%As discussed in the above proof of submodularity, there is no direct interaction $\psi(f'_p,f'_q)$. Therefore, such a chain must contain two types of pairwise terms $\psi_{cd}(f'_a,f'_c)$ and $\psi_{ac}(f'_c,f'_d)$ described above, which become a unary potential and a constant, respectively. Therefore, there is no such chain of pairwise interactions connecting $p$ and $q$.
\end{proof}

\subsection{Optimization procedure}
\label{sec:optimization}

\begin{algorithm}[t]
	\small
	\definecolor{mossgreen}{HTML}{308014}
	\caption[Optimization procedure]{\textsc{Optimization procedure}}         
	\label{alg1}   
	Define three levels of grid structures: $H\!=\!\{h_1, h_2, h_3\}$\;
	Initialize the current solution $f$ randomly\;
	Initialize the perturbation size $|\Delta|$\;
	\Repeat{convergence}{ 
		\ForEach{grid structure $h \in H$}{
			\ForEach{disjoint group $k = 0, 1,..., 15$}{
				
				\ForEach(\textcolor{mossgreen}{[in parallel]}){cell $(i,j)$ in the group $k$}{
					Do iterative $\alpha_{ij}$ expansion $(f, (i,j), |\Delta|)$\;
				}
			}
		}
		$|\Delta| \gets |\Delta|/2$\;
	}
	Do post processing\;
\end{algorithm}

Using the local expansion moves shown in the previous section, we present the overall optimization procedure in this section and summarize it in Algorithm~\ref{alg1}.

This algorithm begins with defining grid structures. Here, we use thee different sizes of grid structures for better balancing localization and spatial propagation.

At line 2 of Algorithm~\ref{alg1}, the solution $f$ is randomly initialized.
To evenly sample the allowed solution space, we take the initialization strategy described in~\cite{Bleyer11}. 
Specifically, for each \hbox{$f_p = (a_p, b_p, c_p)$} we select a random disparity $z_0$ in the allowed disparity range $[0, {\rm dispmax}]$. Then,  a random unit vector \hbox{$n\!=\!(n_x, n_y, n_z)$} and $z_0$  are converted to the plane representation by \hbox{$a_p = -n_x/n_z$}, \hbox{$b_p = -n_y/n_z$}, and \hbox{$c_p = -(n_x p_u + n_y p_v + n_z z_0)/n_z$}.

In the main loop through  lines 4--13, we select one grid level $h$ from the pre-defined grid structures (line 5), and apply the iterative $\alpha_{ij}$ expansions of Algorithm~\ref{alg2} for each cell of the selected structure $h$ (lines 6--10). As discussed in Sec.~\ref{sec:localexp}, we perform the iterative $\alpha_{ij}$ expansions by dividing into disjoint groups defined by Eq.~(\ref{eq:groupindex}). Because of this grouping, the $\alpha_{ij}$ expansions in the loop at lines 7--9 are mutually independent and can be performed in parallel.

The perturbation at line 9 of Algorithm~\ref{alg2} is implemented as described in~\cite{Bleyer11}. Namely, each candidate label $\alpha_{ij} = (a,b,c)$ is converted to the form of a disparity $d$ and normal vector $n$. We then add a random disparity \hbox{$\Delta_{d}' \in [-r_d, r_d]$} and a random  vector $\Delta_{n}'$ of size  $\|\Delta_{n}'\|_2=r_n$  to them, respectively, to obtain $d'$ and $n'$. Finally, $d'$ and $n'/| n'|$ are converted to the plane representation $\alpha_{ij} \gets (a',b',c')$ to obtain a perturbed candidate label. The values $r_d$ and $r_n$ define an allowed change of disparity planes.
We initialize them by setting \hbox{$r_d \gets {\rm dispmax}/2$} and \hbox{$r_n \gets 1$} at line 3 of Algorithm~\ref{alg1}, and update them by \hbox{$r_d \gets r_d/2$} and \hbox{$r_n \gets r_n/2$} at line~12 of Algorithm~\ref{alg1} and line~13 of Algorithm~\ref{alg2}. %\footnote{Note the two variables $|\Delta'|$ and $|\Delta|$ in Algorithms~\ref{alg2} and \ref{alg1} are different instances, \ie, we copy the value $|\Delta'| \gets |\Delta|$ when call Algorithm~\ref{alg2}, but the changes to $|\Delta'|$ do not affect $|\Delta|$.}

Finally, we perform post-processing using left-right consistency check and weighted median filtering as described in~\cite{Bleyer11} for further improving the results. This step is widely employed in recent methods~\cite{Bleyer11,Besse12,Lu13,Heise13}.

Note that there are mainly two differences between this algorithm and our previous version~\cite{Taniai14}. 
For one thing, we have removed a per-pixel label refinement step of \cite{Taniai14}. This step is required for updating a special label space structure named \emph{locally shared labels} (LSL) used in \cite{Taniai14}, but can be removed in our new algorithm by using local $\alpha$-expansions instead.
For the other, the previous algorithm proceeds rather in a batch cycle; \ie, it produces all local $\alpha$-expansions with all grid structures at once and then applies them to the current solution $f$ in one iteration. %This way we can minimize the overhead of data transferring between GPU in unary cost computation, but results in 
This leads to
slower convergence than our new algorithm that produces each local $\alpha$-expansion always from the latest current solution~$f$. %Our new algorithm rather concedes the increased overhead of GPU because it is also intended for a fast CPU implementation as we describe below.

\subsection{Fast implementation}
\label{sec:fastimp}
Our method has two major computation parts: the calculations of matching costs $\phi_p (f_p)$ of Eq.~(\ref{eq:unary_term}), and application of GC in Algorithm~\ref{alg2}. 
In Sec.~\ref{sec:localexp} and Sec.~\ref{sec:optimization}, we have shown that the latter part can be accelerated by performing disjoint local $\alpha$ expansions in parallel. In this section, we discuss the former part.

The calculations of the matching costs $\phi_p (f_p)$ are very expensive, since they require $O(|W|)$ of computation for each term, where $|W|$ is the size of the matching window.
In our previous algorithm~\cite{Taniai14}, we accelerate this part by using GPU.
The use of GPU is reasonable for our method because $\phi_p (f_p)$ of all pixels can be individually computed in parallel during the inference, which contrasts to PMBP~\cite{Besse12} that can only sequentially process each pixel. Still, the computation complexity is $O(|W|)$ and it becomes very inefficient if we only use CPU~\cite{Taniai14}. %\footnote{In \cite{Taniai14} we reported that it took more than 10 hours to get a single disparity map using only a single CPU core.}

In the following part, we show that we can approximately achieve $O(1)$ of complexity for computing each $\phi_p (f_p)$. 
%, which approximately achieves $O(1)$ of computation for each $\phi_p (f_p)$.
The key observation here is that, we only need to compute this matching cost $\phi_p (f_p)$ during the $\alpha_{ij}$-expansions in Algorithm~\ref{alg2}, where $\phi_p (f_p)$ is computed as $\phi_p (\alpha_{ij})$ for all $p \in R_{ij}$.
With this consistent-label property, we can effectively incorporate the fast subregion cost-filtering technique used in \cite{Lu13}, by which $\phi_p (\alpha_{ij})$ for all $p \in R_{ij}$ are {efficiently computed at once}.

%In our method, we only need to compute this matching cost $\phi_p (f_p)$ as the form of $\phi_p (\alpha_{ij})$ during the $\alpha_{ij}$ expansions in Algorithm~\ref{alg2}.  An important observation here is that, the candidate label $f_p = \alpha_{ij}$ is constant throughout the expansion region $R_{ij}$. With this property, we can effectively employ the fast edge-aware filtering technique used in \cite{Lu13} for computing $\phi_p (\alpha_{ij})$ at once for all $p \in R_{ij}$, which approximately achieves $O(1)$ of computation.

To more specifically discuss it, we first separate the computation of $\phi_p (f_p)$ into two steps: the calculations of raw matching costs
\begin{equation}
\rho_{f_p}(s)=\rho(s|f_p)  \label{eq:rawcost}
\end{equation}
for the support pixels $s \in W_p$, and the aggregation of the raw matching costs using an edge-aware filter kernel $\omega_{ps}$
\begin{equation}
\phi_{f_p} (p) = \sum_{s \in W_p} \omega_{ps} \, \rho_{f_p}(s). \label{eq:aggregation}
\end{equation}
%Here, we express each term as a function of a coordinate $p$ instead of a label $f_p$.
We also define a filtering region $M_{ij}$ as the joint matching windows in $R_{ij}$ as
\begin{equation}
M_{ij} = \bigcup_{p \in R_{ij}} W_p.
\end{equation}
As shown in Fig.~\ref{fig:filtering}, this $M_{ij}$ is typically a square region margining $R_{ij}$ with $r$ pixels of width around $R_{ij}$, where $r$ is the radius of the matching windows.

In the raw matching part of Eq.~(\ref{eq:rawcost}), the calculations of the raw costs $\rho_{f_p}(s)$, ${}^\forall s \in W_p$ for all $p \in R_{ij}$ generally require $O(|W||R_{ij}|)$ of total computation.
However, with the consistent-label property (\ie, $f_p=\alpha_{ij}$ for all $p \in R_{ij}$), they can be computed at once in $O(|M_{ij}|)$ by computing $\rho_{\alpha_{ij}}(s)$ for ${}^\forall s \in M_{ij}$.
Here, the computation complexity {for} each unary term is $O(|M_{ij}|/|R_{ij}|)$. Therefore, if $|M_{ij}| \simeq |R_{ij}|$, we can approximately achieve $O(|M_{ij}|/|R_{ij}|) \simeq O(1)$~\cite{Lu13}.

Similarly, the cost aggregation part of Eq.~(\ref{eq:aggregation}) can be done in approximately $O(1)$,  if we apply a constant-time edge-aware filtering $\omega_{ps}$ to $\rho_{\alpha_{ij}}(s)$, ${}^\forall s \in M_{ij}$~\cite{Lu13}.
As such an example we have chosen guided image filtering~\cite{He13} for Eq.~(\ref{eq:gf}) but could use more sophisticated filtering~\cite{Lu12}.
%Unfortunately, our weight function $w_{ps}$ in Eq.~(\ref{eq:weight}) represents a variant of the joint bilateral filtering~\cite{Petschnigg04}, which generally requires $O(|W|)$ of computation. However, there are several constant-time edge-aware filterings that work similarly to or even better than the bilateral filtering, such as the guided image filtering~\cite{He13} and the cross-based local multipoint filterings~\cite{Lu12}. %Hence, we them to approximate the bilateral filtering.
%Therefore, if we replace the adaptive window weight $w_{ps}$ of Eq.~(\ref{eq:weight}) with those filter kernels, the overall computation complexity for each unary cost $\phi_p (f_p)$ becomes approximately $O(1)$.
%Formally, we use the following filter kernel of the guided image filtering~\cite{He13}:
%\begin{equation}
%\omega_{ps} = \frac{1}{|W'|^2} \sum_{k:(p,s)\in W'_k} \!\! \left(1 + (I_p - \mu_k)^T(\Sigma_k + e)^{-1}(I_s - \mu_k) \right) \label{eq:gf}
%\end{equation}
%Here, $I_p = I_L(p)/255$ is a normalized color vector, $\mu_k$ and $\Sigma_k$ are the mean and co-variance matrix of $I_p$ in a local regression window $W'_k$, and $e$ is an identity matrix with a small positive coefficient for avoiding over-fitting. This filtering can be computed in $O(1)$ using integral image.

Note that as the consistent-label property is the key to being able to use the filtering techniques, this scheme cannot be used in other PatchMatch based methods~\cite{Besse12,Bleyer11,Heise13} except for \cite{Lu13} that uses superpixels as computation units. % and the fusion moves~\cite{Lempitsky10}.

\section{Experiments}
\label{sec:experiments}
In the experiments, we first evaluate our method on the Middlebury benchmark V2 and V3. While the latest V3 benchmark allows us to compare with state-of-the-art methods, we also use the older V2 as there are more related methods~\cite{Bleyer11,Besse12,Heise13,Lu13,Taniai14} registered only in this version.

We further assess the effect of sizes of grid-cells and also the effectiveness of the proposed acceleration schemes in comparison to our previous method~\cite{Taniai14}.
Our method is further compared with the PMBP~\cite{Besse12}, PMF~\cite{Lu13} and Olsson~\etal~\cite{Olsson13} methods that are closely related to our approach. 
{Additional results for these analyses and comparisons on more image pairs are provided in the supplementary material.}

\subsubsection*{Settings}
We use the following settings throughout the experiments. We use a desktop computer with a Core i7 CPU (3.5 GHz $\times$ $4$ physical cores) and NVIDIA GeForce GTX Titan Black GPU. All methods are implemented using C++ and OpenCV.

The parameters of our data term are set as $\{ e, \tau_\text{col}, \tau_\text{grad}, \alpha \} = \{ 0.01^2, 10, 2, 0.9 \}$ following \cite{Lu13,Rhemann11} and \cite{Bleyer11}. The size of matching windows $W_p$ is set to $41\times41$ (\ie, the size of local regression windows $W'_k$ in Eq.~(\ref{eq:gf}) is set to $21\times 21$), which is the same setting with PMBP~\cite{Besse12} and \cite{Bleyer11,Heise13}. %, but slightly bigger than $35\times35$ that was used in \cite{Bleyer11}.
For the smoothness term, we use $\{\lambda, \tau_\text{dis}, \epsilon, \gamma \} = \{1, 1, 0.01, 10 \}$ and eight neighbors for $\myN$.

For our algorithm, we use three grid structures with cell sizes of $5\times 5$, $15\times 15$, and $25\times 25$ pixels. The iteration numbers $\{K_\text{prop}, K_\text{rand}\}$ in Algorithm~\ref{alg2} are set to  $\{1, 7\}$ for the first grid structure, and $\{2, 0\}$ (only propagation step) for the other two. Except in Sec.~\ref{sec:benchmark_v3} and \ref{sec:olsson} we disable the RANSAC proposer for fair comparisons with related methods or for avoiding cluttered analysis. When enabled, we set $K_\text{RANS} = 1$ for all the grid structures. We iterate the main loop ten times. We use a GC implementation of~\cite{Boykov04}.

We use two variants of our method.
\textbf{LE-GF} uses guided image filtering~\cite{He13} for $w_{ps}$ as proposed in Sec.~\ref{sec:formulation}. %The size of local regression windows $W'_k$ is set to $21\times 21$ so that the size of actual matching windows becomes $41\times 41$.  
We only use a CPU implementation for this method. 
\textbf{LE-BF} uses bilateral filtering for $\omega_{ps}$ defined similarly to Eq.~(\ref{eq:smweight}) and $\lambda=20$ so that the energy function corresponds to one we used in~\cite{Taniai14}.
The computation of matching costs is thus as slow as $O(|W|)$. 
In a GPU implementation, we accelerate this calculation by computing each unary term individually in parallel on GPU.
%In a CPU implementation, we compute them as filtering to raw matching costs. This way we can still accelerate the computation at the first step of cost filtering. 
For both LE-BF and LE-GF, we perform disjoint local $\alpha$-expansions  in parallel using four CPU cores.

\subsection{Evaluation on the Middlebury benchmark V2}
\label{sec:benchmark}
%In this section, we evaluate our method using the Middlebury stereo benchmark~\cite{Scharstein02,Scharstein03}. 
%\rev{
We show in~\tref{middlebury} selected rankings on the Middlebury stereo benchmark V2 for $0.5$-pixel accuracy, comparing with other PatchMatch-based methods~\cite{Bleyer11,Besse12,Lu13,Heise13,Taniai14}.
%For a fair comparison with other PatchMatch based methods~\cite{Bleyer11,Besse12,Lu13,Heise13,Taniai14} registered in this benchmark, we here disable the RANSAC proposer in Algorithm~\ref{alg1} by setting \mbox{$K_\text{RANS} = 0$}.
The proposed LE-GF method achieves the current best average rank ($3.9$) and bad-pixel-rate ($5.97\%$) amongst more than $150$ stereo methods including our previous method (GC+LSL)~\cite{Taniai14}. Even without post-processing, our LE-GF method still outperforms the other methods in average rank, despite that all the methods in \tref{middlebury} use the post-processing.
On the other hand, the proposed LE-BF method achieves comparable accuracy with our previous algorithm~\cite{Taniai14}, since both methods optimize the same energy function in very similar ways.
%More thorough comparisons with closely related methods (our previous method~\cite{Taniai14}, PMBP~\cite{Besse12} and PMF~\cite{Lu13}) are provided later.
%Compared with closely related approaches (PMBP~\cite{Besse12} and PatchMatch stereo~\cite{Bleyer11}), which are ranked eleventh and fourteenth in \tref{middlebury},  although their results of Cones are slightly better than ours in some evaluations, our LE-GF and LE-BF methods consistently outperform the two methods in the other evaluations. 
%We summarize the results of our LE-GF method in \fref{middlebury}. 

\begin{table*}[!t]
	\centering
	\caption{Middlebury benchmark V2 for $0.5$-pixel accuracy. Our method using the guided image filtering (GF) achieves the current best average rank $3.9$. Error percents for \emph{all} pixels, non-occluded pixels (\emph{nonocc}), and pixels around depth discontinuities (\emph{dis}) are shown. (Snapshot on April 22, 2015)}
	%In \emph{all}, results are evaluated for all pixels, while only for non-occluded pixels in \emph{nonocc}, and around depth discontinuities in \emph{disc}. (March 23, 2015)}
	\label{tab:middlebury}
	\includegraphics[width=0.9\textwidth]{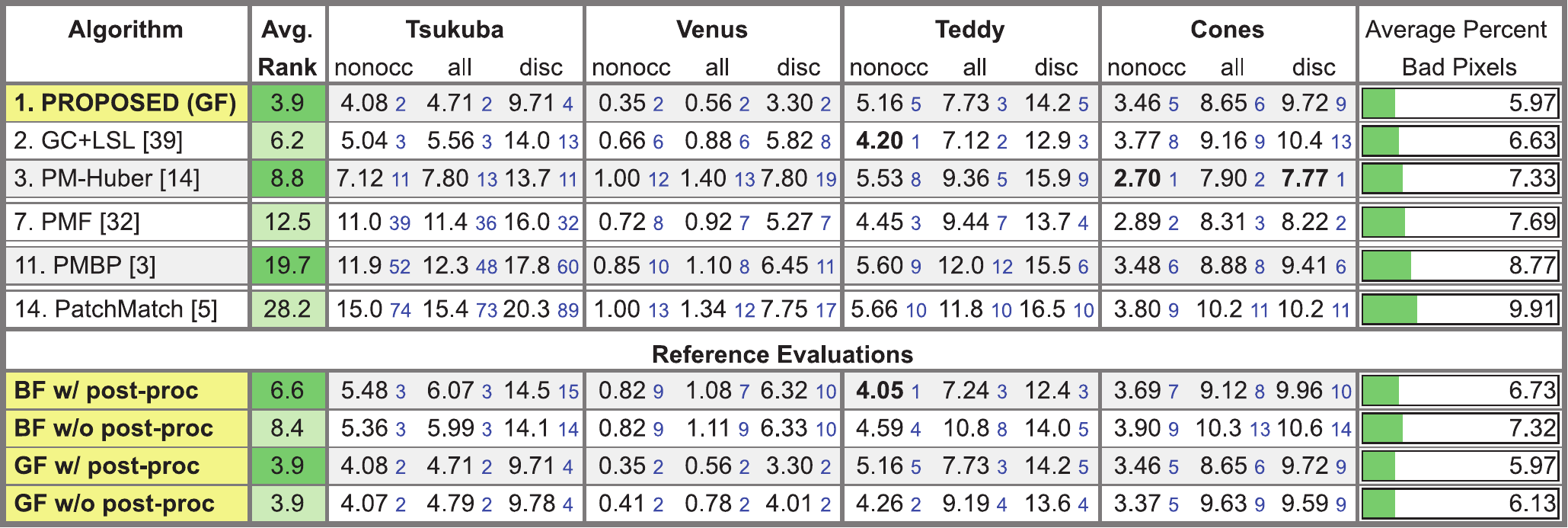}
\end{table*}

\begin{table*}[!t]
	\centering
	\caption{Middlebury benchmark V3 for the \emph{bad 2.0 nonocc} metric. Our method achieves the current best average error rate among 64 existing algorithms. We here list current top methods~\cite{Zbontar16,Li17mst,Li17pmsc,Drouyer17,Kim16} that use the matching costs of MC-CNN-acrt~\cite{Zbontar16} like ours. (Snapshot on July 4, 2017)
	}
	\label{tab:middlebury3}
	\includegraphics[width=0.9\textwidth]{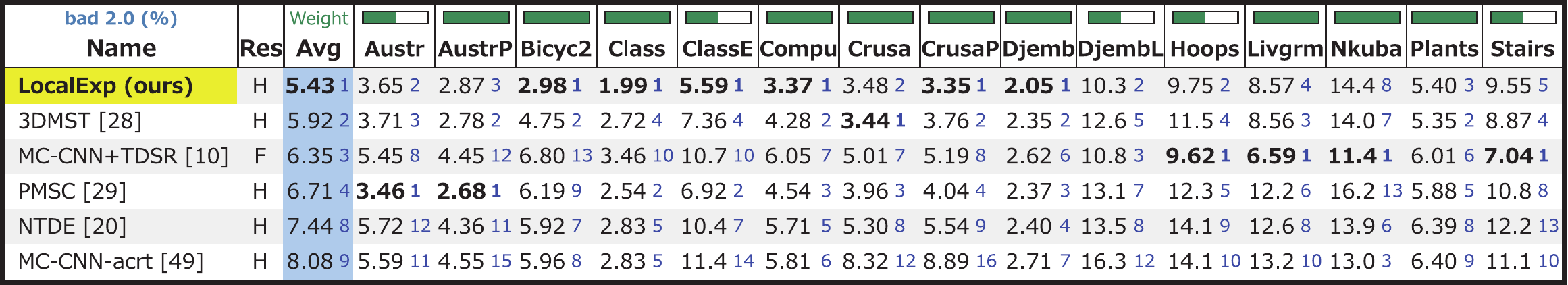}
\end{table*}

\subsection{Evaluation on the Middlebury benchmark V3}
\label{sec:benchmark_v3}
As major differences from the older version, the latest Middlebury benchmark introduces more challenging difficulties such as high-resolution images, different exposure and lighting between an image pair, imperfect rectification, etc. 
Since our model, especially our data term borrowed from \cite{Bleyer11}, does not consider such new difficulties, we here slightly modify our model for adapting it to the latest benchmark.
To this end, we incorporate the state-of-the-art CNN-based matching cost function by Zbontar~and~LeCun~\cite{Zbontar16}, by following manners of current top methods~\cite{Li17mst,Li17pmsc,Drouyer17,Kim16} on the latest benchmark. Here, we only replace our pixelwise raw matching cost function $\rho(\cdot)$ of Eq.~(\ref{eq:photofunc}) with their matching cost function as follows.
\begin{equation}
\rho'(s|f_p) = \min( C_\text{CNN}(s, s') , \tau_\text{CNN}).
\label{eq:mccnn}
\end{equation}
The function $C_\text{CNN}(s, s')$ computes a matching cost between left and right image patches of $11\times 11$-size centered at $s$ and $s'$, respectively, using a  neural network (called MC-CNN-acrt in \cite{Zbontar16}).\footnote{Using the authors' code and pre-trained model, we pre-compute the matching costs $C_\text{CNN}$ at all allowed integer disparities using fronto-parallel patches.
	During the inference, we linearly interpolate the function $C_\text{CNN}(s, s')$ at non-integer coordinates $s'$. MC-CNN learns a patch similarity metric invariantly to the left-right patch distortion. Thus, the bias due to using fronto-parallel patches in $C_\text{CNN}$ is small.% especially after the slanted window cost aggregation~\cite{Li17mst,Li17pmsc}.
}
In our method we truncate the matching cost values at $\tau_\text{CNN} = 0.5$ and aggregate them for the support window pixels $s \in W_p$  using slanted patch matching of Eq.~(\ref{eq:unary_term}). In addition to the change in Eq.~(\ref{eq:photofunc}), the widths of square cells of three grid structures are changed proportionally to $1\%$, $3\%$, and $9\%$ of the image width. Since only one submission is allowed for the test data by the new benchmark regulation, we here only use our LE-GF method with $\lambda = 0.5$ and $K_\text{RANS} = 1$. We keep all the other parameters (\ie, $\{e,W_p,\tau_\text{dis},\epsilon,\gamma,K_\text{prop},K_\text{rand}\}$) as default.

Table~\ref{tab:middlebury3} shows selected rankings for the \emph{bad 2.0} metric (\ie, percentage of bad pixels by the error threshold of 2.0 pixels at full resolution) for non-occluded regions as the default metric. Our method  outperforms all existing 63 methods on the table not only for the default metric but for all combinations of $\{\text{\emph{bad} } \emph{0.5}, \emph{1.0}, \emph{2.0}, \emph{4.0}\} \times \{\text{\emph{nonocc}}, \text{\emph{all}}\}$ except for \emph{bad 4.0 - all}.
In Table~\ref{tab:training} we analyze effects of our post-processing and RANSAC proposer using training data.
Again, our method is ranked first even without post-processing for the default and all other \emph{bad nonocc} metrics.
Also, the RANSAC proposer reduces errors by one point (see Sec.~\ref{sec:olsson} for more analysis).
The avr and stdev show that our inference is stable by different random initializations.

\subsection{Effect of grid-cell sizes}
To observe the effect of grid-cell sizes, we use the three sizes of grid-cells, $5\times 5$ pixels (denoted as ``S''mall), $15\times 15$ pixels (denoted as ``M''edium), $25\times 25$ pixels (denoted as ``L''arge), in different combinations and assess the performance using the following five different settings;
(S, S, S): the small-size cells for all grid-structures;
(M, M, M): the medium-size cells for all grid-structures; 
(L, L, L): the large-size cells for all grid-structures; 
(S, M, M): the small and medium-size cells for the first and the other two grid-structures, respectively; 
(S, M, L): the small, medium, and large-size cells for the first, second, and third grid-structures, respectively (the default setting for our method described above).
Here, the iteration numbers for the first to third grid-structures are kept as default.
We use the LE-GF method so as to access the effect on cost filtering acceleration as well. We use  $\lambda = 0.5$ but keep the other parameters as default. Using these settings, we observe the performance variations by estimating the disparities of only the left image of the \emph{Reindeer} dataset without post-processing.

The plots in Fig.~\ref{fig:regionlabel_energy}~(a) show the transitions of the energy function values over iterations. %Here, energies are evaluated by the re-defined energy function using Eq.~(\ref{eq:gf}).
The plots in Fig.~\ref{fig:regionlabel_energy}~(b) show the temporal transitions of error rates with subpixel accuracy.
As shown, the joint use of multiple cell sizes improves the performance in both energy reduction and accuracy. Although (S, M, M) and (S, M, L) show almost the same energy transitions, the proposed combination (S, M, L) shows faster and better convergence in accuracy.

Figure~\ref{fig:gridsize_image} compares the results of the five settings with error maps for (S, M, M) and (S, M, L).
The use of larger grid-cells helps to obtain smoother disparities, and it is especially effective for occluded regions.

Comparing the running times in Fig.~\ref{fig:regionlabel_energy}~(b), the use of the small-size cells is inefficient due to the increased overhead in cost filtering, whereas the use of the medium and large-size cells achieves almost the same efficiency.

\begin{figure*}[!t]
	\begin{minipage}{0.34\textwidth}
		\scriptsize
		\tabcaption{Effect of our post-processing (PP) and RANSAC proposer (RP). We show {weighted} average bad 2.0 scores of our LE-GF method with 3DMST~\cite{Li17mst} that ranks second following ours. We use 15 training image pairs from Middlebury V3. We also show average (avr) and standard deviation (stdev) by ten sets of different random initializations. See also Sec.~\ref{sec:olsson} for more analysis on the RANSAC proposer.}
		\label{tab:training}
		\begin{center}
			\begin{tabular}{c|cc|cc}
				&  PP	& RP & \scriptsize{nonocc} & \scriptsize{all} \\ 
				\hline 
				& \checkmark	& \checkmark &  \textbf{6.52} & \textbf{12.1} \\ 
				%\cline{2-4} 
				Ours &  & \checkmark &  \textbf{6.65} & 13.6 \\ 
				% \cline{2-4} 
				&  & & 7.72 & 14.6 \\ 
				\hline 
				\multirow{2}{6em}{\centering avr $\pm$ stdev } & \multirow{2}{1em}{\checkmark}	& \multirow{2}{1em}{\checkmark} &  \textbf{6.63} & \textbf{12.3} \\ 
				& 	&  &  $\!\!\!\!\pm$0.12 & $\!\!\pm$0.2 \\ 
				\hline 
				\emph{ref.} 3DMST~\cite{Li17mst} & \checkmark  &  & 7.08 & 12.9 \\ 
		\end{tabular}\end{center}
	\end{minipage}\hfill
	\begin{minipage}{0.65\textwidth}
		\centering
		\begin{minipage}{0.495\textwidth}
			\begin{center}
				\includegraphics[width=\textwidth]{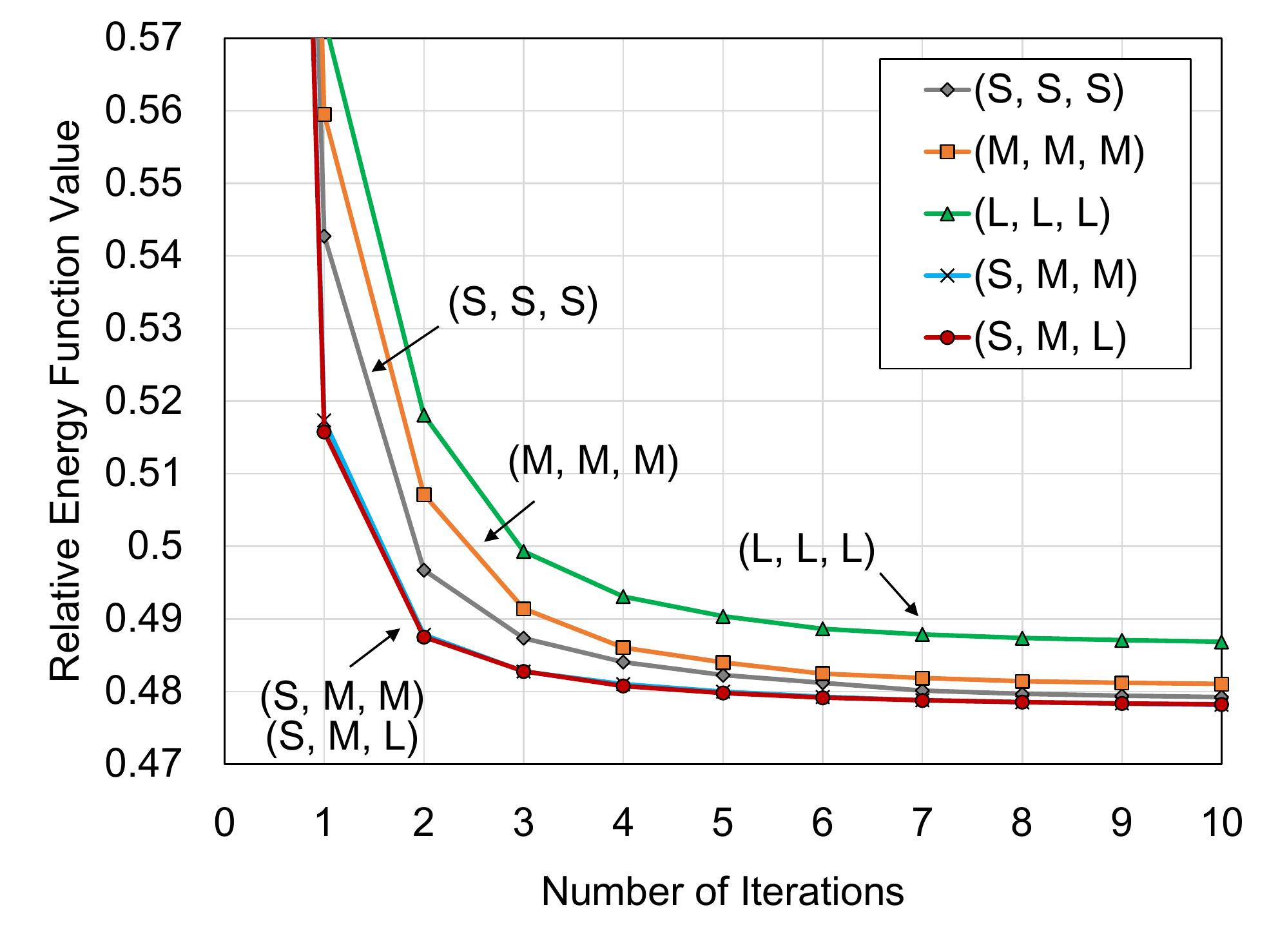}\\
				\capskip
				\small{(a) Energy function values \wrt iterations}
			\end{center}
		\end{minipage}
		\hfil
		\begin{minipage}{0.495\textwidth}
			\begin{center}
				\includegraphics[width=\textwidth]{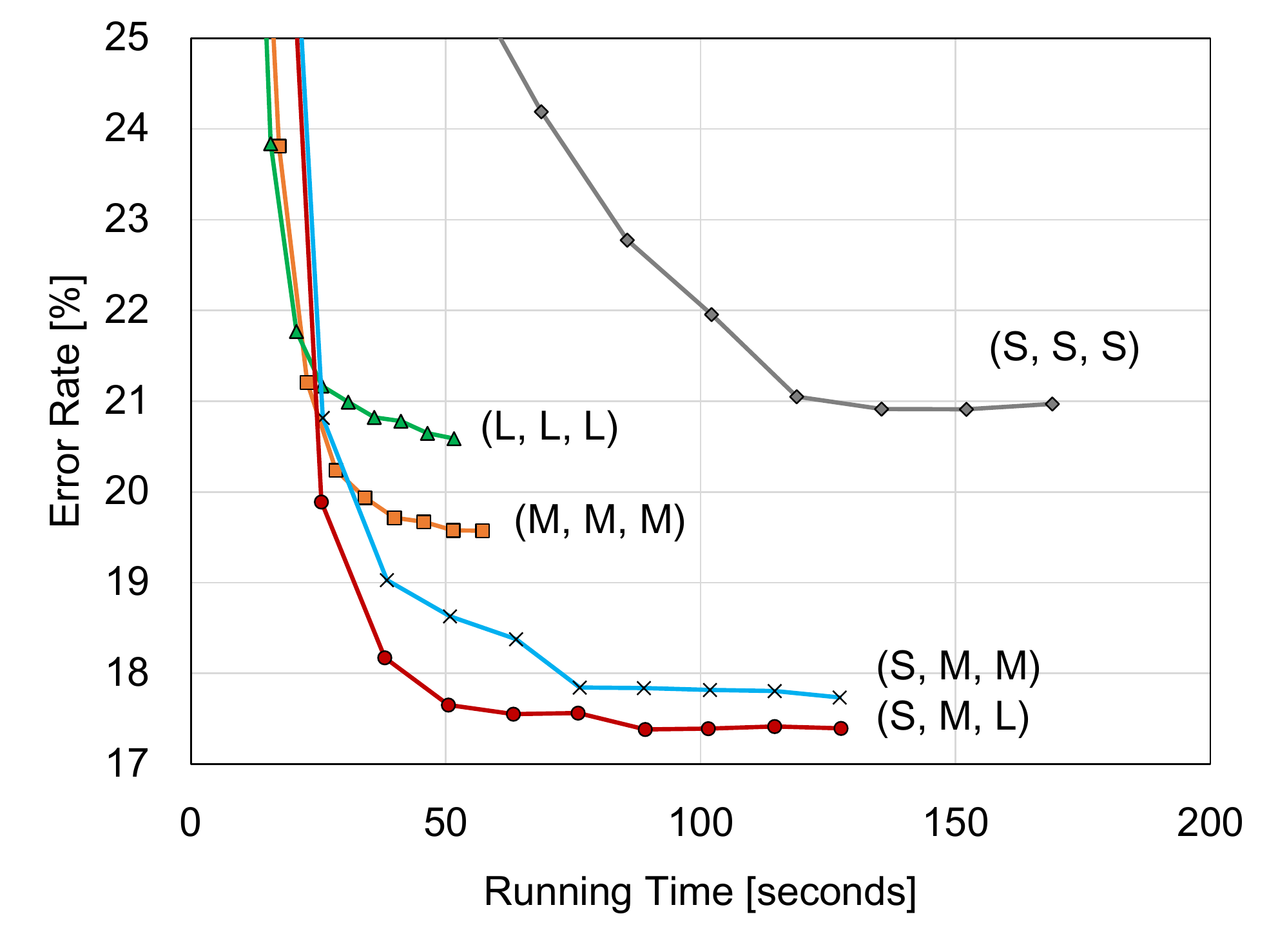}\\
				\capskip
				\small{(b) Temporal error-rate transitions}
			\end{center}
		\end{minipage}
		\caption{Effect of grid-cell sizes. We use LE-GF with different combinations of grid structures. The S, M, and L denote small, medium, and large grid-cells, respectively.
			The joint use of different sizes of grid-cells improves the performance. See also Fig.~\ref{fig:gridsize_image} for visual comparison.}
		\label{fig:regionlabel_energy}
	\end{minipage}
	\vskip 5mm
	\begin{minipage}{0.19\textwidth}
		\begin{center}
			\includegraphics[width=\textwidth]{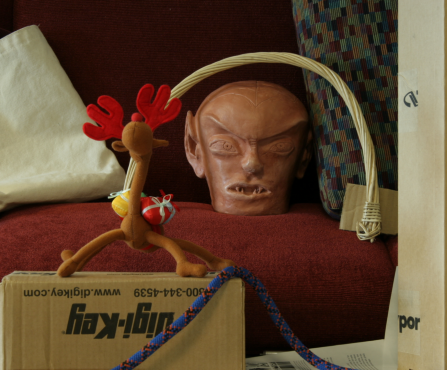}
		\end{center}
	\end{minipage}
	\hfil
	\begin{minipage}{0.19\textwidth}
		\begin{center}
			\includegraphics[width=\textwidth]{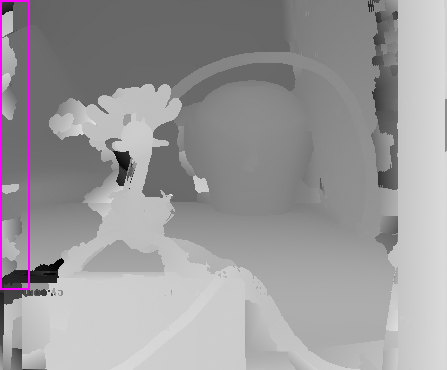}
		\end{center}
	\end{minipage}
	\hfil
	\begin{minipage}{0.19\textwidth}
		\begin{center}
			\includegraphics[width=\textwidth]{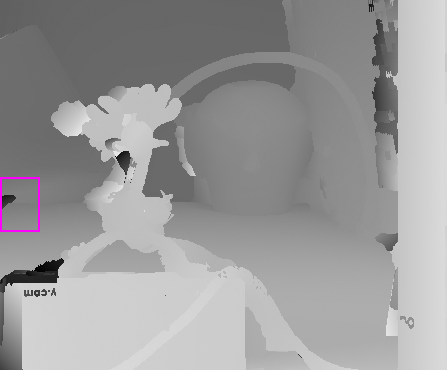}
		\end{center}
	\end{minipage}
	\hfil
	\begin{minipage}{0.19\textwidth}
		\begin{center}
			\includegraphics[width=\textwidth]{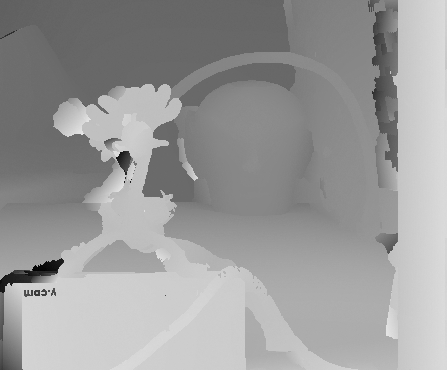}
		\end{center}
	\end{minipage}
	\hfil
	\begin{minipage}{0.19\textwidth}
		\begin{center}
			\includegraphics[width=\textwidth]{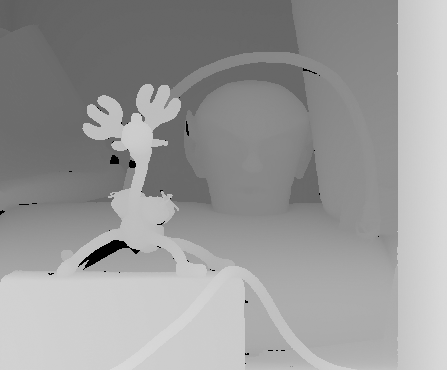}
		\end{center}
	\end{minipage}
	\vskip 1mm
	\begin{minipage}{0.19\textwidth}
		\begin{center}
			\small{\emph{Reindeer}}
		\end{center}
	\end{minipage}
	\hfil
	\begin{minipage}{0.19\textwidth}
		\begin{center}
			\small{(S, S, S)}
		\end{center}
	\end{minipage}
	\hfil
	\begin{minipage}{0.19\textwidth}
		\begin{center}
			\small{(S, M, M)}
		\end{center}
	\end{minipage}
	\hfil
	\begin{minipage}{0.19\textwidth}
		\begin{center}
			\small{\textbf{(S, M, L)}}
		\end{center}
	\end{minipage}
	\hfil
	\begin{minipage}{0.19\textwidth}
		\begin{center}
			\small{Ground truth}
		\end{center}
	\end{minipage}
	\vskip 2mm
	\begin{minipage}{0.19\textwidth}
		\begin{center}
			\includegraphics[width=\textwidth]{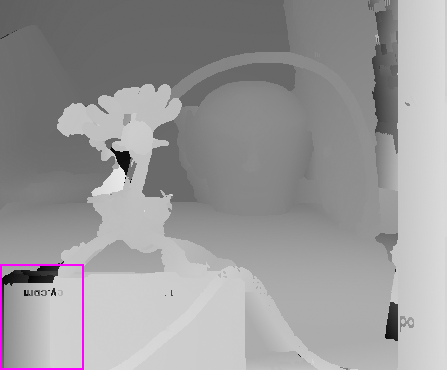}
		\end{center}
	\end{minipage}
	\hfil
	\begin{minipage}{0.19\textwidth}
		\begin{center}
			\includegraphics[width=\textwidth]{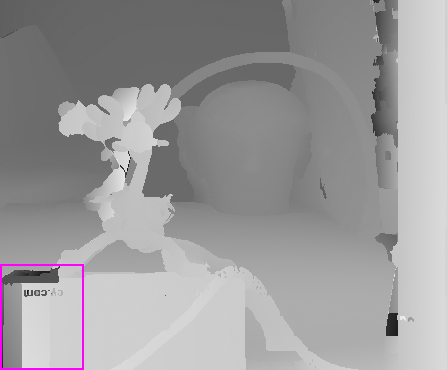}
		\end{center}
	\end{minipage}
	\hfil
	\begin{minipage}{0.19\textwidth}
		\begin{center}
			\includegraphics[width=\textwidth]{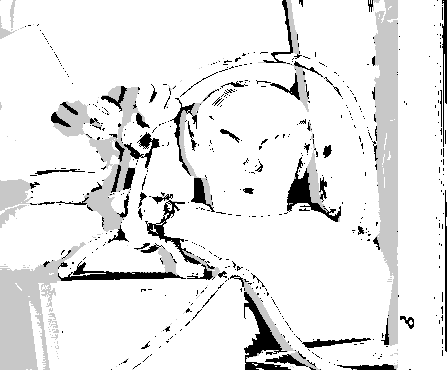}
		\end{center}
	\end{minipage}
	\hfil
	\begin{minipage}{0.19\textwidth}
		\begin{center}
			\includegraphics[width=\textwidth]{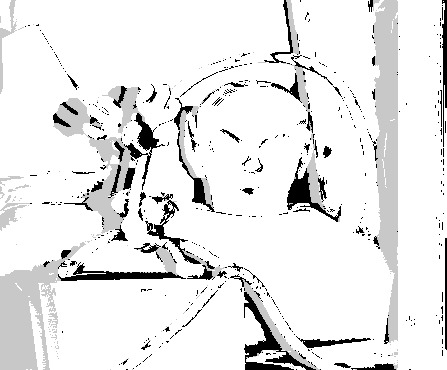}
		\end{center}
	\end{minipage}
	\hfil
	\begin{minipage}{0.19\textwidth}
		\begin{center}
			\includegraphics[width=\textwidth]{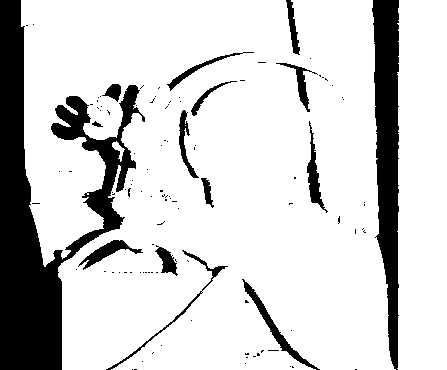}
		\end{center}
	\end{minipage}
	\vskip 1mm
	\begin{minipage}{0.19\textwidth}
		\begin{center}
			\small{(M, M, M)}
		\end{center}
	\end{minipage}
	\hfil
	\begin{minipage}{0.19\textwidth}
		\begin{center}
			\small{(L, L, L)}
		\end{center}
	\end{minipage}
	\hfil
	\begin{minipage}{0.19\textwidth}
		\begin{center}
			\small{Error map of (S, M, M)}
		\end{center}
	\end{minipage}
	\hfil
	\begin{minipage}{0.19\textwidth}
		\begin{center}
			\small{Error map of (S, M, L)}
		\end{center}
	\end{minipage}
	\hfil
	\begin{minipage}{0.19\textwidth}
		\begin{center}
			\small{Visibility map}
		\end{center}
	\end{minipage}
	\caption{Visual effect of grid-cell sizes. The use of larger grid-cells leads to smoother solutions and effective for occluded regions. The proposed combination (S, M, L) well balances localization and spatial propagation and performs best.
		These are all raw results without post-processing.}
	\label{fig:gridsize_image}
\end{figure*}

\def\plotw{0.32\textwidth}
\begin{figure*}[!t]
	\centering
	\begin{minipage}{\plotw}
		\begin{center}
			\includegraphics[width=\textwidth]{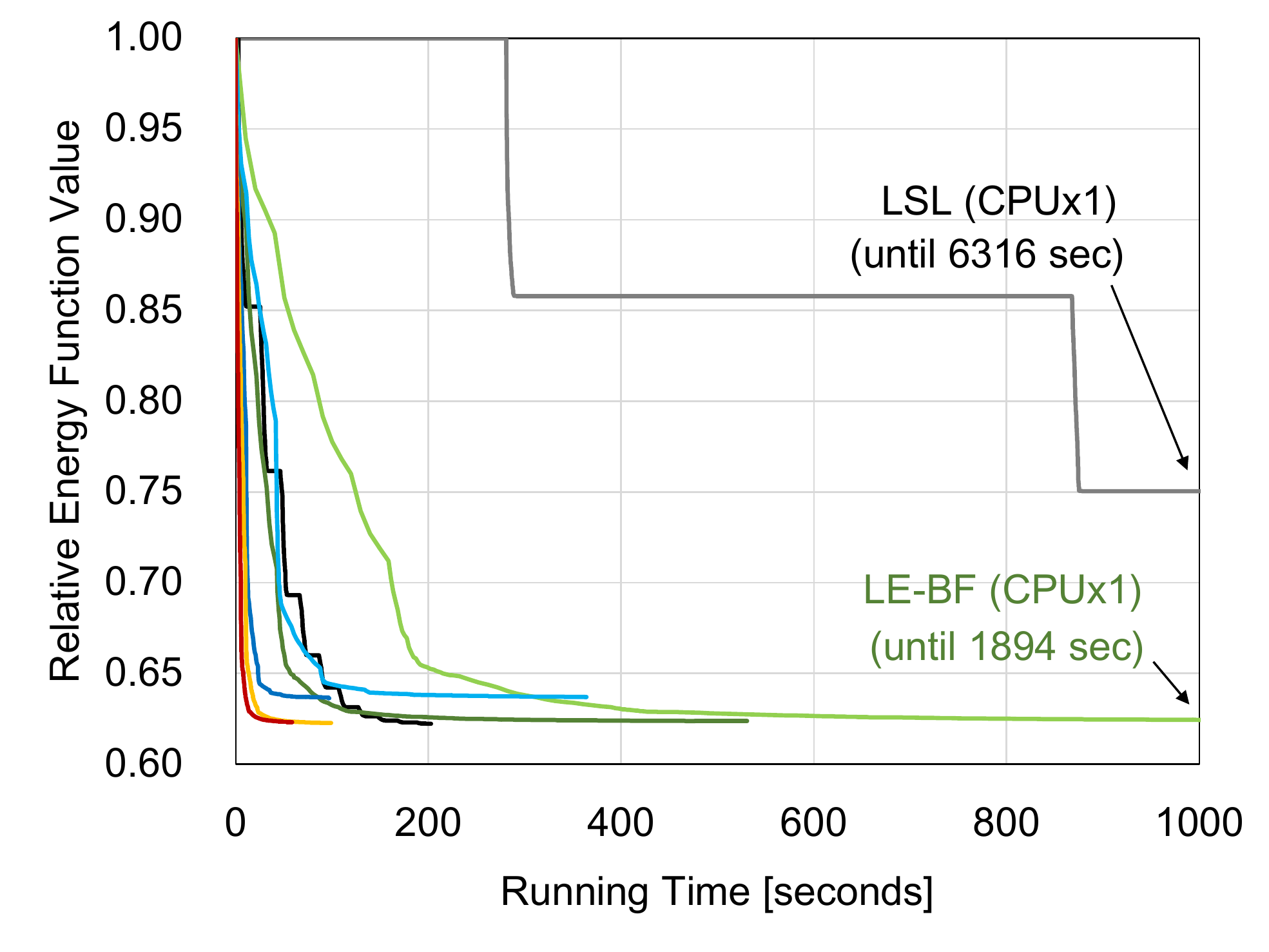}\\
			\capskip
			\small{(a) Time-energy transitions (full)}
		\end{center}
	\end{minipage}
	\hfil
	\begin{minipage}{\plotw}
		\begin{center}
			\includegraphics[width=\textwidth]{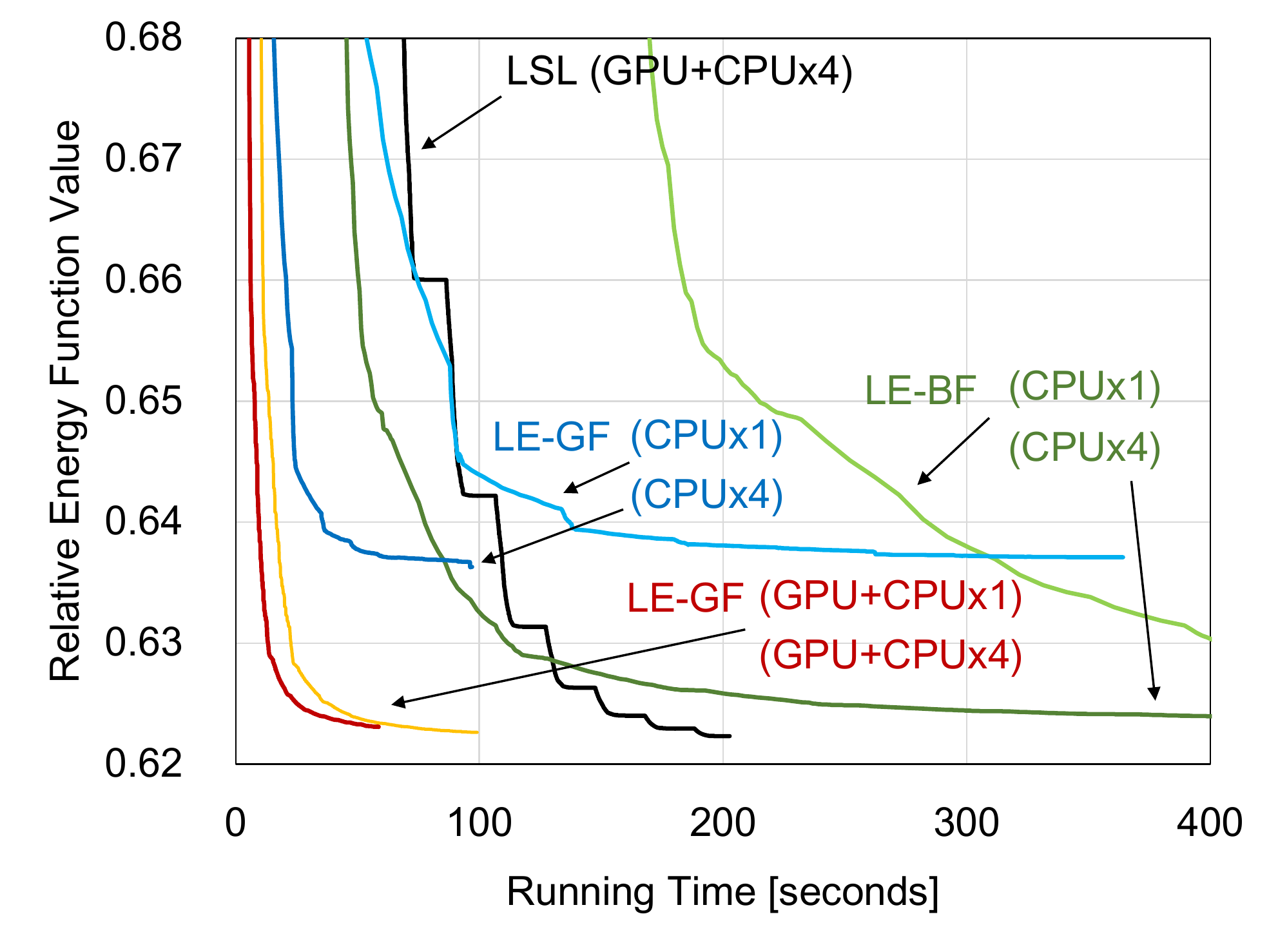}\\
			\capskip
			\small{(b) Time-energy transitions (zoomed)}
		\end{center}
	\end{minipage}
	\hfil
	\begin{minipage}{\plotw}
		\begin{center}
			\includegraphics[width=\textwidth]{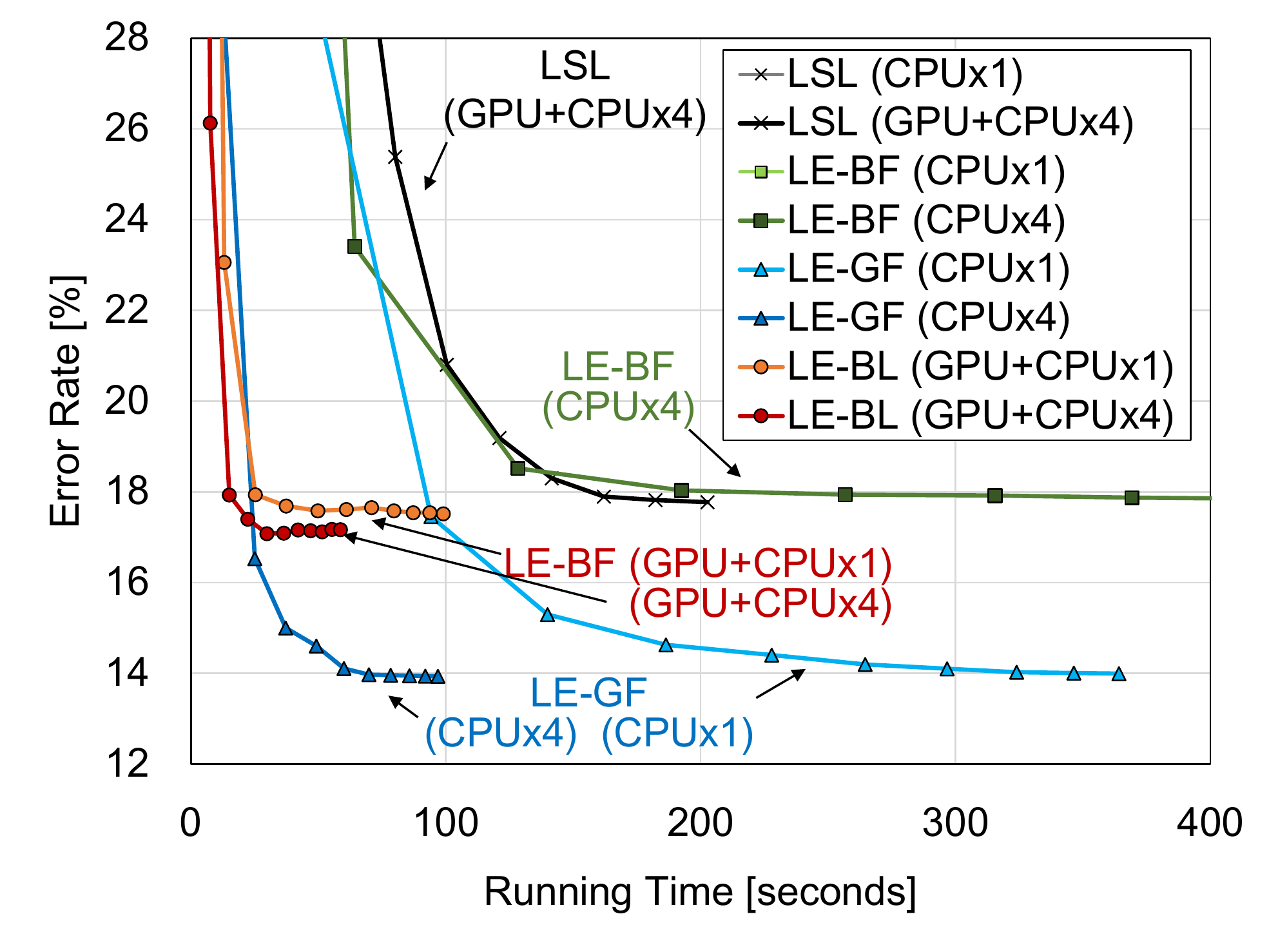}\\
			\capskip
			\small{(c) Error rate transitions}
		\end{center}
	\end{minipage}
	\caption{Efficiency evaluation in comparison to our previous algorithm (LSL) \cite{Taniai14}.
		Accuracies are evaluated for all-regions at each iteration. }
	\label{fig:self_energy}
\end{figure*}

\subsection{Efficiency evaluation in comparison to our previous algorithm~\cite{Taniai14}}
We evaluate the effectiveness of the three acceleration techniques: 1) parallelization of disjoint $\alpha$-expansions, and 2) acceleration of unary cost computation by GPU and 3) by cost filtering.
For this, we compare following six variants of our method: LE-BF and LE-GF using one or four CPU cores (denoted as CPUx1 and CPUx4), and LE-BF using GPU and one or four CPU cores (denoted as GPU+CPUx1 and GPU+CPUx4). Additionally, we compare with our previous algorithm of \cite{Taniai14} with CPU and GPU implementations (denoted as LSL with CPUx1 and GPU+CPUx4). We use the \textit{Rocks1} dataset by estimating the disparities of only the left image without post-processing.
Figures \ref{fig:self_energy}~(a)--(c) show the temporal transitions of the energy values in  the full and zoomed scales, and the subpixel error rates, respectively.
Note that energy values in Fig.~\ref{fig:self_energy} are evaluated by the energy function of LE-BF and LSL for better visualization.

\textbf{Parallelization of local $\alpha$-expansions on CPU}:
Comparing CPUx1 and CPUx4, we observe about 3.5x of speed-up for both LE-BF and LE-GF.
Note that the converged energy values of LE-GF are relatively higher than the others because it optimizes a different energy function. However, if we see Fig.~\ref{fig:self_energy}~(c), it actually finds better solutions in this example.
On the other hand, speed-up for LE-BF from GPU+CPUx1 to GPU+CPUx4 is limited to about 1.7x. This is because the unary cost computation is already parallelized by GPU and this speed-up is accounted for by the parallelization of other parts, \eg, the computation of pairwise terms and min-cut.

\textbf{Parallelization of unary cost computation on GPU}: By comparing GPU+CPUx1 and CPUx1 of LE-BF, we obtain about 19x of speed-up.
This is relatively smaller than 32x of speed-up in our previous LSL method, mainly due to the larger overhead of GPU data transferring. %  than the batch-cycle algorithm of \cite{Taniai14}.

\textbf{Fast cost filtering}: By comparing LE-GF and LE-BF methods, we observe about 5.3x speed-up for both CPUx1 and CPUx4 versions. Notably, even a CPU implementation of LE-GF (CPUx4) achieves almost comparable efficiency with a GPU implementation of LE-BF (GPU+CPUx1).

\textbf{Comparison to our previous method~\cite{Taniai14}}:
In comparison to our previous LSL method~\cite{Taniai14}, our LE-BF shows much faster convergence than LSL in the same single-core results (CPUx1). We consider there are mainly three factors contributing to this speed-up; First, the batch-cycle algorithm adopted in \cite{Taniai14} makes its convergence slower; Second, we use cost filtering and compute its first step of Eq.~(\ref{eq:rawcost}) in $O(1)$, while \cite{Taniai14} computes each unary term individually in $O(|W|)$;
Finally, we have removed a per-pixel label refinement step of \cite{Taniai14} which requires additional unary-cost computations.
Similar speed-up can be observed from LSL (GPU+CPUx4) to LE-BF (GPU+CPUx4). %Note that we  obtain only a few percents of speed-up by CPU parallelization in \cite{Taniai14}, since it does not perform min-cut computation in parallel as disjoint local $\alpha$-expansions.

\subsection{Comparison with PMBP~\cite{Besse12}}
\label{sec:pmbp}
We compare our method with PMBP~\cite{Besse12} that is the closest method to ours. For a fair comparison, we use four neighbors for $\myN$ in \eref{mrf} as the same setting with PMBP. For a comparable smoothness weight with the default setting (eight-neighbor $\myN$), we use $\lambda = 40$ for LE-BF and $\lambda = 2$ for LE-GF, and keep the other parameters as default. For PMBP, we use the same model as ours; the only difference from the original PMBP is the smoothness term, which does not satisfy the submodularity of \eref{axregularity}. In PMBP, it defines $K$ candidate labels for each pixel, for which we set $K = 1$ and $K = 5$ (original paper uses $K=5$). We show the comparison using the \textit{Cones} dataset by estimating the disparities of only the left image without post-processing.

Figures~\ref{fig:pmbp_energy}~(a)--(c) show the temporal transitions of the energy values in the full and zoomed scales, and the subpixel error rates, respectively. We show the performance of our method using its GPU and CPU (one or four CPU cores) implementations.
For PMBP, we also implemented the unary cost computation on GPU, but it became rather slow, due to the overhead of data transfer.
Efficient GPU implementations for PMBP are not available in literature.\footnote{
	GPU-parallelization schemes of BP are not directly applicable due to PMBP's unique settings. The ``jump flooding'' used in the original PatchMatch~\cite{Barnes09} reports $7$x speed-ups by GPU. However, because it propagates candidate labels to distant pixels, it is not applicable to PMBP that must propagate messages to \emph{neighbors}, and is not as efficient as our $19$x, either.}
Therefore, the plots show PMBP results that use a single CPU core. 
Figures~\ref{fig:pmbp_energy}~(a) and (b) show that, even with a single CPU-core implementation, our LE-BF and LE-GF show comparable or even faster convergence than PMBP.
With CPU and GPU parallelization, our methods achieve much faster convergence than PMBP.
Furthermore, our methods reach lower energies with greater accuracies than PMBP at the convergence.
In Fig.~\ref{fig:pmbp_image} we compare the results of our LE-GF method and PMBP with $K = 5$. While PMBP yields noisy disparities, our method finds smoother and better disparities around edges and at occluded regions.

\subsection{Comparison with PMF~\cite{Lu13}}
\label{sec:pmf}
We also compare our LE-GF method with PMF~\cite{Lu13} using the same data term as ours.
For PMF, the number of superpixels $K$ is set to $300$, $500$, and $700$ as used in \cite{Lu13} ($K=500$ is the default in \cite{Lu13}), and we sufficiently iterate 30 times.
Both LE-GF and PMF are run using a single CPU core.

Figures~\ref{fig:pmf_energy}~(a)--(c) show the temporal transitions of the energy values in the full and zoomed scales, and the subpixel error rates, respectively. %Here, the energy values are evaluated by the re-defined weights using Eq.~(\ref{eq:gf}).
As shown in Figs.~\ref{fig:pmf_energy}~(a) and (b), PMF converges at higher energies than ours, since it cannot explicitly optimize pairwise smoothness terms {because it is a local method}.
Furthermore, although energies are reduced almost monotonically in PMF, the transitions of accuracies are not stable and even degrade after many iterations. This is also because of the lack of explicit smoothness regularizer, and PMF converges at a bad local minimum.
Figure~\ref{fig:pmf_image} compares the results of our method and PMF with $K=500$. Again, our methods find smoother and better disparities  around edges and at occluded regions.

\begin{figure*}[!t]
	\centering
	\begin{minipage}{\plotw}
		\begin{center}
			\includegraphics[width=\textwidth]{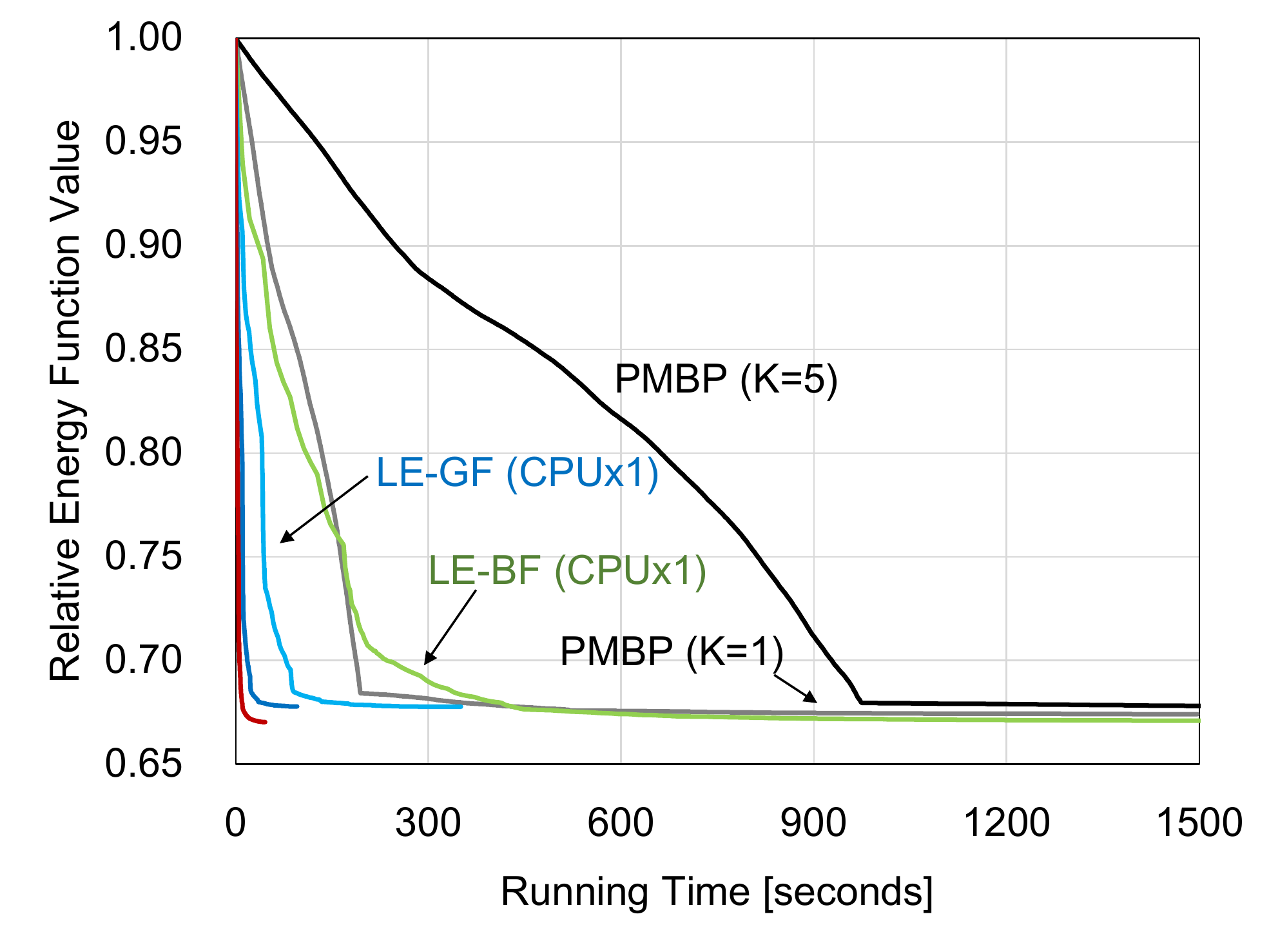}\\
			\capskip
			\small{(a) Time-energy transitions (full)}
		\end{center}
	\end{minipage}
	\hfil
	\begin{minipage}{\plotw}
		\begin{center}
			\includegraphics[width=\textwidth]{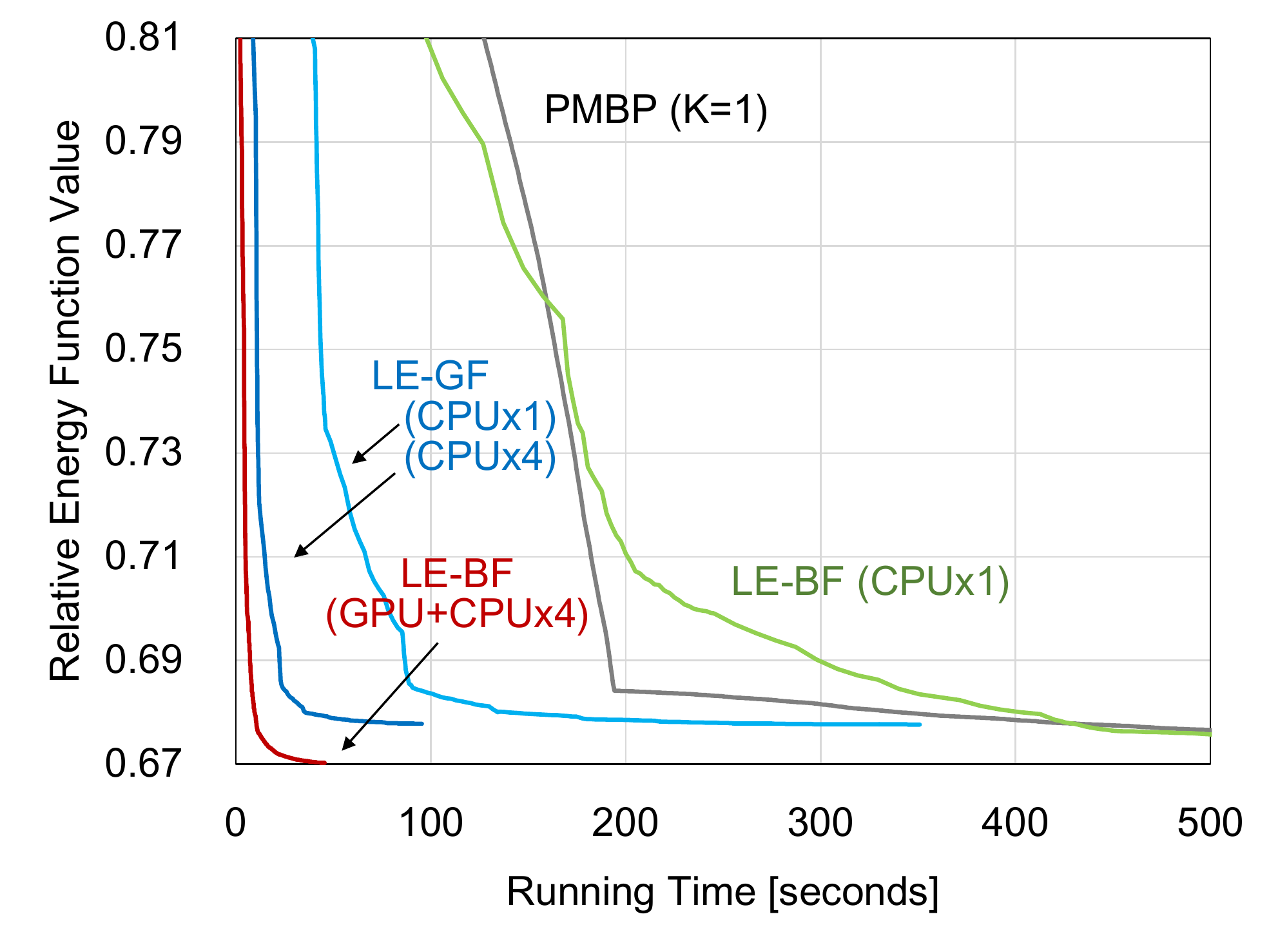}\\
			\capskip
			\small{(b) Time-energy transitions (zoomed)}
		\end{center}
	\end{minipage}
	\hfil
	\begin{minipage}{\plotw}
		\begin{center}
			\includegraphics[width=\textwidth]{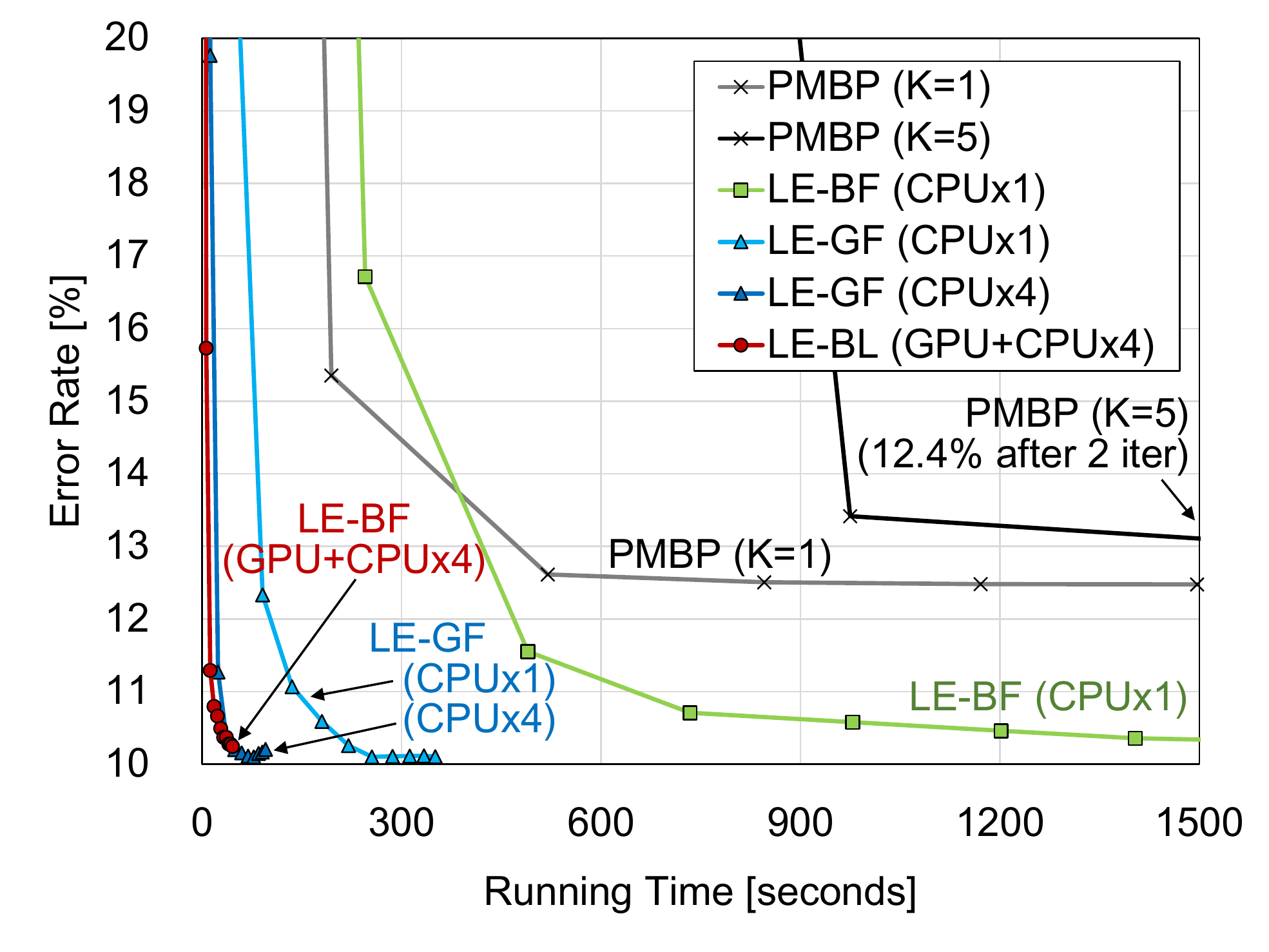}\\
			\capskip
			\small{(c) Error rate transitions}
		\end{center}
	\end{minipage}
	\caption{Efficiency and accuracy comparison with PMBP \cite{Besse12}. Our methods achieve much faster convergence, reaching lower energies and better accuracies at the convergence.
		Accuracies are evaluated for all-regions at each iteration. See also Fig.~\ref{fig:pmbp_image} for visual comparison.}
	\label{fig:pmbp_energy}
	\vskip 5mm
	% 380 85 450 205,
	\begin{minipage}{0.37\textwidth}
		\begin{center}
			\includegraphics[height=3.5cm]{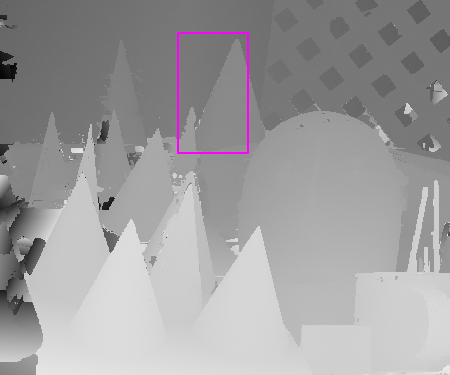} 
			\includegraphics[height=3.5cm, viewport=178 222 248 342, clip]{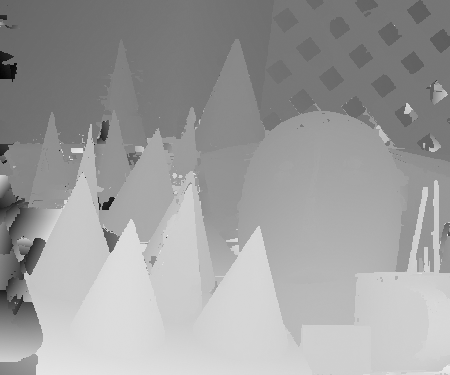}\\
			\small{(a) Our LE-BF}
		\end{center}
	\end{minipage}
	\hfil
	\begin{minipage}{0.37\textwidth}
		\begin{center}
			\includegraphics[height=3.5cm]{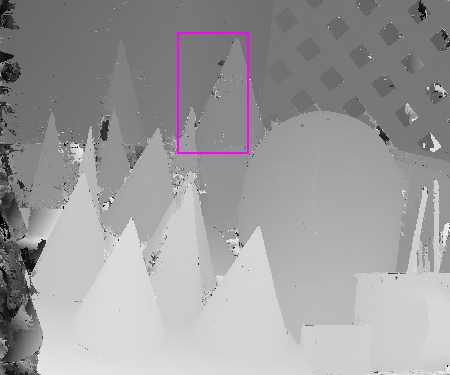} 
			\includegraphics[height=3.5cm, viewport=178 222 248 342, clip]{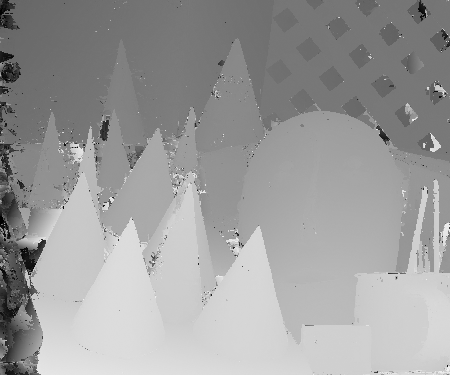}\\
			\small{(b) PMBP~\cite{Besse12}}
		\end{center}
	\end{minipage}
	\hfil
	\begin{minipage}{0.25\textwidth}
		\begin{center}
			\includegraphics[height=3.5cm, viewport=178 222 248 342, clip]{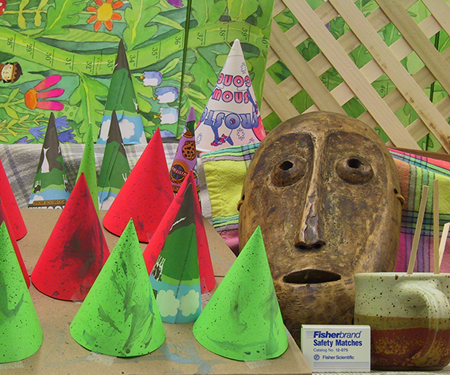} 
			\includegraphics[height=3.5cm, viewport=178 222 248 342, clip]{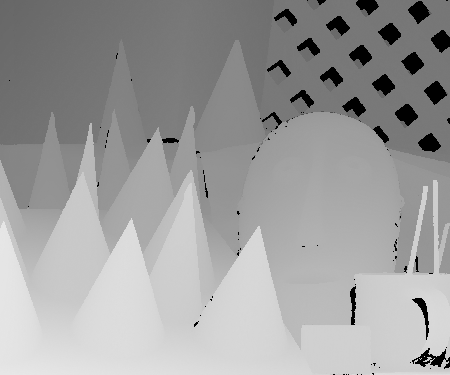}\\
			\small{(c) Ground truth}
		\end{center}
	\end{minipage}
	\caption{Visual comparison with PMBP~\cite{Besse12}. Our method finds smother and better disparities  around edges and at occluded regions.}
	\label{fig:pmbp_image}
\end{figure*}

\begin{figure*}[!t]
	\centering
	\begin{minipage}{\plotw}
		\begin{center}
			\includegraphics[width=\textwidth]{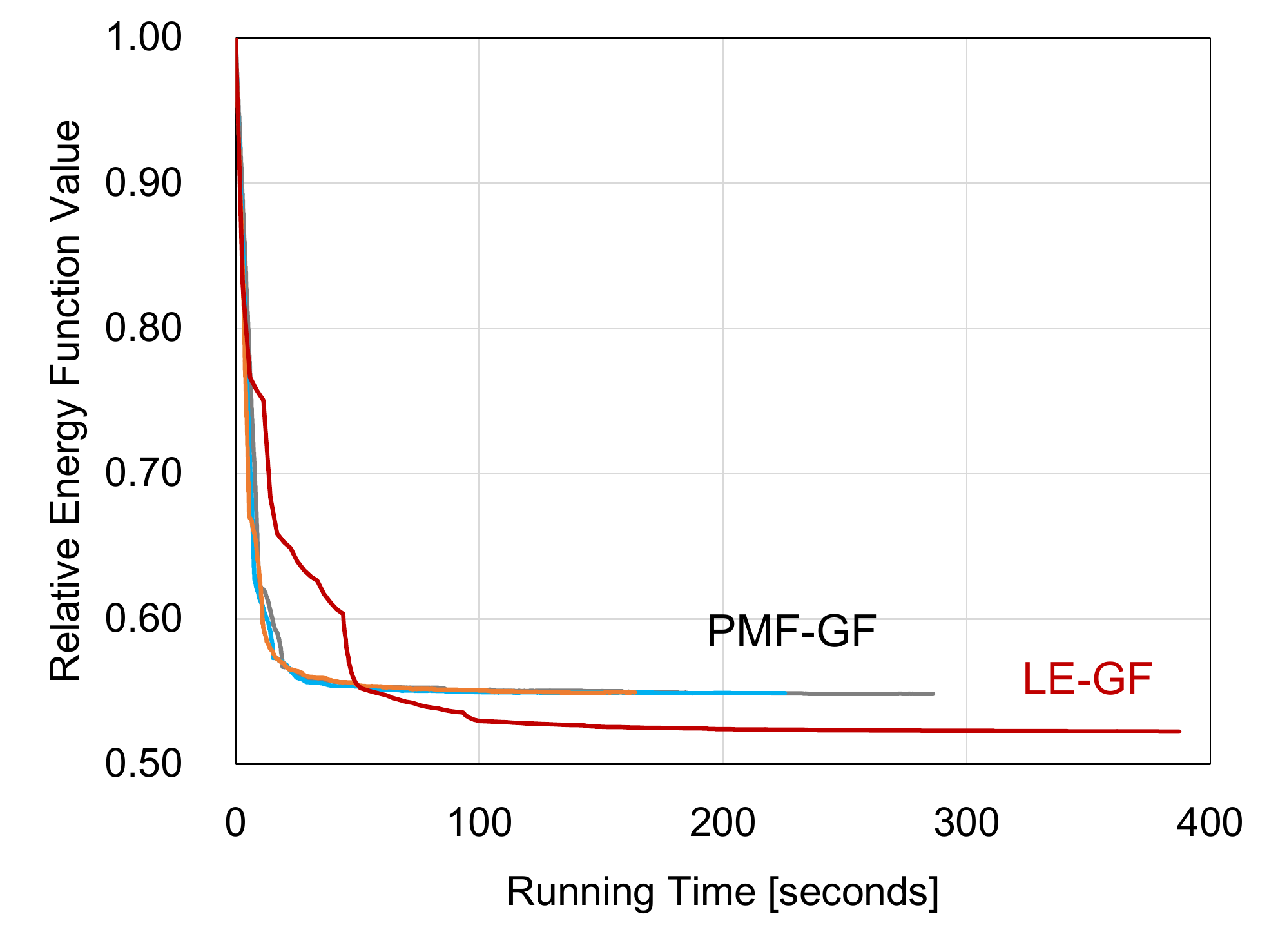}\\
			\capskip
			\small{(a) Time-energy transitions (full)}
		\end{center}
	\end{minipage}
	\hfil
	\begin{minipage}{\plotw}
		\begin{center}
			\includegraphics[width=\textwidth]{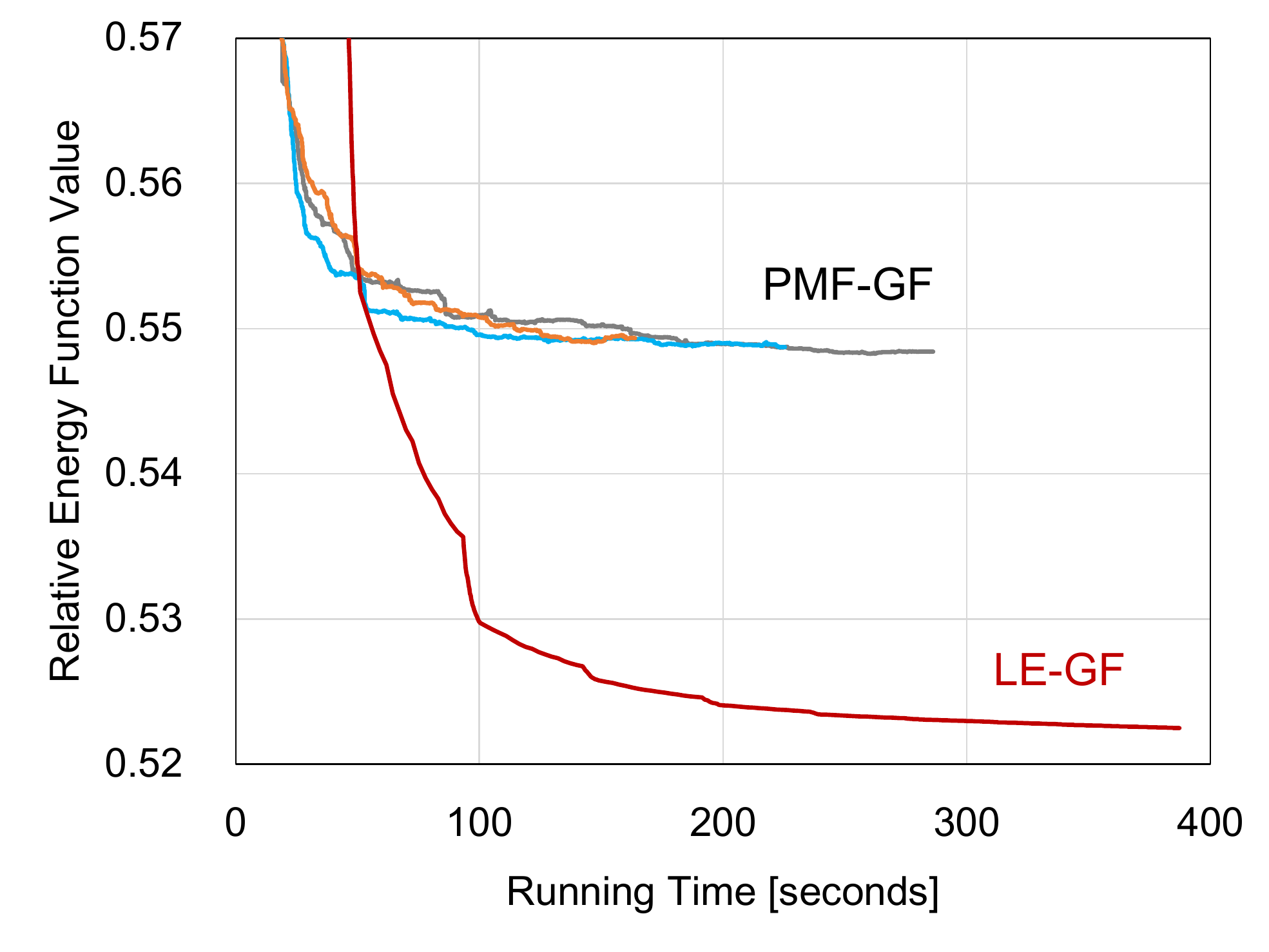}\\
			\capskip
			\small{(b) Time-energy transitions (zoomed)}
		\end{center}
	\end{minipage}
	\hfil
	\begin{minipage}{\plotw}
		\begin{center}
			\includegraphics[width=\textwidth]{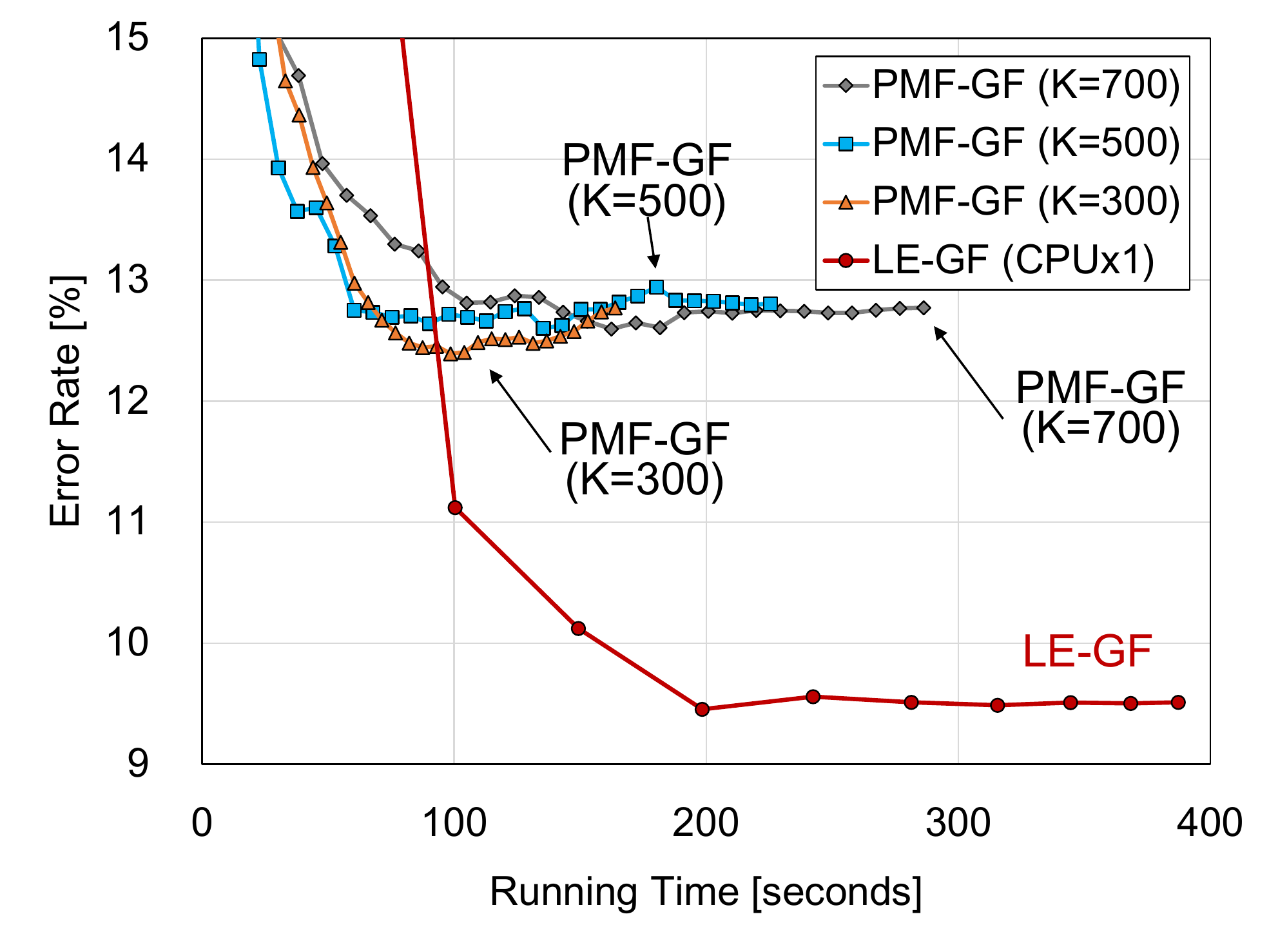}\\
			\capskip
			\small{(c) Error rate transitions}
		\end{center}
	\end{minipage}
	\caption{Efficiency and accuracy comparison with PMF \cite{Lu13}. Our method stably improves the solution and reaches a lower energy with greater accuracy at the convergence.
		Accuracies are evaluated for all-regions at each iteration. See also Fig.~\ref{fig:pmf_image} for visual comparison.}
	\label{fig:pmf_energy}
	\vskip 5mm
	\begin{minipage}{0.37\textwidth}
		\begin{center}
			\includegraphics[height=3.5cm]{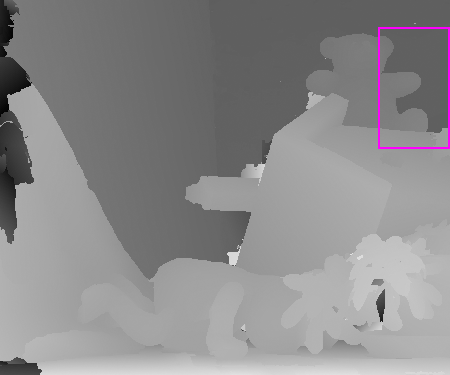} 
			\includegraphics[height=3.5cm, viewport=375 226 450 346, clip]{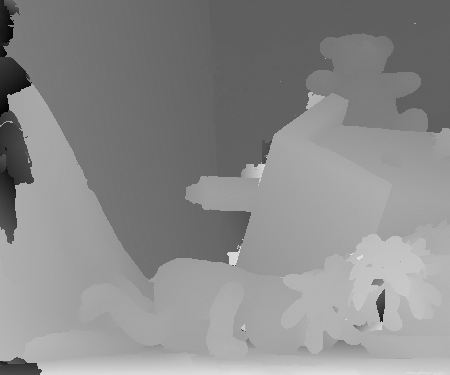}\\
			\small{(a) Our LE-GF}
		\end{center}
	\end{minipage}
	\hfil
	\begin{minipage}{0.37\textwidth}
		\begin{center}
			\includegraphics[height=3.5cm]{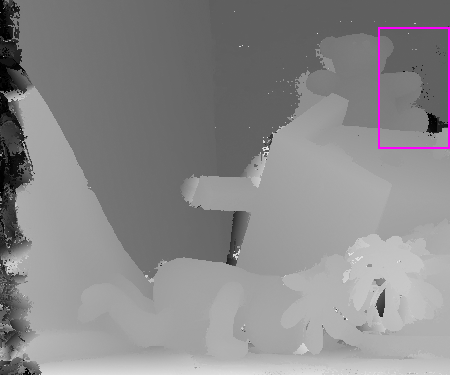} 
			\includegraphics[height=3.5cm, viewport=375 226 450 346, clip]{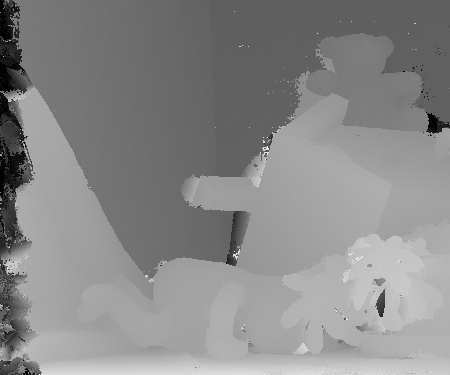}\\
			\small{(b) PMF~\cite{Lu13}}
		\end{center}
	\end{minipage}
	\hfil
	\begin{minipage}{0.25\textwidth}
		\begin{center}
			\includegraphics[height=3.5cm, viewport=375 226 450 346, clip]{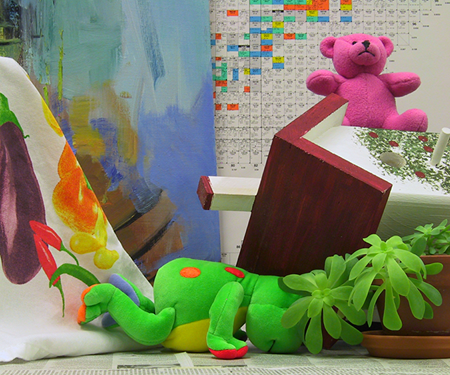} 
			\includegraphics[height=3.5cm, viewport=375 226 450 346, clip]{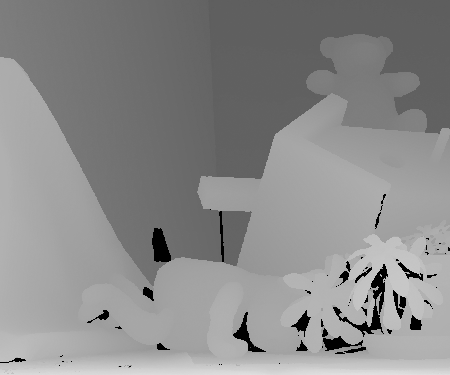}\\
			\small{(c) Ground truth}
		\end{center}
	\end{minipage}
	\caption{Visual comparison with PMF~\cite{Lu13}. With explicit smoothness regularization, our method finds smoother and better disparities  around edges and at occluded regions.}
	\label{fig:pmf_image}
\end{figure*}

\begin{table*}[t]
	\footnotesize
	\newcolumntype{C}[1]{>{\centering}m{#1}}
	\caption{{Accuracy comparison with Olsson~\etal~\cite{Olsson13}. We use our energy function using 3D patch matching for our method and \cite{Olsson13}.} \emph{Teddy} dataset is from Middlebury V2 (Sec.~\ref{sec:benchmark}) and the other three from V3 (Sec.~\ref{sec:benchmark_v3}). Our local expansion move method (LE) outperforms fusion-based optimization of \cite{Olsson13}. While our method has comparably good accuracies regardless of the RANSAC proposer, it is effective for large texture-less planes in \emph{Vintage}. }
	\label{tab:olsson}
	\begin{center}\begin{tabular}{|l||c|c||c|c|c||c|c||c|c||c|c|}
			\hline 
			&  \multicolumn{2}{c||}{Options}	& \multicolumn{3}{c||}{Teddy (from V2)} & \multicolumn{2}{c||}{Adirondack} & \multicolumn{2}{c||}{ArtL} & \multicolumn{2}{c|}{Vintage} \\ 
			%	\cline{5-11}
			{}  & {Reg.} & {Optimization}   & {nonocc} & {all}  & {disc} & {nonocc} & {all} & {nonocc} & {all} & {nonocc} & {all}\\
			\hline 
			\hline 
			Ours ({proposed LE-GF})  & \cite{Olsson13}  & LE+RANSAC & \textbf{3.98} &  9.46 & \textbf{12.8} & \textbf{1.19} & \textbf{4.36} & \textbf{3.97} & \textbf{15.5} &  \textbf{5.65} & \textbf{11.7} \\ 
			\hline 
			Ours w/o RANSAC  & \cite{Olsson13}  & LE & 4.26 & \textbf{9.19}  & 13.6 & \textbf{1.20} & {4.63} & \textbf{3.98} &  \textbf{15.3} & 13.7 & 18.7 \\ 
			\hline
			Ours w/o regularization & None  &  LE+RANSAC & 5.47 & 13.0 & 16.0 & 2.56 & 7.74 & 4.12 &  17.6 & 22.8 & 27.9 \\ 
			\hline 
			Olsson~\etal~\cite{Olsson13} + 3DPM  & \cite{Olsson13}  & Fusion+RANSAC & 5.21 & 9.98 &  15.8 & 2.74 & 5.42 & 4.41 & 15.8 & 5.94 & \textbf{11.8}  \\ 
			\hline 
	\end{tabular}\end{center}
\end{table*}

\begin{figure*}[!t]
	\begin{minipage}[!t]{\linewidth}
		\begin{minipage}[!t]{0.315\linewidth}\centering
			\includegraphics[width=\linewidth]{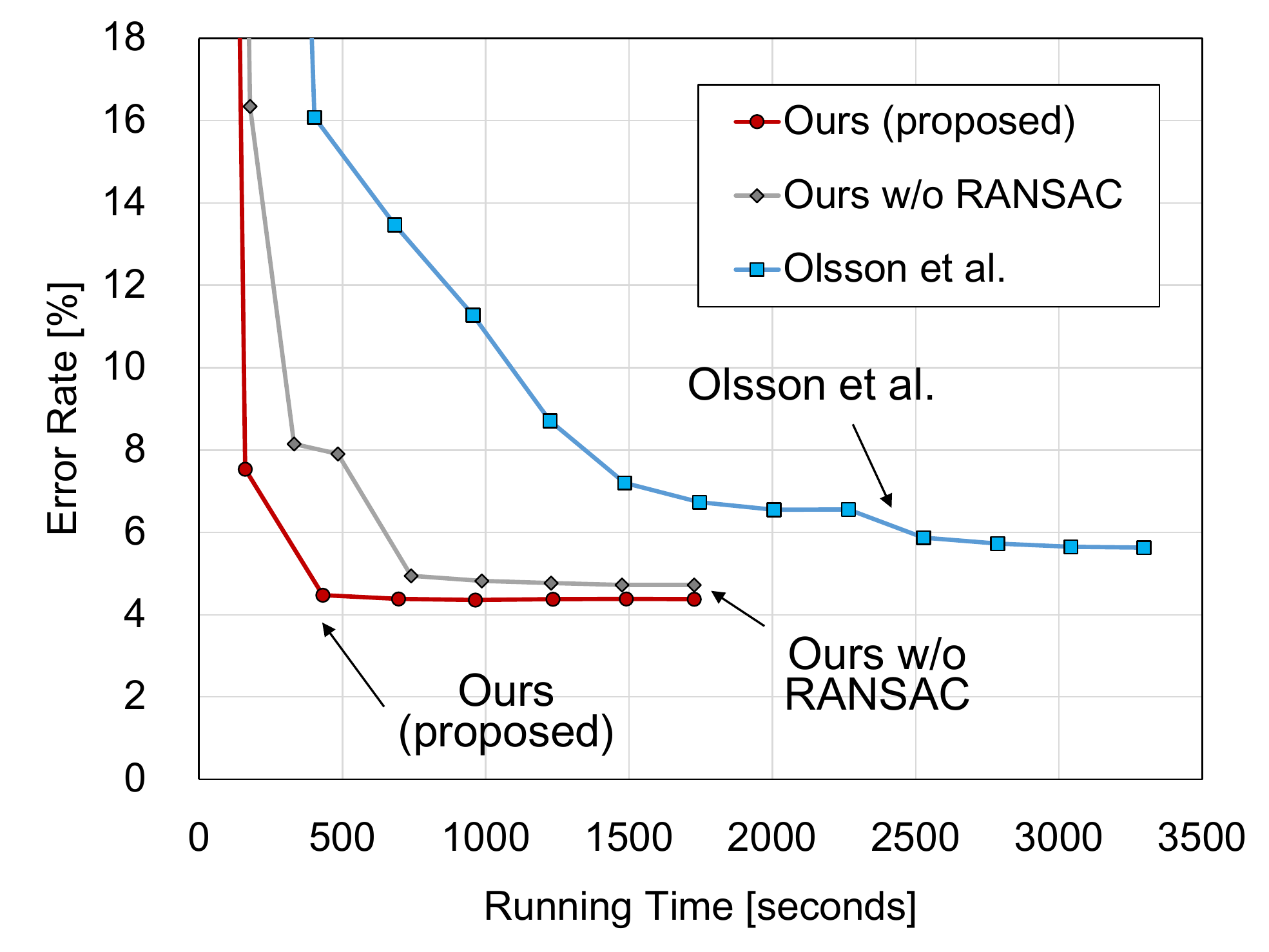}\\\small{(For \emph{Andirondack}, by a single CPU core)}\end{minipage}\hfill
		\begin{minipage}[!t]{0.315\linewidth}\centering
			\begin{minipage}[!t]{0.45\linewidth}\begin{center}\includegraphics[width=\linewidth]{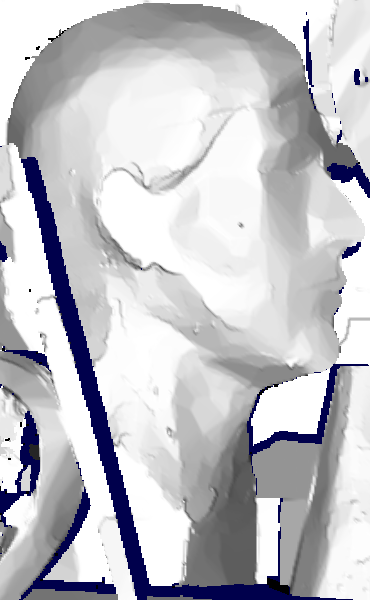}\\\small{Our LE-GF}\end{center}\end{minipage}\hfil
			\begin{minipage}[!t]{0.45\linewidth}\begin{center}\includegraphics[width=\linewidth]{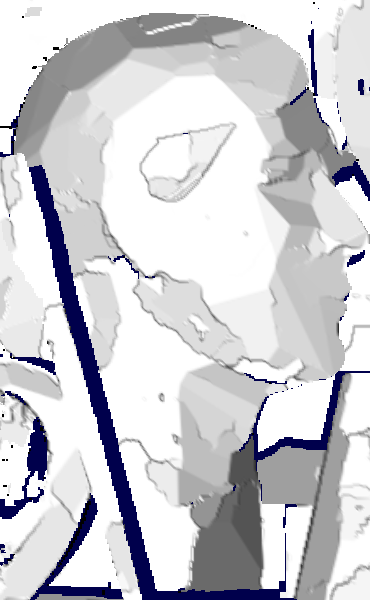}\\\small{Olsson~\etal~\cite{Olsson13}}\end{center}\end{minipage}
		\end{minipage}\hfill
		\begin{minipage}[!t]{0.315\linewidth}%
			\begin{minipage}[!t]{0.49\linewidth}\begin{center}\includegraphics[width=\linewidth]{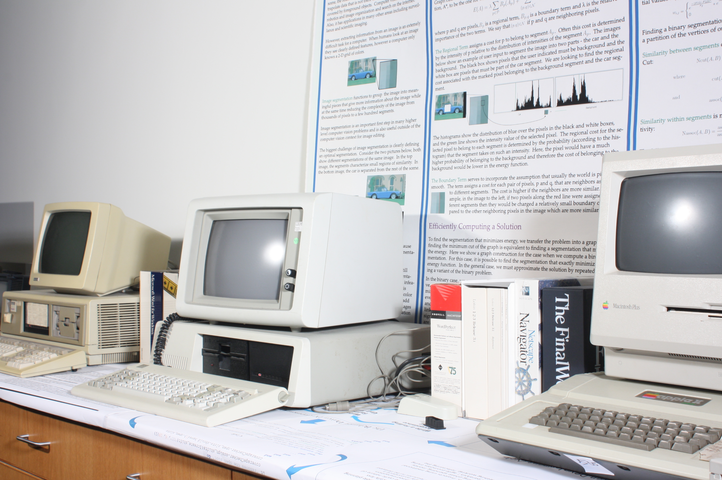}\\\small{\textit{Vintage}}\end{center}\end{minipage}\hfill
			\begin{minipage}[!t]{0.49\linewidth}\begin{center}\includegraphics[width=\linewidth]{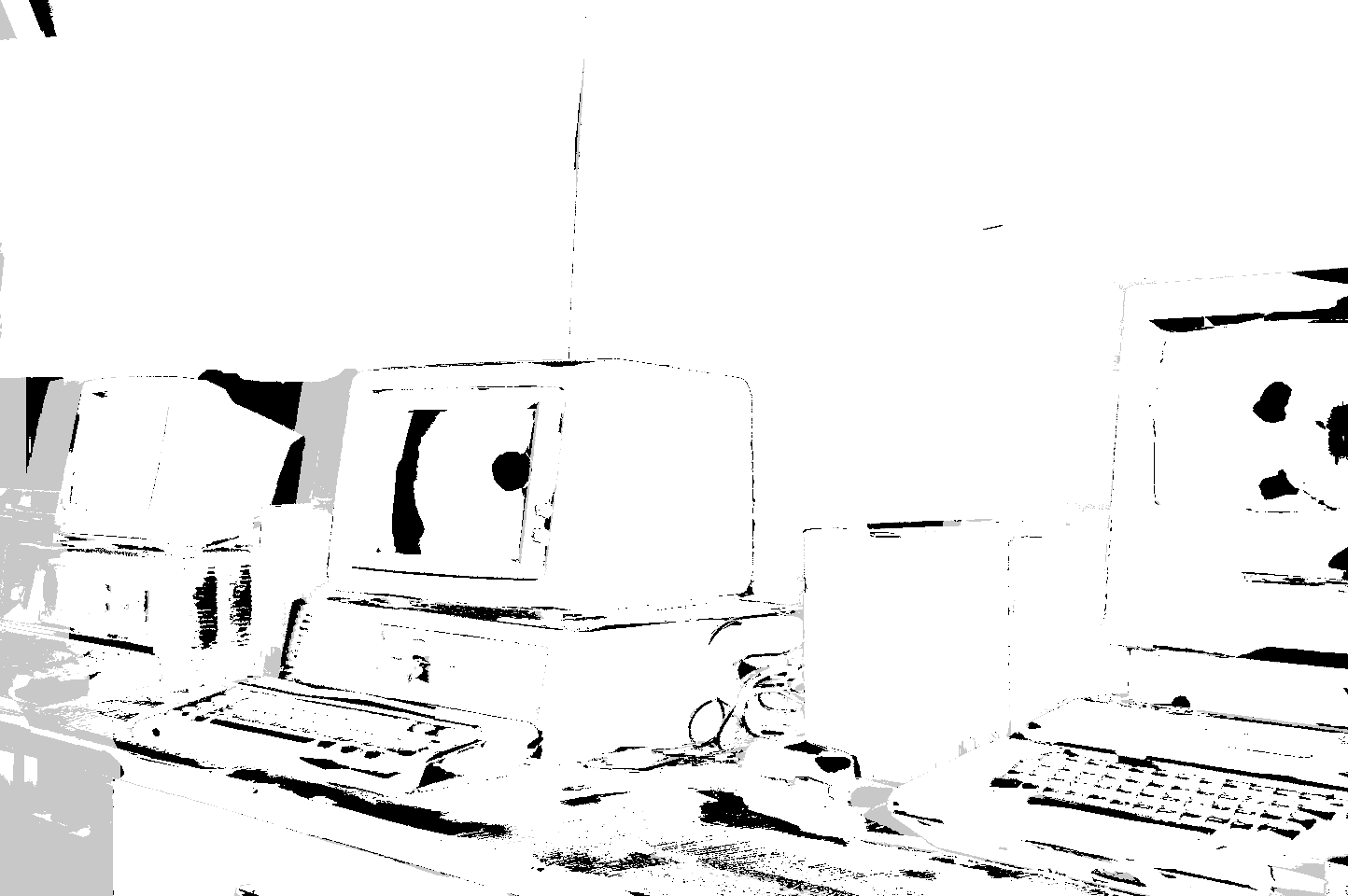}\\\small{Ours (proposed)}\end{center}\end{minipage}\\
			\begin{minipage}[!t]{0.49\linewidth}\begin{center}\includegraphics[width=\linewidth]{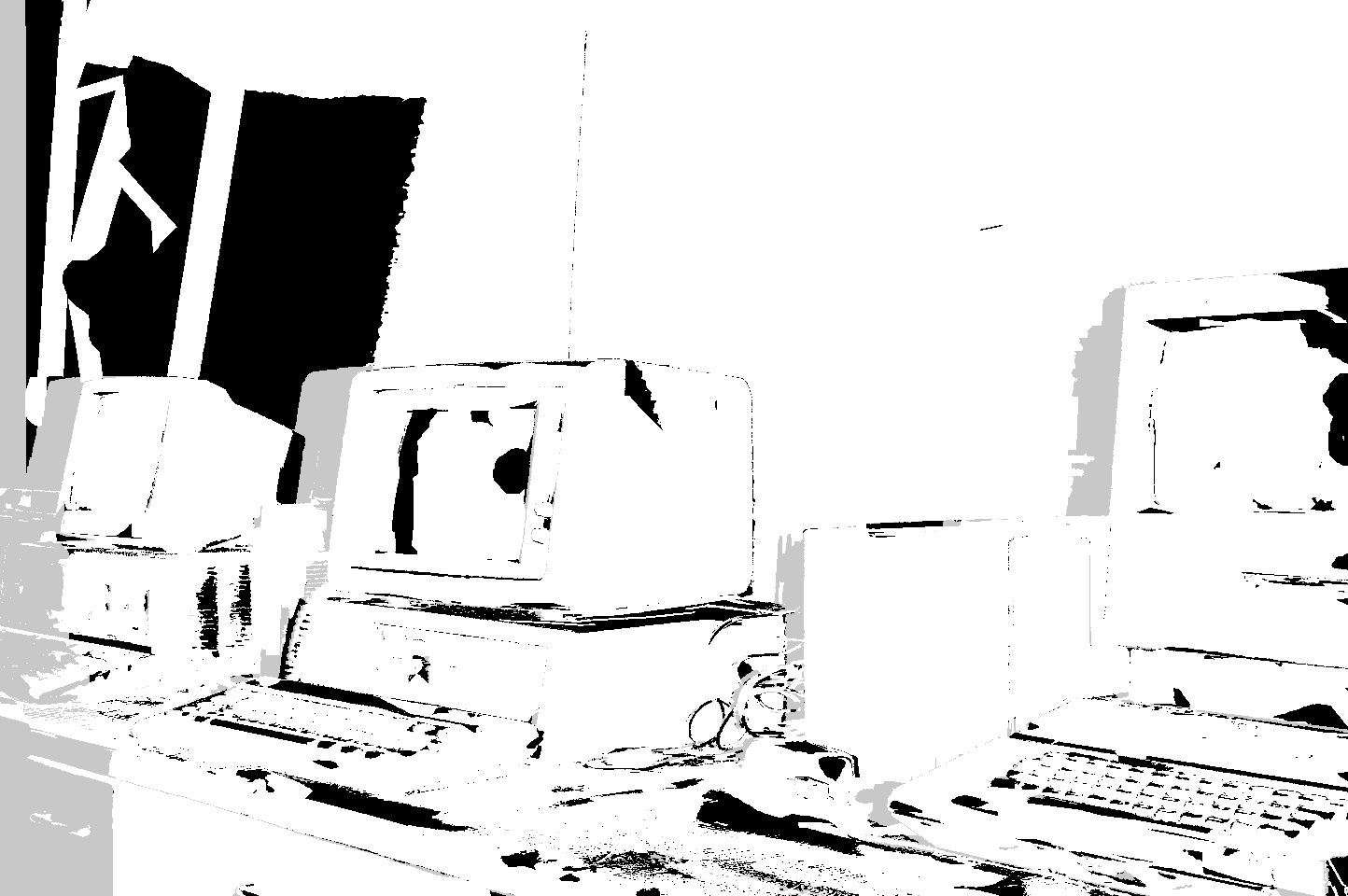}\\\small{Ours~w/o~RANSAC}\end{center}\end{minipage}\hfill
			\begin{minipage}[!t]{0.49\linewidth}\begin{center}\includegraphics[width=\linewidth]{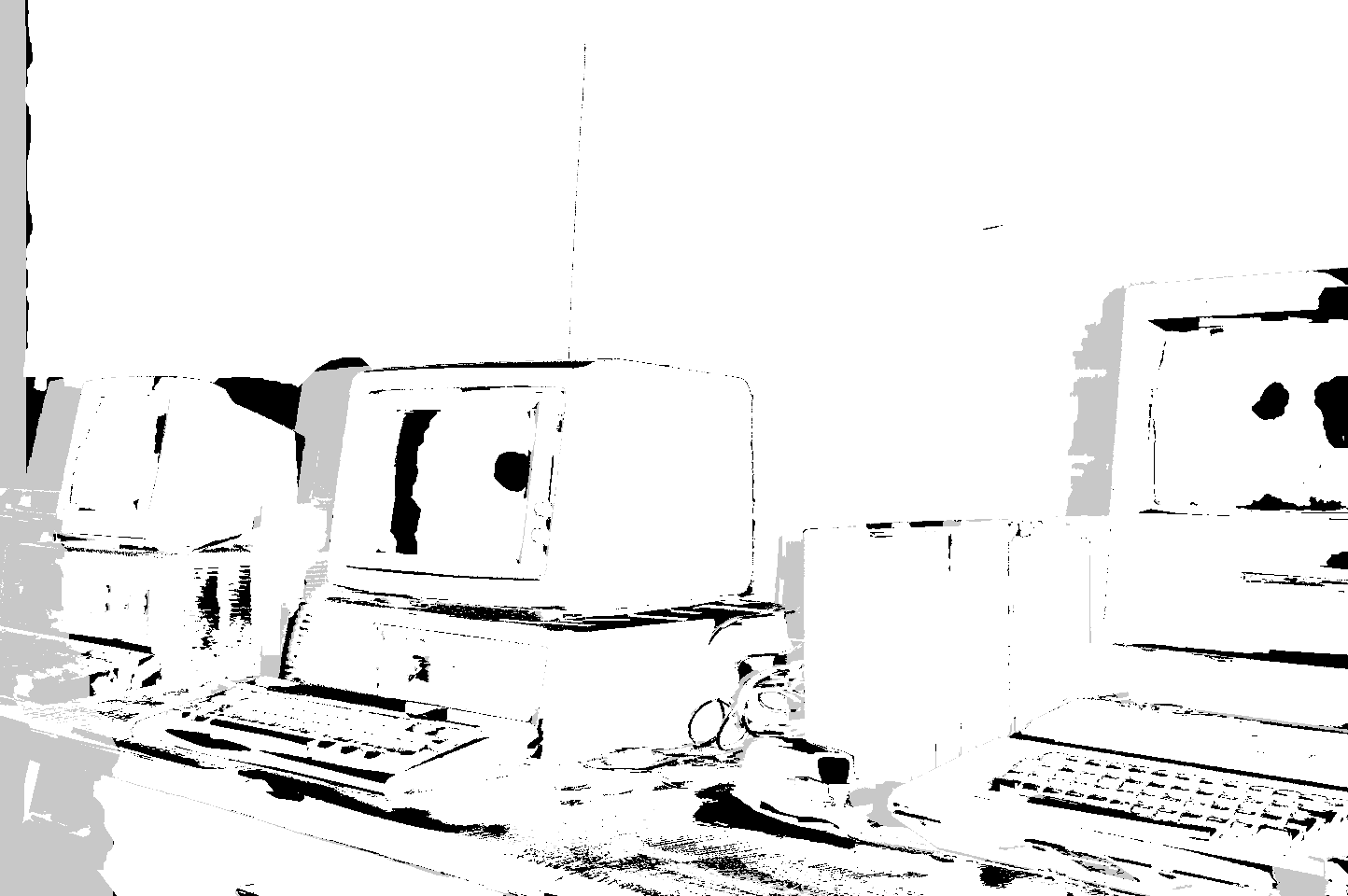}\\\small{Olsson~\etal~\cite{Olsson13}}\end{center}\end{minipage}
		\end{minipage}
	\end{minipage}
	\begin{minipage}[t]{0.315\linewidth}\begin{center}\caption{Efficiency and accuracy comparison with Olsson~\etal~\cite{Olsson13}. Our method is faster than \cite{Olsson13} even without parallelization. The RANSAC proposer  promotes the inference.}\label{fig:olsson_plots}\end{center}\end{minipage}\hfill
	\begin{minipage}[t]{0.315\linewidth}
		\caption{{3D visual comparison with Olsson~\etal~\cite{Olsson13} for \textit{ArtL} dataset. Although we use the same energy function,} the result by \cite{Olsson13} is strongly biased to piecewise planar surfaces.}
		\label{fig:art_3d}
	\end{minipage}\hfill
	\begin{minipage}[t]{0.315\linewidth}
		\caption{Effectiveness of the RANSAC-based plane proposer for large texture-less regions. We show error maps of our method and Olsson~\etal~\cite{Olsson13} for \textit{Vintage} dataset.}
		\label{fig:vintage}
	\end{minipage}
\end{figure*}

\subsection{{Comparison with Olsson~\etal~\cite{Olsson13}}}
\label{sec:olsson}
We compare with a stereo method by Olsson~\etal~\cite{Olsson13}, from which we borrow our smoothness term.
As differences from our method, their data term does not use accurate slanted patch matching %\footnote{\rev{\cite{Olsson13} pre-computes matching costs by normalized cross correlation with fronto-parallel patches for all integer disparities, and then interpolates the cost function at non-integer disparities by local quadric fitting.}}
and they use conventional fusion-based optimization using planar proposals generated by RANSAC-based local plane fitting (see \cite{Olsson13} for details).
{To compare optimization approaches of ours and \cite{Olsson13}, we replace the data term of \cite{Olsson13} with ours so that the two methods use the same energy function as ours.}
We here use {two variants of our LE-GF method that remove the RANSAC proposer and the regularization term}, which are components closely related to \cite{Olsson13}.
We disable post-processing in all methods.

Table~\ref{tab:olsson} shows error rates for \emph{Teddy} from Middlebury V2 and three more datasets from V3,
where the proposed method outperforms \cite{Olsson13} and others.
{As shown in Figs.~\ref{fig:olsson_plots} and \ref{fig:art_3d}, our method has faster convergence and produces qualitatively better smooth surfaces than \cite{Olsson13}. Spatial label propagation not only accelerates inference but also introduces an implicit smoothness effect~\cite{Bleyer11}, which helps our method to find smoother surfaces.}
%Spatial label propagation has an implicit smoothness effect~\cite{Bleyer11}, which helps our method to find smoother surfaces.}
%our method converges faster and produces qualitatively better smooth surfaces than \cite{Olsson13} thanks to local propagation similar to \cite{Bleyer11}.
Our method  mostly performs comparably {well} regardless of the RANSAC proposer, outperforming \cite{Olsson13}.
On the other hand, the RANSAC proposer promotes convergence (see Fig.~\ref{fig:olsson_plots}) and also effectively recovers large texture-less planes in \emph{Vintage} (Fig.~\ref{fig:vintage}).
%Also note that, except for \emph{Vintage} in Table~\ref{tab:olsson}, our method without regularization shows comparable accuracies with \cite{Olsson13} for non-occluded regions.
{Using a more advanced proposer~\cite{Olsson14} may lead to further improvements.}
Also note that, except for \emph{Vintage} in Table~\ref{tab:olsson}, our method performs well for non-occluded regions even without regularization  (as comparable to \cite{Olsson13}).
The slanted patch matching term alone has an implicit second-order smoothness effect~\cite{Bleyer11}, and the regularizer of \cite{Olsson13} further enforces it especially for occluded or texture-less regions. 
%Other candidate proposers such as local refinement by ADMM~\cite{Olsson14} can be also used.

\section{Conclusions}
In this paper, we have presented an accurate and efficient stereo matching method for continuous 3D plane label estimation. Unlike previous approaches that use fusion moves~\cite{Lempitsky10,Olsson13,Woodford09}, our method is subproblem optimal and only requires a randomized initial solution. %Their proposal generation heuristics can add further improvements when incorporated in our method.}
By comparing with a recent continuous MRF stereo method, PMBP~\cite{Besse12}, our method has shown an advantage in efficiency and comparable or greater accuracy. The use of a GC-based optimizer makes our method advantageous.

Furthermore, by using the subregion cost-filtering scheme developed in a local stereo method PMF~\cite{Lu13}, we achieve a fast algorithm and greater accuracy than PMF.
As shown in \cite{Lu13}, our method will be even faster by using a more light-weight constant-time filtering such as \cite{Lu12}.

{Follow-up works~\cite{Taniai16,Hur17}} suggest that our optimization strategy can be used for more general correspondence field estimation than stereo. We also believe that occlusion handling schemes using GC~\cite{Kolmogorov01,Wei05} can be incorporated into our framework, which may yield even greater accuracy.

\section*{Acknowledgments}
The authors would like to thank the anonymous reviewers including those of the conference paper.
% for their helpful comments needed to improve this paper.
This work was supported by JSPS KAKENHI Grant Number 14J09001 and Microsoft Research Asia Ph.D. Fellowship.

{
	\bibliographystyle{ieee}
	\bibliography{egbib}
}

\newpage
\clearpage
\onecolumn

\setcounter{figure}{0}
\setcounter{equation}{0}  \renewcommand\thefigure{A\arabic{figure}}
\renewcommand\theequation{A\arabic{equation}}

\makeatletter
\def\@thanks{}
\makeatother

\title{Supplementary Material for \\Continuous 3D Label Stereo Matching using\\ Local Expansion Moves}

% make the title area
\maketitle
%\IEEEdisplaynontitleabstractindextext
%\IEEEpeerreviewmaketitle

\appendices
\section{Proof of Lemma 1}
\label{sec:app_submodularity}
\begin{proof}
	For the completeness, we repeat the original proof given in \cite{Olsson13} with our notations.
	Regarding  $\bar{\psi}_{pq}(f_p, f_q)$,
	%Obviously, $\bar{\psi}_{pq}(\alpha, \alpha) = 0$. Therefore,
	\begin{align}
	\bar{\psi}_{pq}(\alpha, \alpha) + \bar{\psi}_{pq}(\beta, \gamma)   &=  \bar{\psi}_{pq}(\beta, \gamma) \notag  \\
	&= | d_p(\beta) - d_p(\gamma)  |  + | d_q(\beta) - d_q(\gamma)  |  \notag \\
	&= | \left(d_p(\beta) - d_p(\alpha)\right) - \left(d_p(\gamma)-d_p(\alpha)\right)  | + \left(d_q(\beta)-d_q(\alpha)\right) - \left(d_q(\gamma)-d_q(\alpha)\right)  | \notag  \\
	&\le |d_p(\beta) - d_p(\alpha)| + |d_p(\gamma)-d_p(\alpha)| + | d_q(\beta)-d_q(\alpha)| + |d_q(\gamma)-d_q(\alpha)  |  \notag \\
	&= \bar{\psi}_{pq}(\beta, \alpha) + \bar{\psi}_{pq}(\alpha, \gamma). \notag 
	\end{align}
	Thus, $\bar{\psi}_{pq}(f_p, f_q)$ satisfies the submodularity of expansion moves. For its truncated function,
	\begin{align}
	\min\left(\bar{\psi}_{pq}(\alpha, \alpha), \tau \right) + \min\left(\bar{\psi}_{pq}(\beta, \gamma), \tau \right) &=  \min\left(\bar{\psi}_{pq}(\beta, \gamma), \tau \right)  \notag \\
	&\le \min\left(\bar{\psi}_{pq}(\beta, \alpha) + \bar{\psi}_{pq}(\alpha, \gamma), \tau \right)  \notag \\
	&\le \min\left(\bar{\psi}_{pq}(\beta, \alpha), \tau \right) + \min\left(\bar{\psi}_{pq}(\alpha, \gamma), \tau \right). \notag 
	\end{align}
	Therefore, ${\psi}_{pq}(f_p, f_q)$ also satisfies the submodularity.
\end{proof}

\section{Additional results}
We here provide additional results for the analyses on effects of grid-cell sizes (Sec.~4.3) and parallelization (Sec.~4.4) as well as the comparisons with the methods of PMBP~\cite{Besse12} (Sec.~4.5), PMF~\cite{Lu13} (Sec.~4.6) and Olsson~\etal~\cite{Olsson13} (Sec.~4.7).
We evaluate performances by showing plots of error rate transitions over running time, where error rates are evaluated by the \emph{bad 2.0} metric for all regions. To provide diverse examples, we use 15 image pairs from the Middlebury V3 training dataset.

Unless noted otherwise, all methods are run on a single CPU core to optimize the same energy function that we use in Sec.~4.2 for the benchmark V3 evaluation. Settings of our method are also the same as we use in Sec.~4.2. Post-processing is disabled in all methods.

\subsubsection*{Effect of grid-cell sizes}
Figure~\ref{figa:grid} shows analyses on effects of grid-cell sizes. Similarly to Sec.~4.3, we compare the proposed grid-cell combination (S, M, L) with other four combinations (L, L, L), (M, M, M), (S, S, S), and (S, M, M). Here, the sizes of ``S''mall, ``M''edium, and ``L''arge grid-cells are proportionally set to $1\%$, $3\%$ and $9\%$ of the image width, respectively. For most of the image pairs, the proposed combination (S, M, L) performs best among the five combinations.

\subsubsection*{Efficiency comparison by parallelization}
Figure~\ref{figa:efficiency} shows analyses on parallelization of our local expansion move algorithm. When our method is performed in parallel on four CPU cores, we observe about 3.8x of average speed-up among the 15 image pairs.

\subsubsection*{Comparison with PMBP~\cite{Besse12}}
Figure~\ref{figa:pmbp} compares our method with PMBP~\cite{Besse12}. Similarly to Sec.~4.4, we compare 
PMBP (using $K=1$ and $5$) with two variants of our method using guided image filtering (GF) and bilateral filtering (BF).
Because PMBP cannot use the cost-map filtering acceleration of GF, we use only BF for PMBP. Therefore, our method using BF optimizes the same energy function with PMBP, while ours using GF optimizes a different function. All methods use four neighbors of pairwise terms as the same setting with PMBP.

As shown in Fig.~\ref{figa:pmbp}, our methods both using GF and BF consistently outperform PMBP in accuracy at convergence, while ours using GF performs best and shows much faster convergence than the others.
Our method can be further accelerated by parallelization as demonstrated in Fig.~\ref{figa:efficiency}.

\subsubsection*{Comparison with PMF~\cite{Lu13}}
Figure~\ref{figa:pmf} compares our method with PMF~\cite{Lu13}. Both methods use the same energy function but PMF can only optimize its data term. With explicit regularization, our method consistently outperforms PMF for all the image pairs.

\subsubsection*{Comparison with Olsson~\etal~\cite{Olsson13}}
Figure~\ref{figa:olsson} compares our method with the method by Olsson~\etal~\cite{Olsson13}. Both methods use the same energy function but optimize it differently. For most of the image pairs, our method is about 6x faster to reach comparable or better accuracies than \cite{Olsson13}.
Our method can be further accelerated by parallelization as demonstrated in Fig.~\ref{figa:efficiency}.

\begin{figure}
	\begin{center}
		\includegraphics[width=0.3\textwidth]{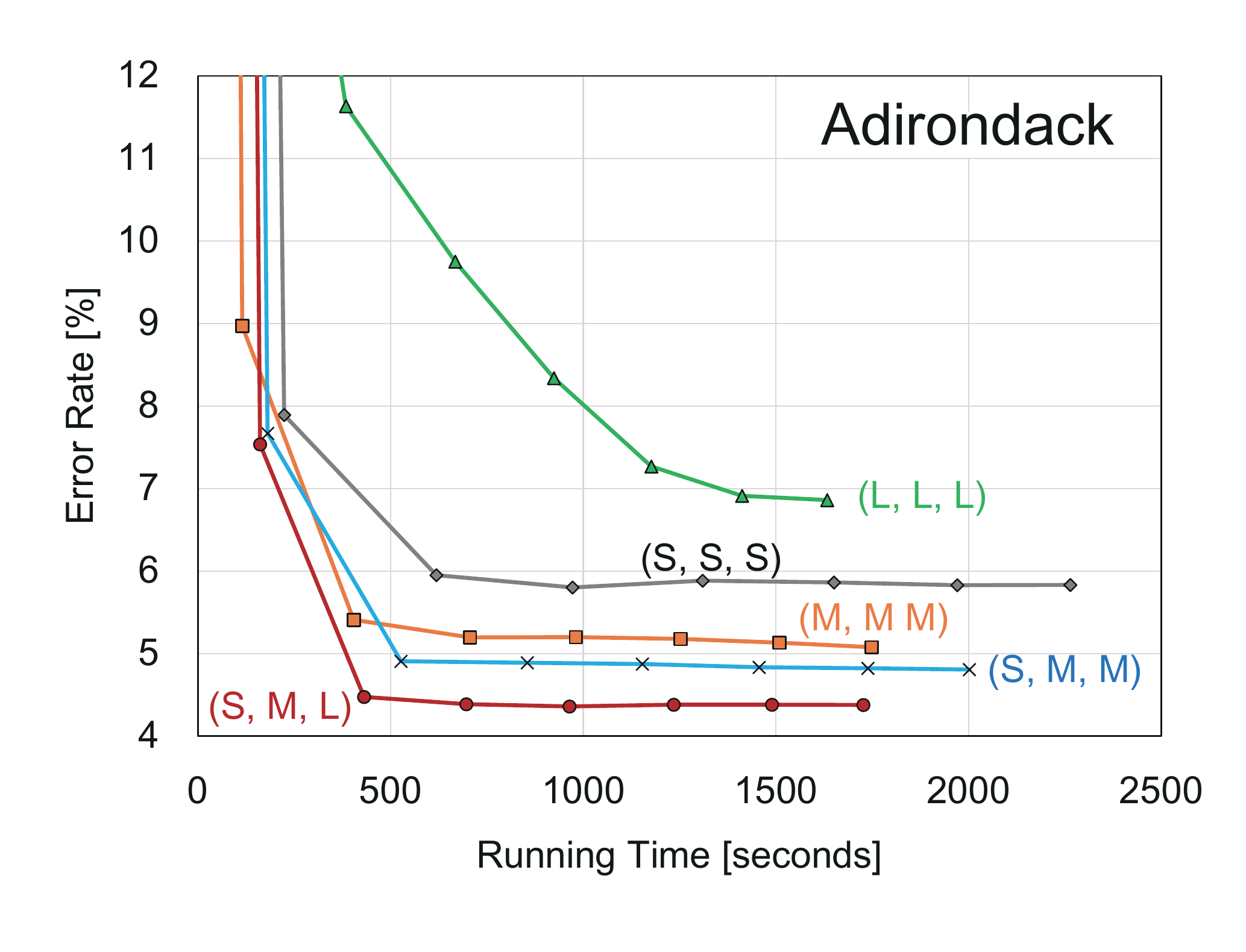}\hfil
		\includegraphics[width=0.3\textwidth]{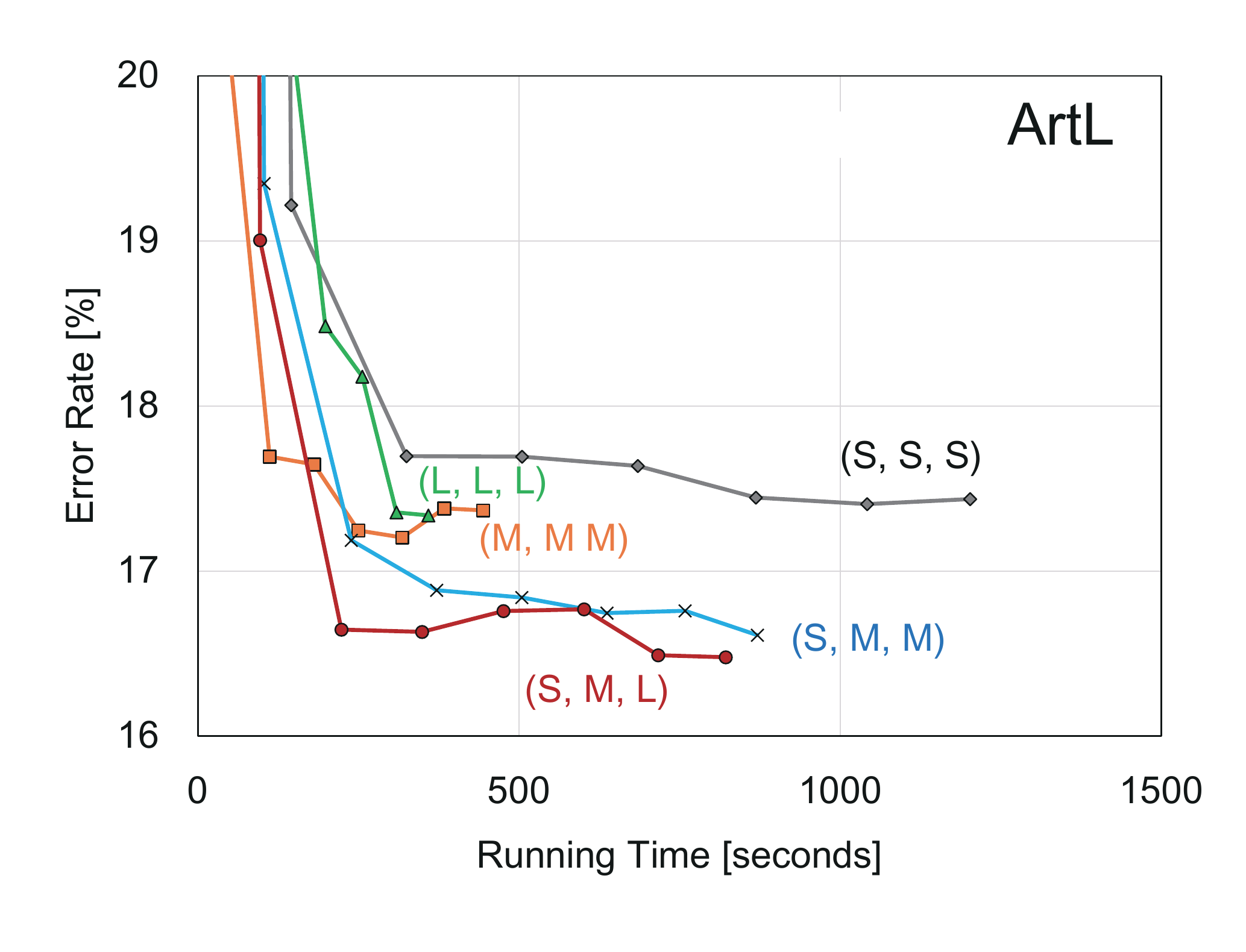}\hfil
		\includegraphics[width=0.3\textwidth]{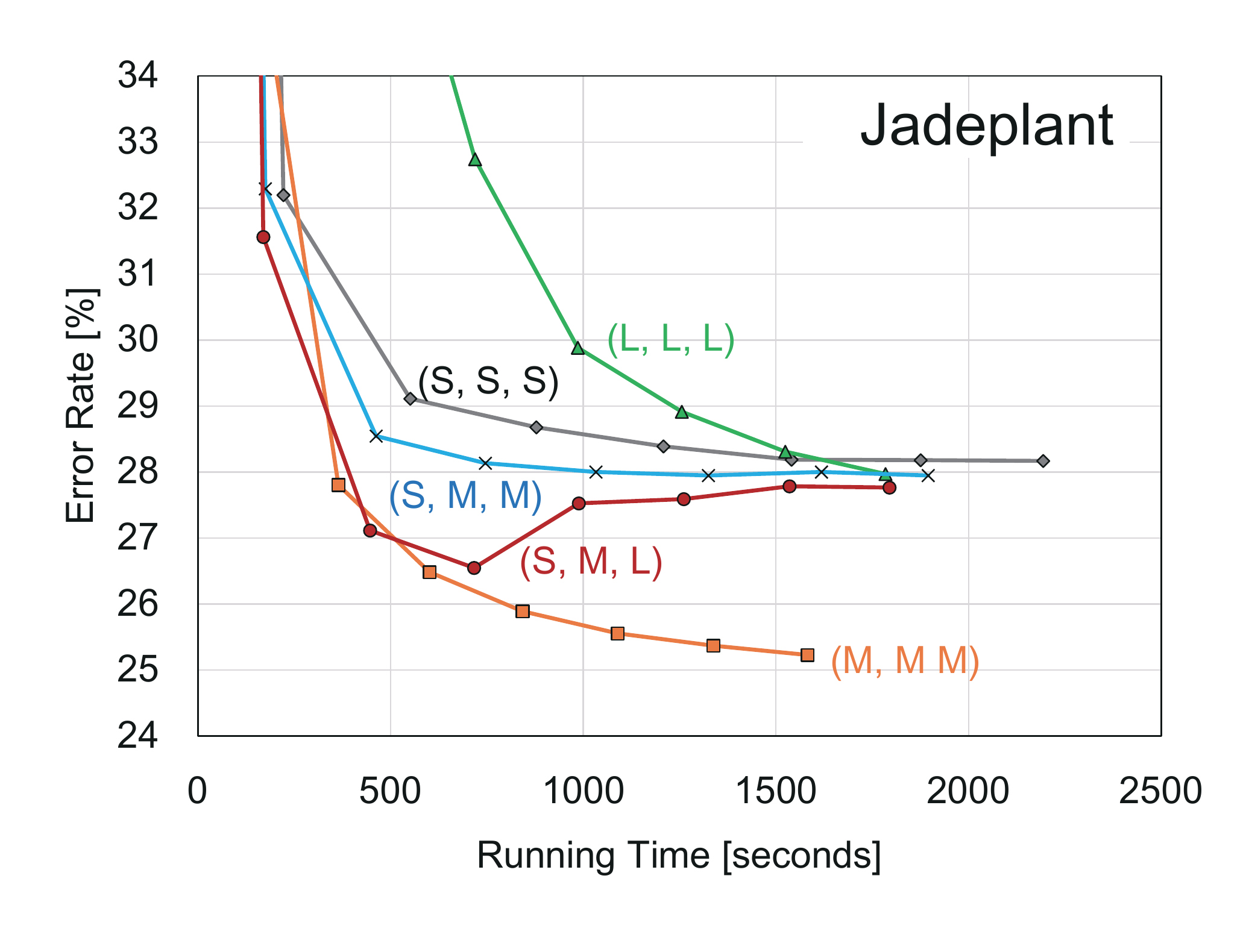}\\
		\includegraphics[width=0.3\textwidth]{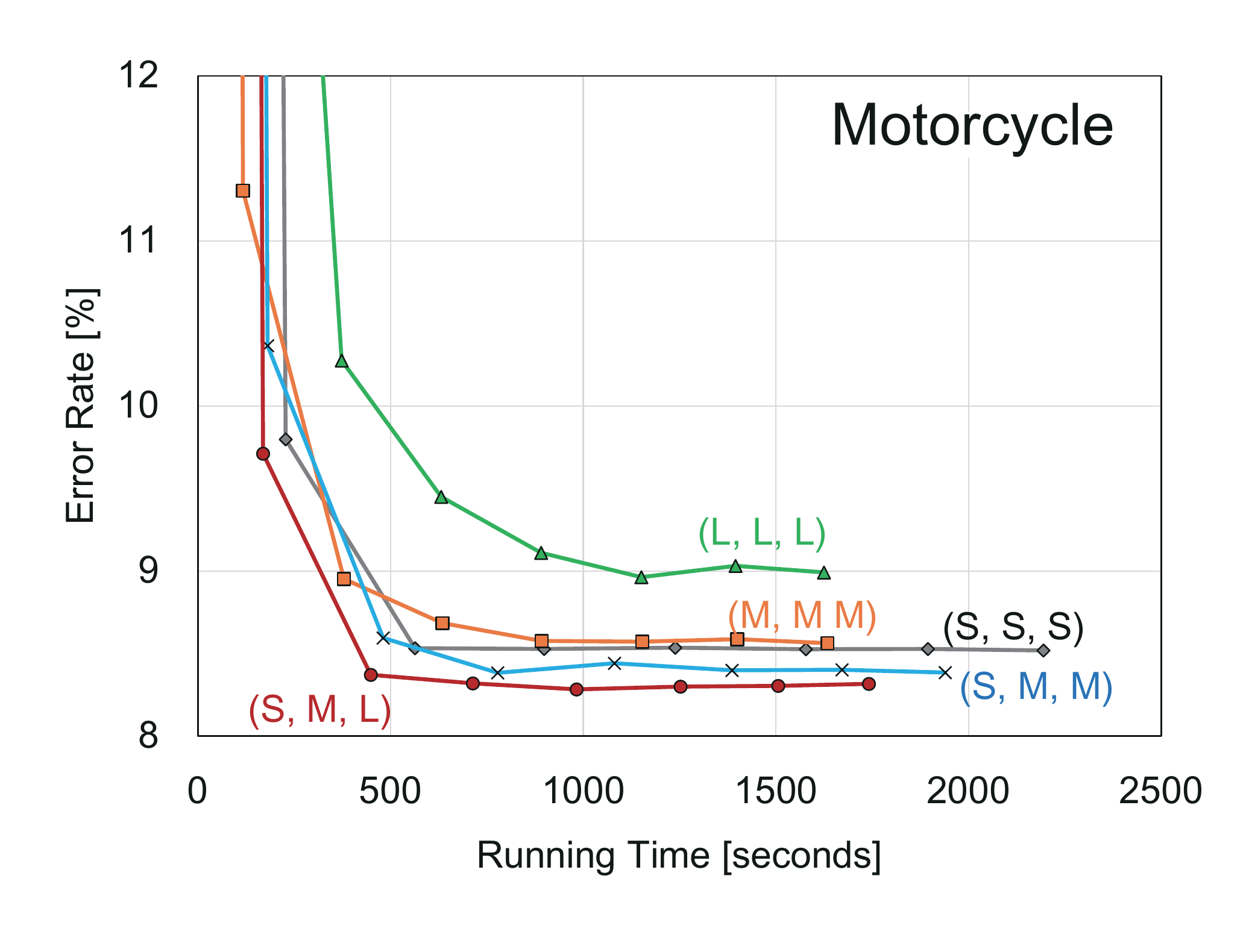}\hfil
		\includegraphics[width=0.3\textwidth]{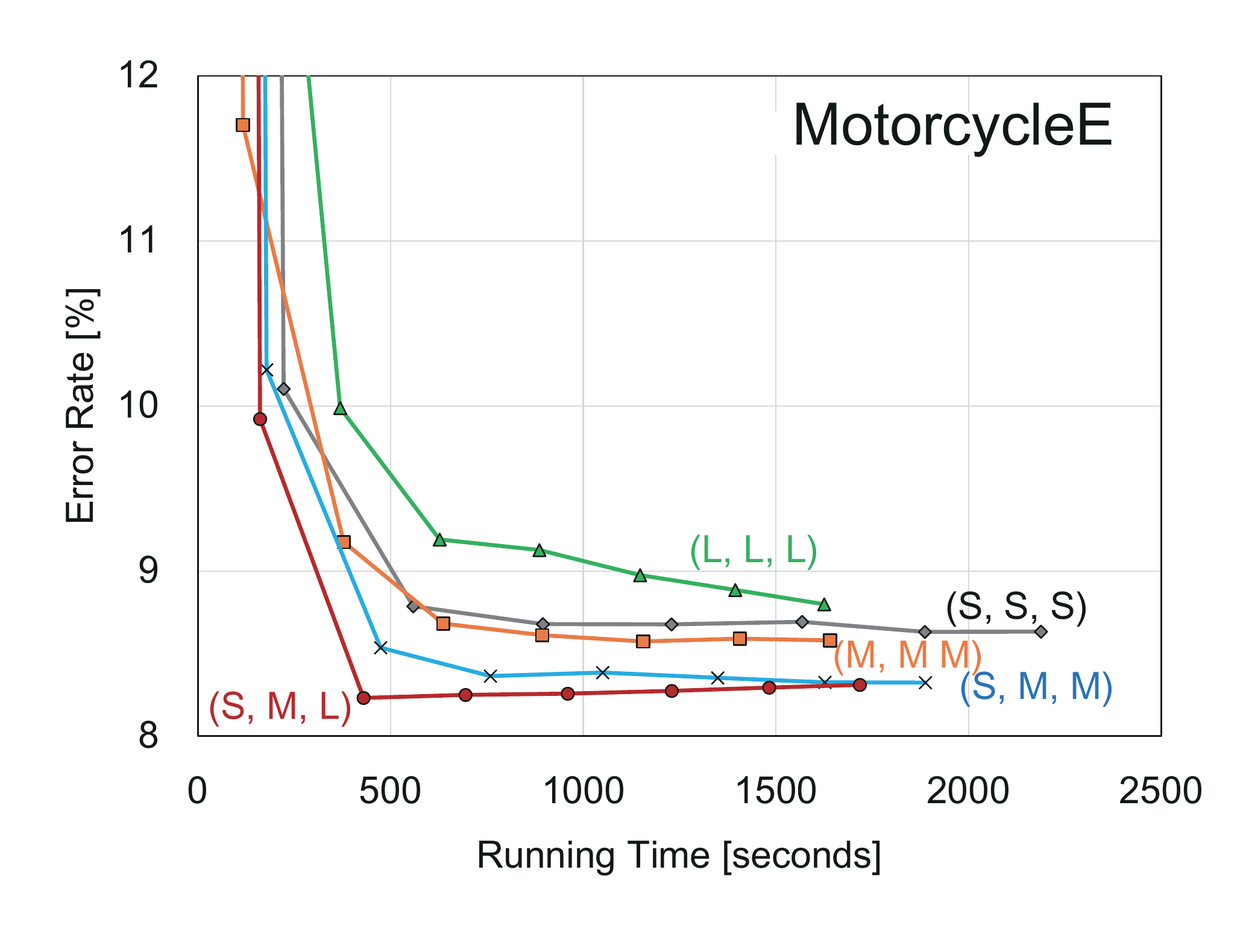}\hfil
		\includegraphics[width=0.3\textwidth]{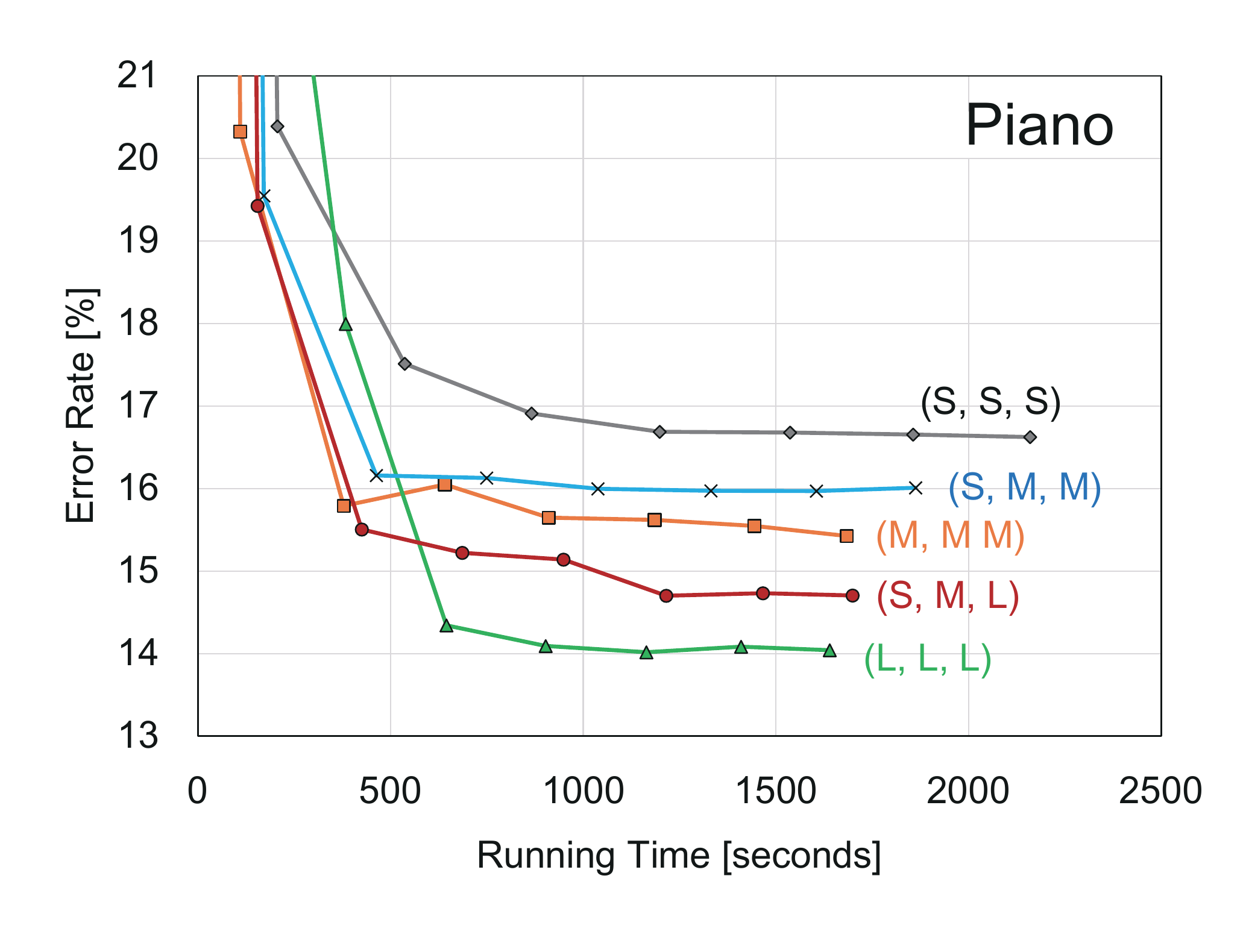}\\
		\includegraphics[width=0.3\textwidth]{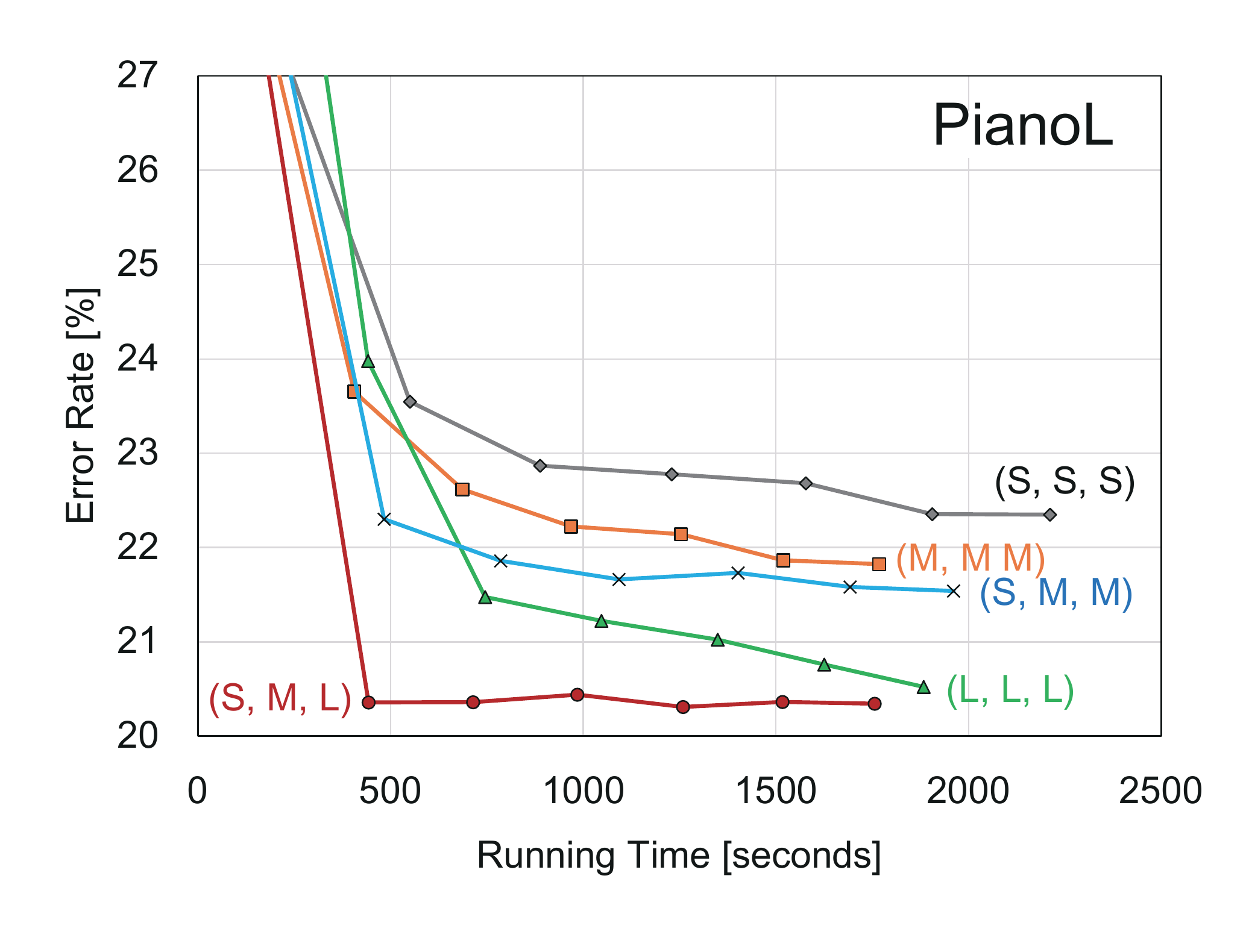}\hfil
		\includegraphics[width=0.3\textwidth]{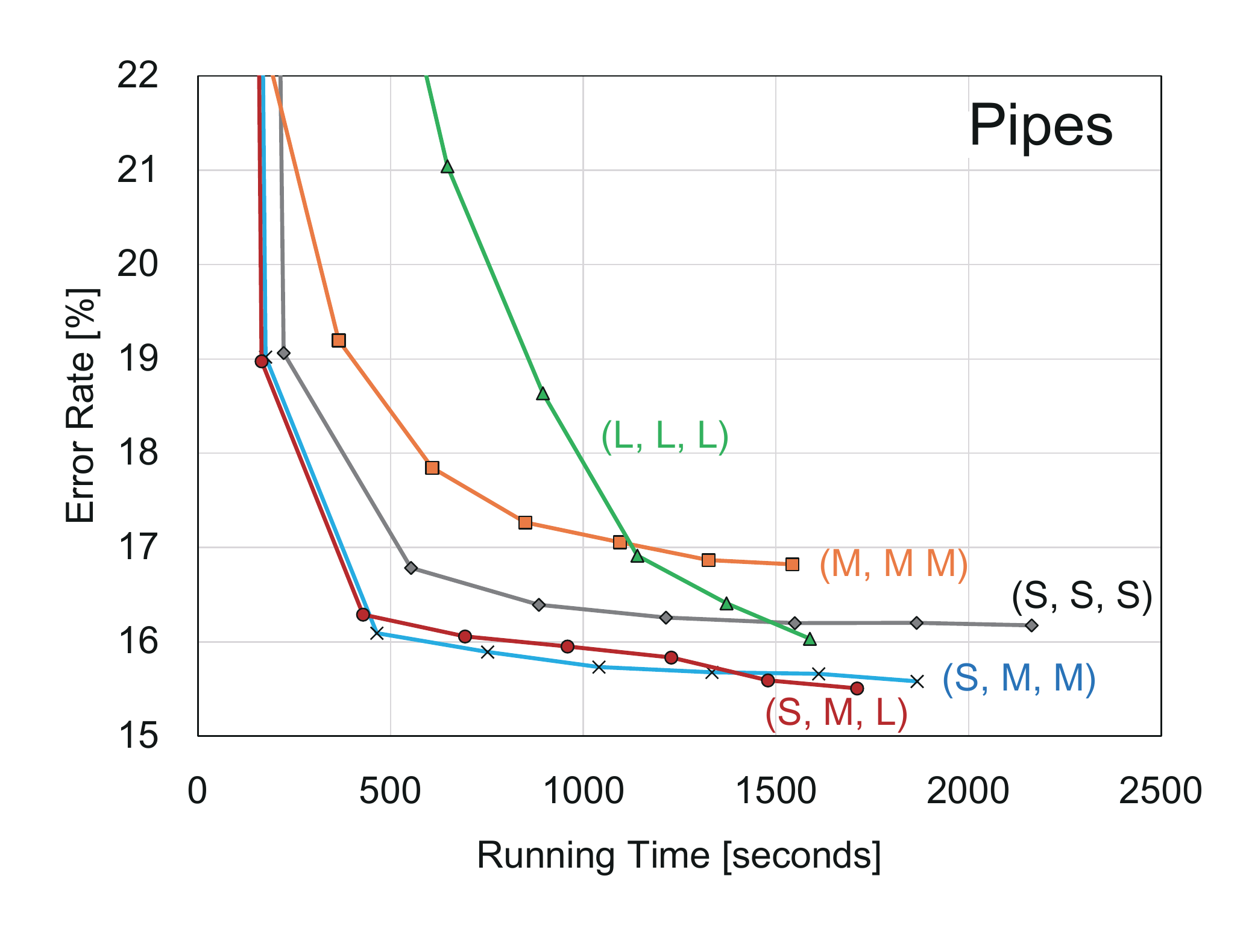}\hfil
		\includegraphics[width=0.3\textwidth]{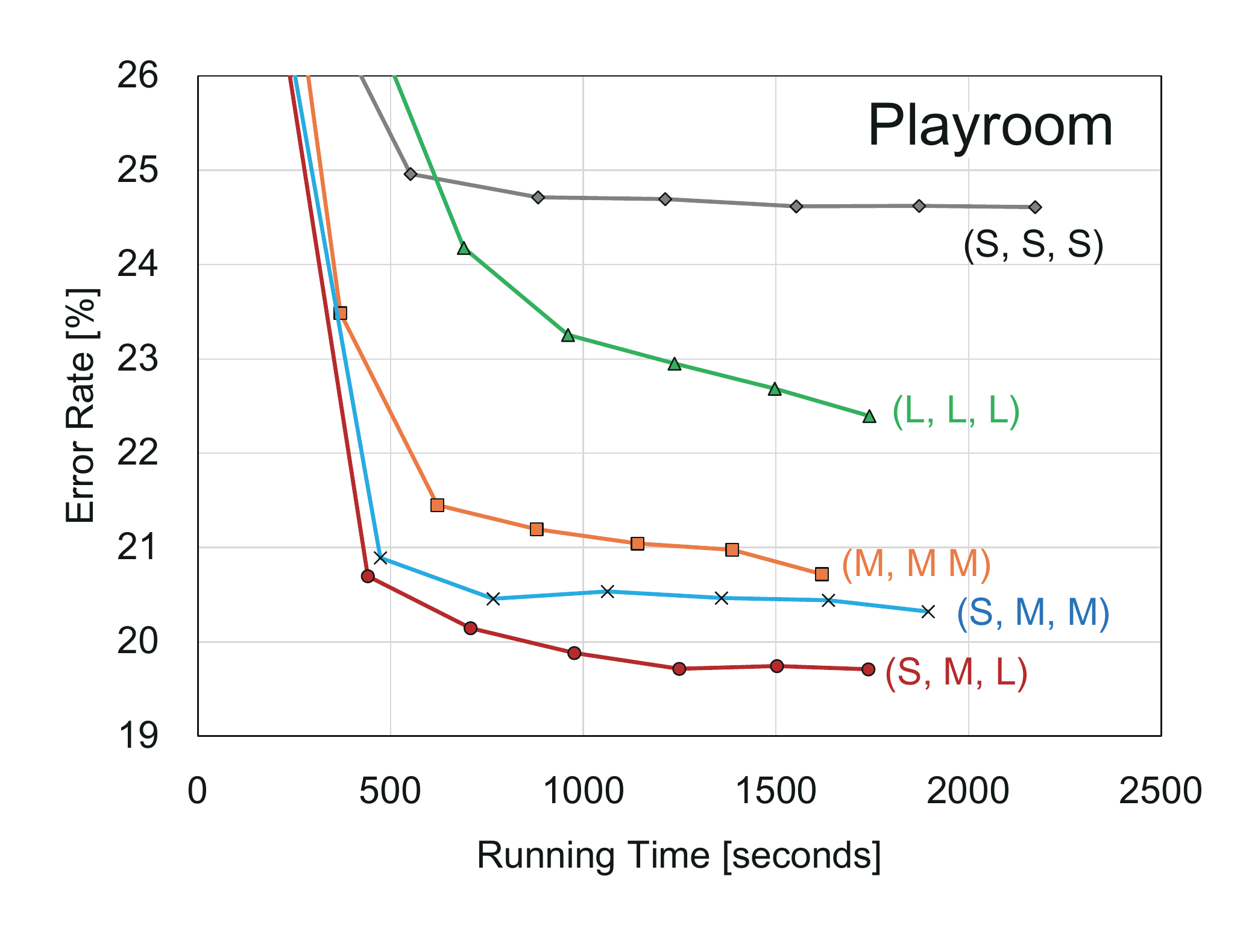}\\
		\includegraphics[width=0.3\textwidth]{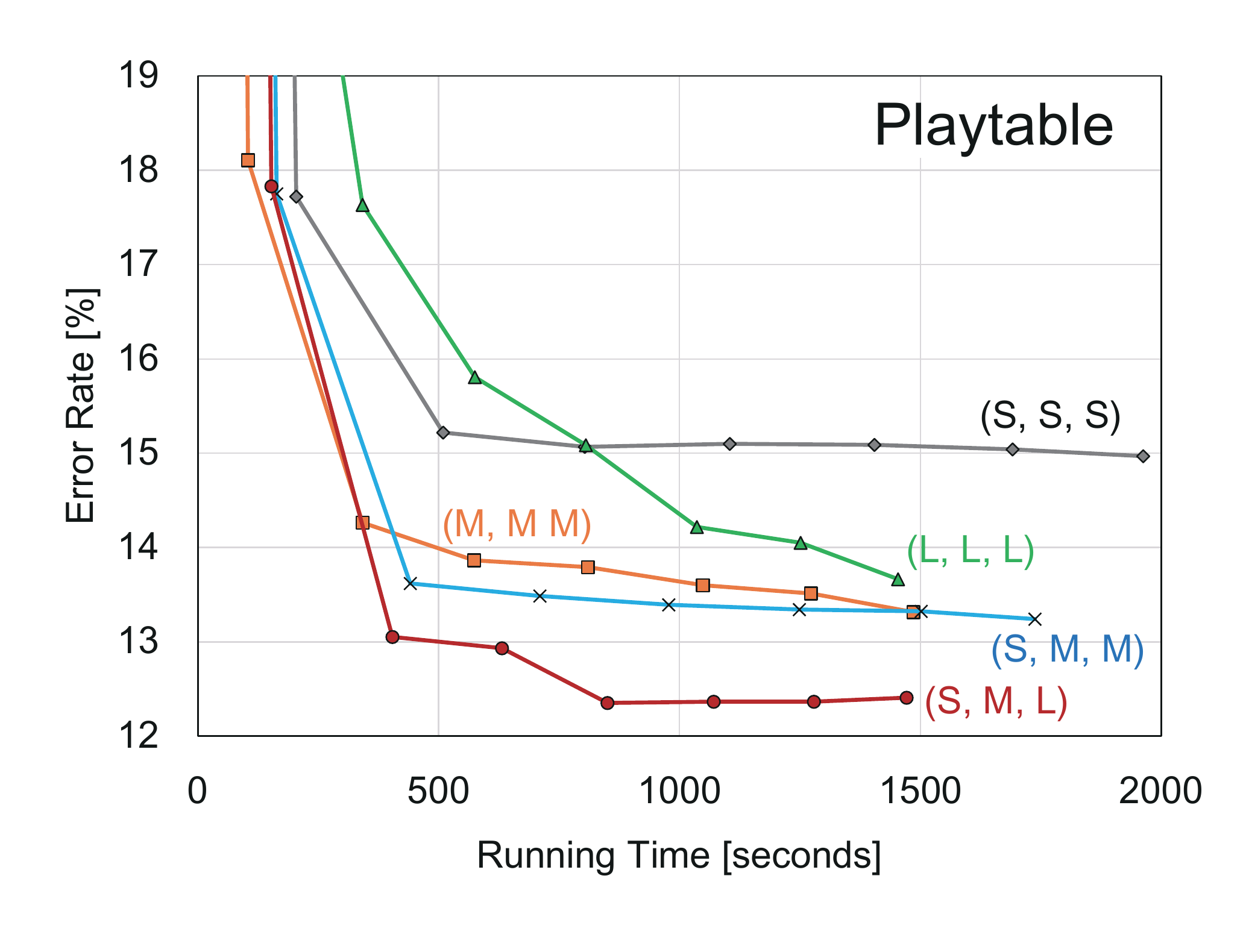}\hfil
		\includegraphics[width=0.3\textwidth]{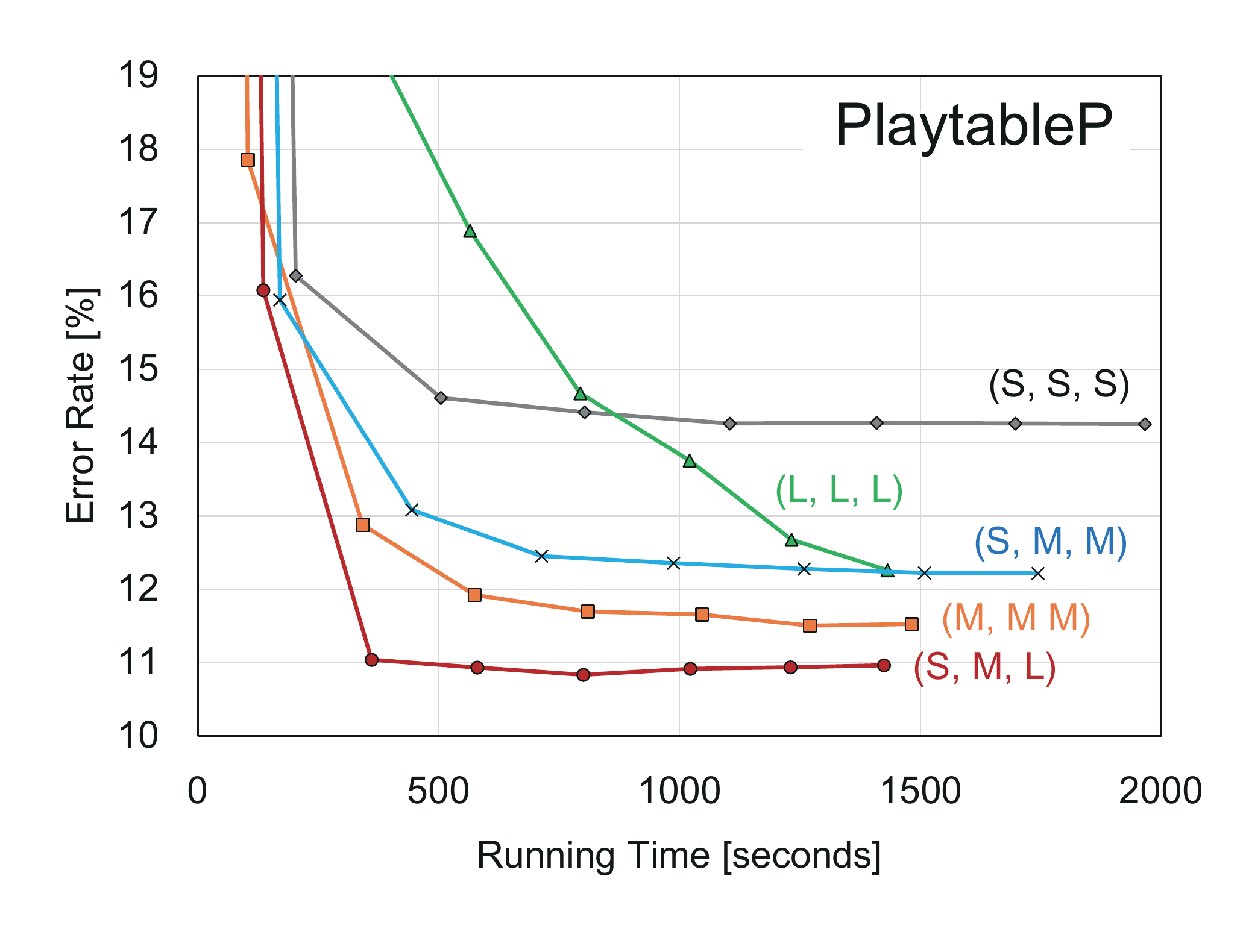}\hfil
		\includegraphics[width=0.3\textwidth]{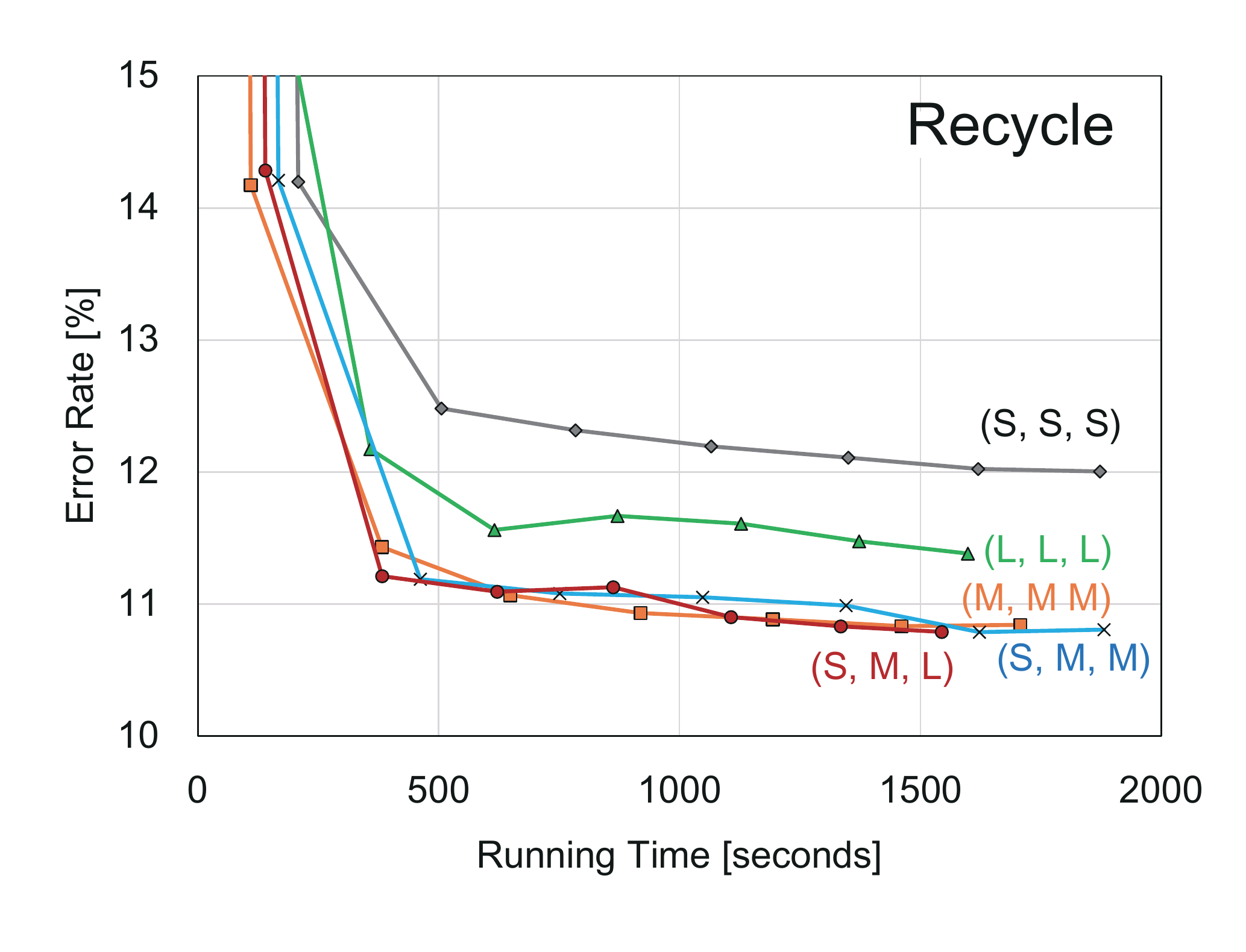}\\
		\includegraphics[width=0.3\textwidth]{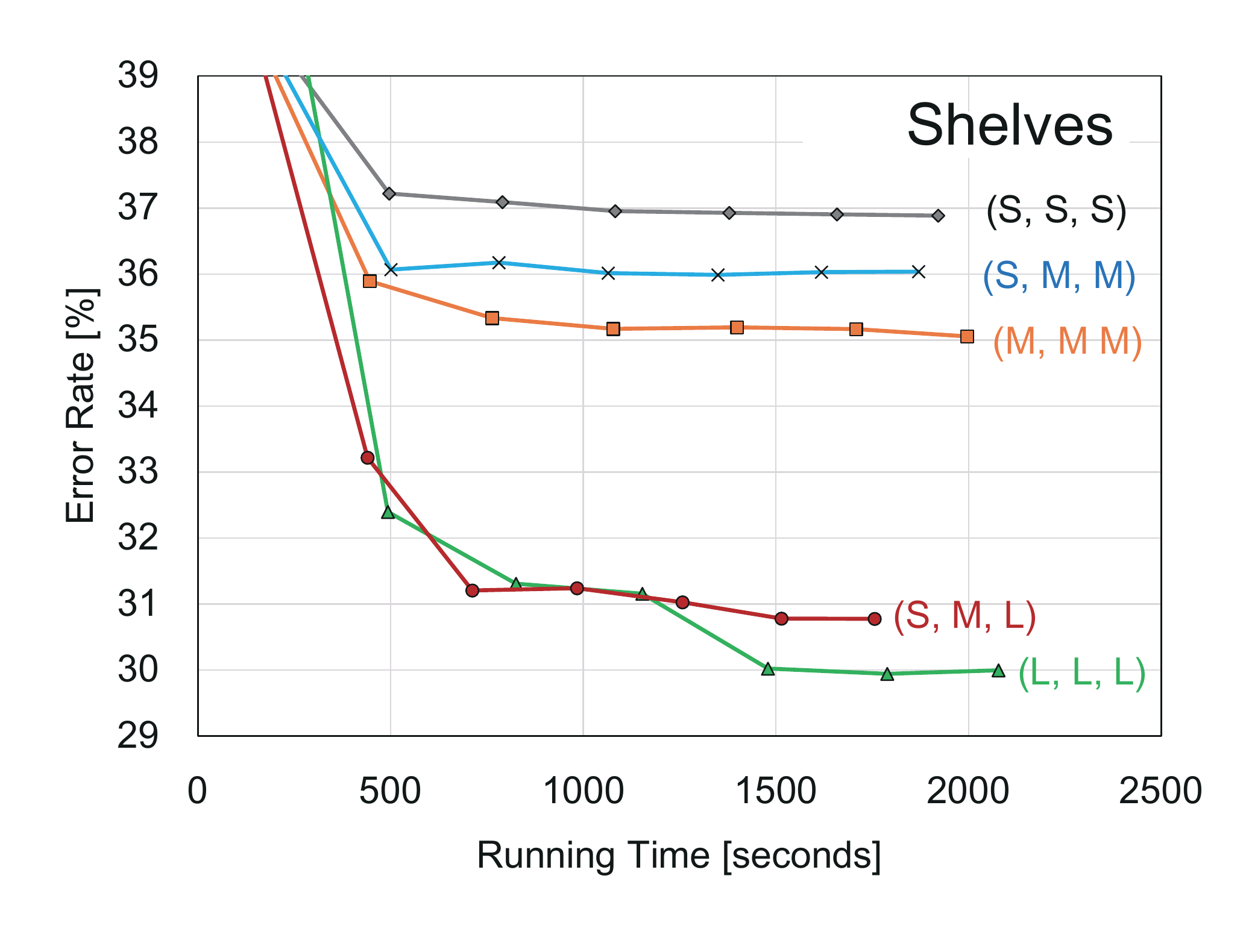}\hfil
		\includegraphics[width=0.3\textwidth]{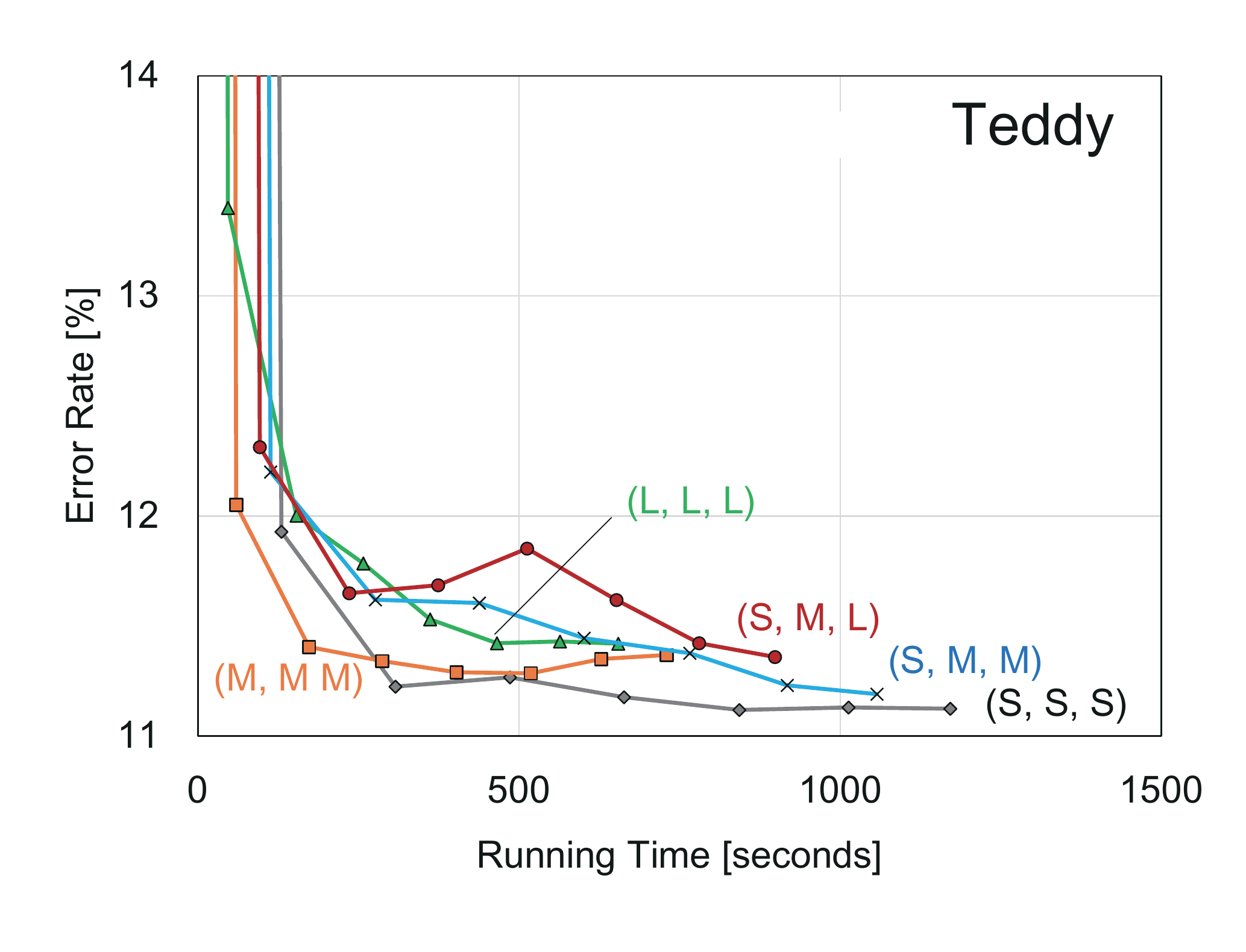}\hfil
		\includegraphics[width=0.3\textwidth]{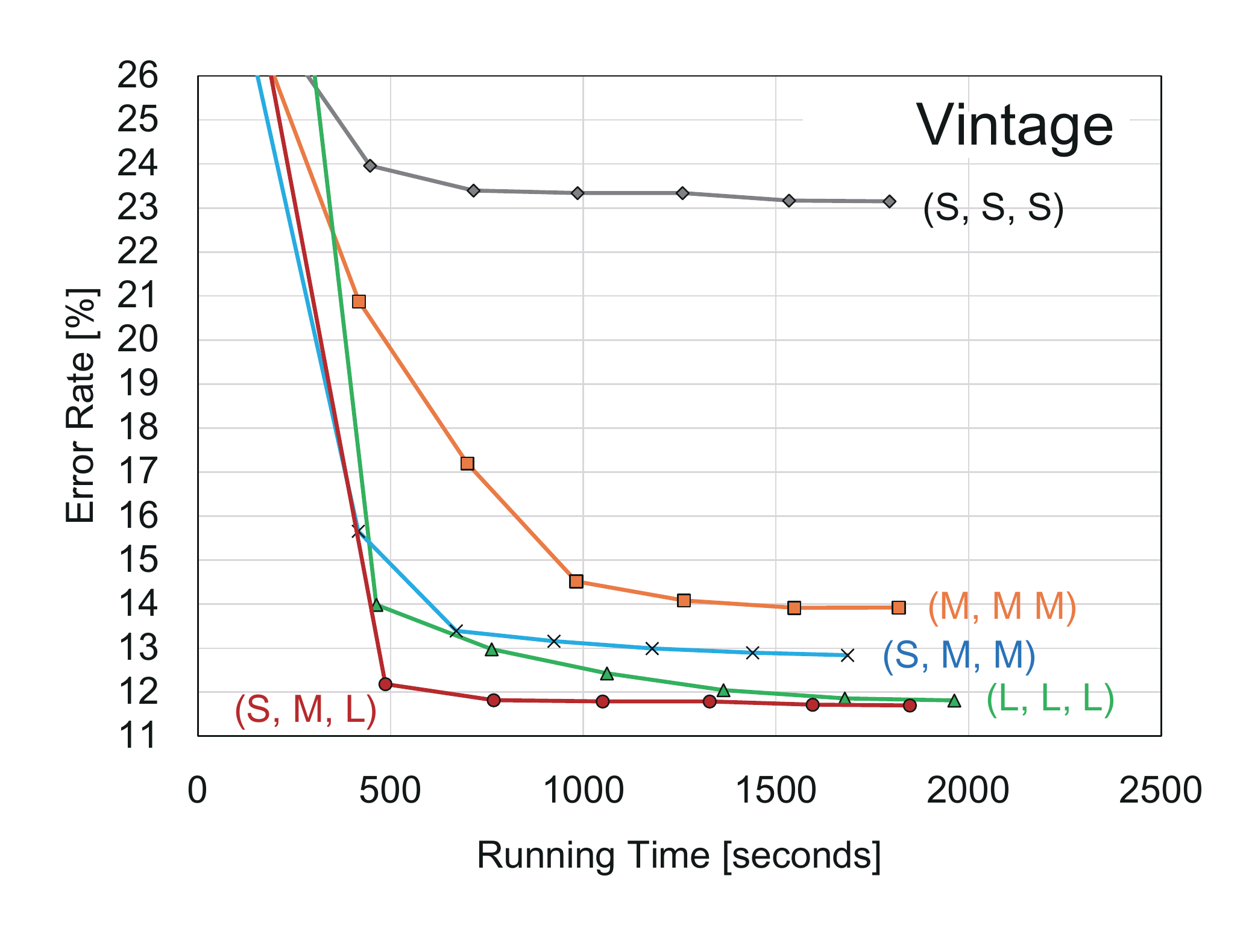}
	\end{center}
	\caption{Convergence comparison for different grid-cell sizes on 15 image pairs from the Middlebury V3 training dataset. Error rates are evaluated by the \emph{bad 2.0} metric for all regions. For most of the image pairs, the proposed grid-cell combination (S, M, L) outperforms the other four combinations (L, L, L), (M, M, M), (S, S, S), and (S, M, M). Here, the sizes of ``S''mall, ``M''edium, and ``L''arge grid-cells are proportionally set to $1\%$, $3\%$ and $9\%$ of the image width, respectively. All methods are run on a single CPU core to optimize the same energy function used in Sec~4.2 without post-processing.
	}
\label{figa:grid}
\end{figure}

\begin{figure}
	\begin{center}
		\includegraphics[width=0.3\textwidth]{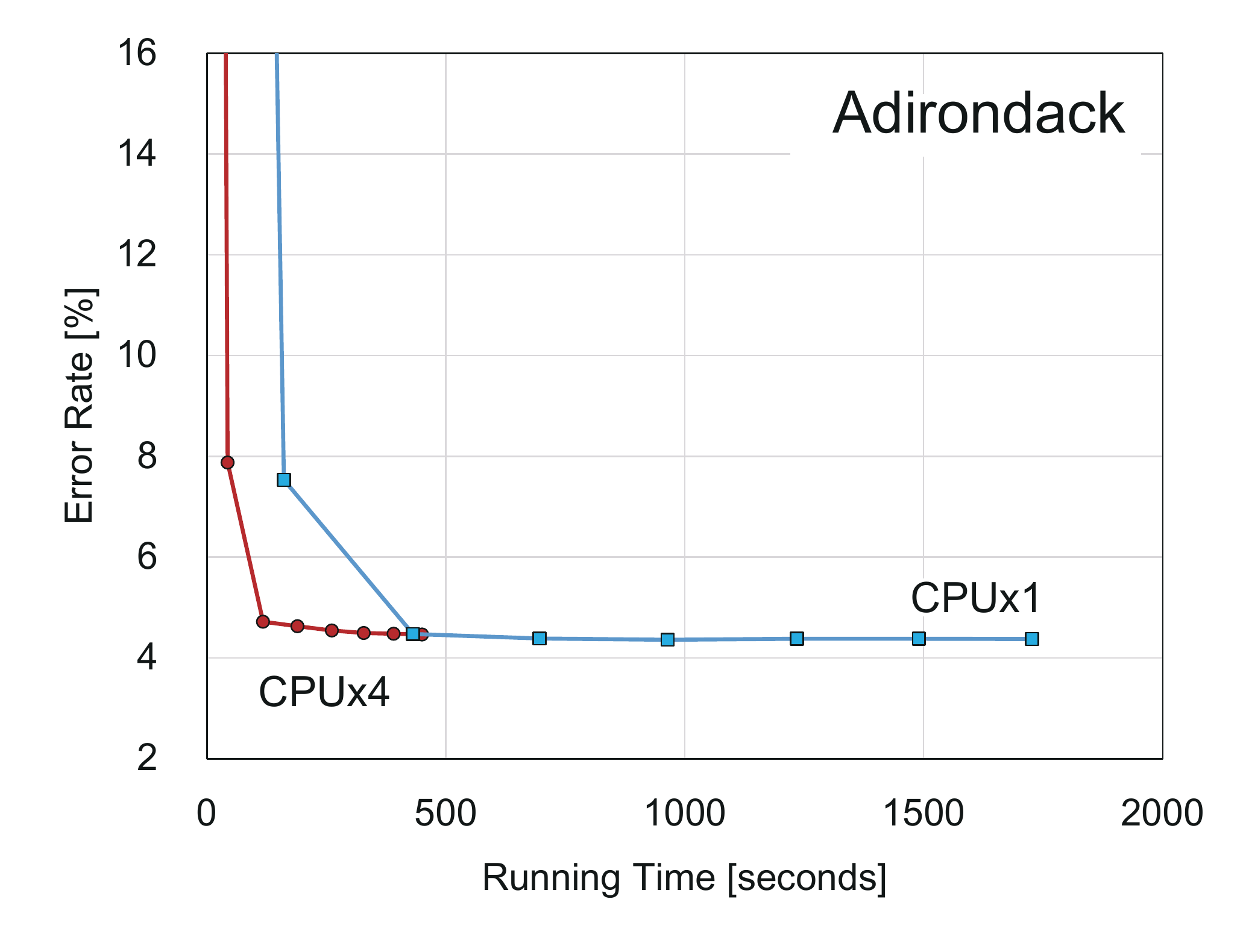}\hfil
		\includegraphics[width=0.3\textwidth]{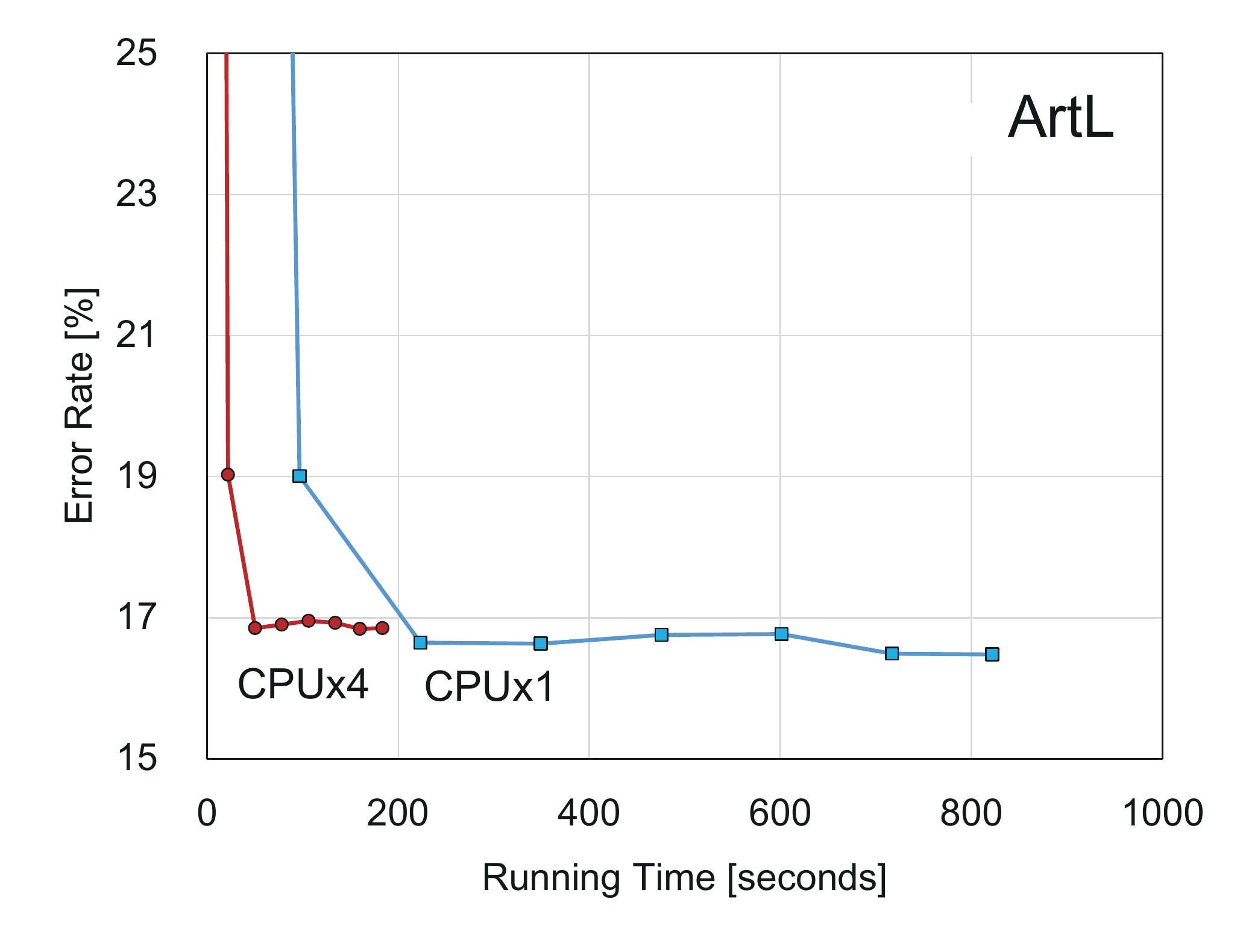}\hfil
		\includegraphics[width=0.3\textwidth]{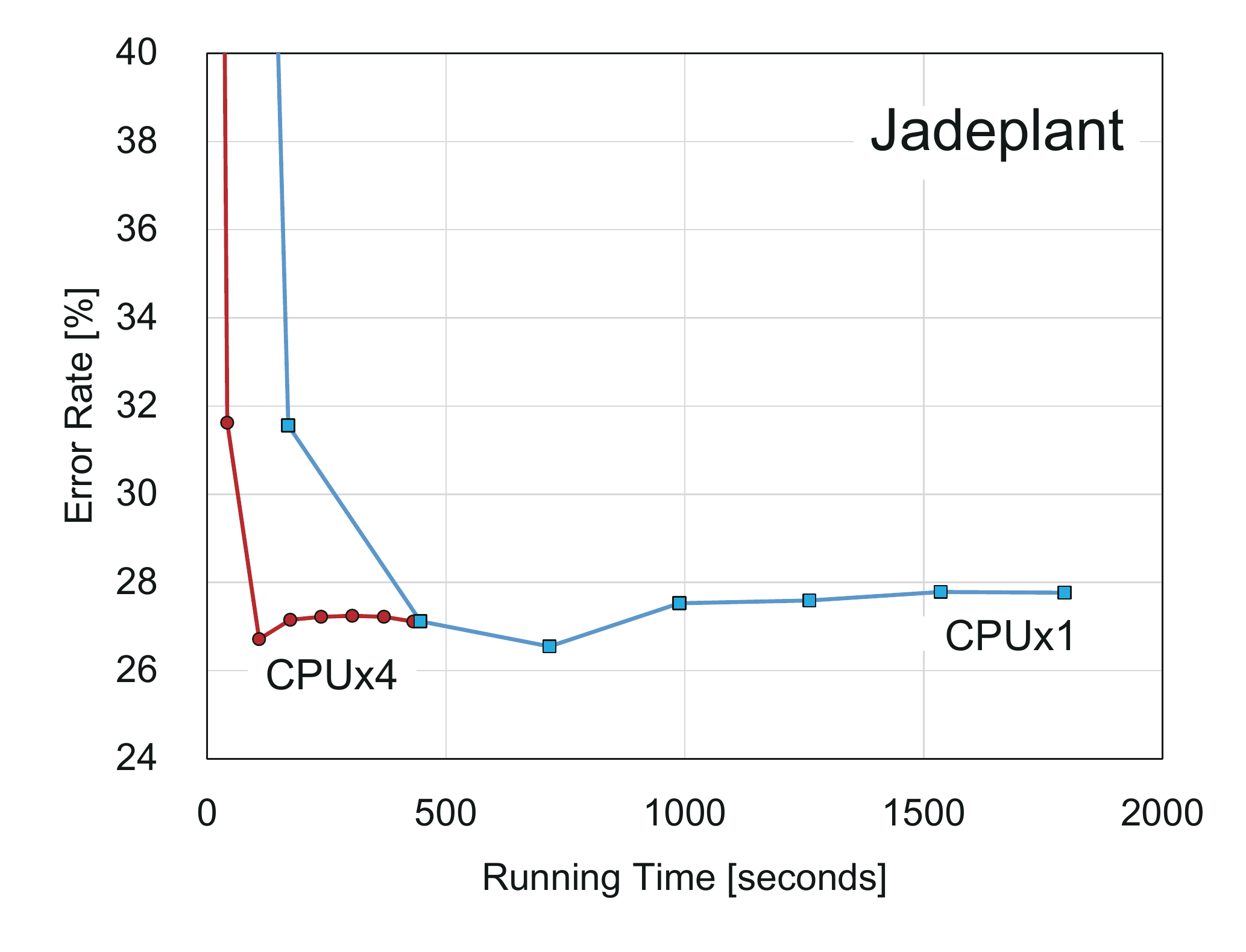}\\
		\includegraphics[width=0.3\textwidth]{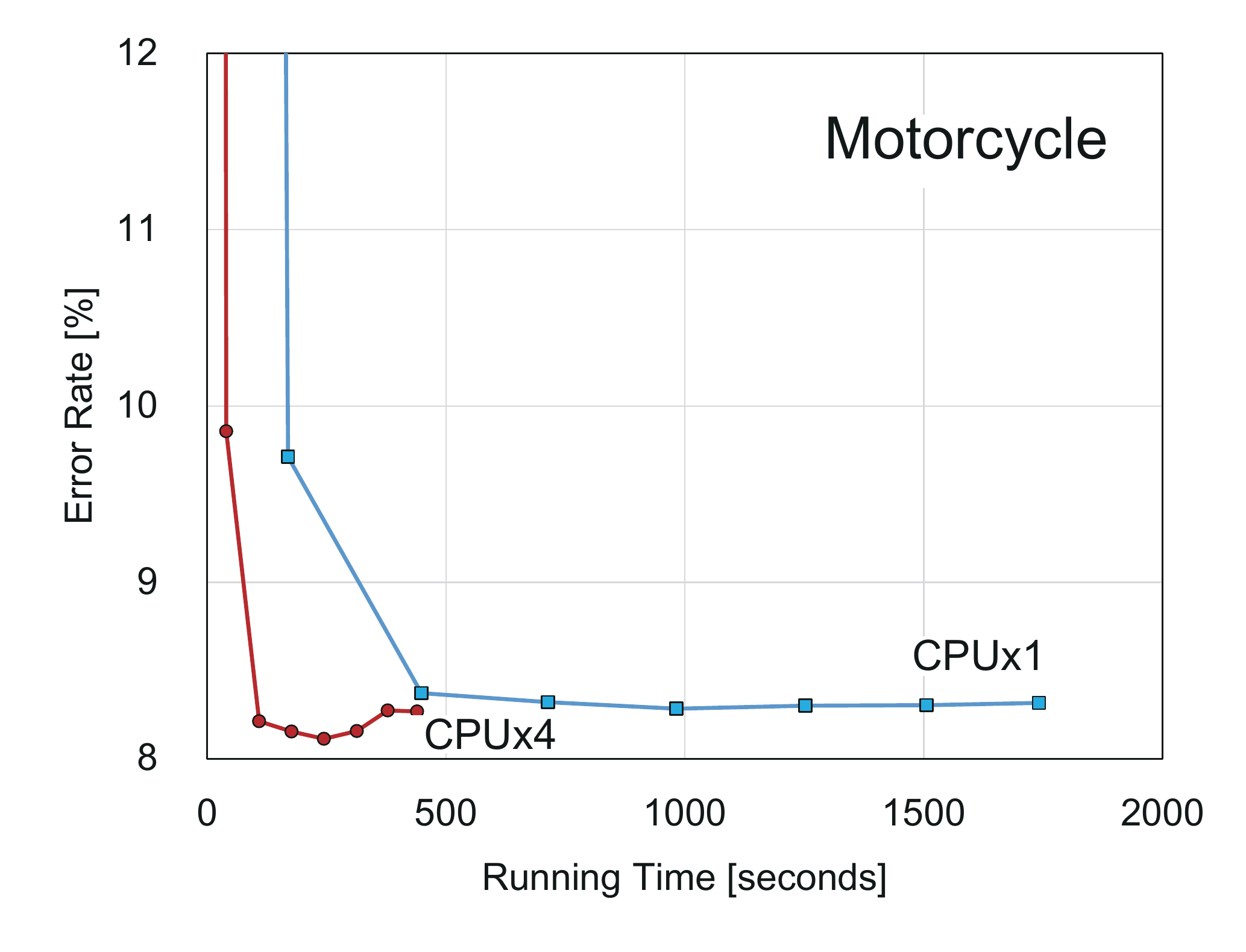}\hfil
		\includegraphics[width=0.3\textwidth]{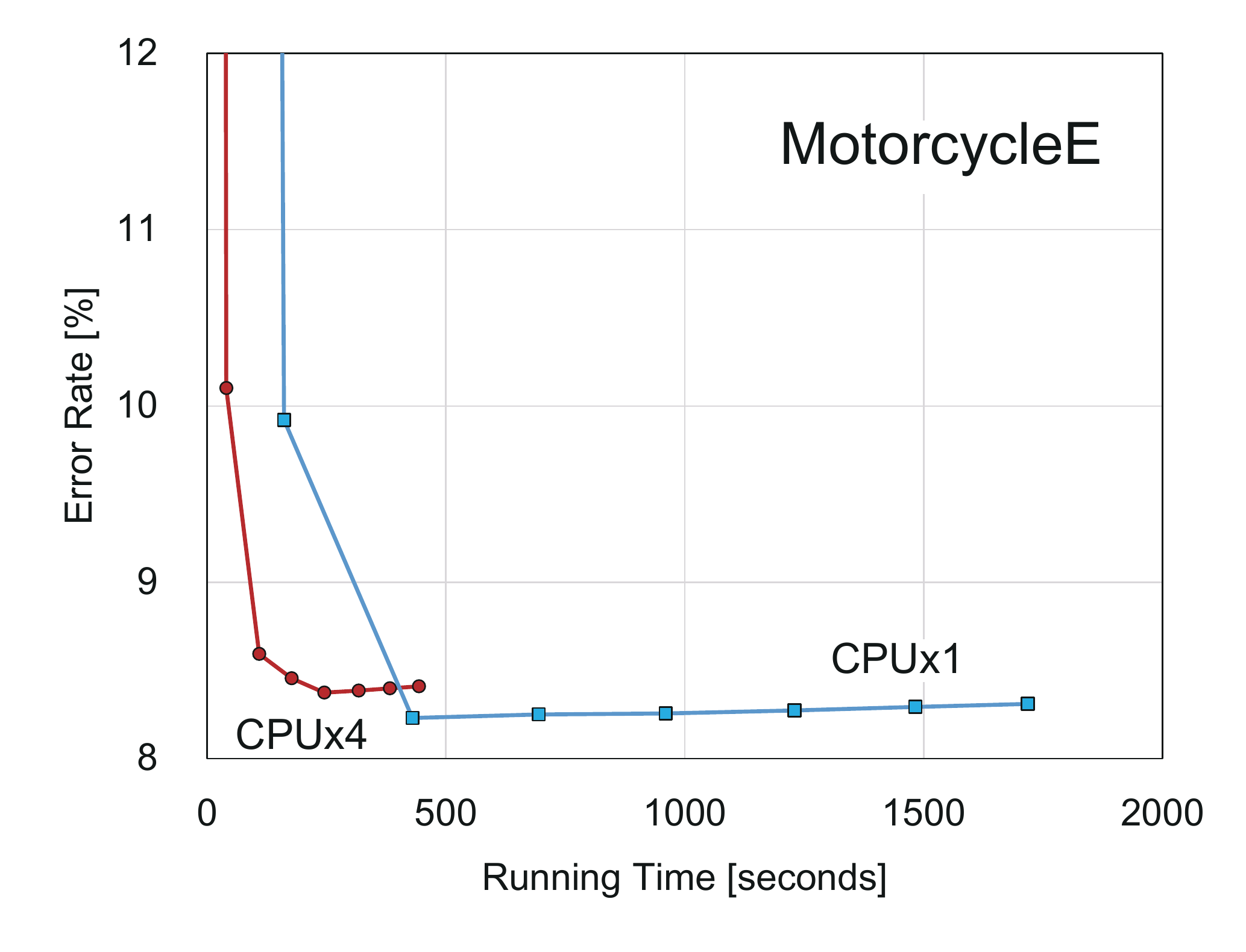}\hfil
		\includegraphics[width=0.3\textwidth]{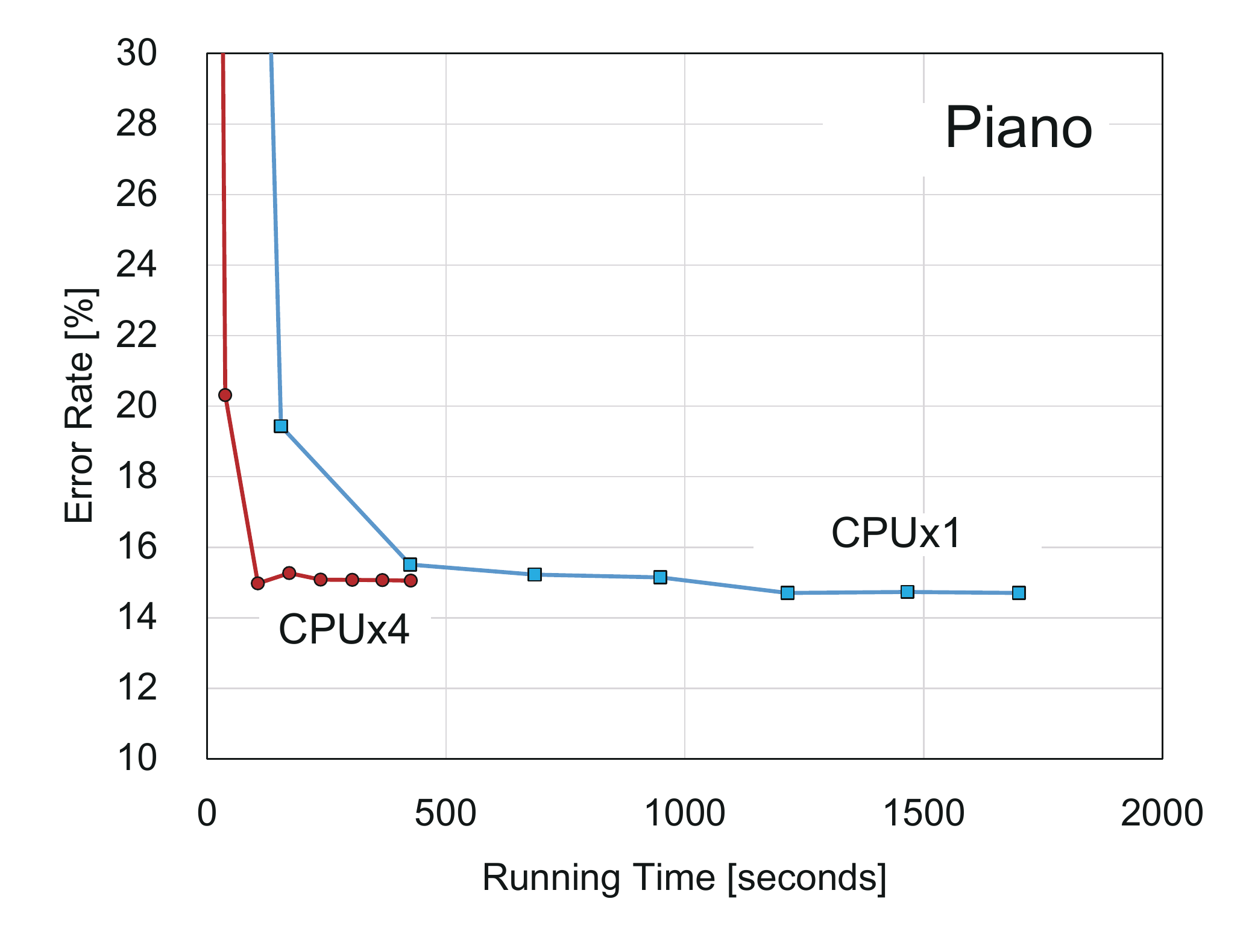}\\
		\includegraphics[width=0.3\textwidth]{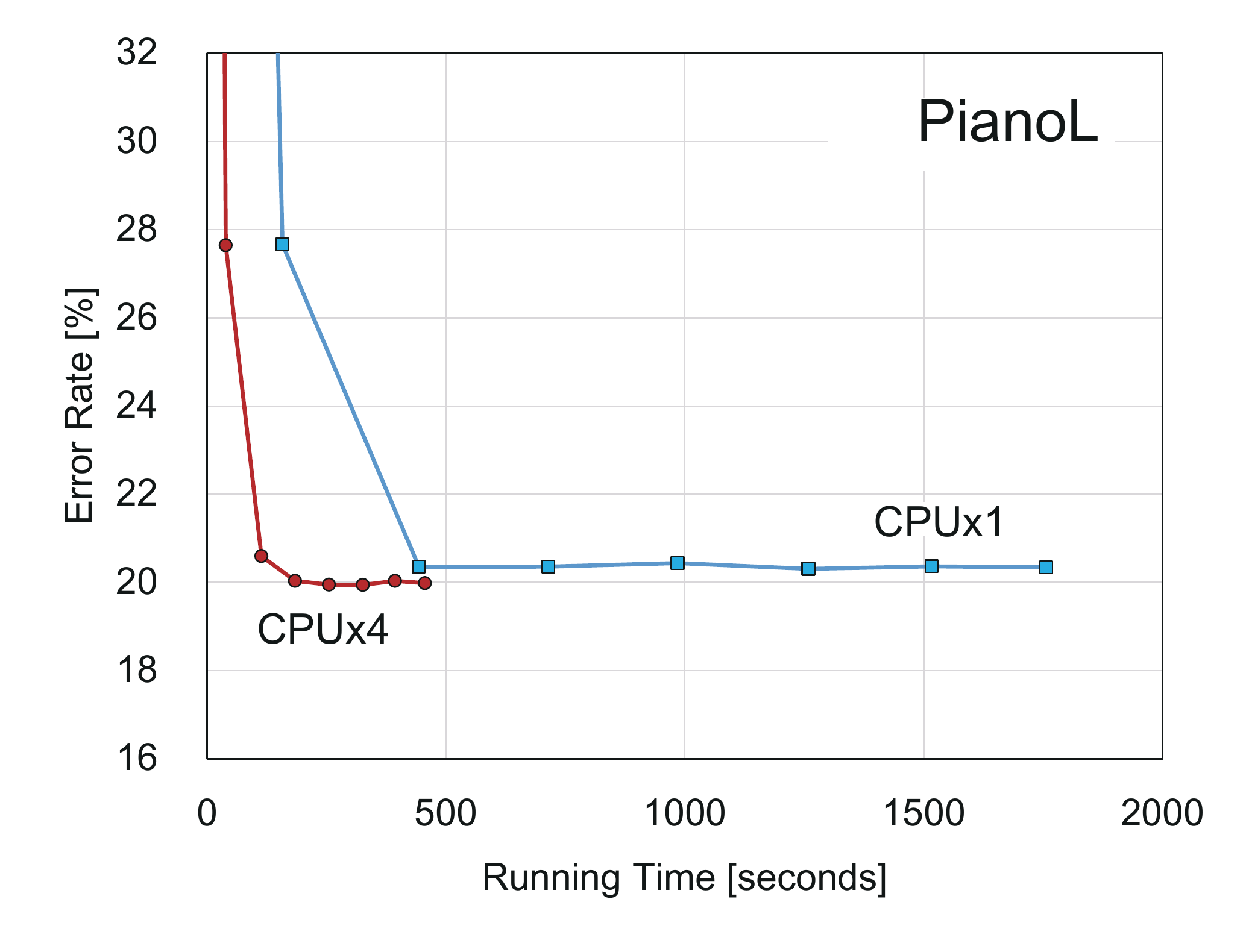}\hfil
		\includegraphics[width=0.3\textwidth]{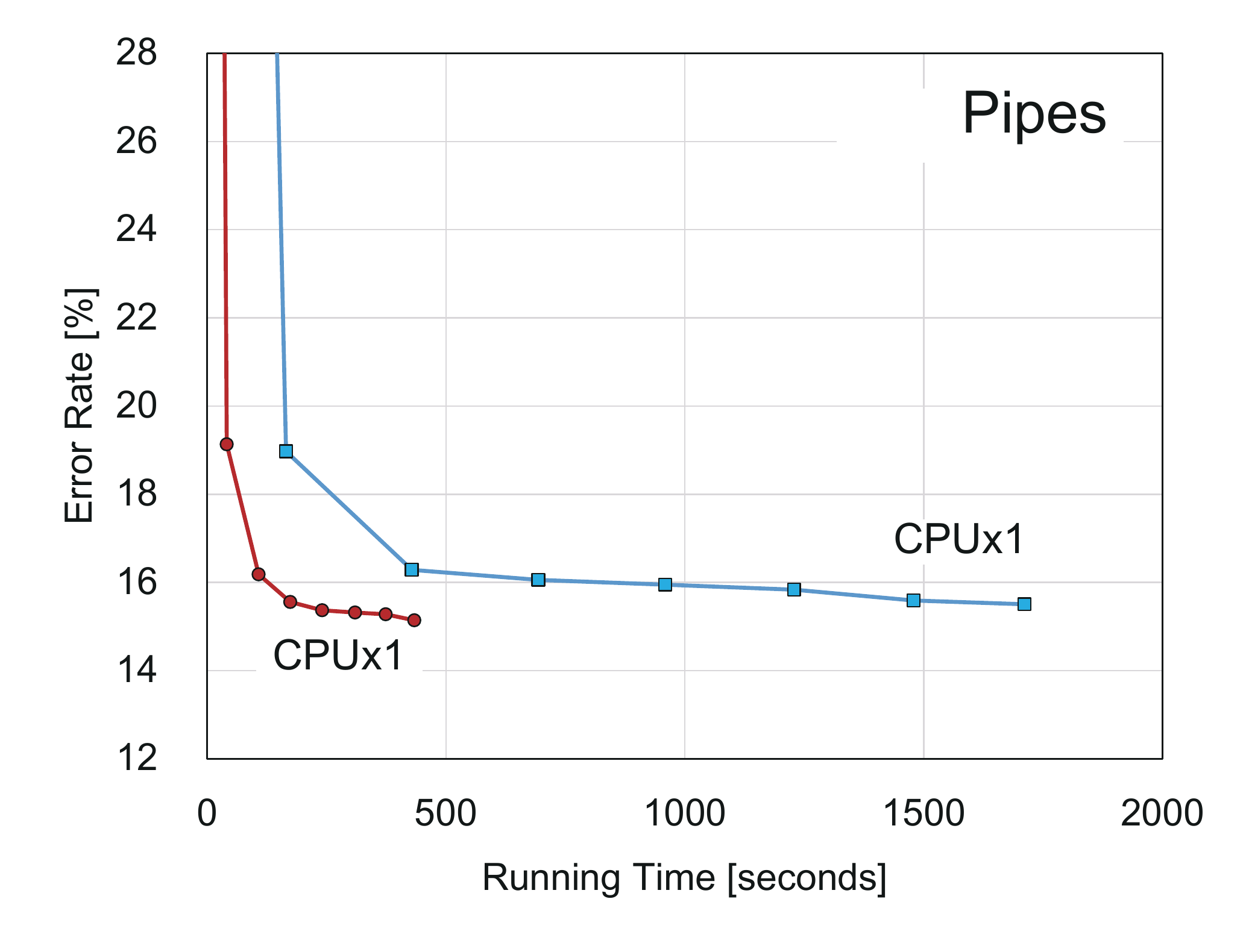}\hfil
		\includegraphics[width=0.3\textwidth]{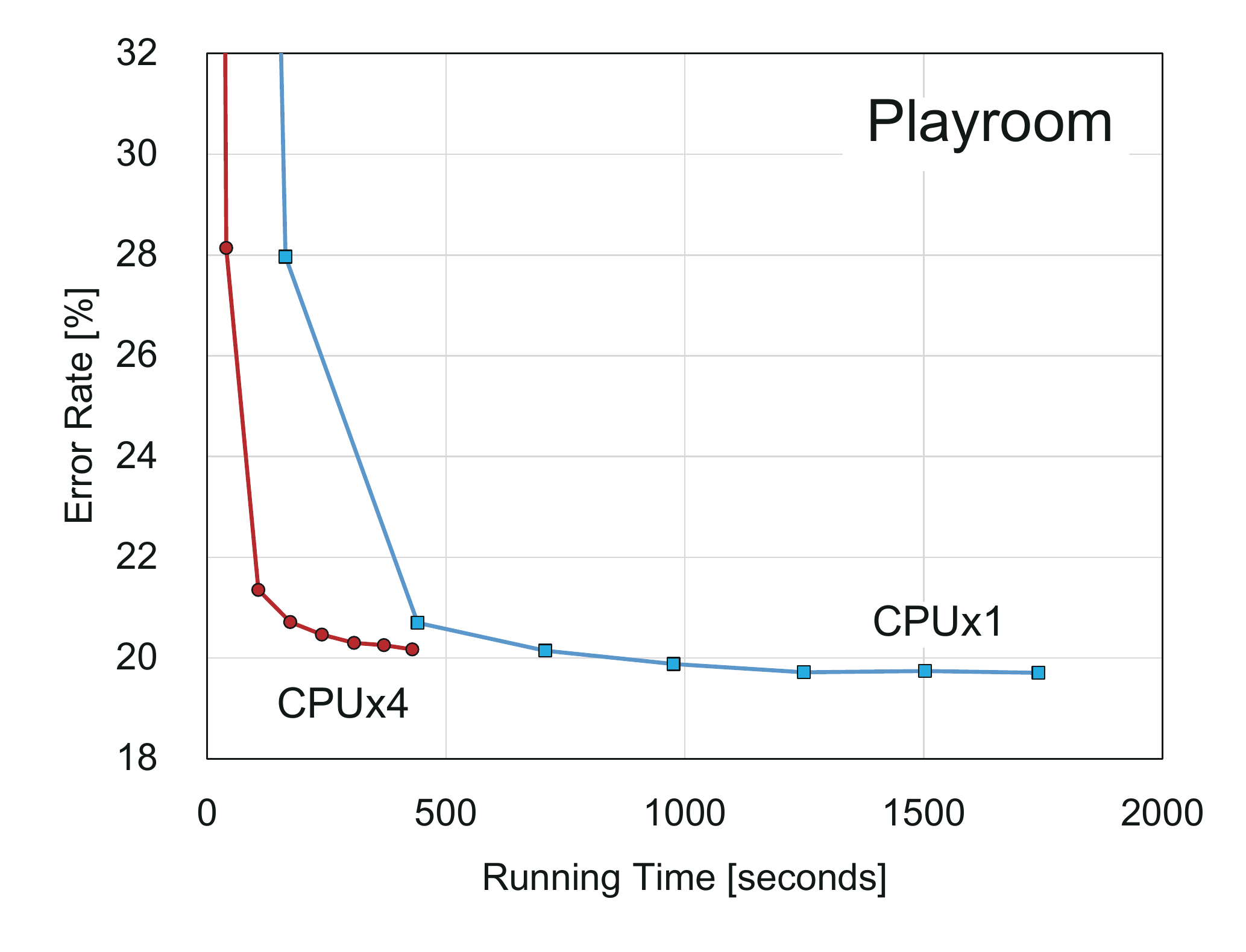}\\
		\includegraphics[width=0.3\textwidth]{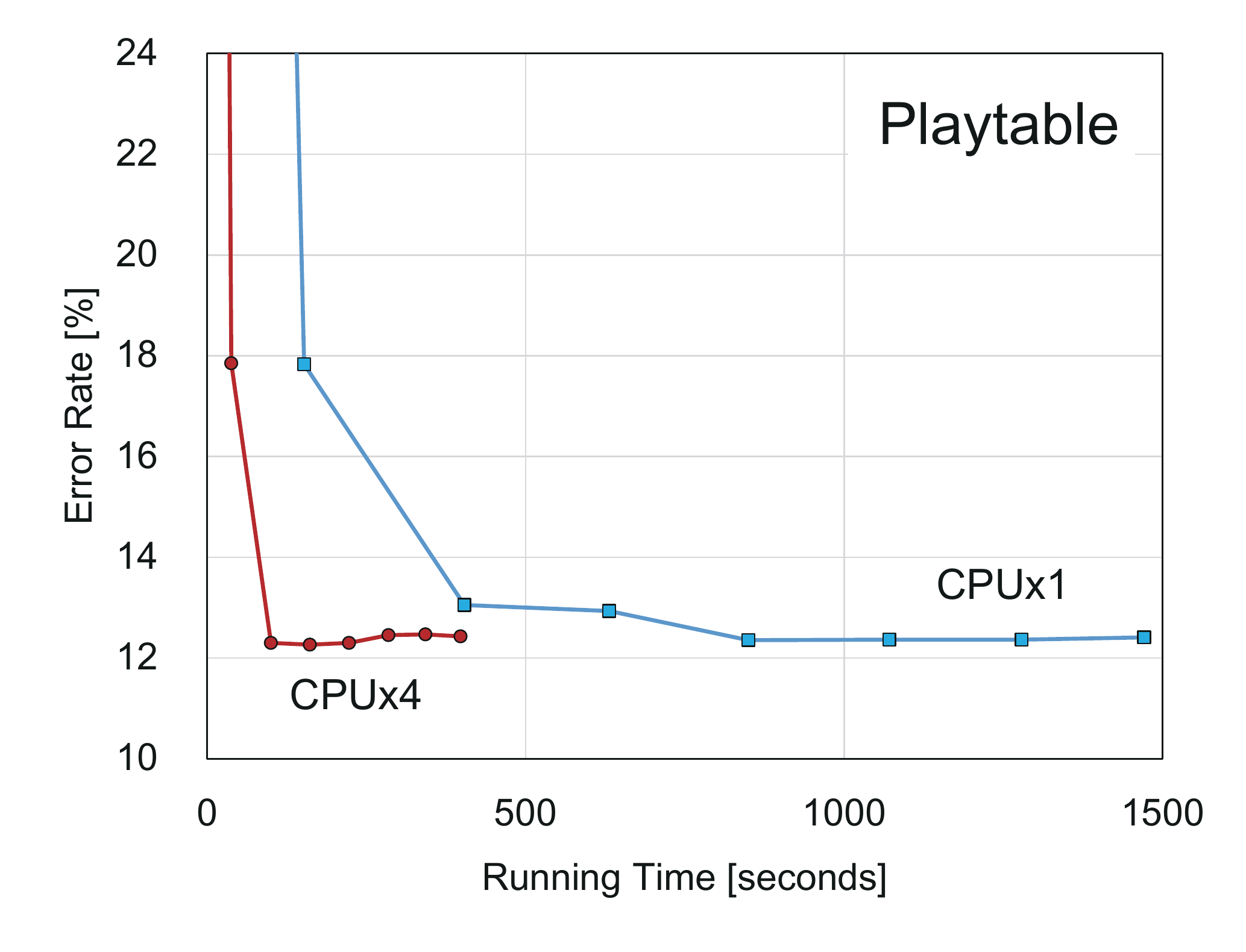}\hfil
		\includegraphics[width=0.3\textwidth]{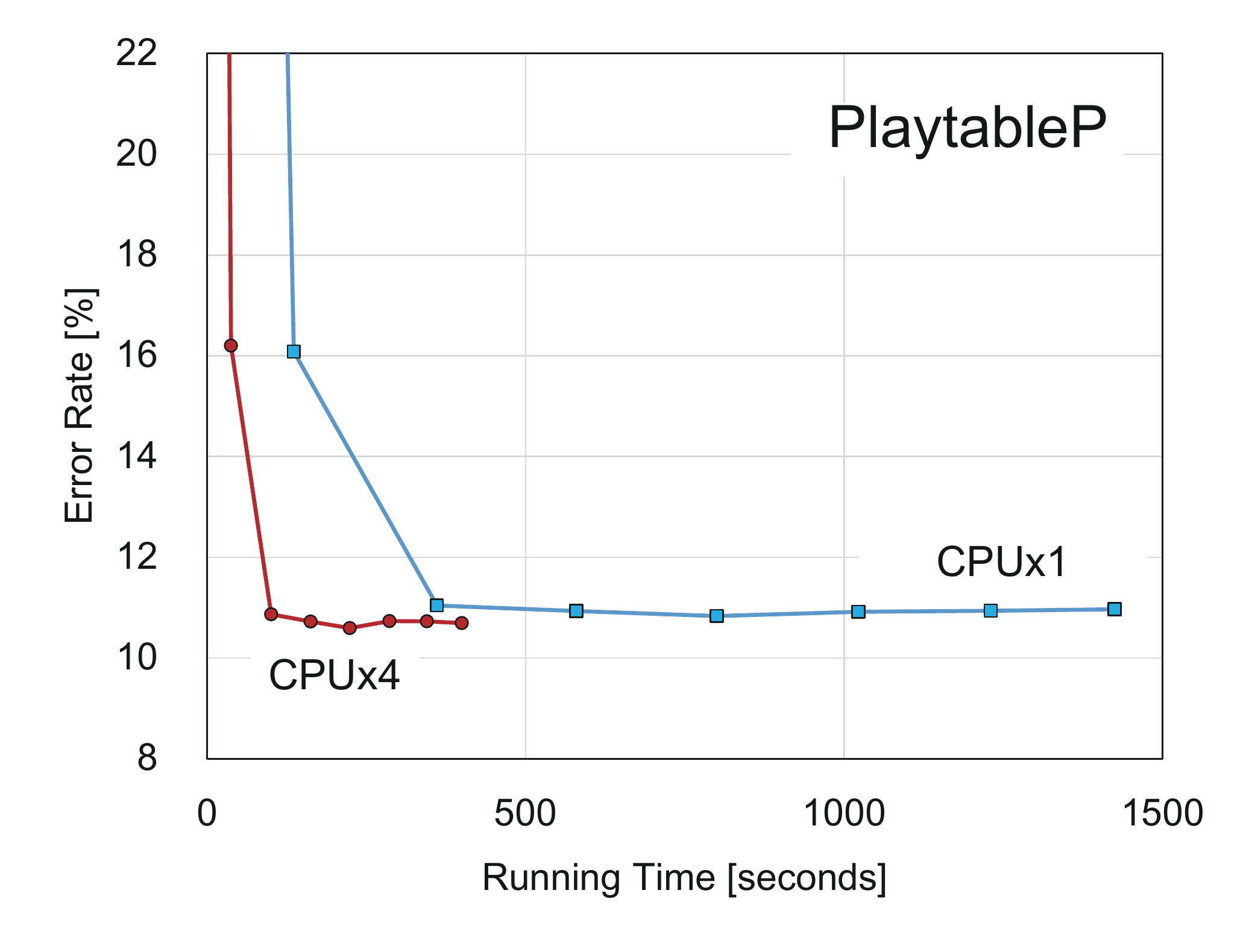}\hfil
		\includegraphics[width=0.3\textwidth]{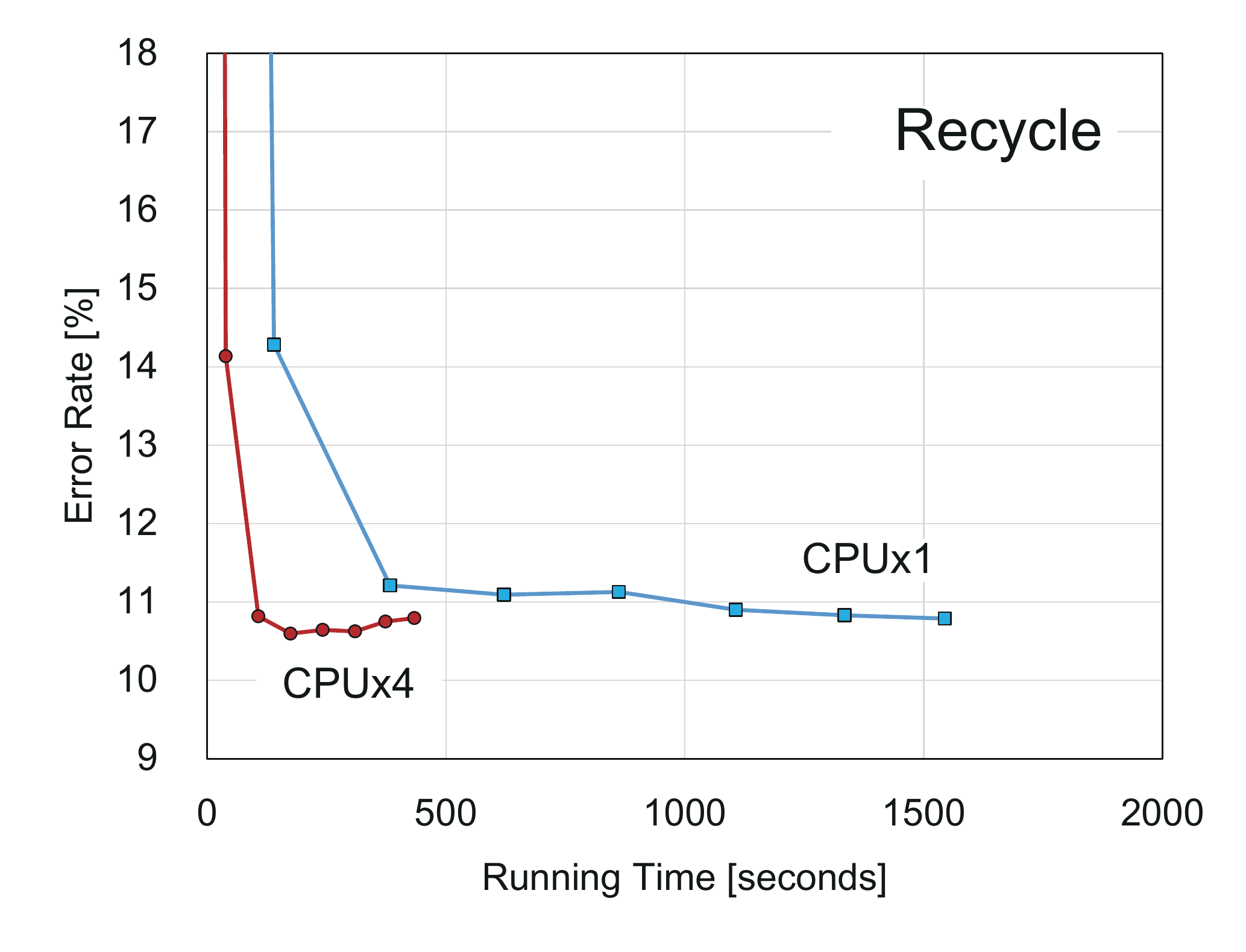}\\
		\includegraphics[width=0.3\textwidth]{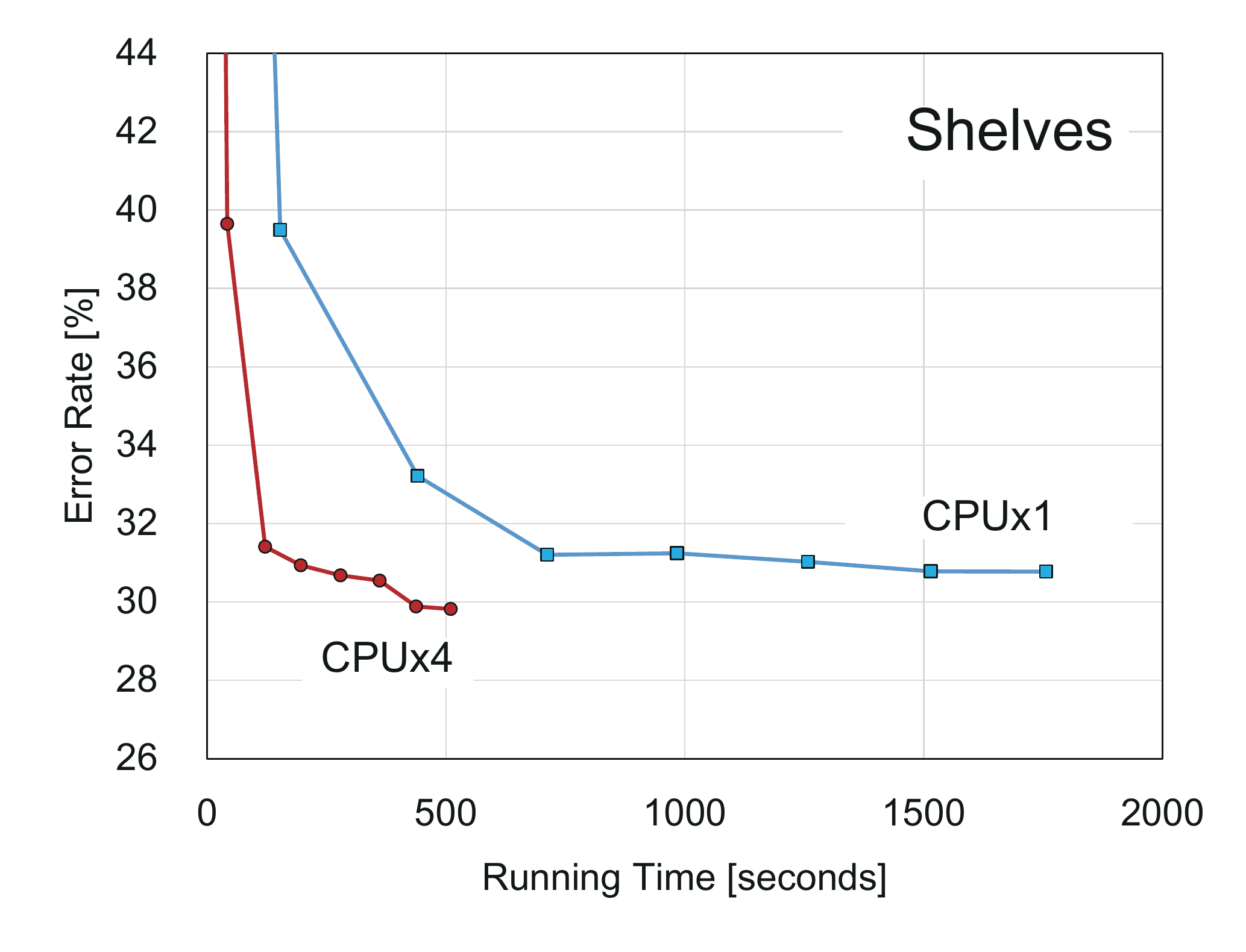}\hfil
		\includegraphics[width=0.3\textwidth]{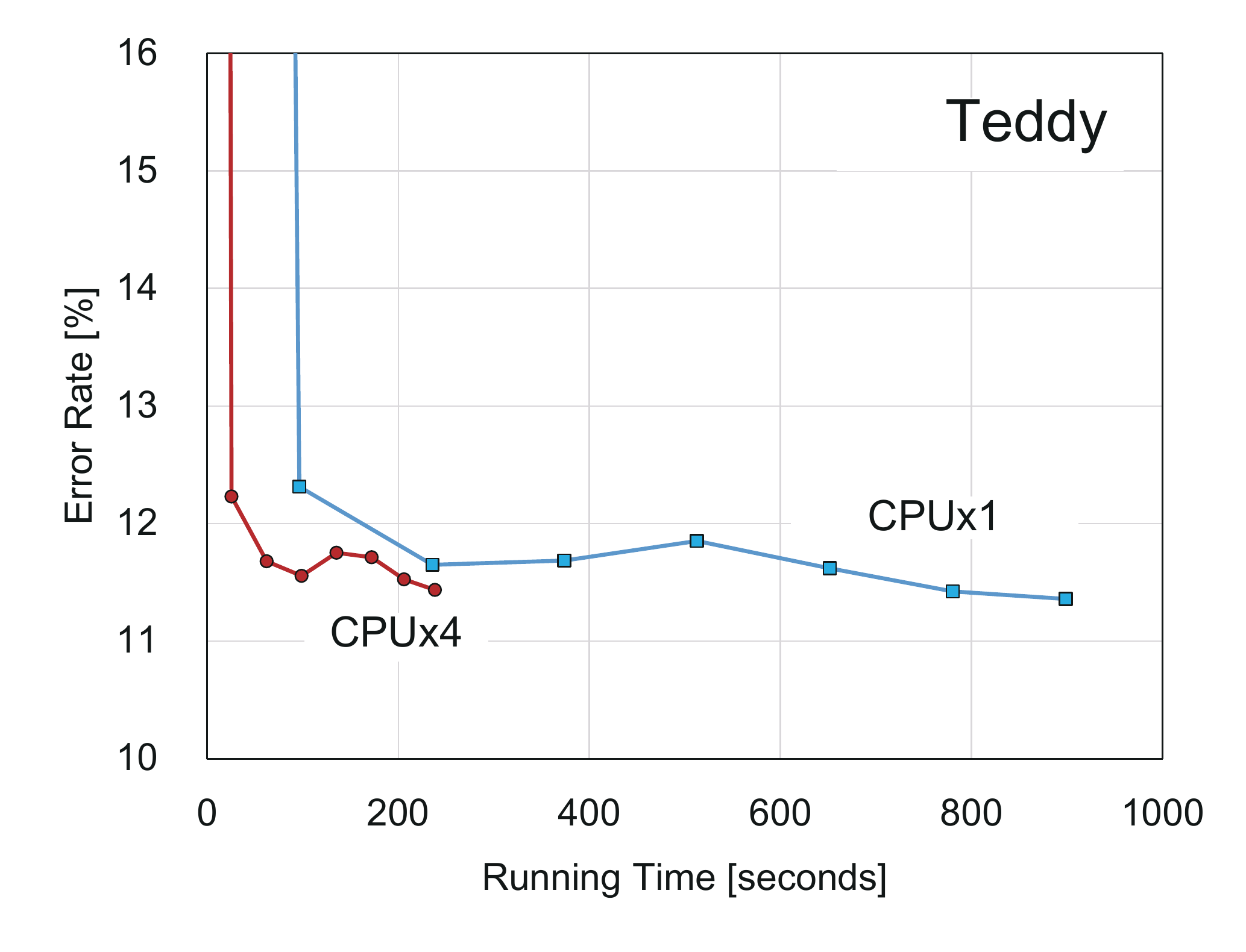}\hfil
		\includegraphics[width=0.3\textwidth]{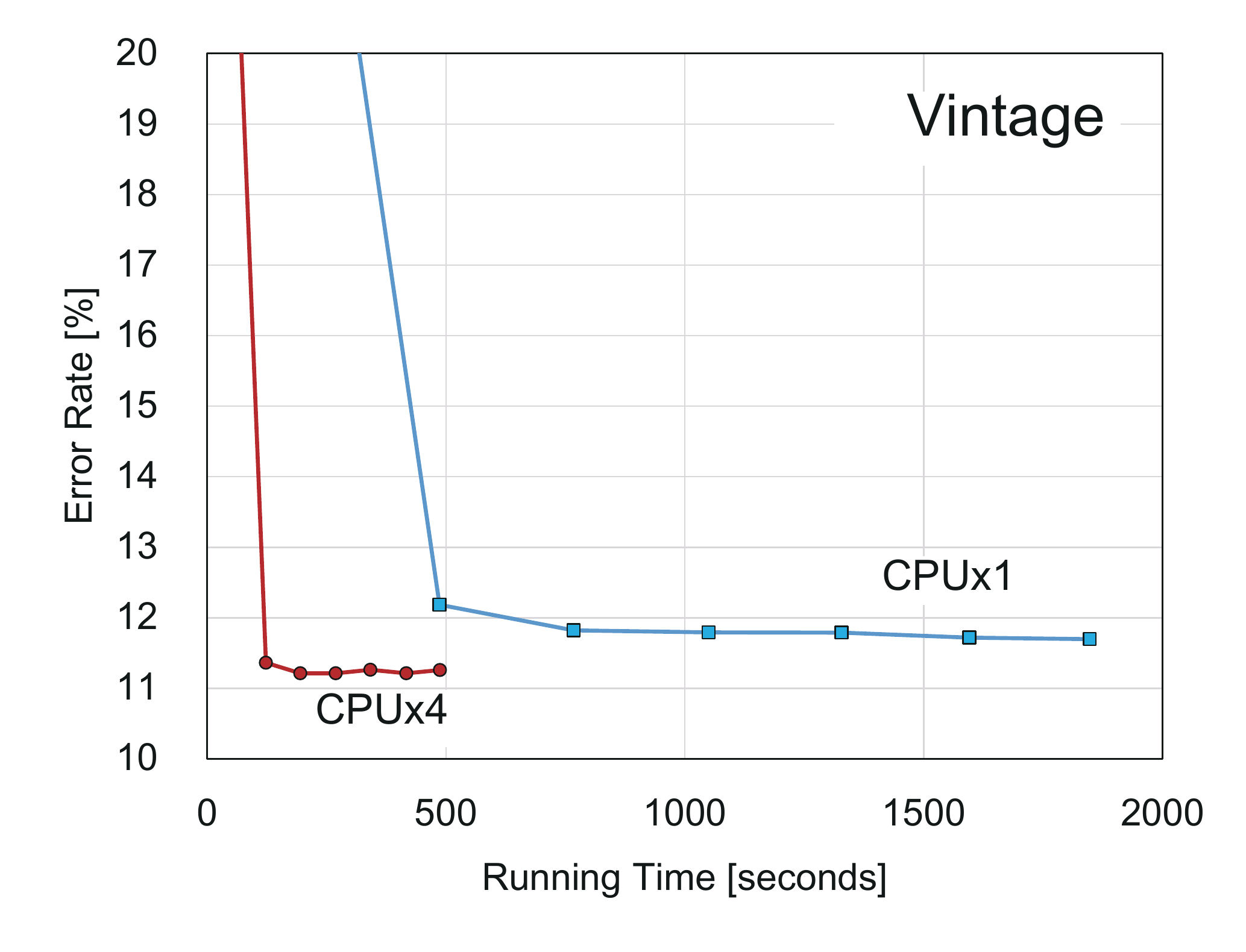}
	\end{center}
	\caption{Convergence comparison for parallelization on 15 image pairs from the Middlebury V3 training dataset. Error rates are evaluated by the \emph{bad 2.0} metric for all regions.
	By performing local expansion moves in parallel on four CPU cores, we observe about 3.8x of average speed-up.  Both methods optimize the same energy function used in Sec~4.2 without post-processing.
	}
	\label{figa:efficiency}
\end{figure}

\begin{figure}
	\begin{center}
		\includegraphics[width=0.3\textwidth]{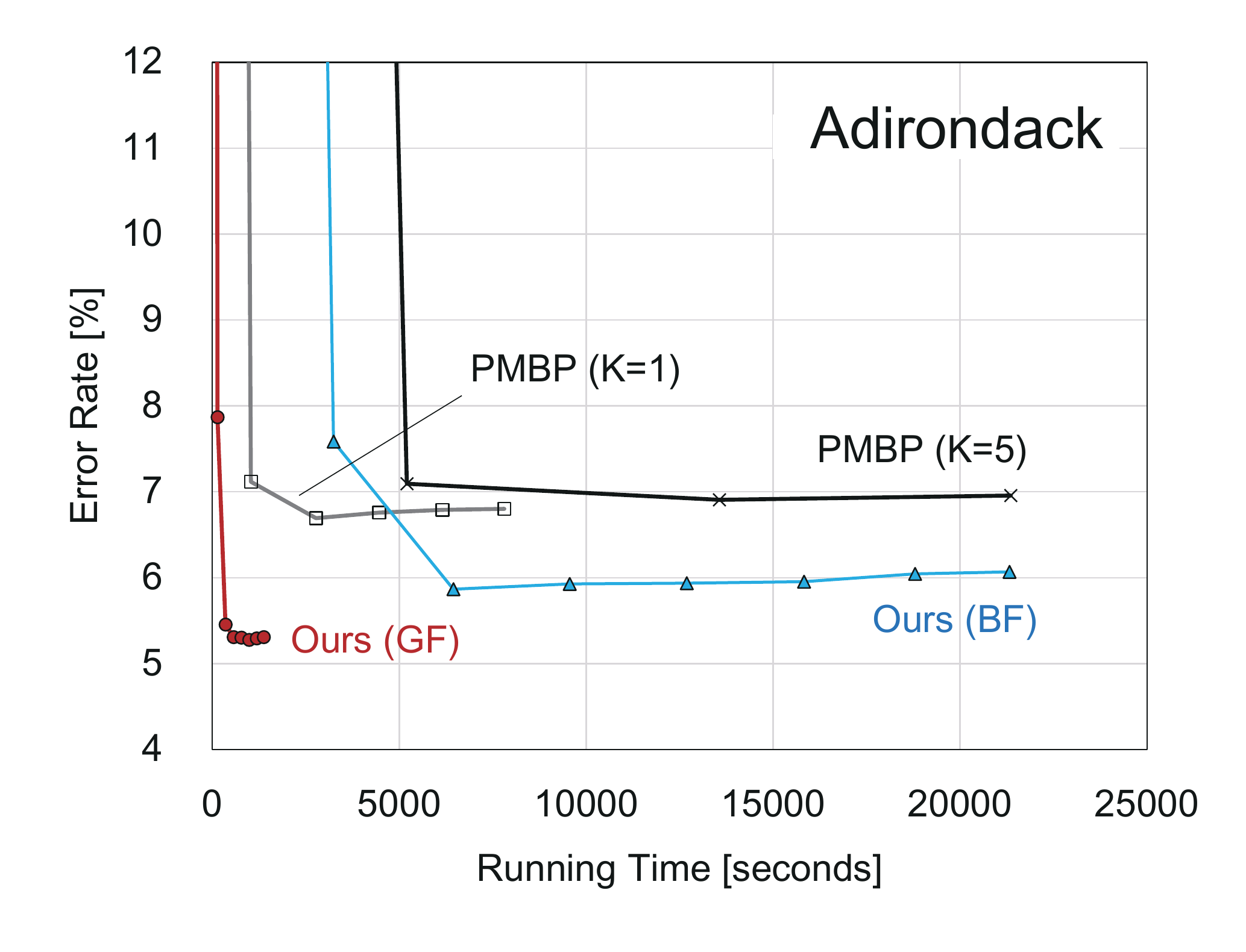}\hfil
		\includegraphics[width=0.3\textwidth]{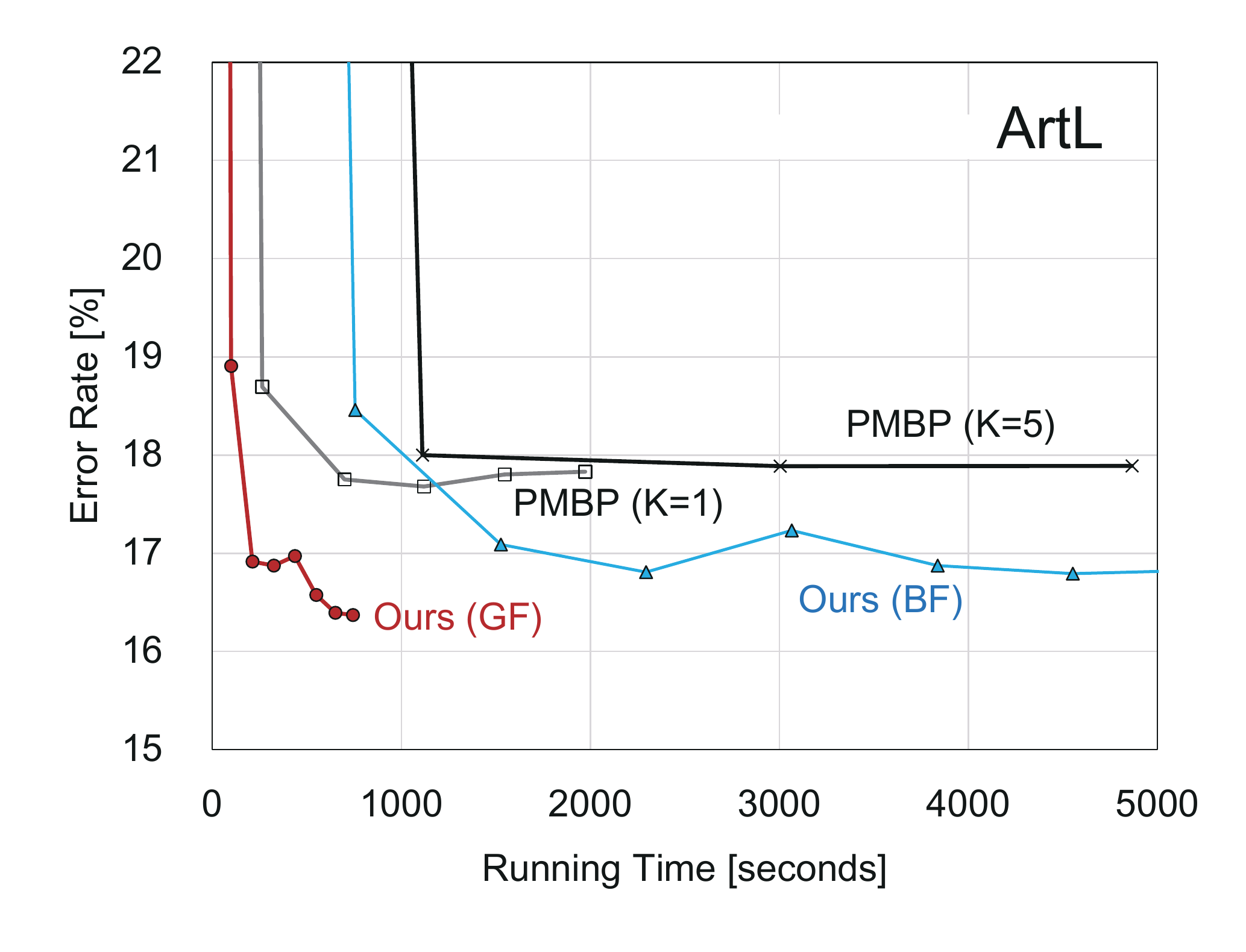}\hfil
		\includegraphics[width=0.3\textwidth]{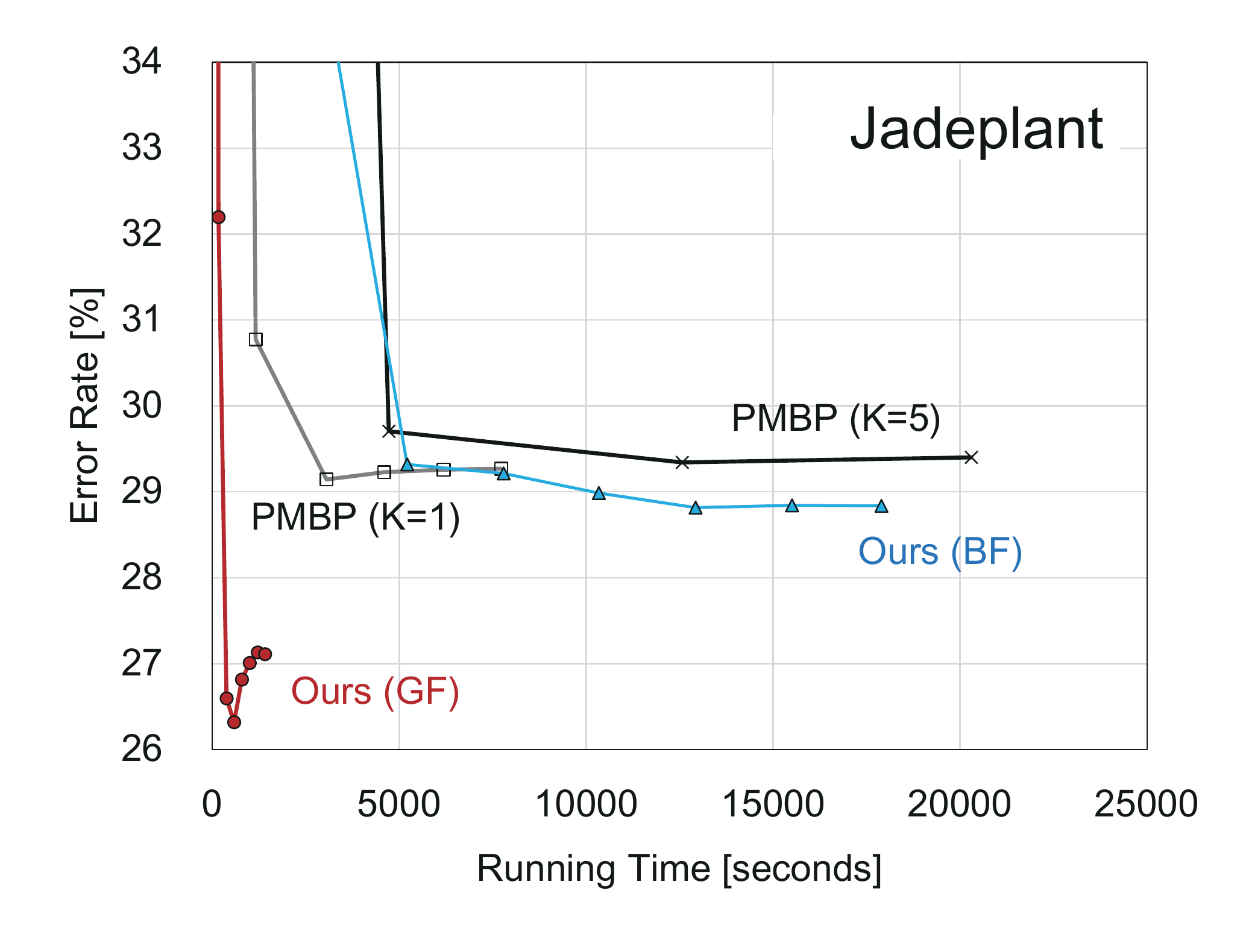}\\
		\includegraphics[width=0.3\textwidth]{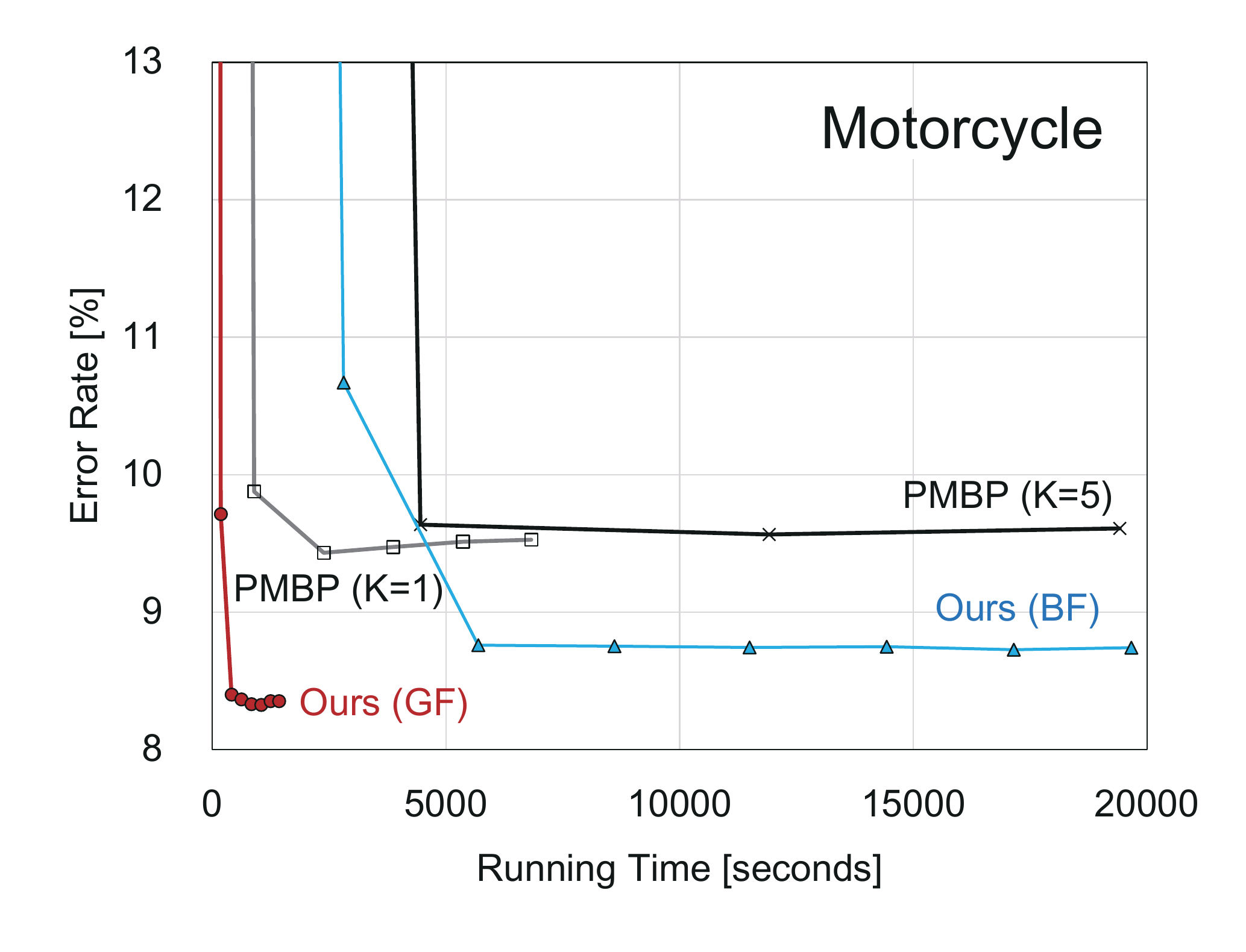}\hfil
		\includegraphics[width=0.3\textwidth]{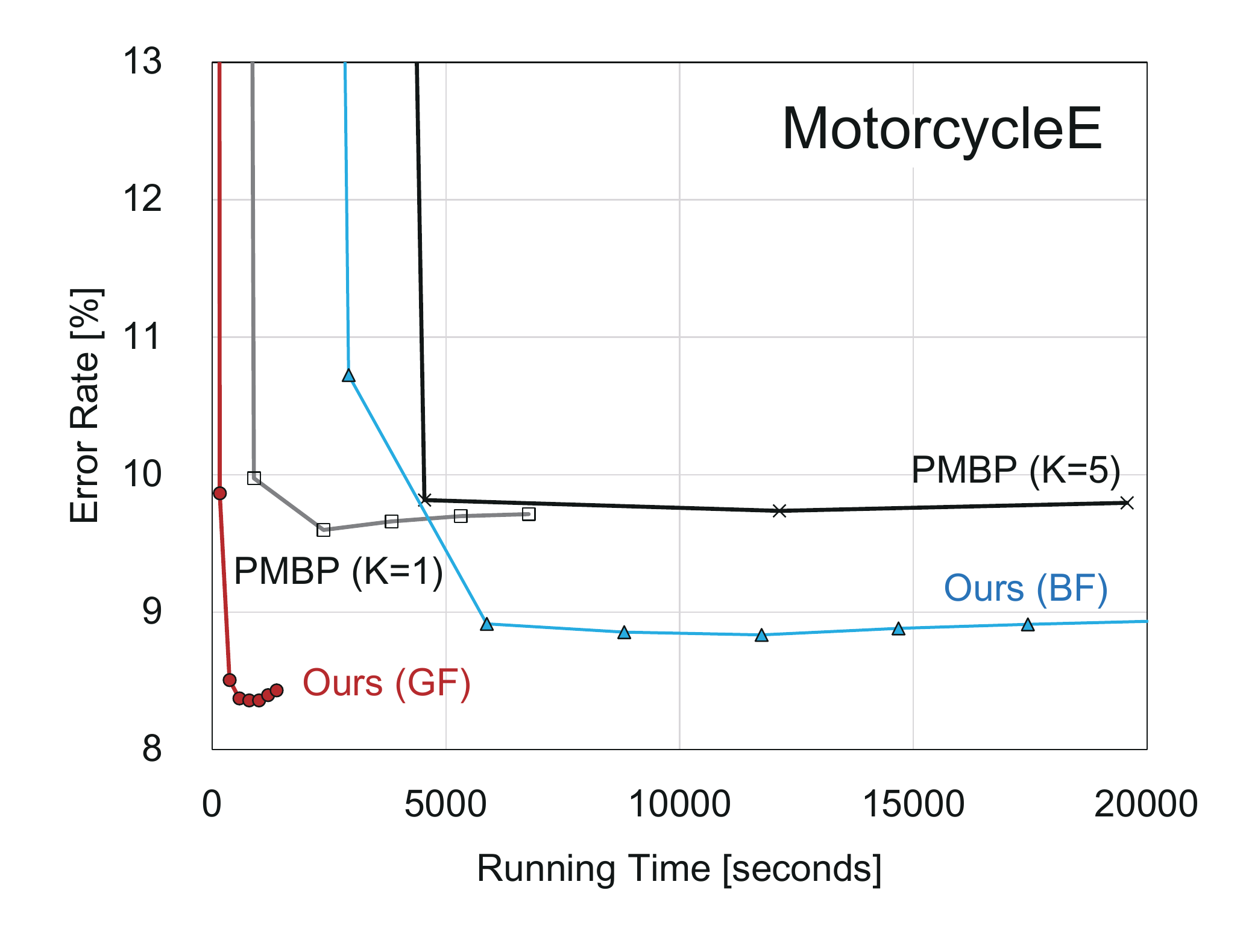}\hfil
		\includegraphics[width=0.3\textwidth]{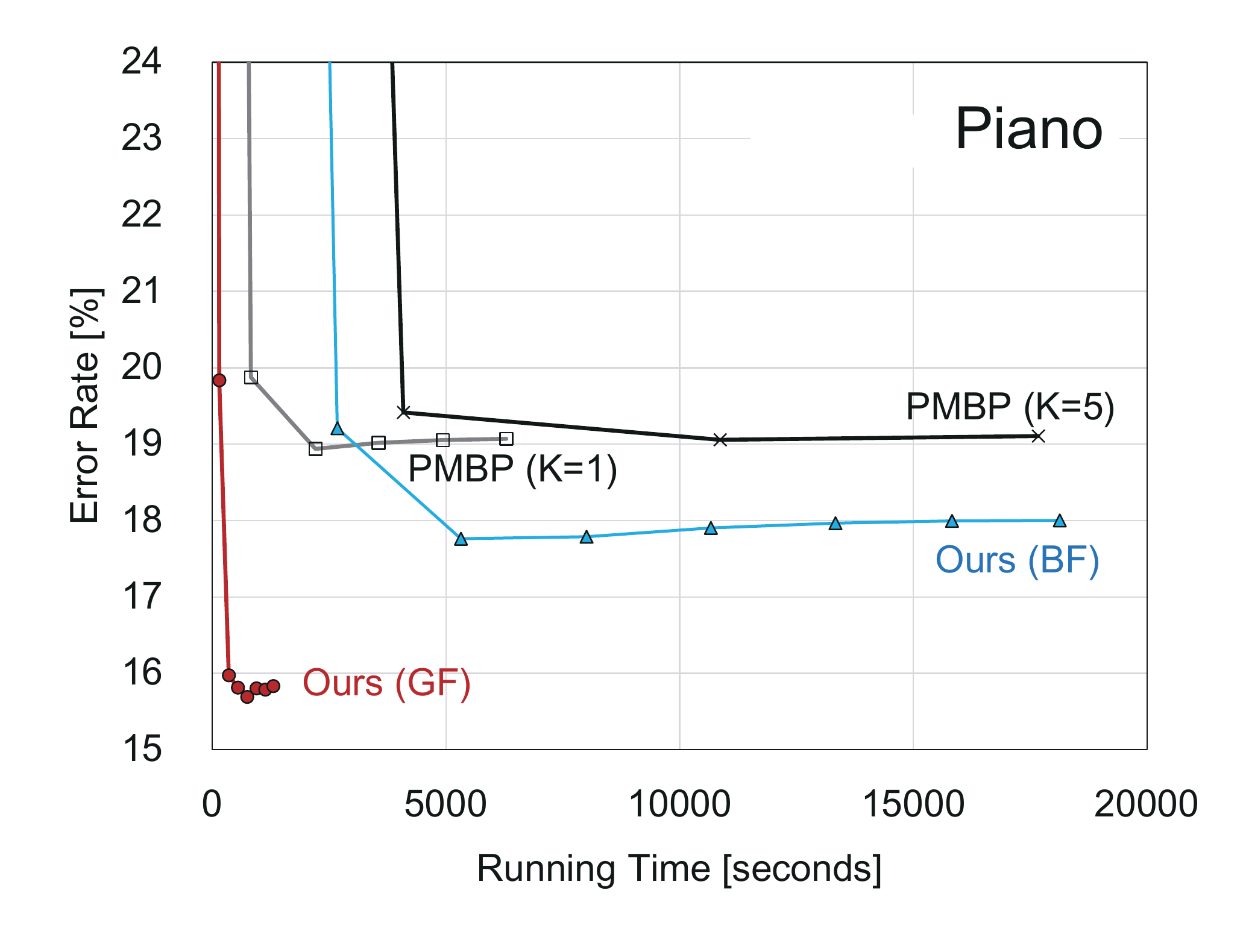}\\
		\includegraphics[width=0.3\textwidth]{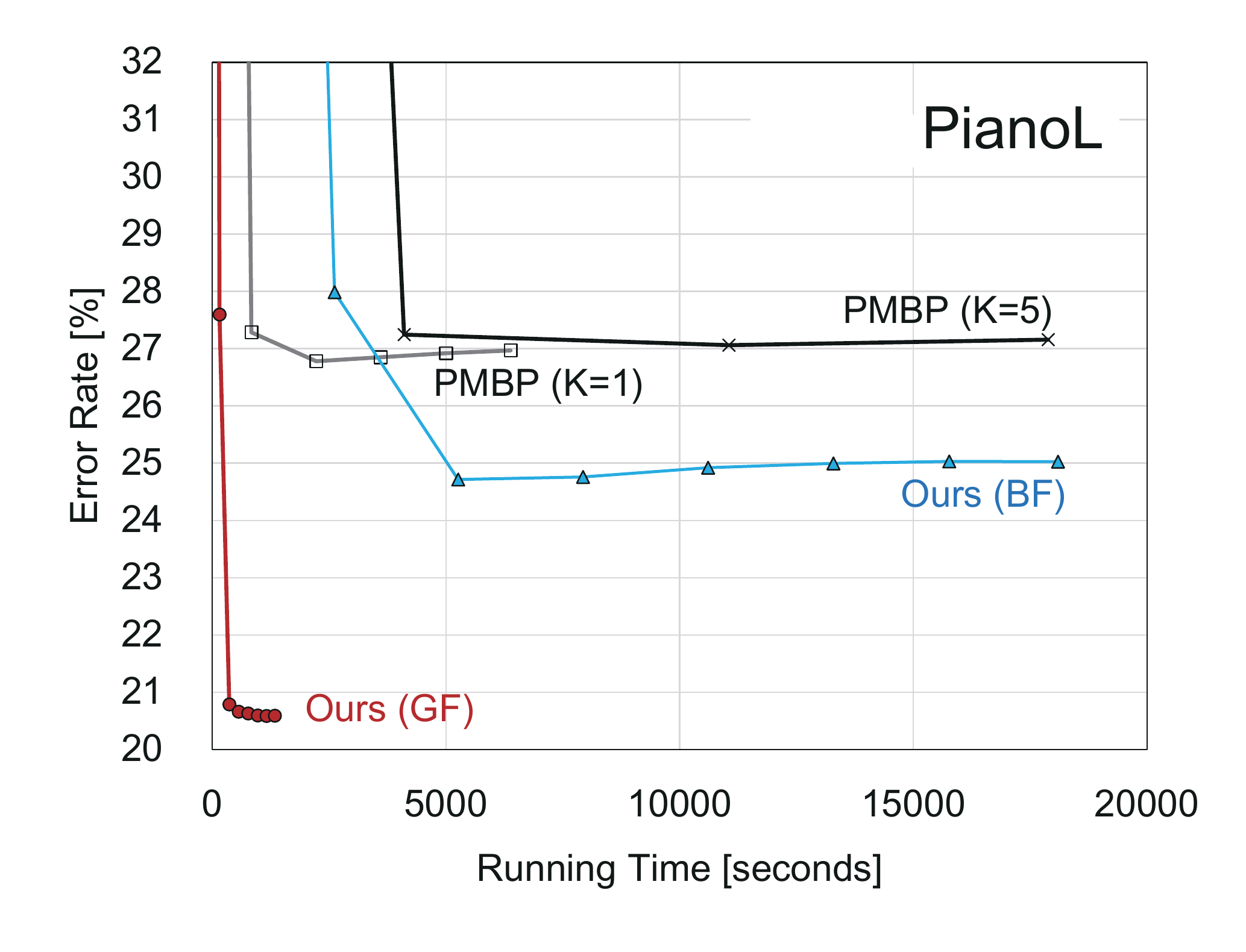}\hfil
		\includegraphics[width=0.3\textwidth]{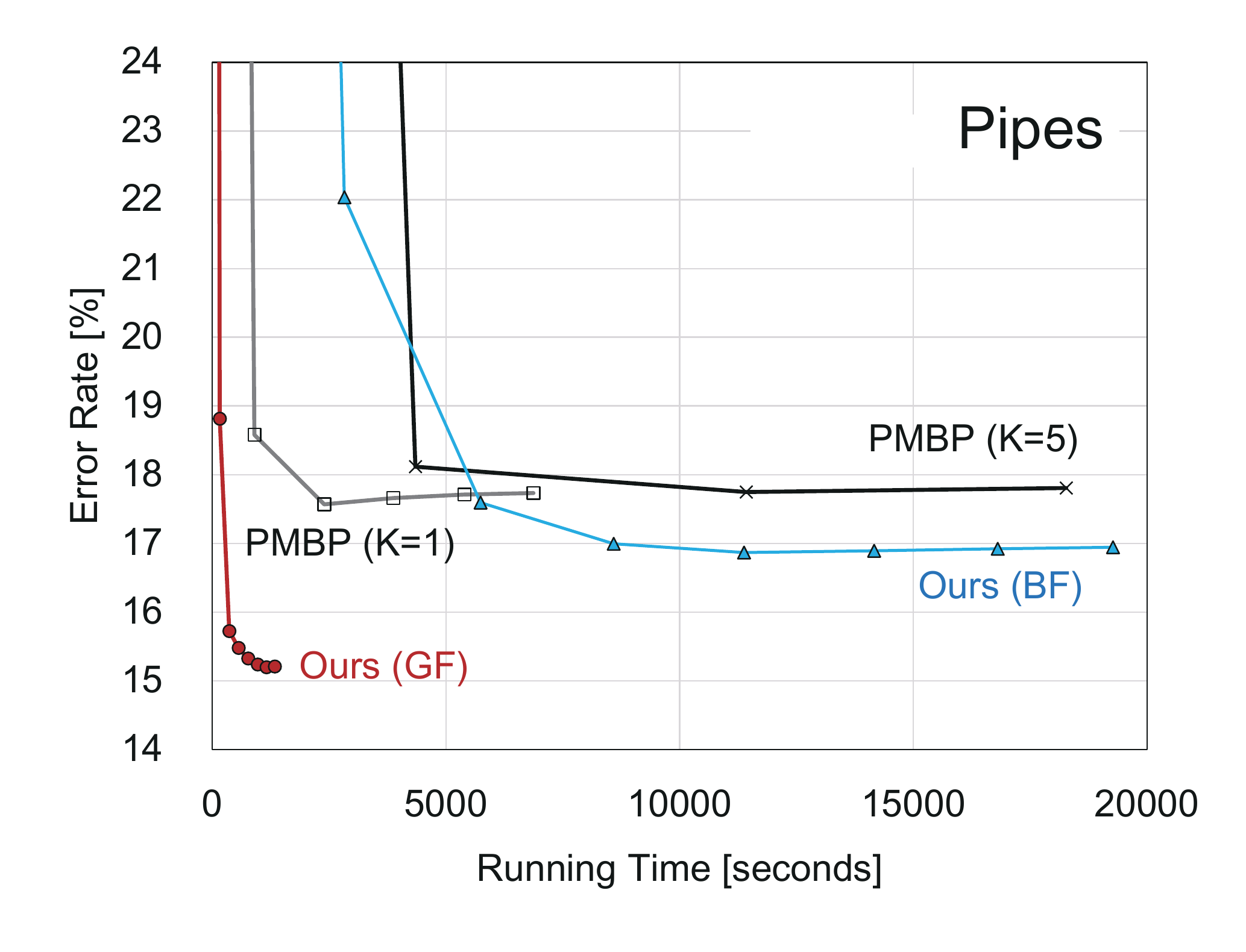}\hfil
		\includegraphics[width=0.3\textwidth]{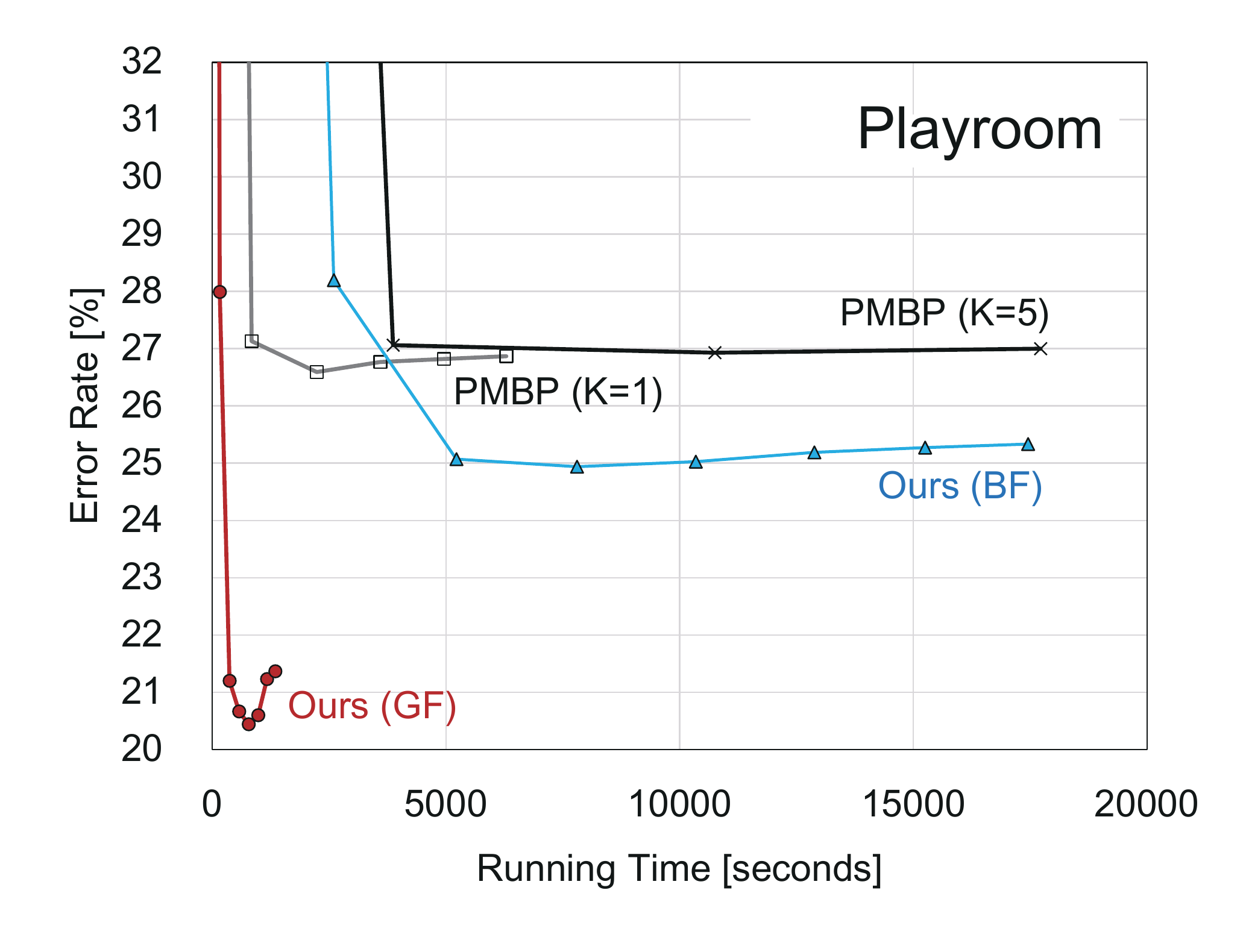}\\
		\includegraphics[width=0.3\textwidth]{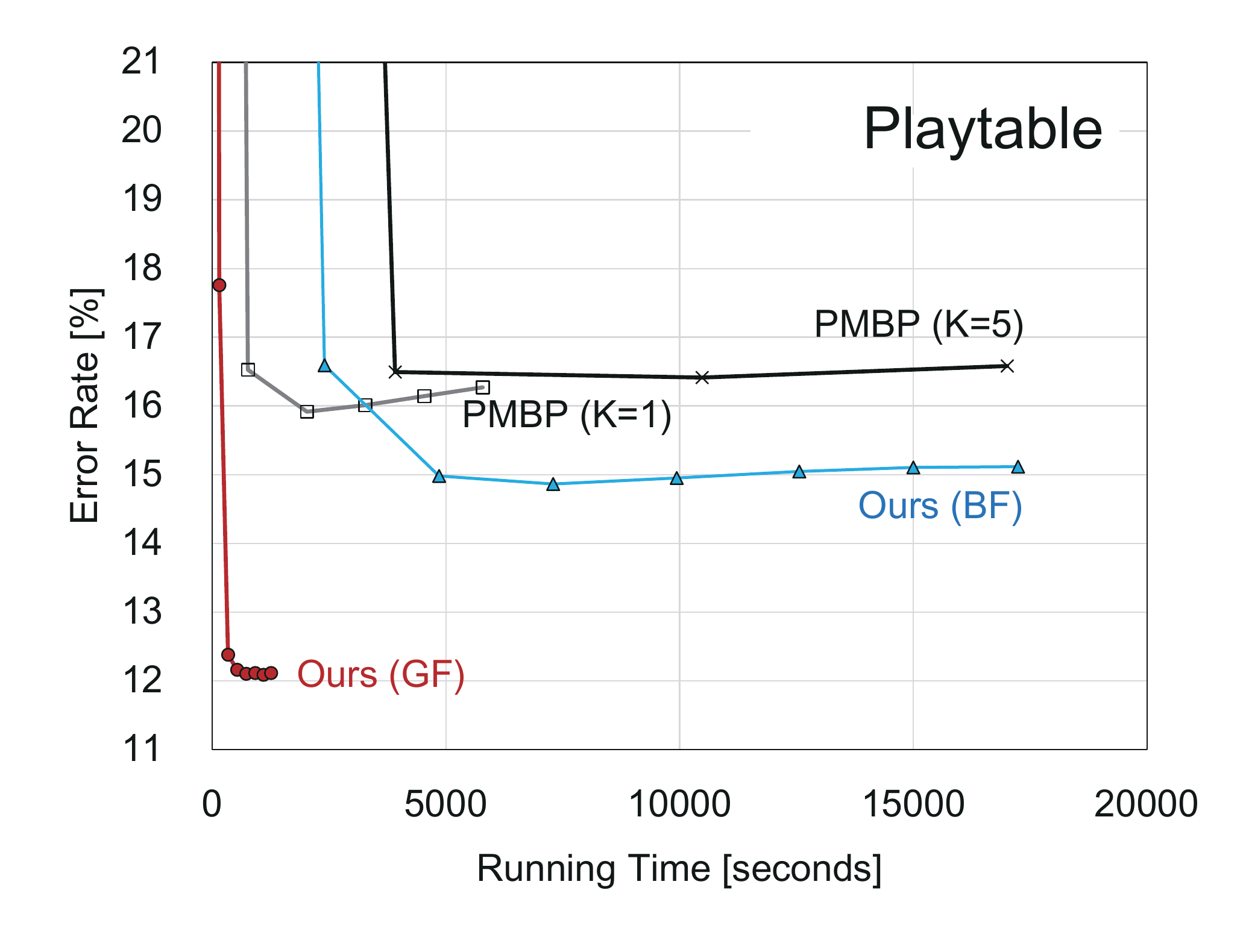}\hfil
		\includegraphics[width=0.3\textwidth]{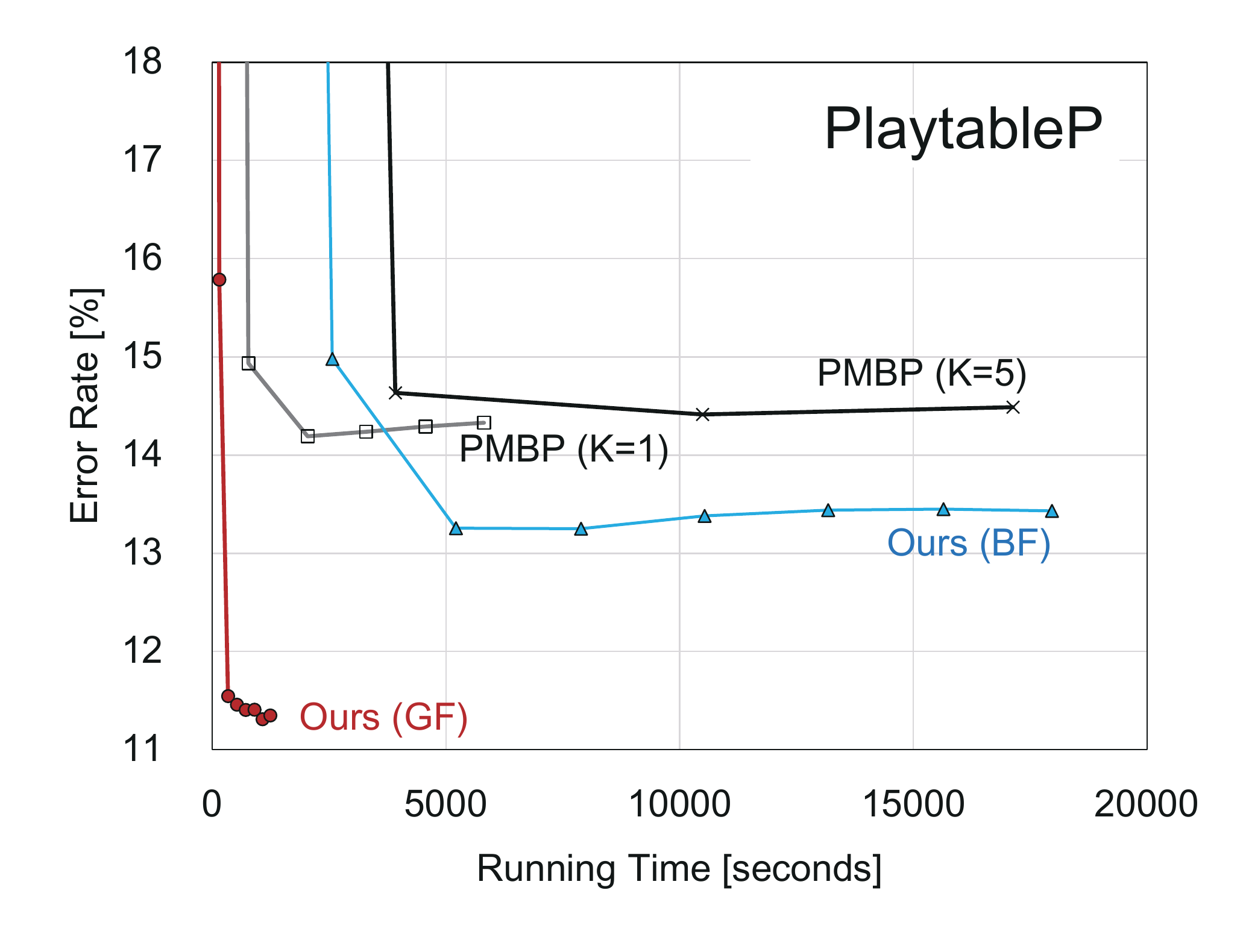}\hfil
		\includegraphics[width=0.3\textwidth]{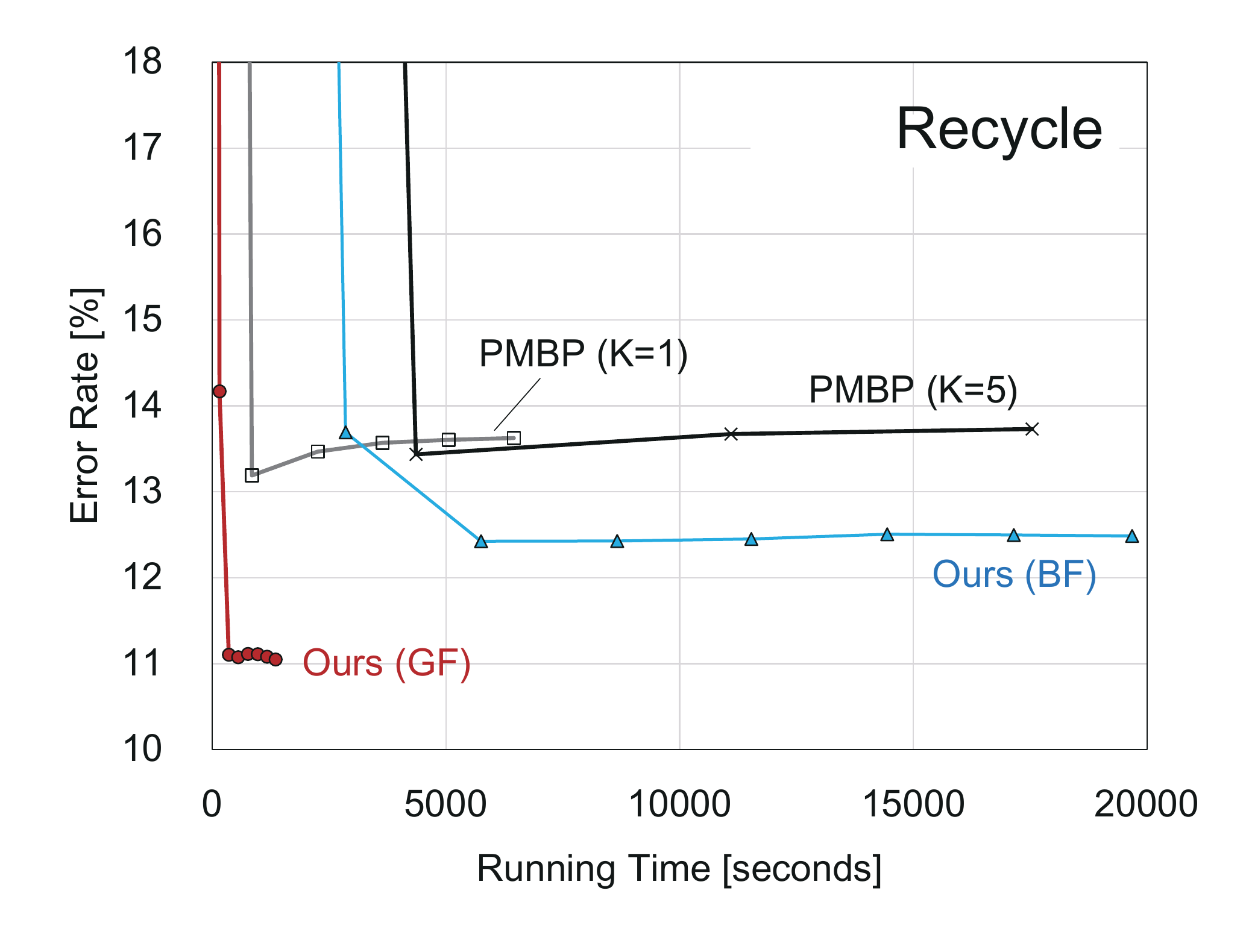}\\
		\includegraphics[width=0.3\textwidth]{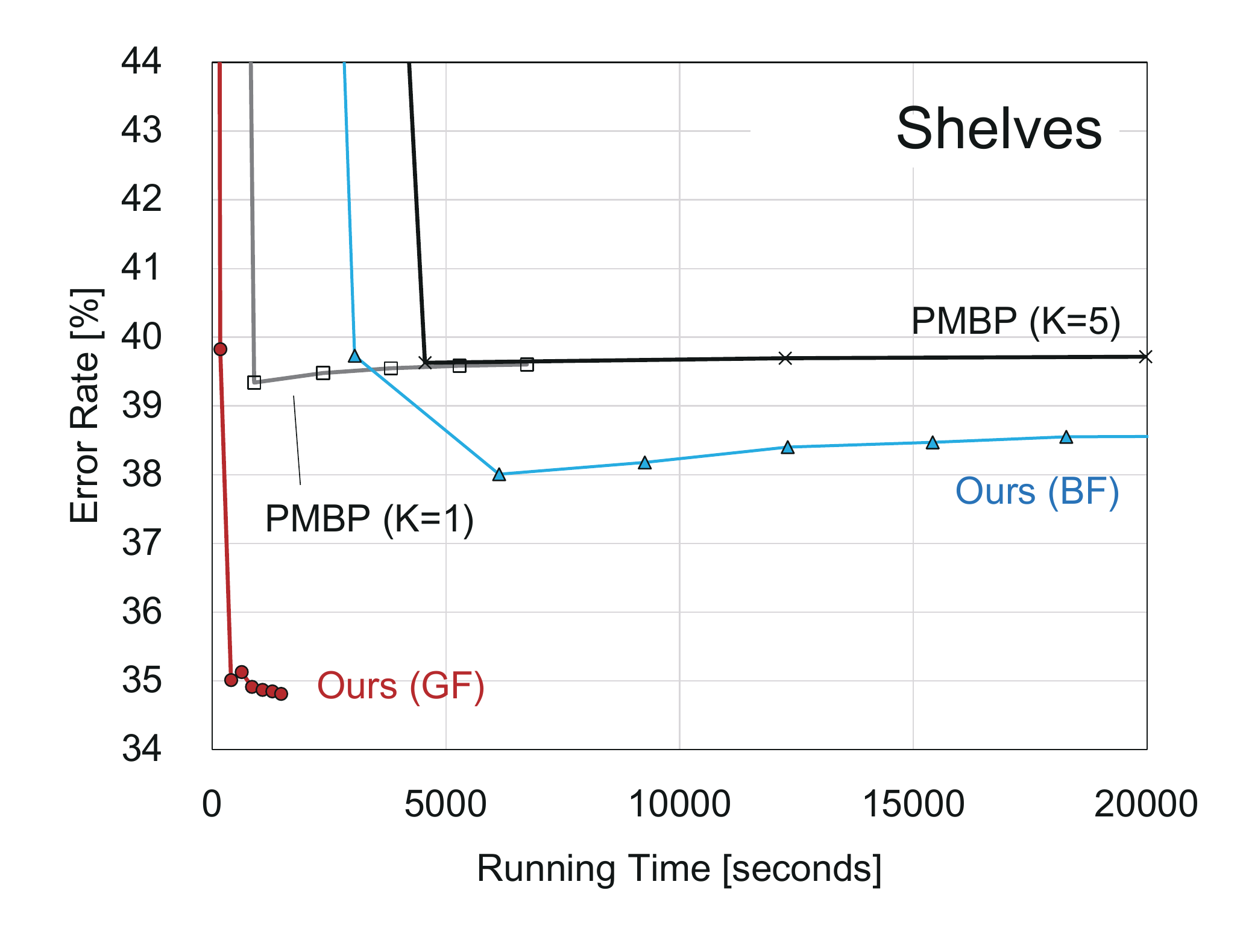}\hfil
		\includegraphics[width=0.3\textwidth]{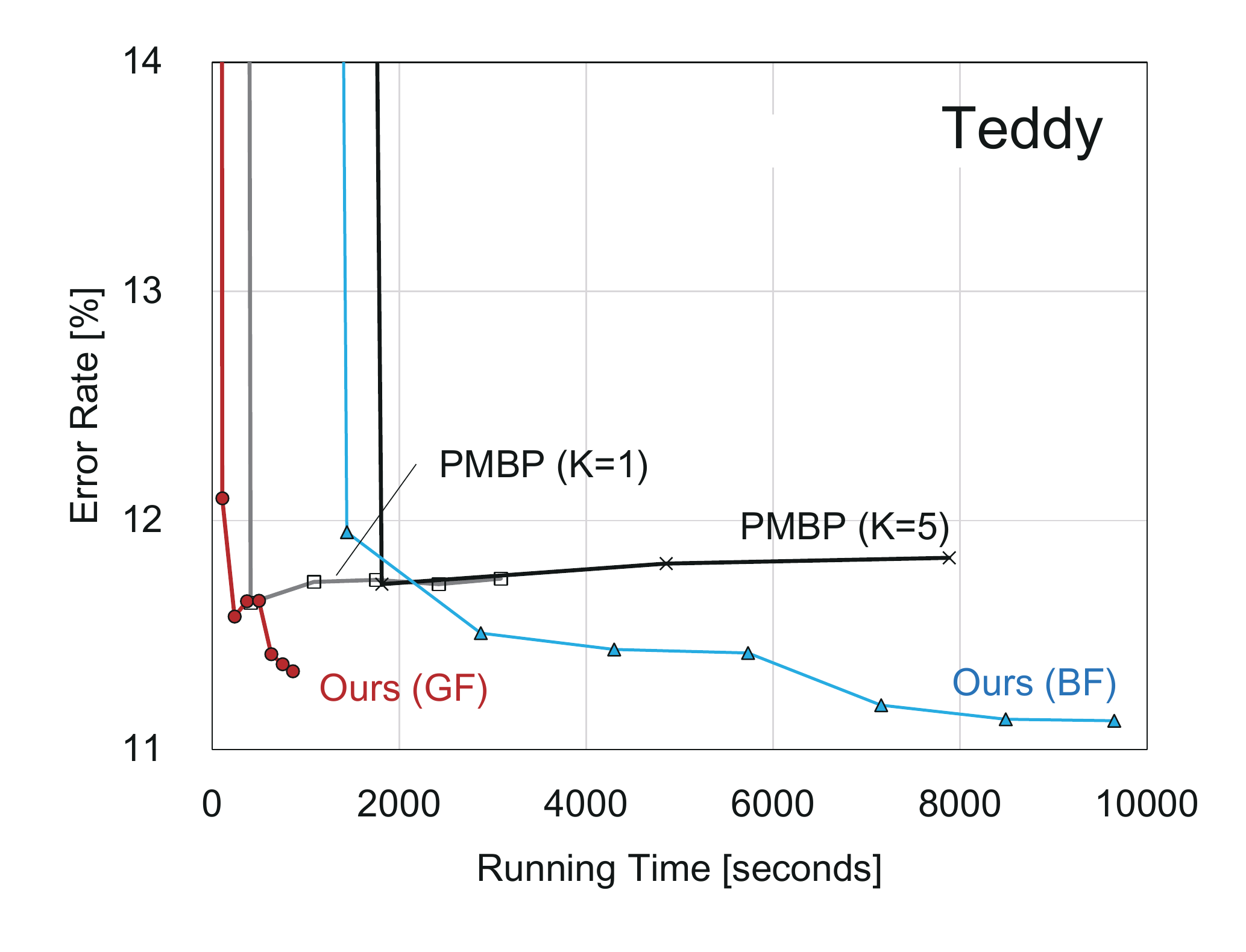}\hfil
		\includegraphics[width=0.3\textwidth]{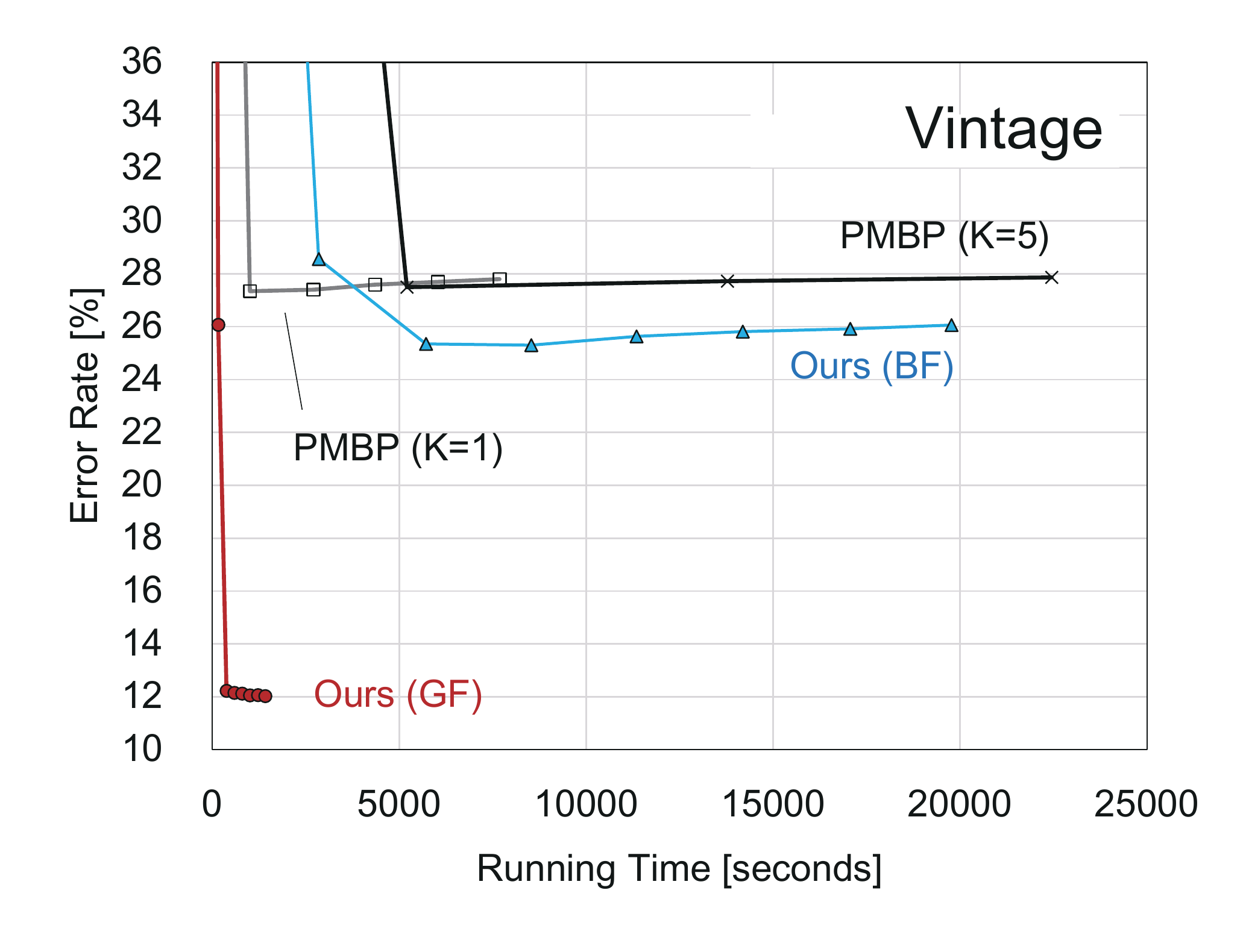}
	\end{center}
	\caption{Convergence comparison with PMBP~\cite{Besse12} on 15 image pairs from the Middlebury V3 training dataset. Error rates are evaluated by the \emph{bad 2.0} metric for all regions. Our methods both using guided image filter (GF) and bilateral filter (BF) consistently outperform PMBP in accuracy at convergence.
	For most of the image pairs, our method using GF performs best, showing much faster convergence than the others.
	Note that both PMBP and ours using BF optimize the same energy function, but ours using GF optimizes a different function. All methods are run on a single CPU core without post-processing.
	Our method can be further accelerated by parallelization as demonstrated in Fig.~\ref{figa:efficiency}.}
	\label{figa:pmbp}
\end{figure}

\begin{figure}
	\begin{center}
		\includegraphics[width=0.3\textwidth]{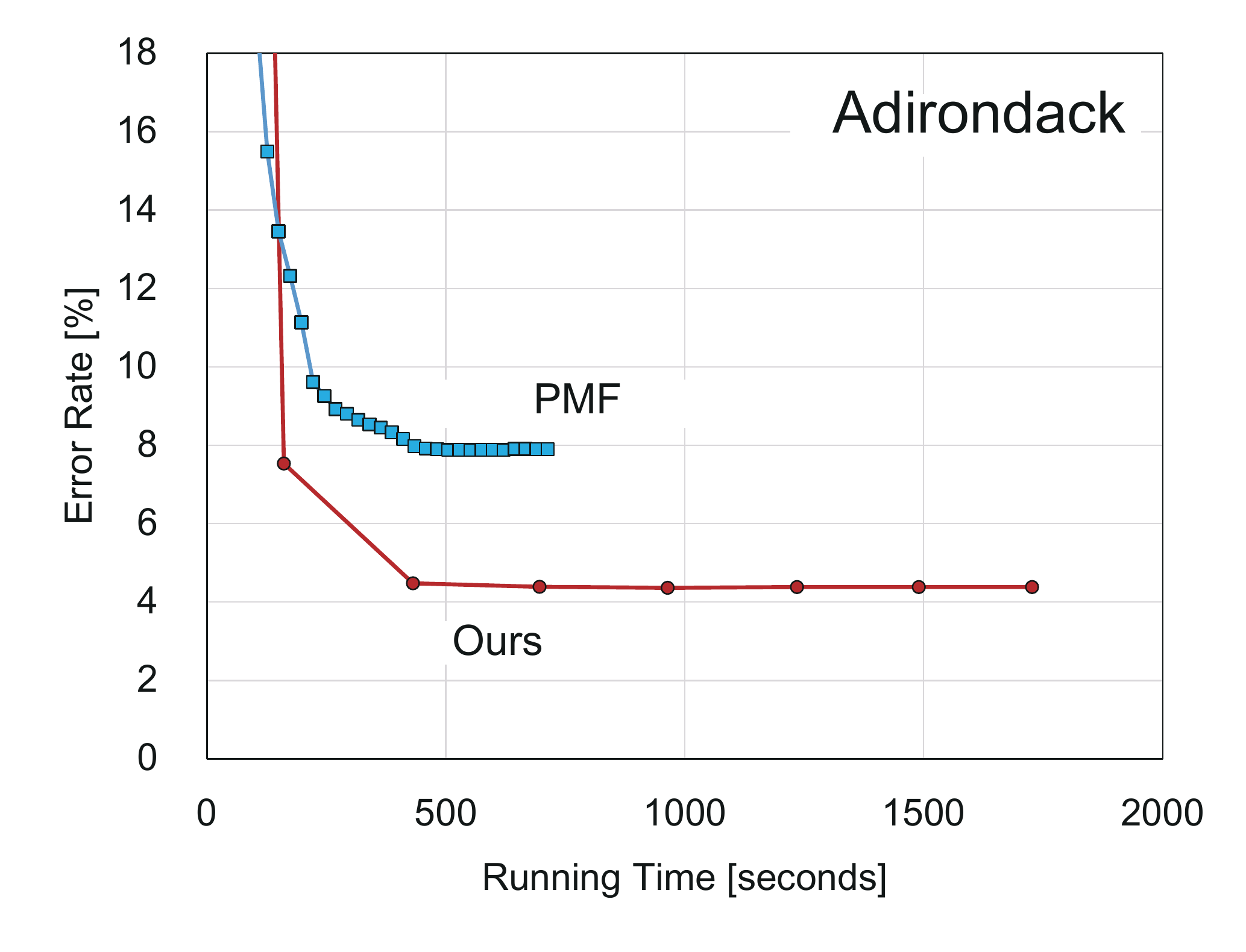}\hfil
		\includegraphics[width=0.3\textwidth]{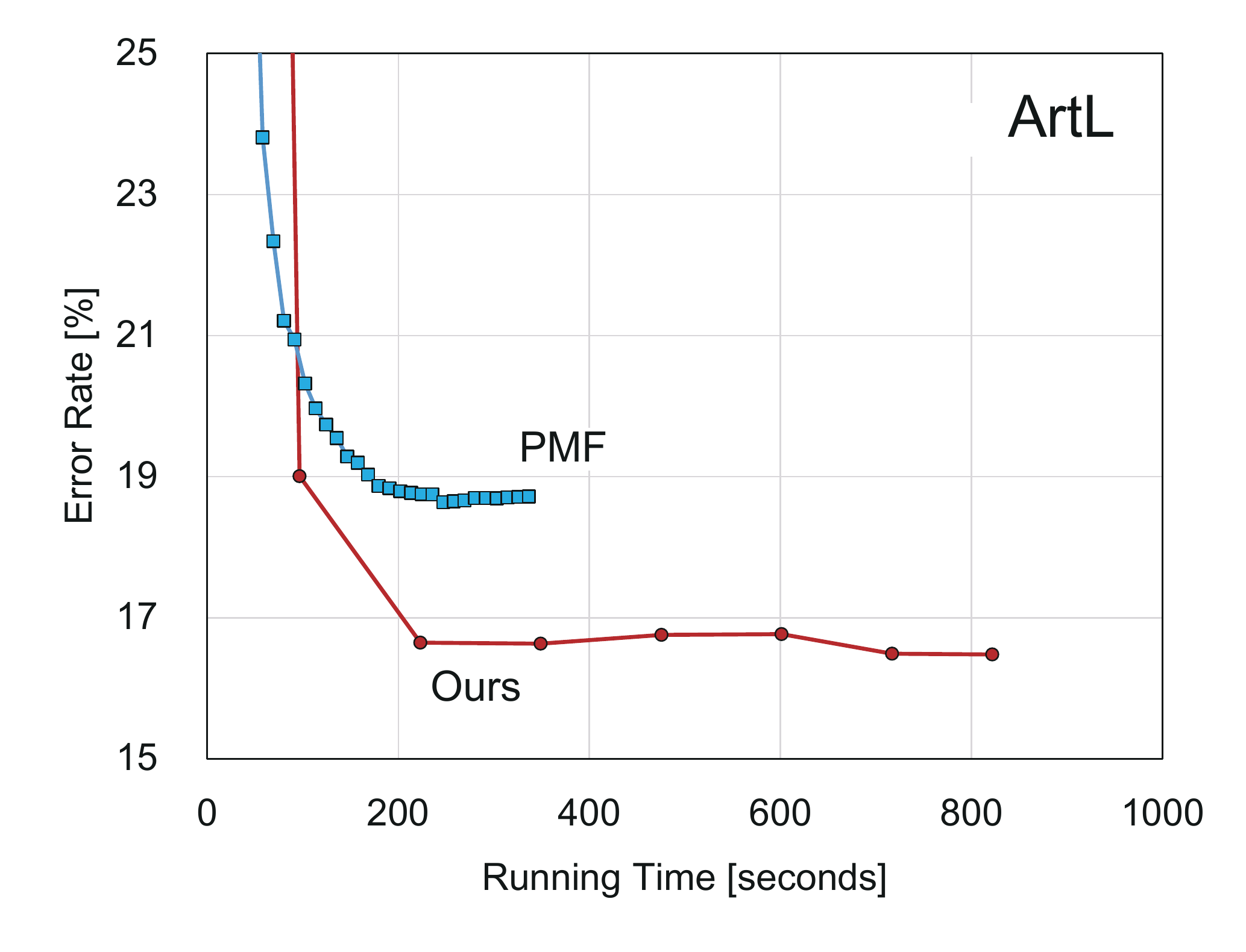}\hfil
		\includegraphics[width=0.3\textwidth]{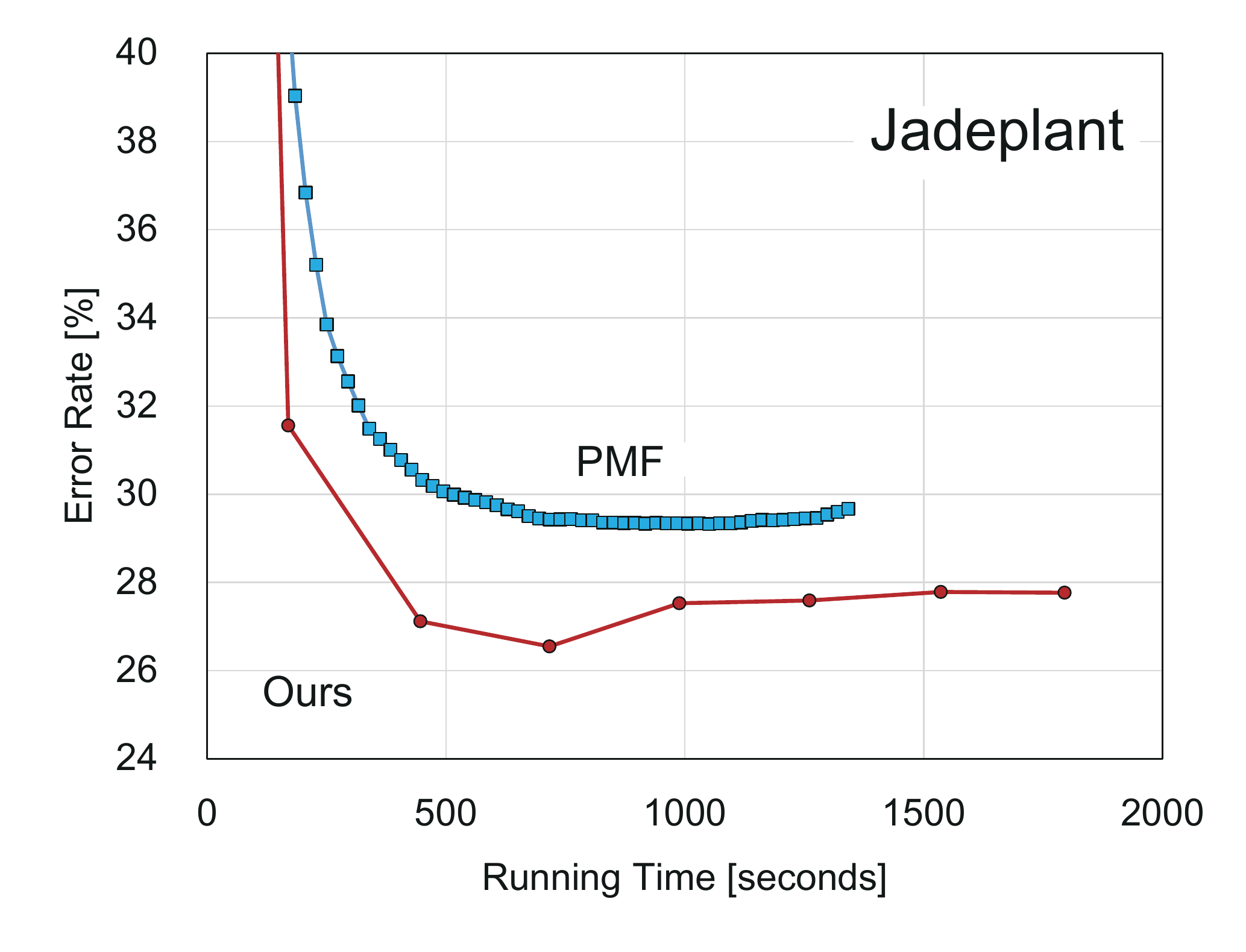}\\
		\includegraphics[width=0.3\textwidth]{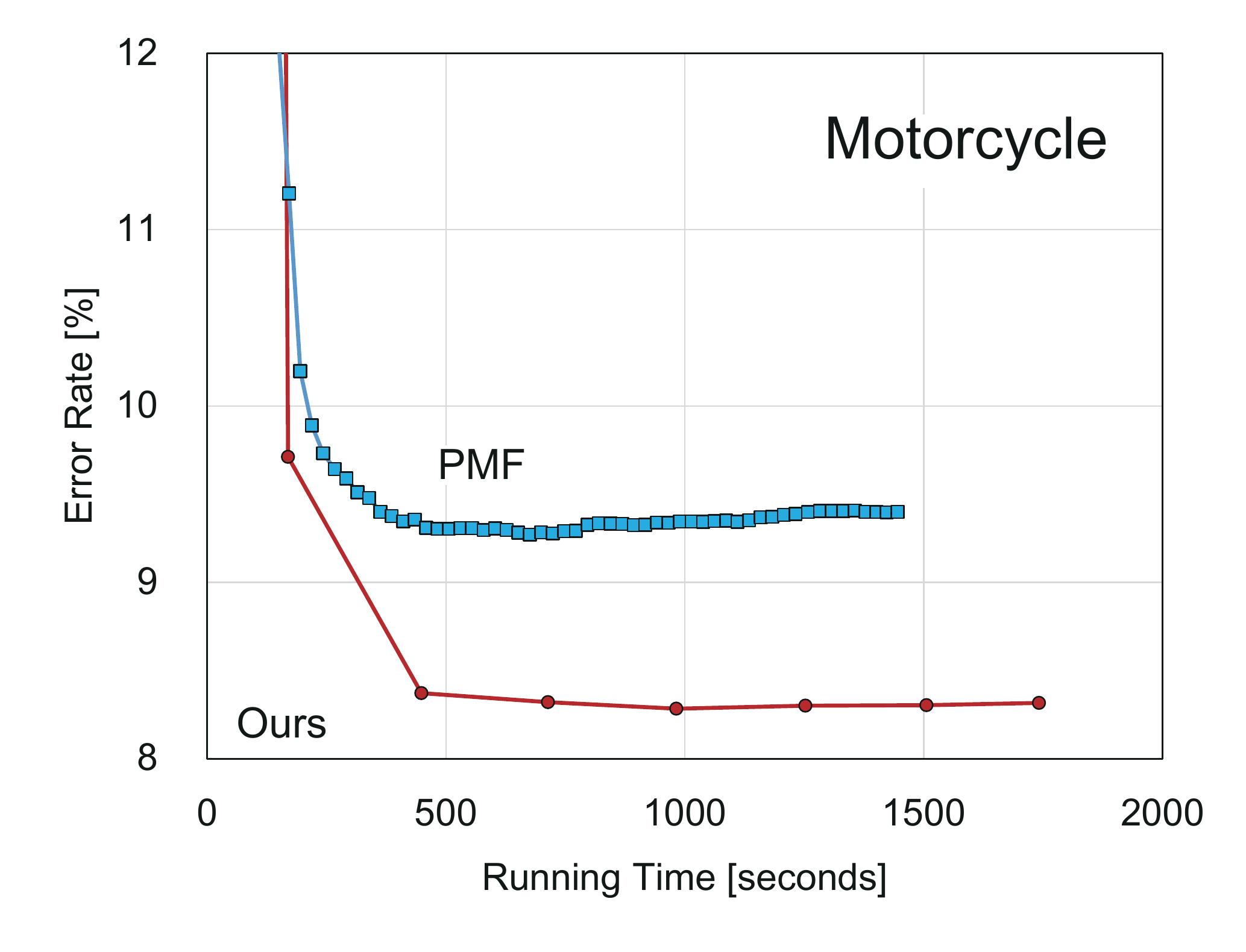}\hfil
		\includegraphics[width=0.3\textwidth]{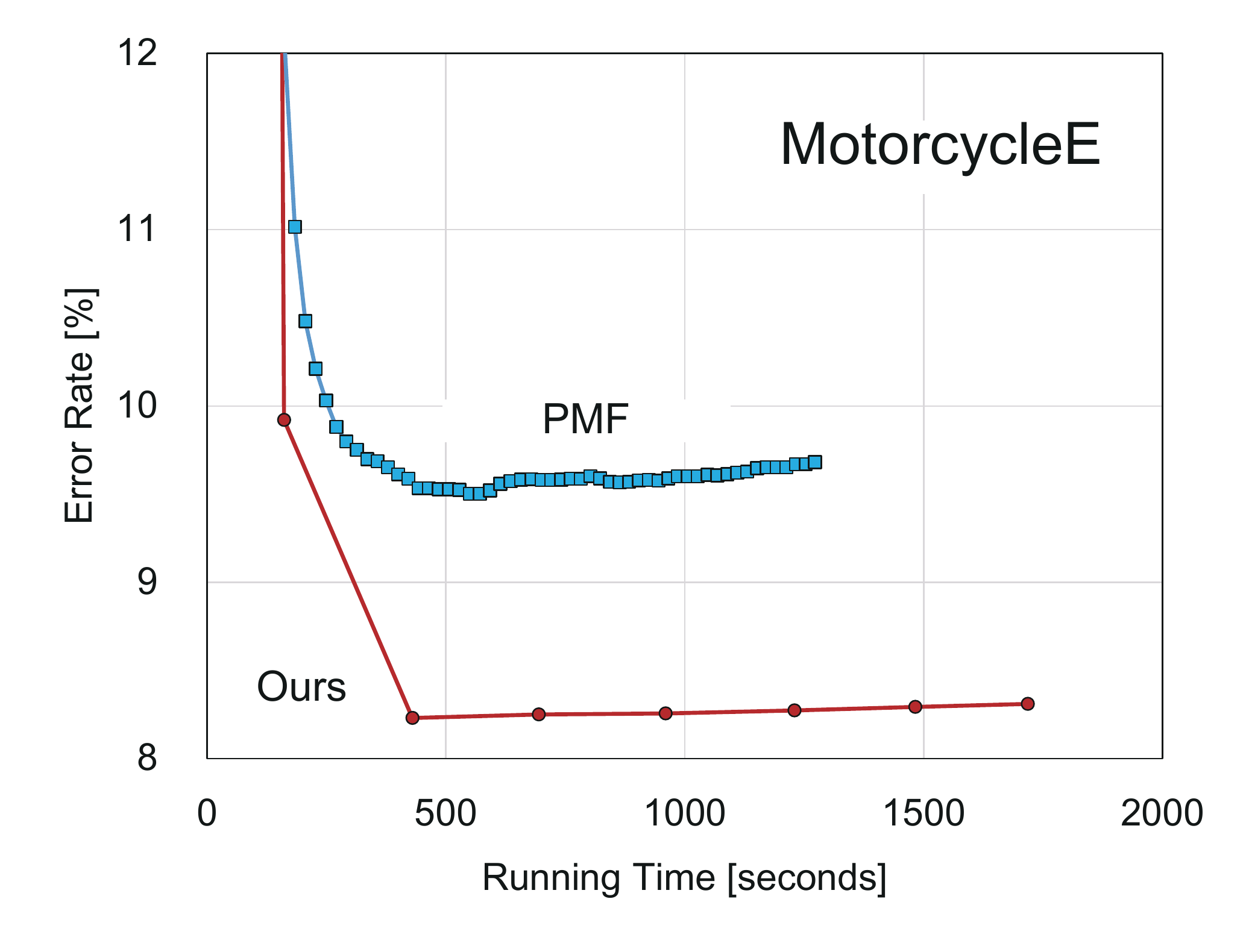}\hfil
		\includegraphics[width=0.3\textwidth]{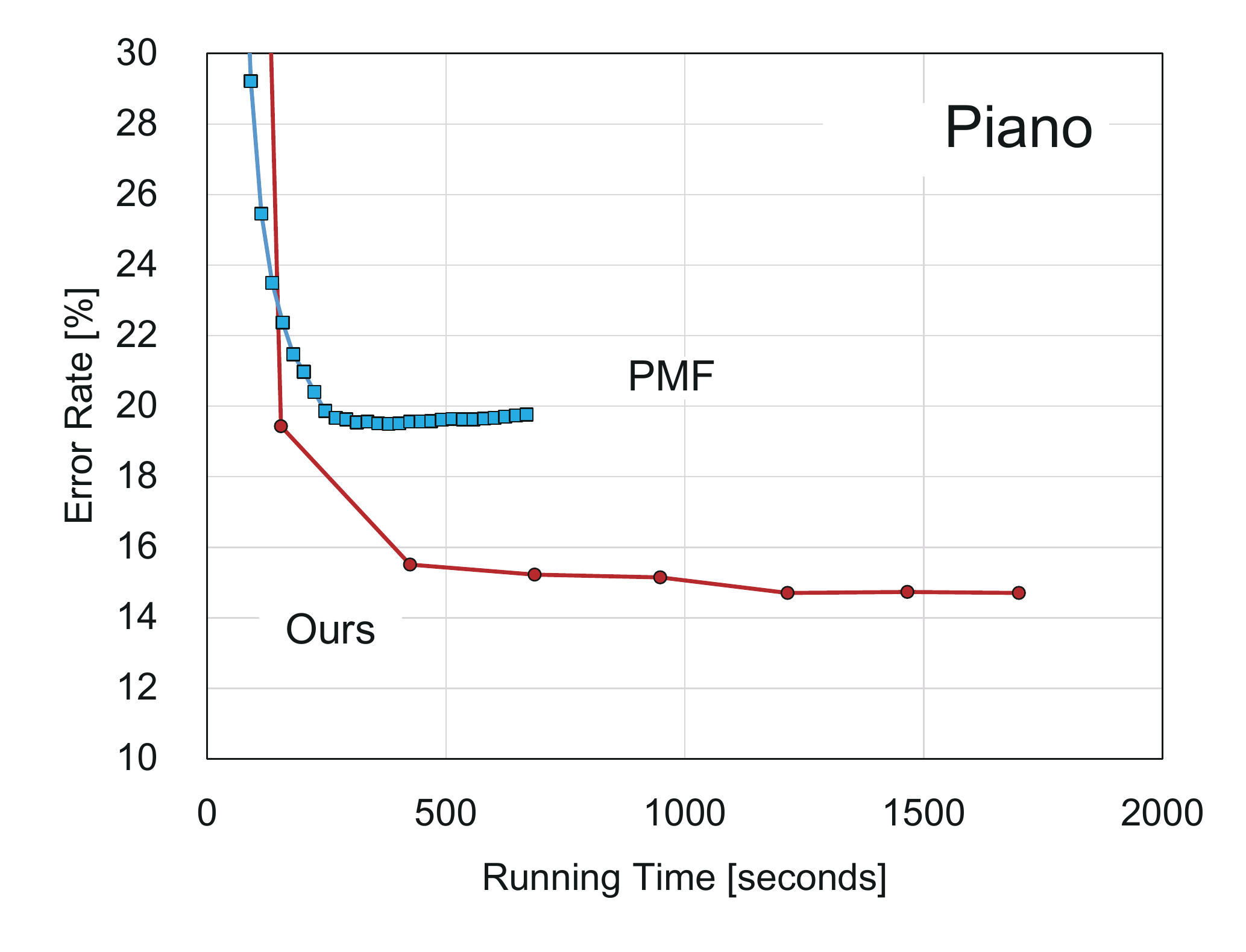}\\
		\includegraphics[width=0.3\textwidth]{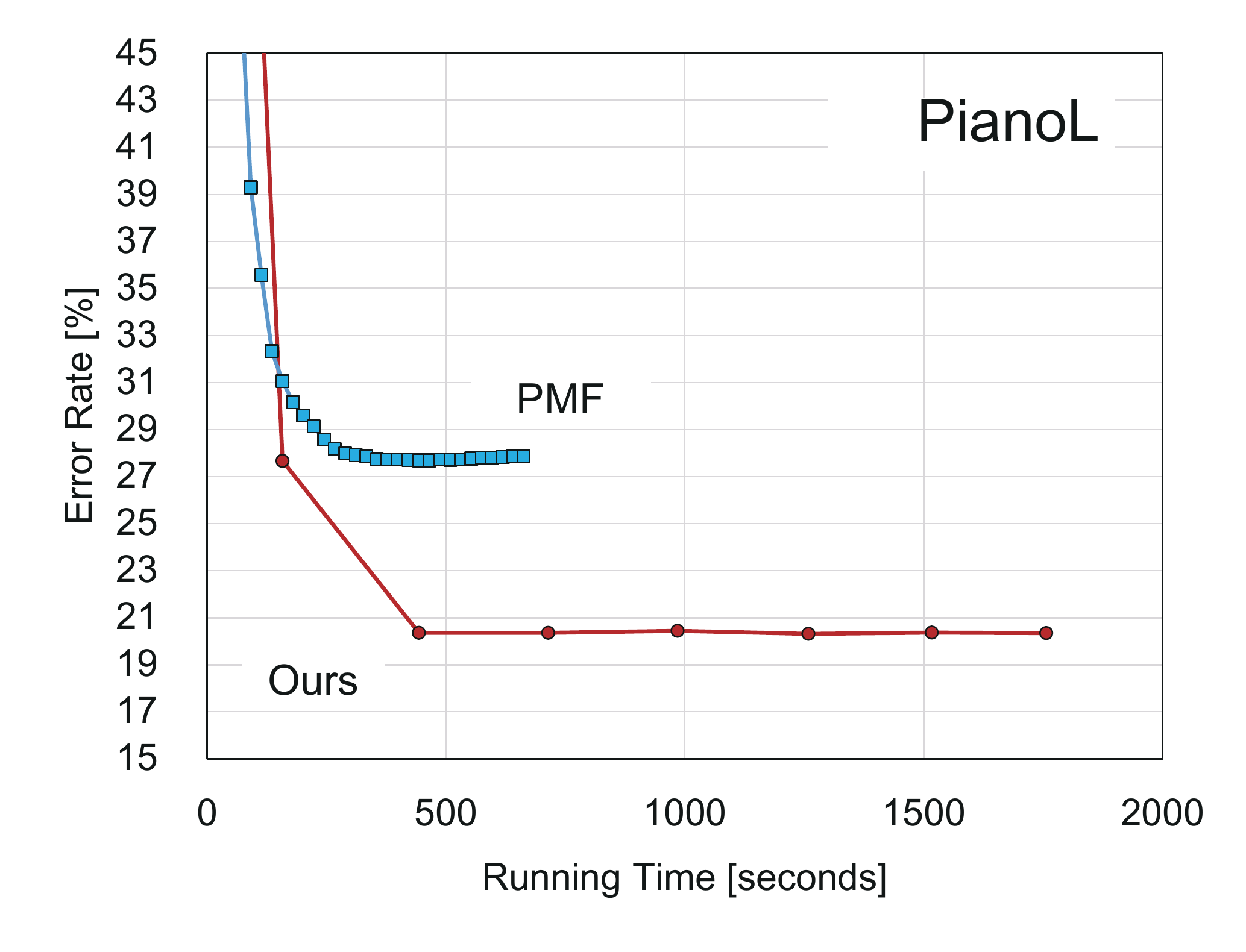}\hfil
		\includegraphics[width=0.3\textwidth]{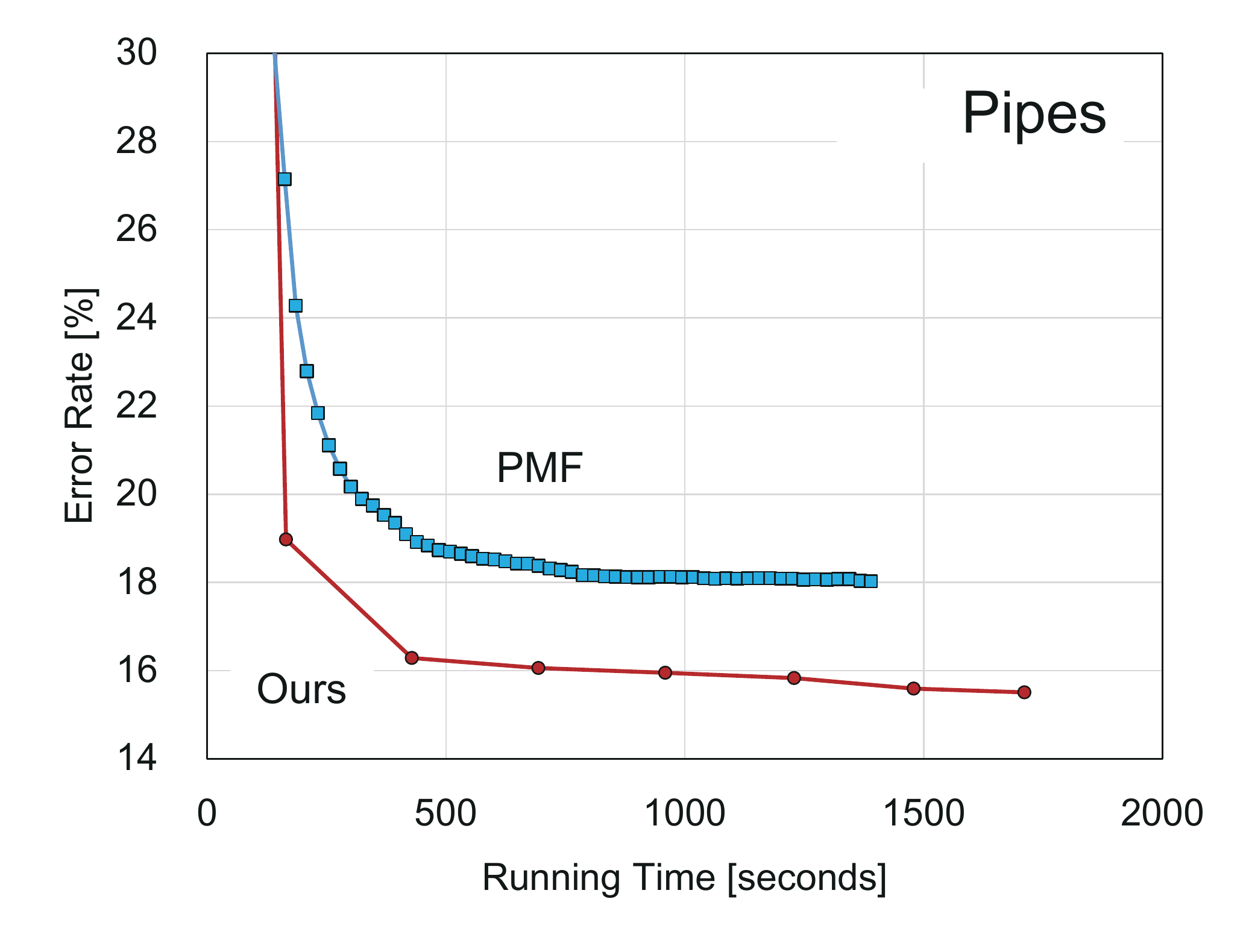}\hfil
		\includegraphics[width=0.3\textwidth]{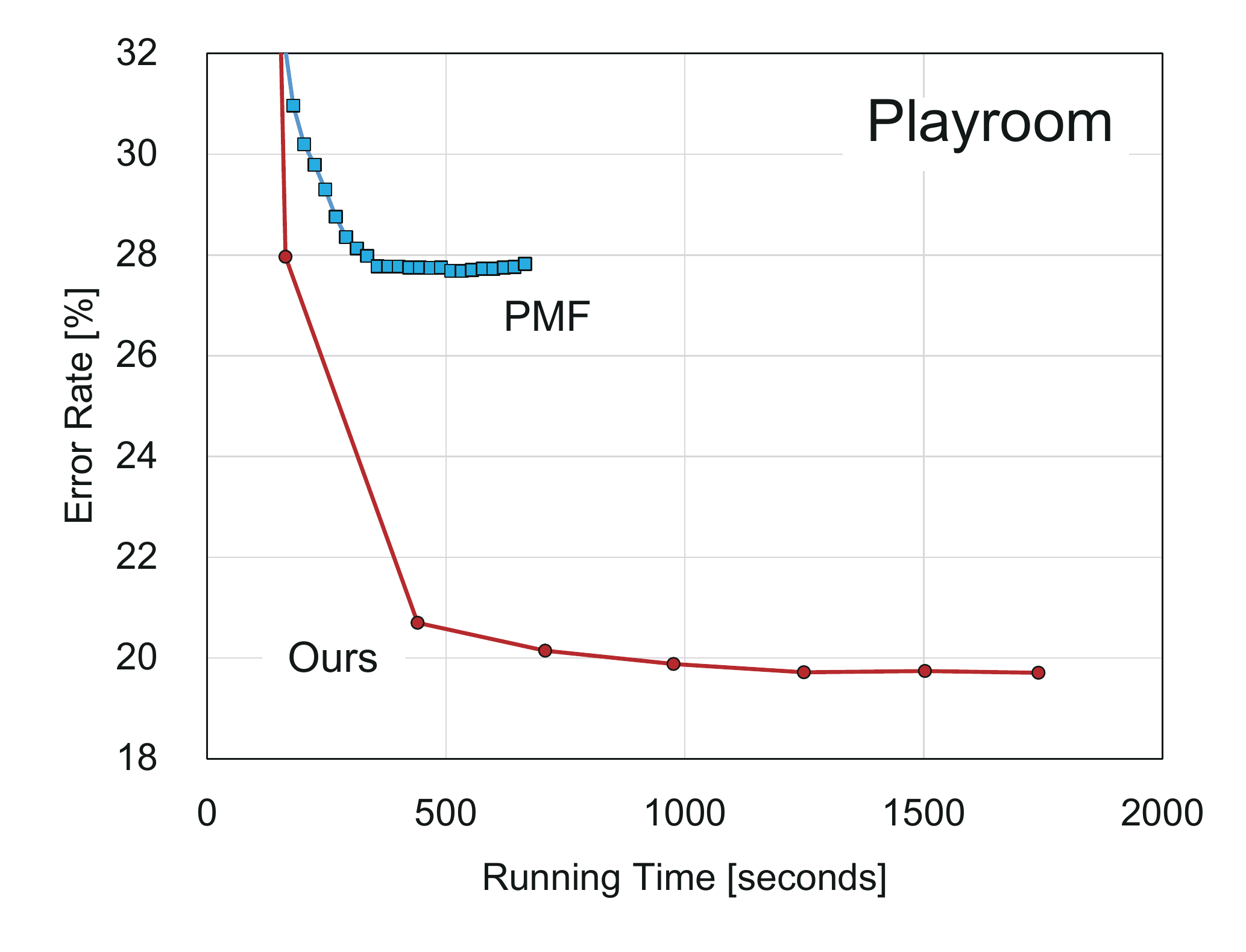}\\
		\includegraphics[width=0.3\textwidth]{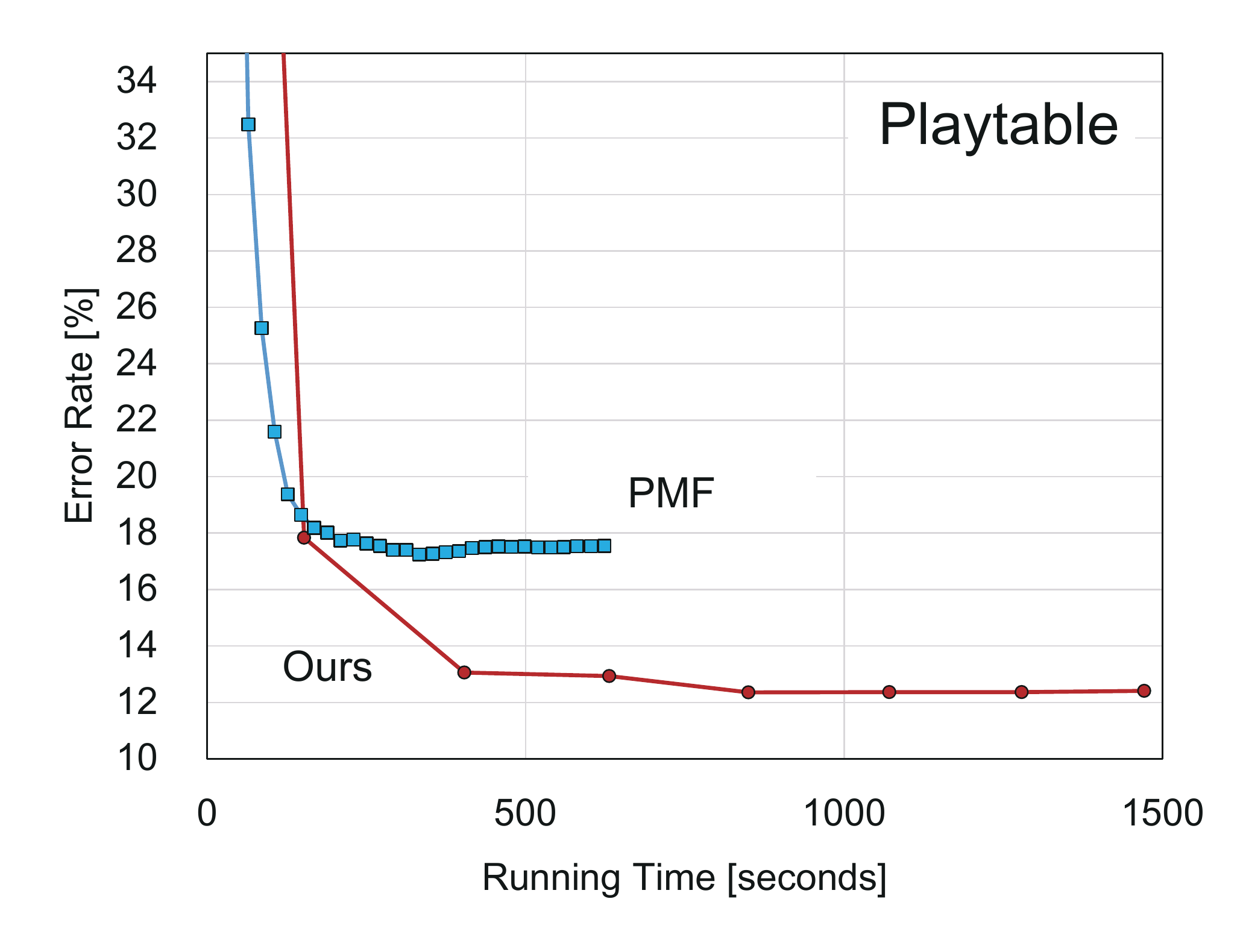}\hfil
		\includegraphics[width=0.3\textwidth]{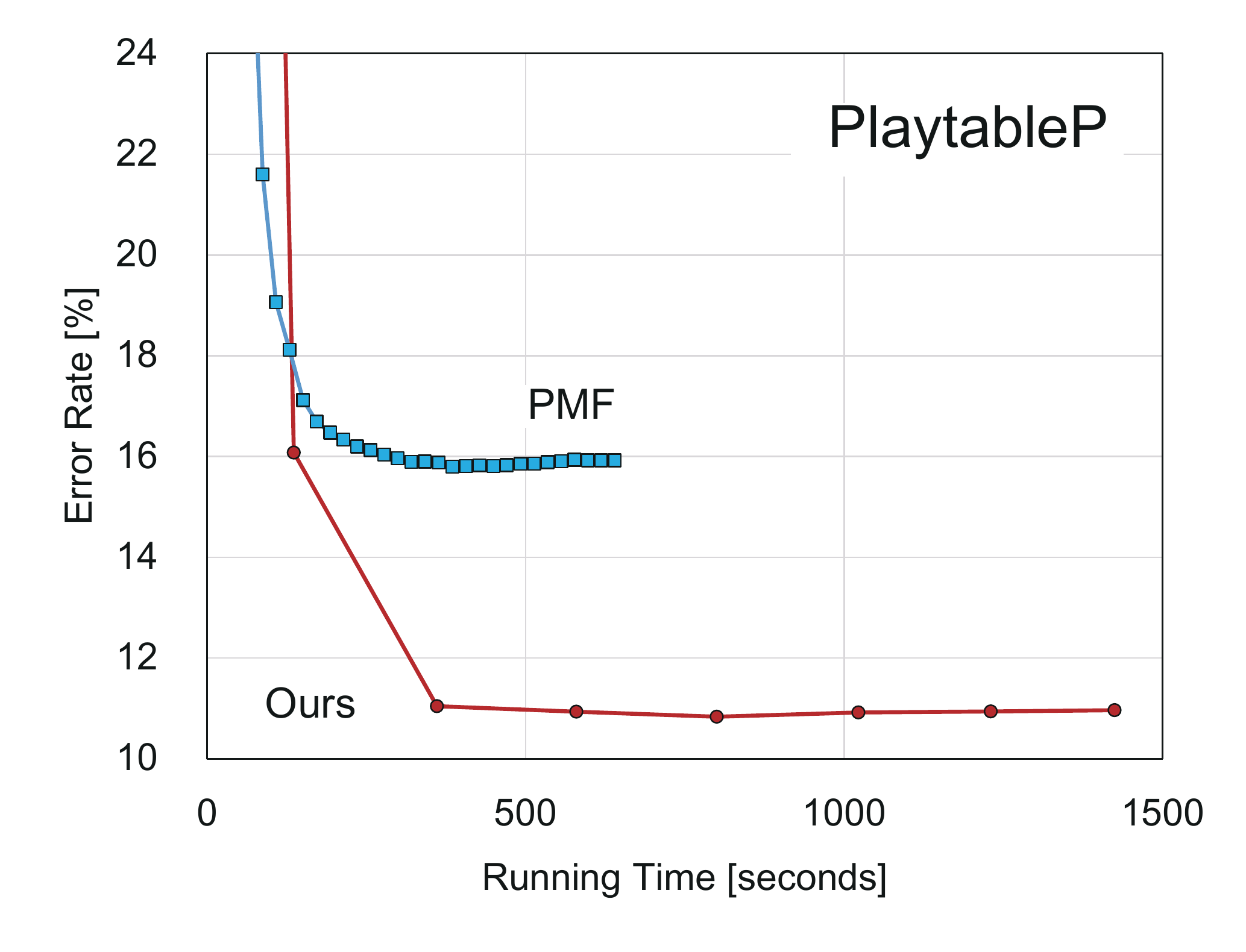}\hfil
		\includegraphics[width=0.3\textwidth]{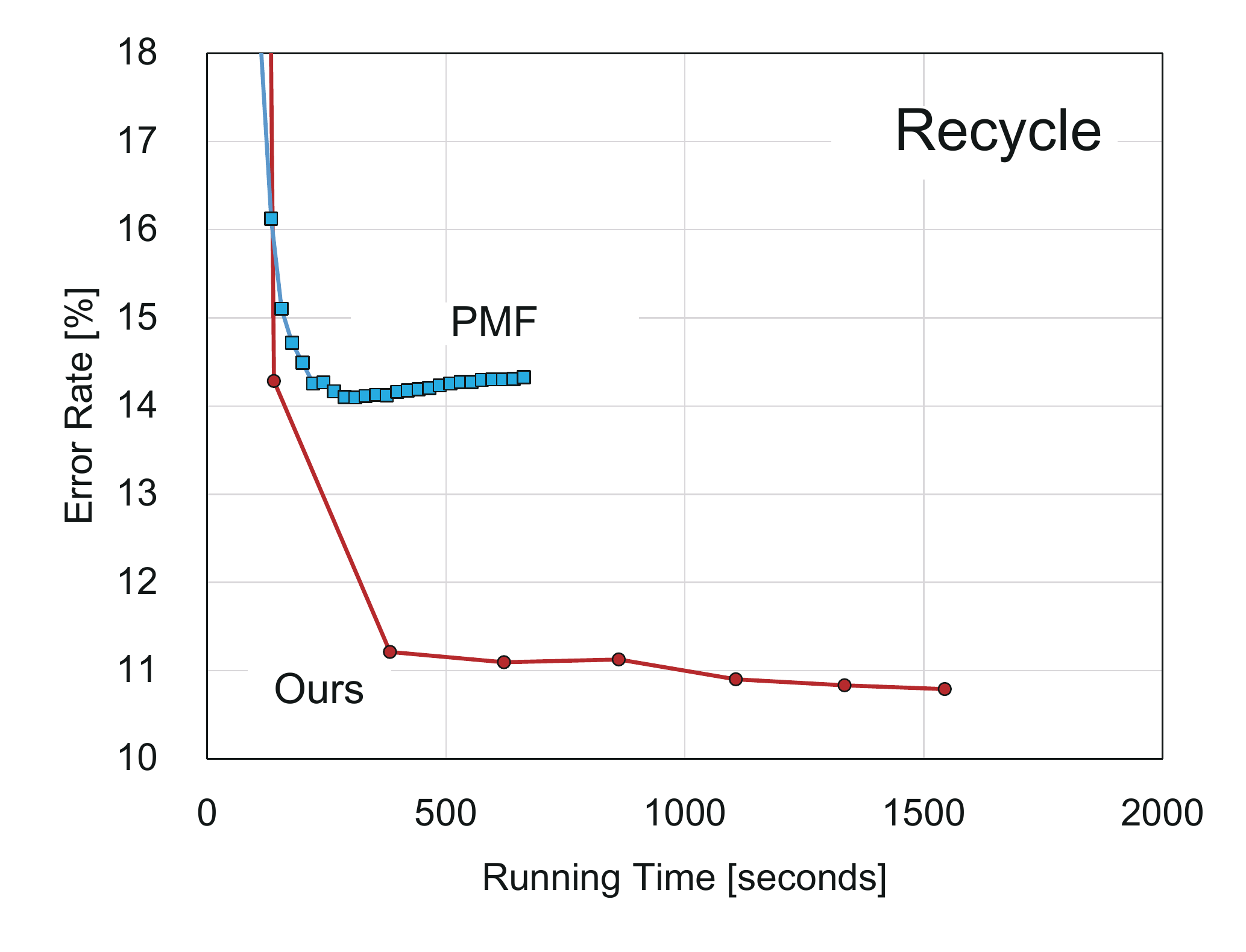}\\
		\includegraphics[width=0.3\textwidth]{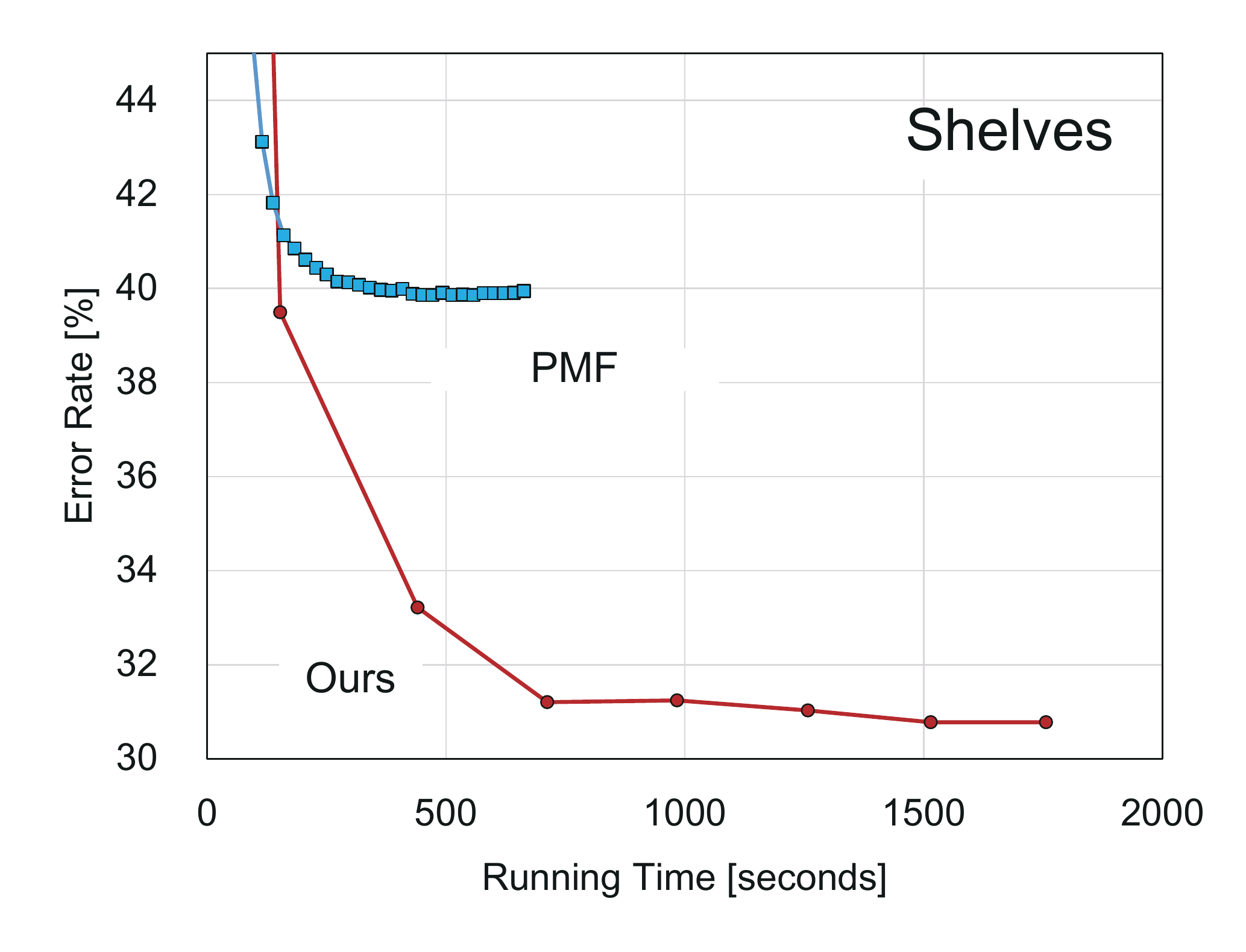}\hfil
		\includegraphics[width=0.3\textwidth]{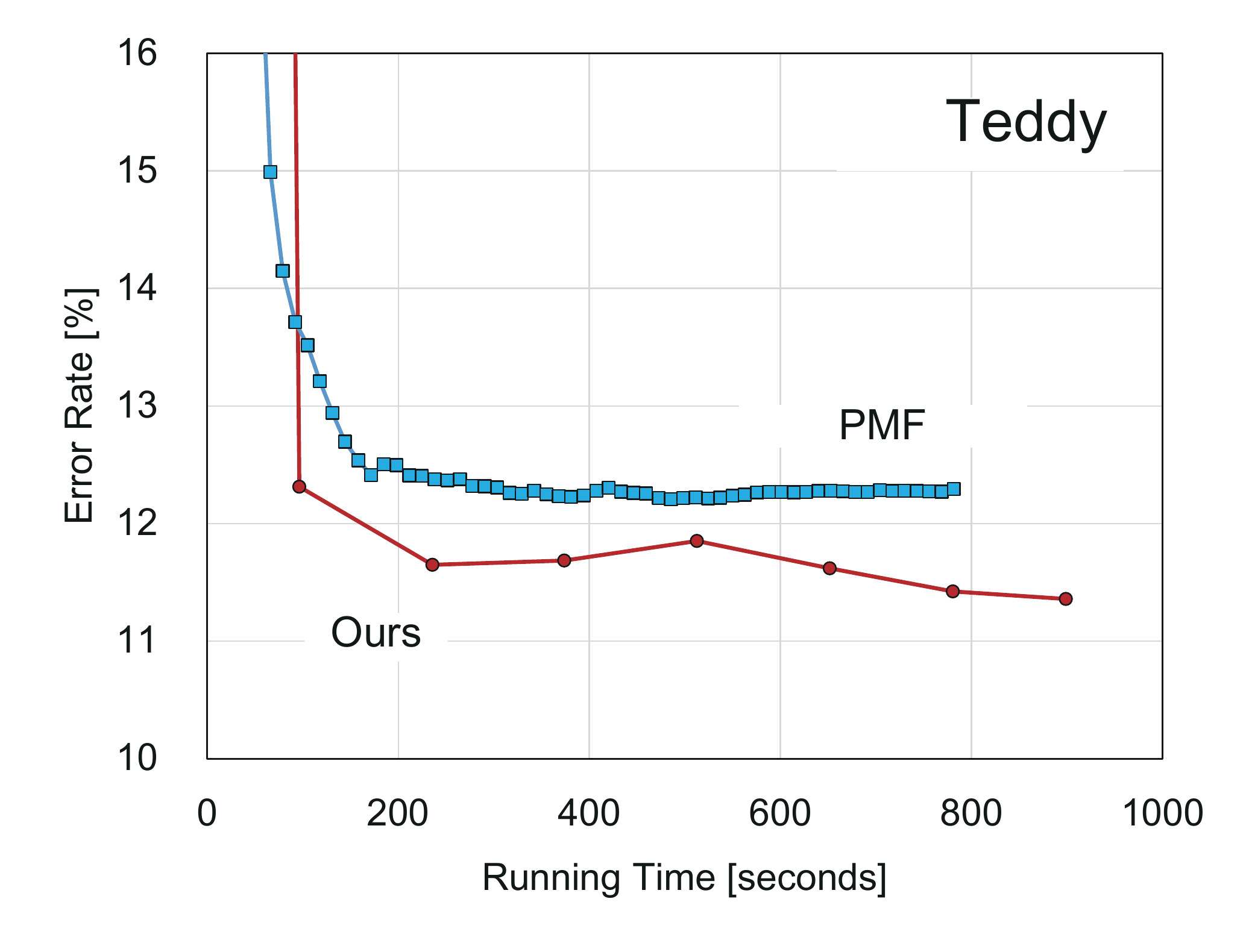}\hfil
		\includegraphics[width=0.3\textwidth]{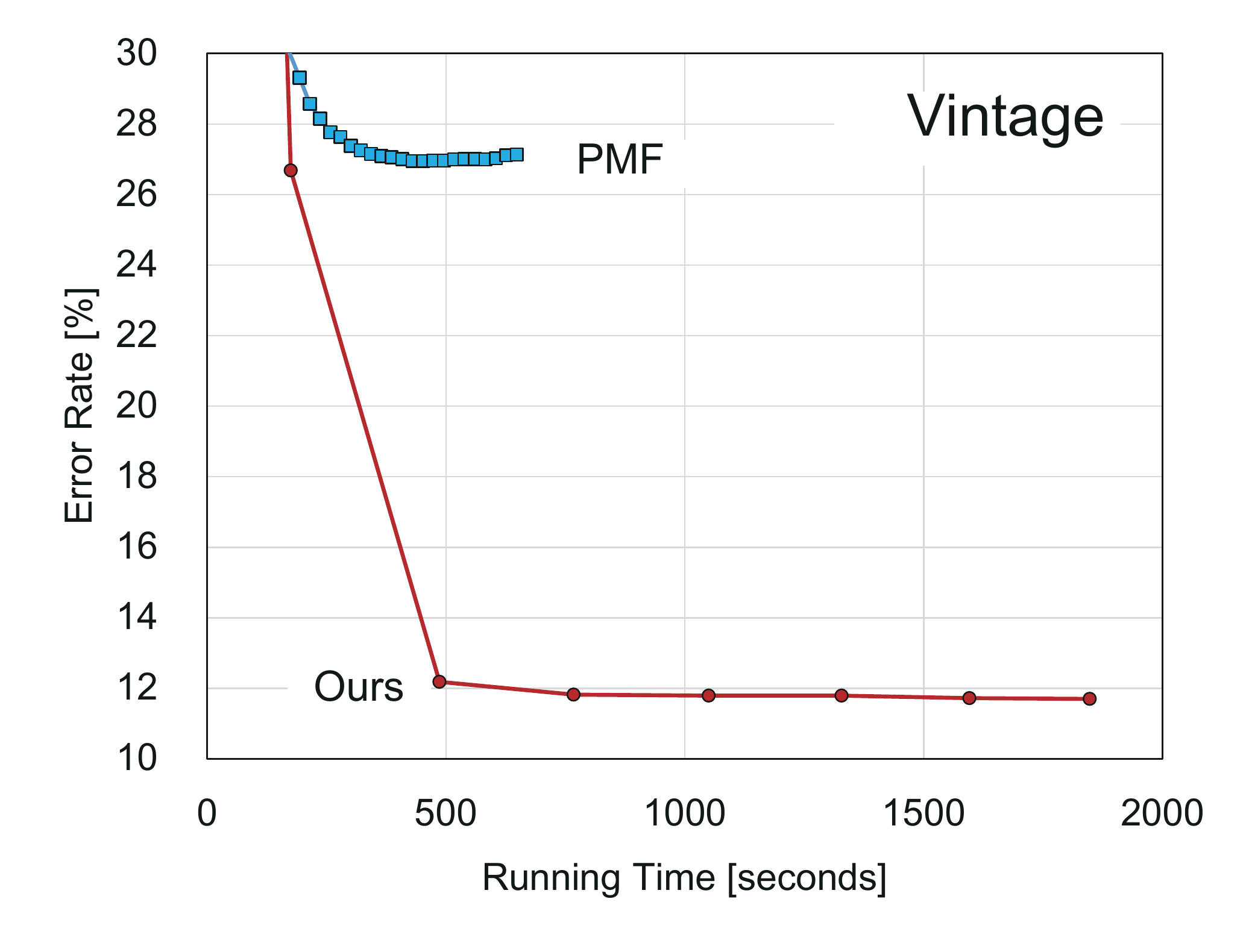}
	\end{center}
	\caption{Convergence comparison with PMF~\cite{Lu13} on 15 image pairs from the Middlebury V3 training dataset. Error rates are evaluated by the \emph{bad 2.0} metric for all regions. Our method consistently outperforms PMF.
	Because performances of PMF by $K=300, 500, 700$ had very similar trends, we here only show results by $K=500$ (default setting) to avoid cluttered profiles.
	Both methods are run on a single CPU core using the same energy function in Sec 4.2 without post-processing, but PMF can optimize only its data term.
	Note that while PMF mostly converged in 30 iterations, we did 60 iterations for some cases (\emph{Jadeplant}, \emph{Motorcycle}, \emph{MotorcycleE}, \emph{Pipes}, and \emph{Teddy}).
	}
	\label{figa:pmf}
\end{figure}

\begin{figure}
	\begin{center}
		\includegraphics[width=0.3\textwidth]{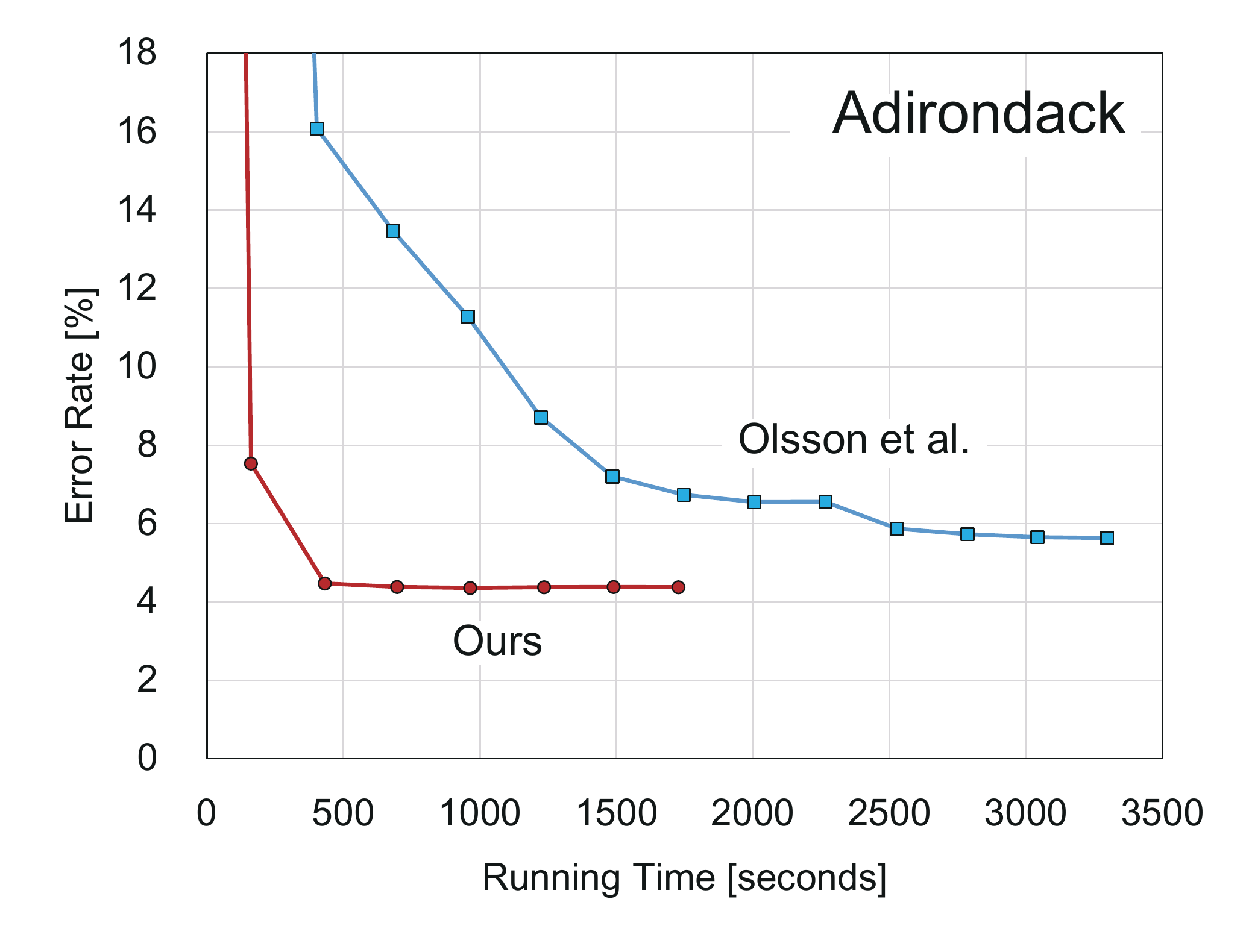}\hfil
		\includegraphics[width=0.3\textwidth]{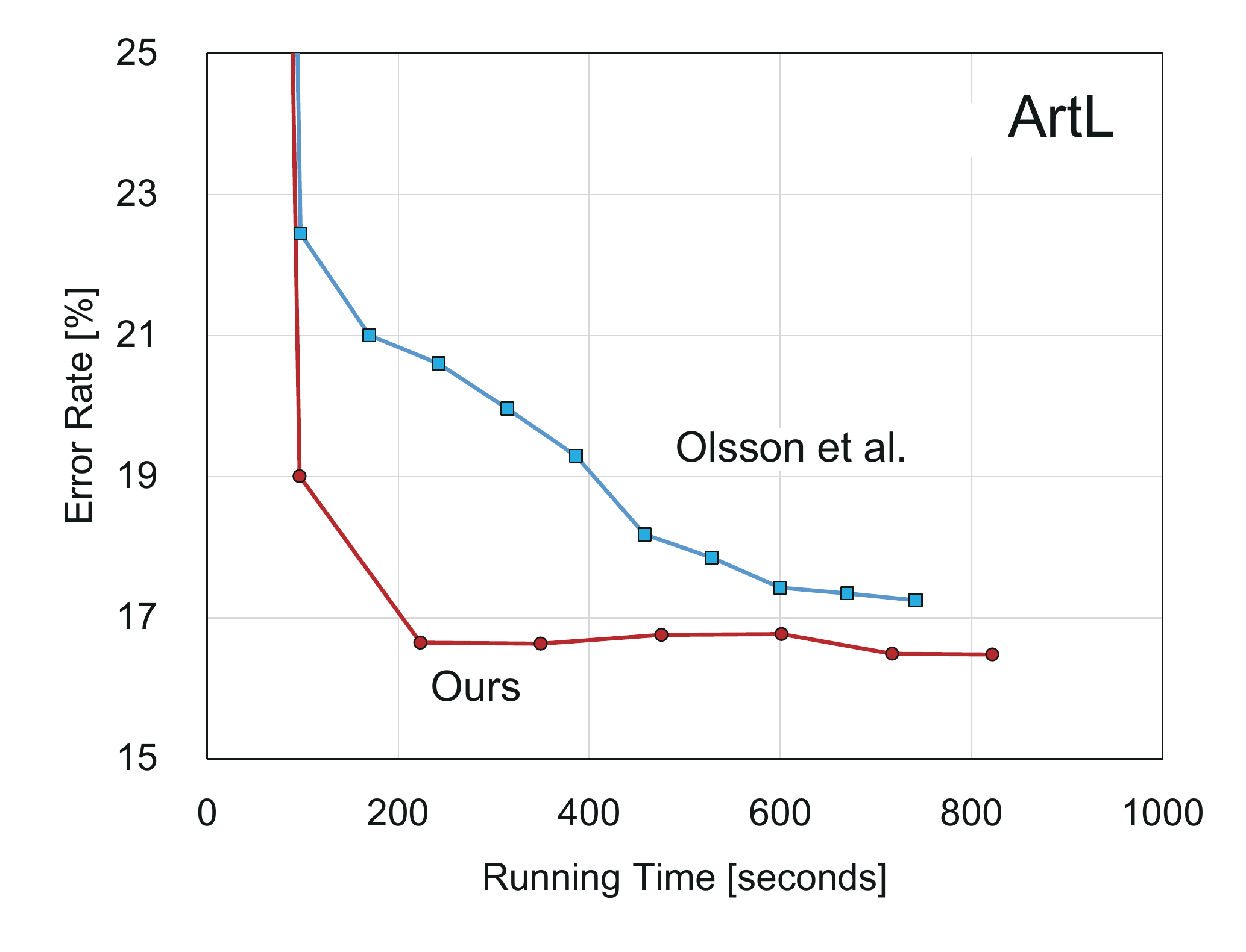}\hfil
		\includegraphics[width=0.3\textwidth]{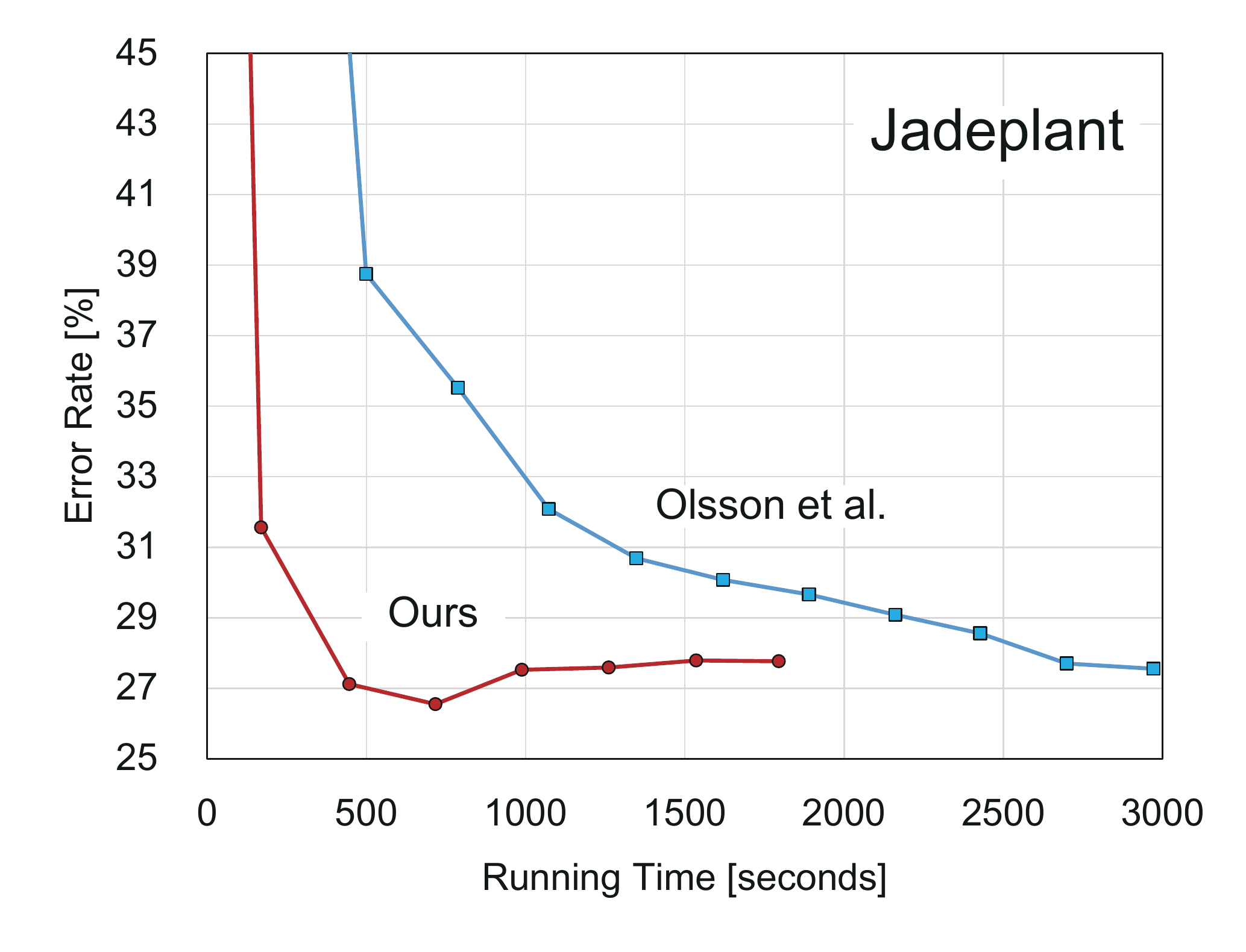}\\
		\includegraphics[width=0.3\textwidth]{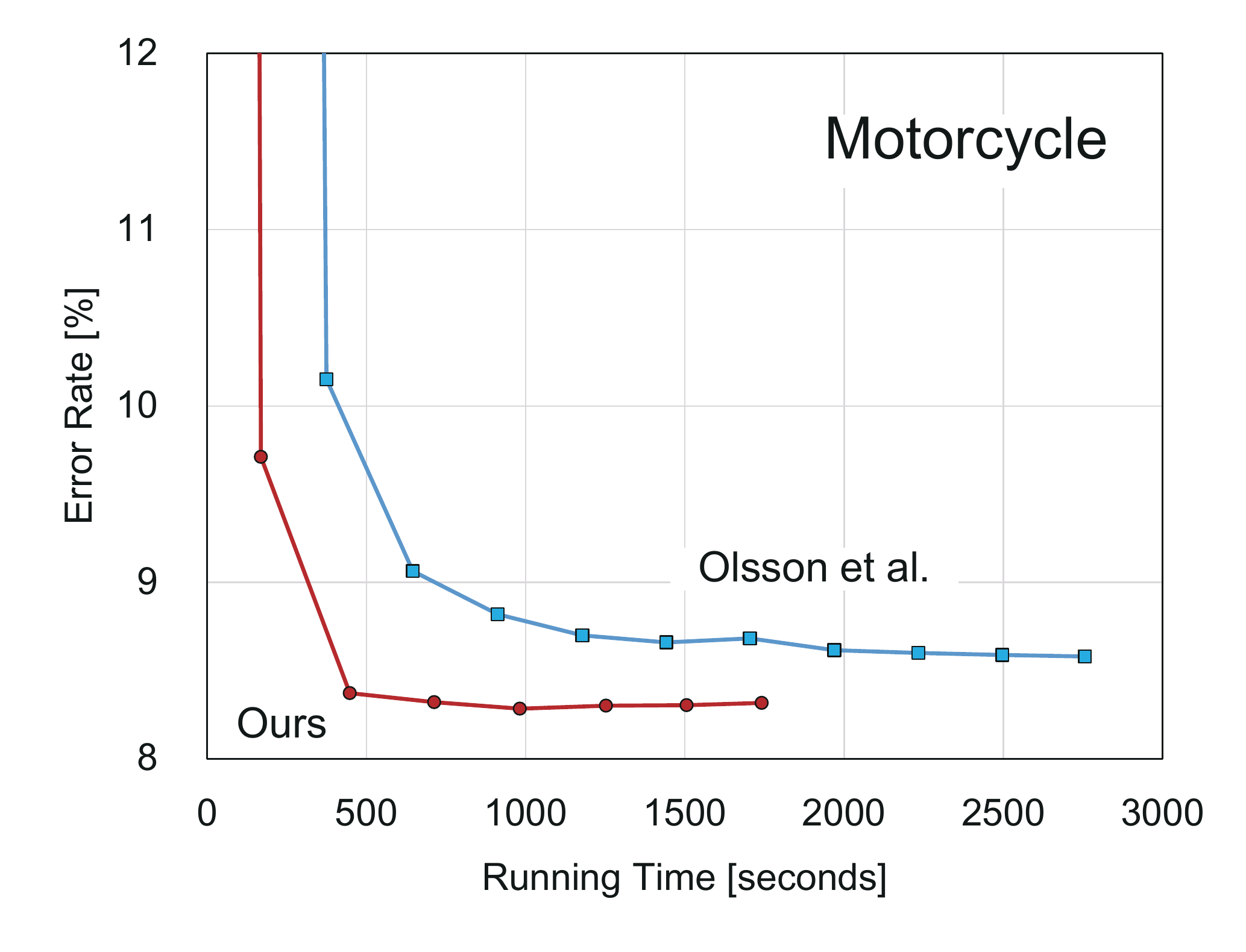}\hfil
		\includegraphics[width=0.3\textwidth]{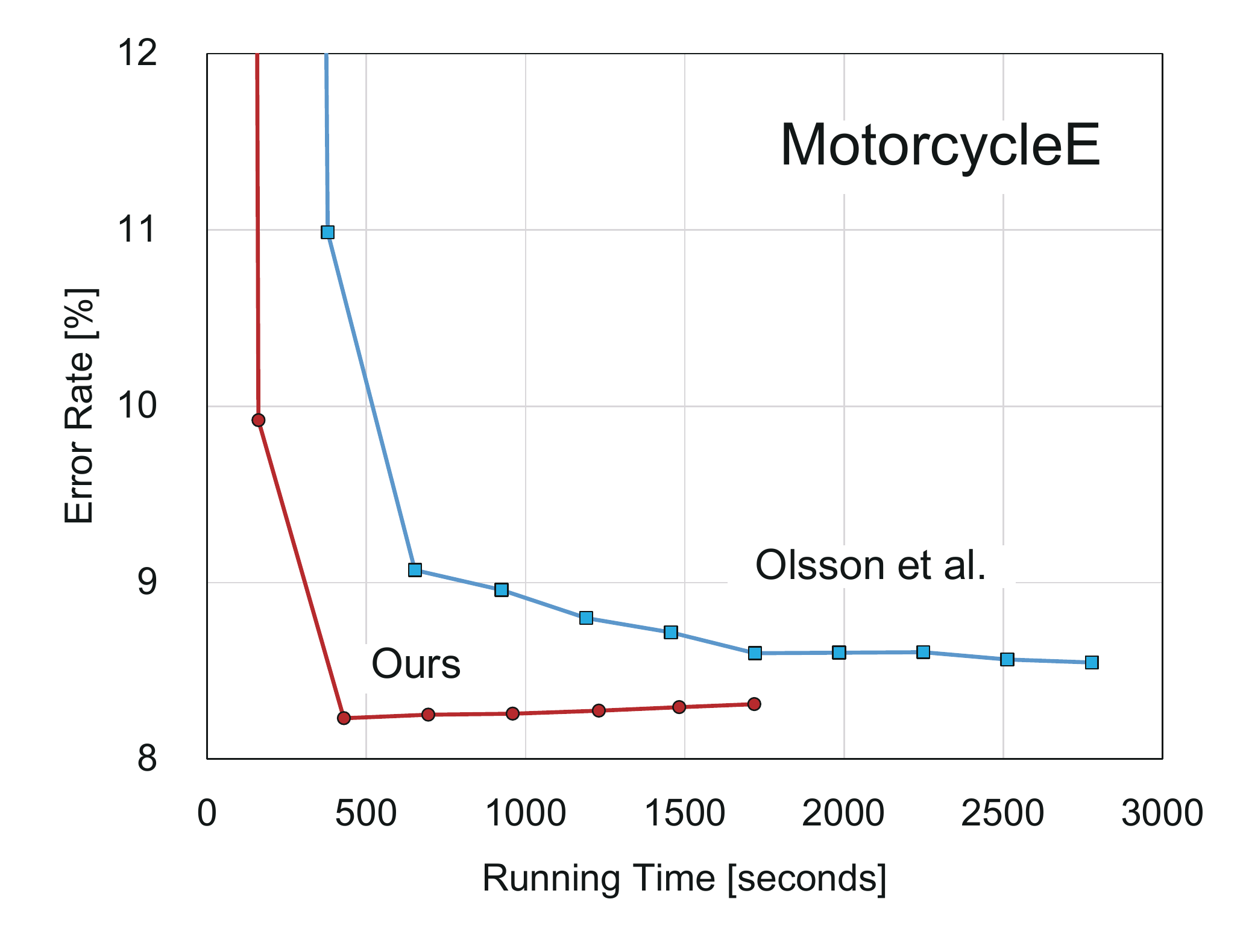}\hfil
		\includegraphics[width=0.3\textwidth]{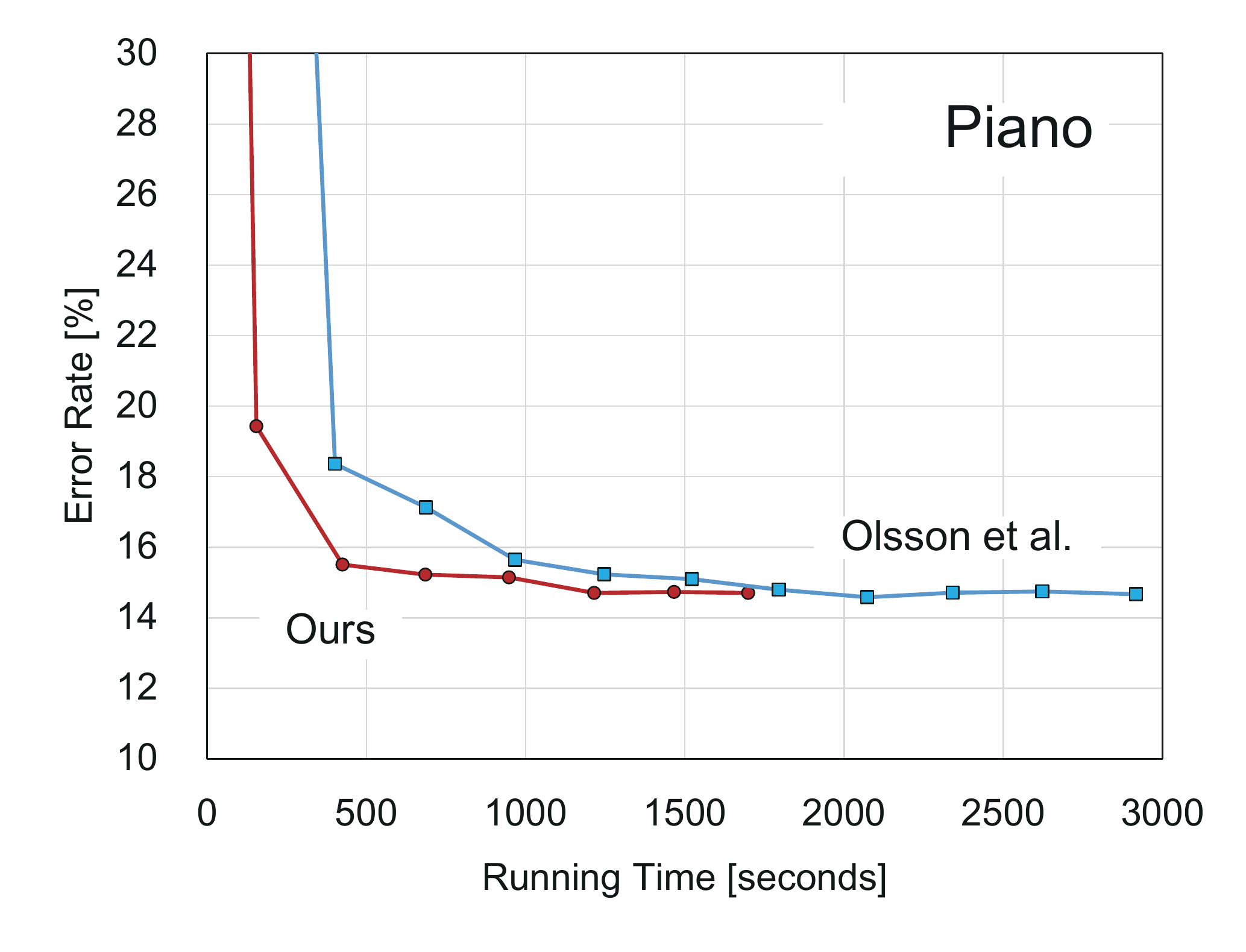}\\
		\includegraphics[width=0.3\textwidth]{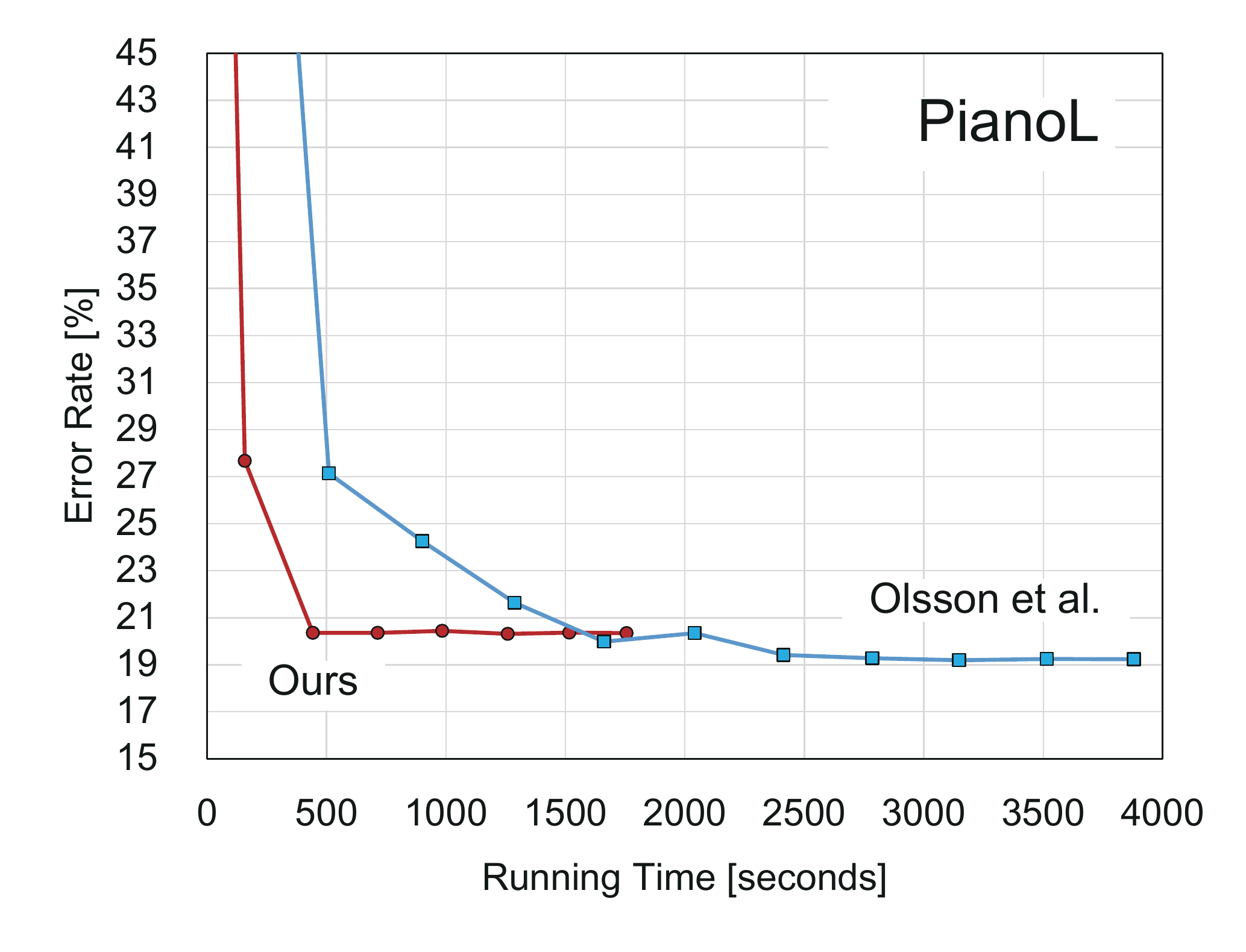}\hfil
		\includegraphics[width=0.3\textwidth]{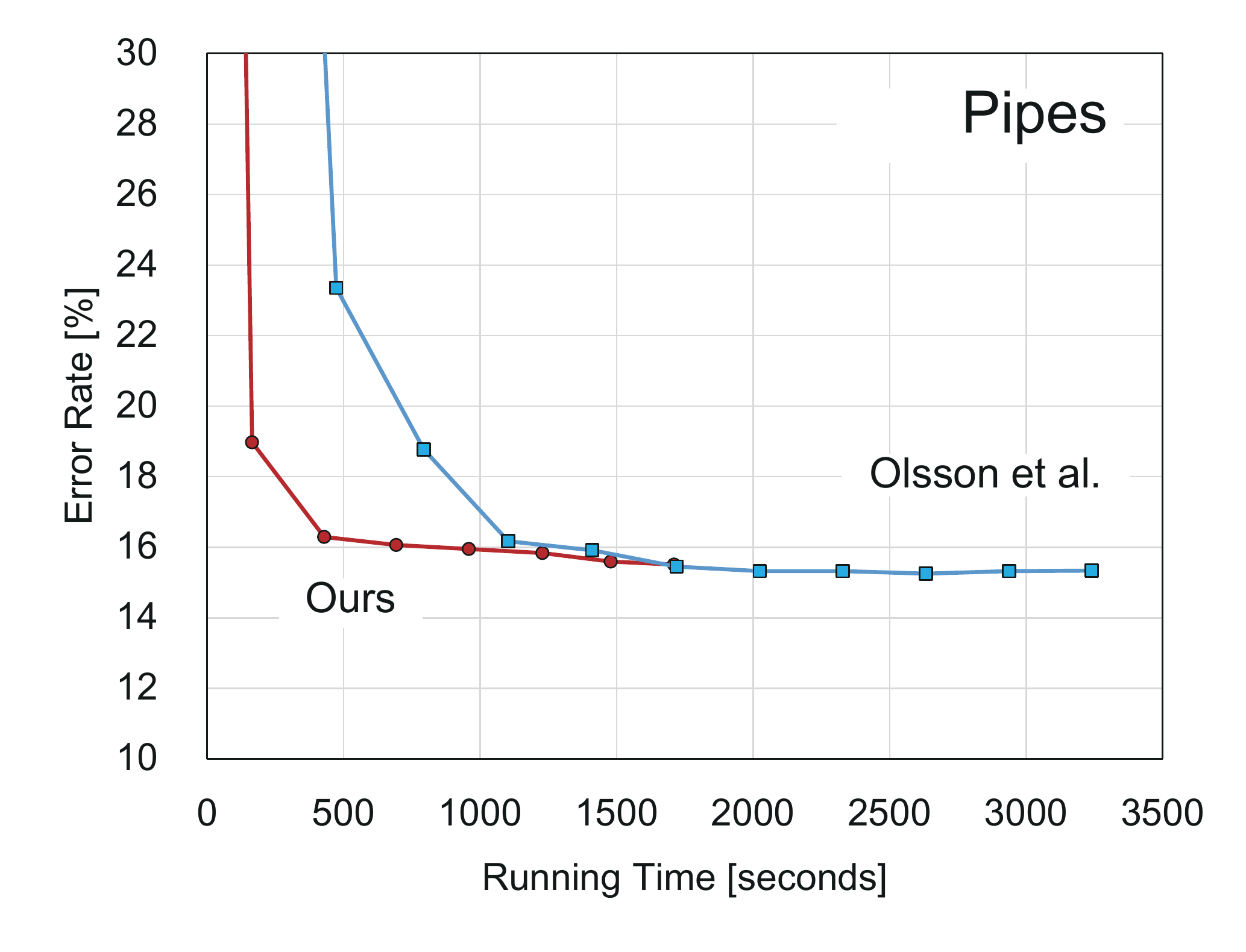}\hfil
		\includegraphics[width=0.3\textwidth]{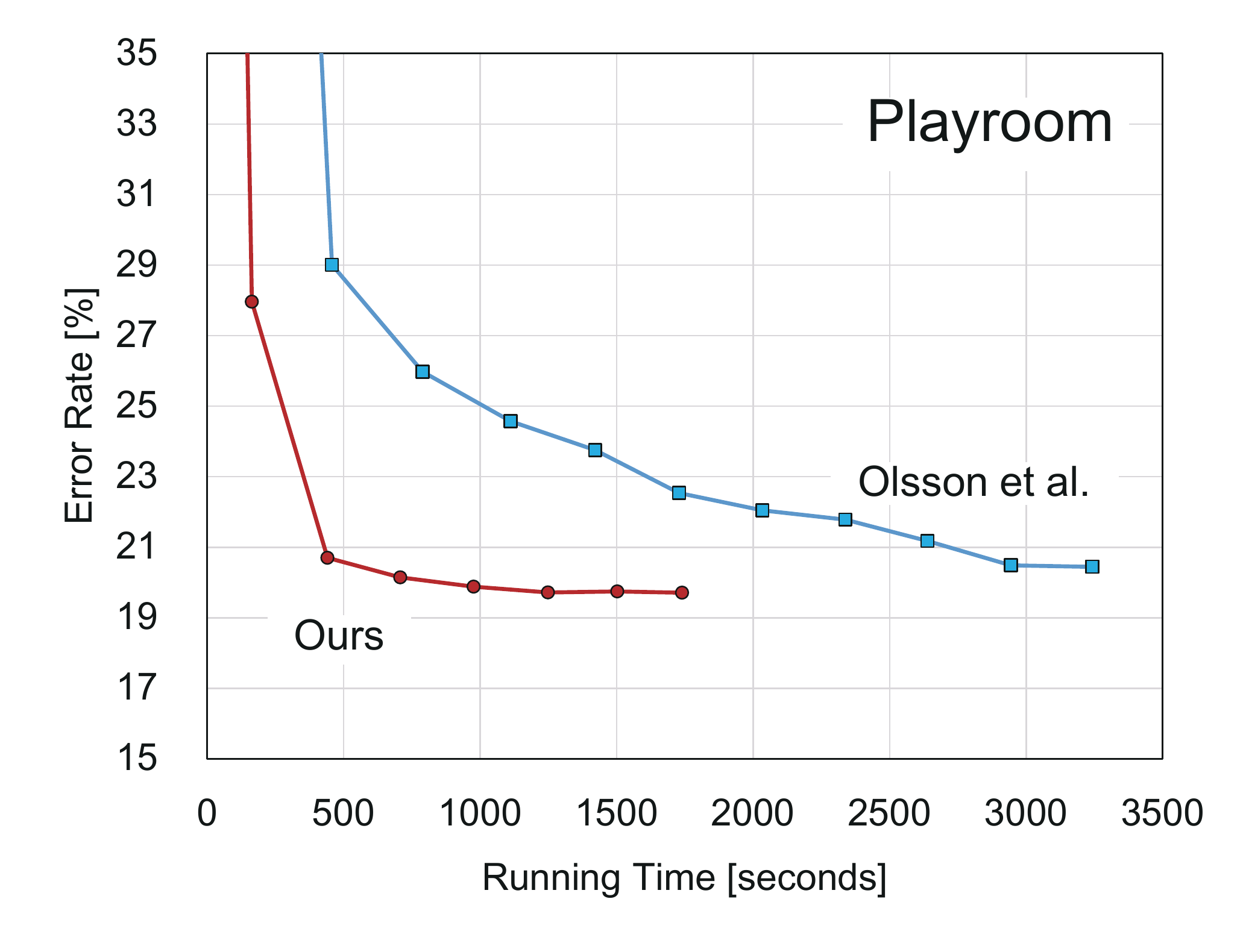}\\
		\includegraphics[width=0.3\textwidth]{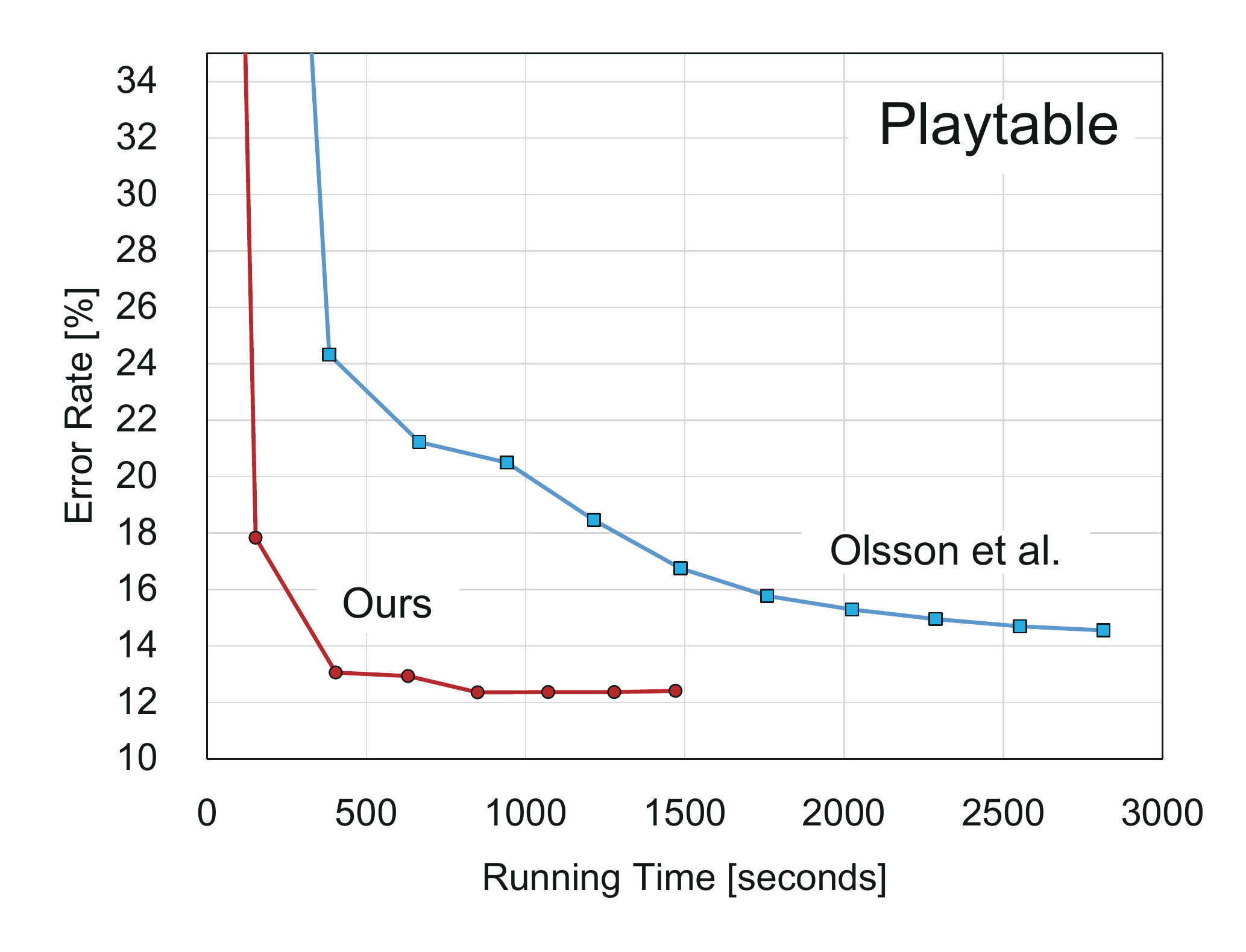}\hfil
		\includegraphics[width=0.3\textwidth]{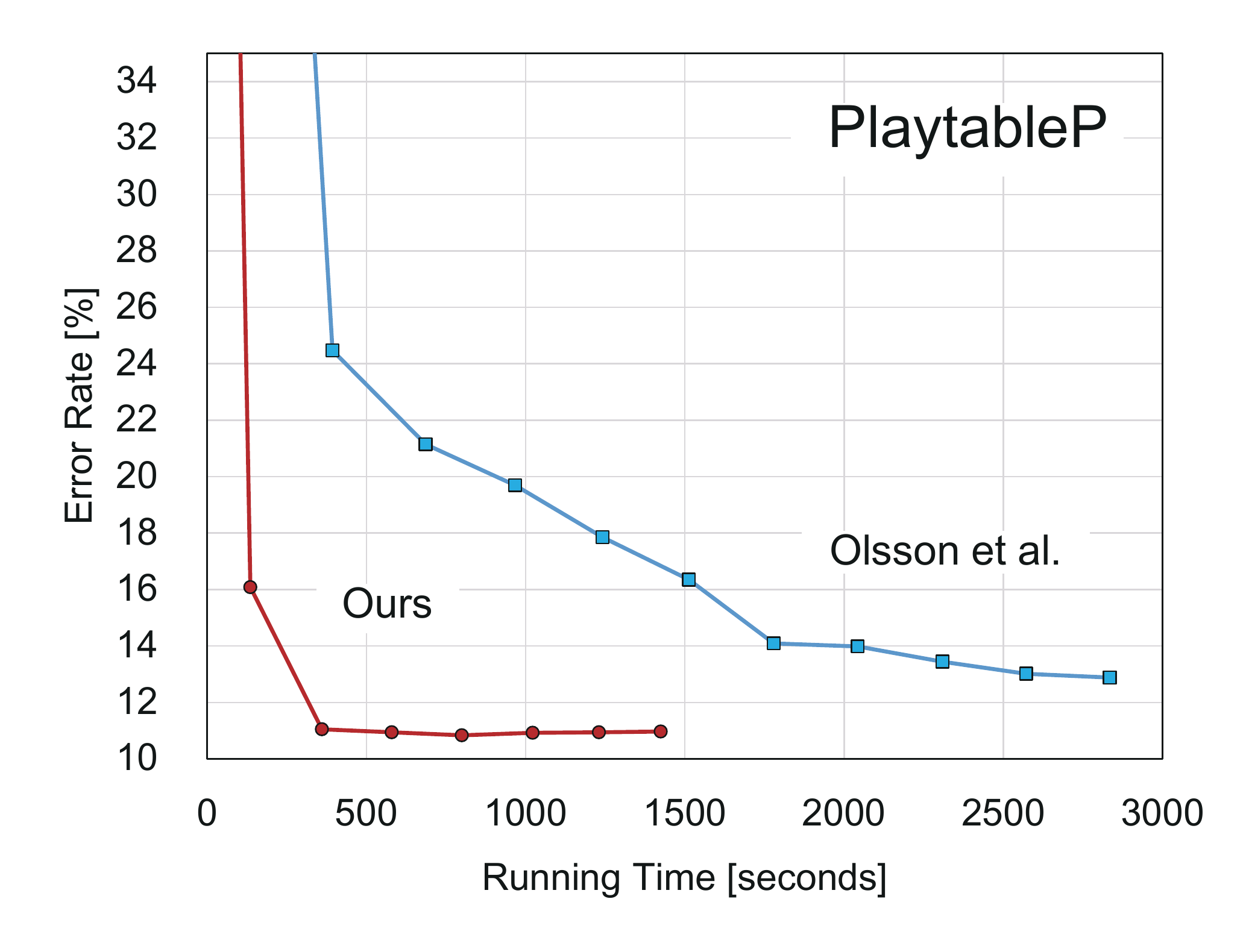}\hfil
		\includegraphics[width=0.3\textwidth]{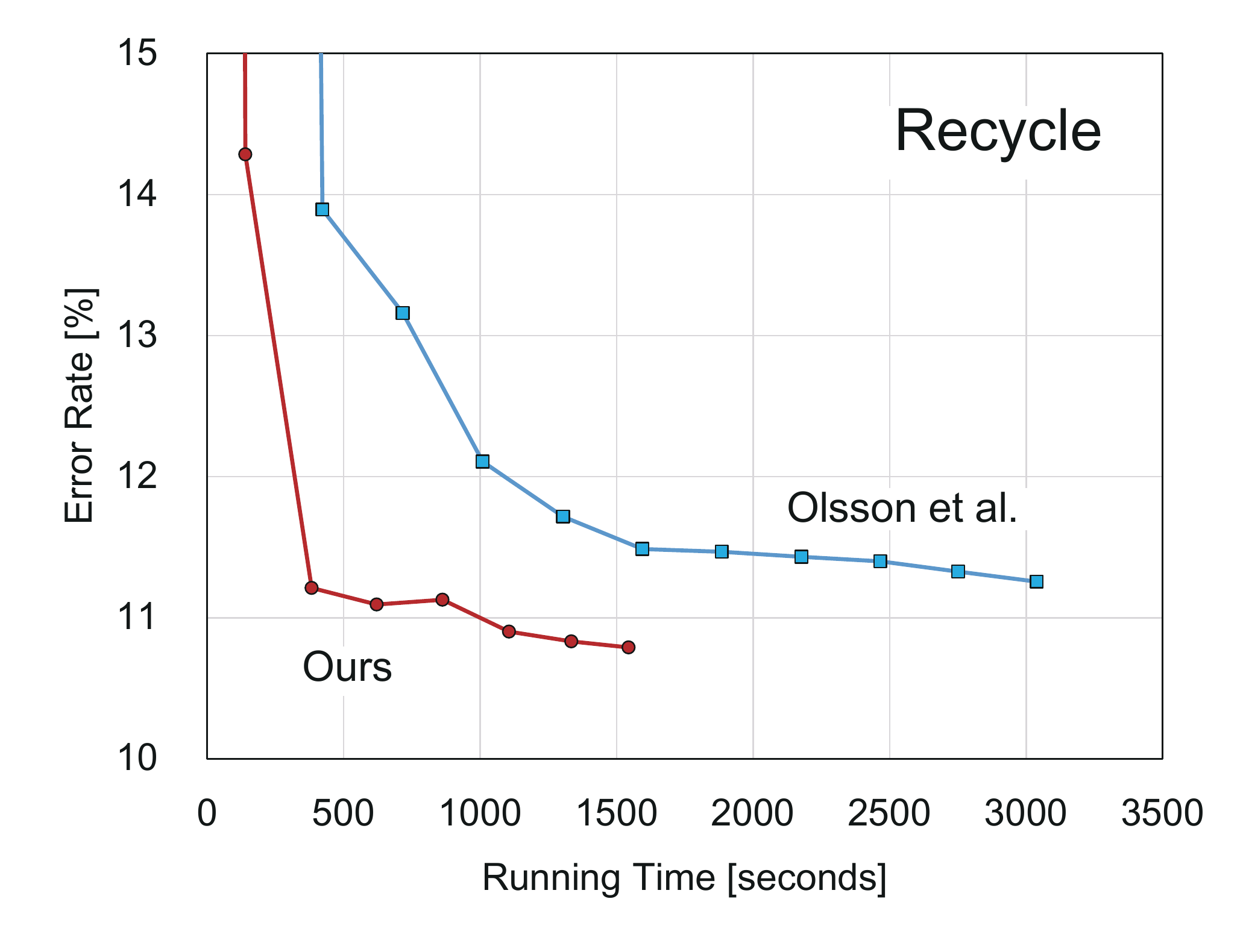}\\
		\includegraphics[width=0.3\textwidth]{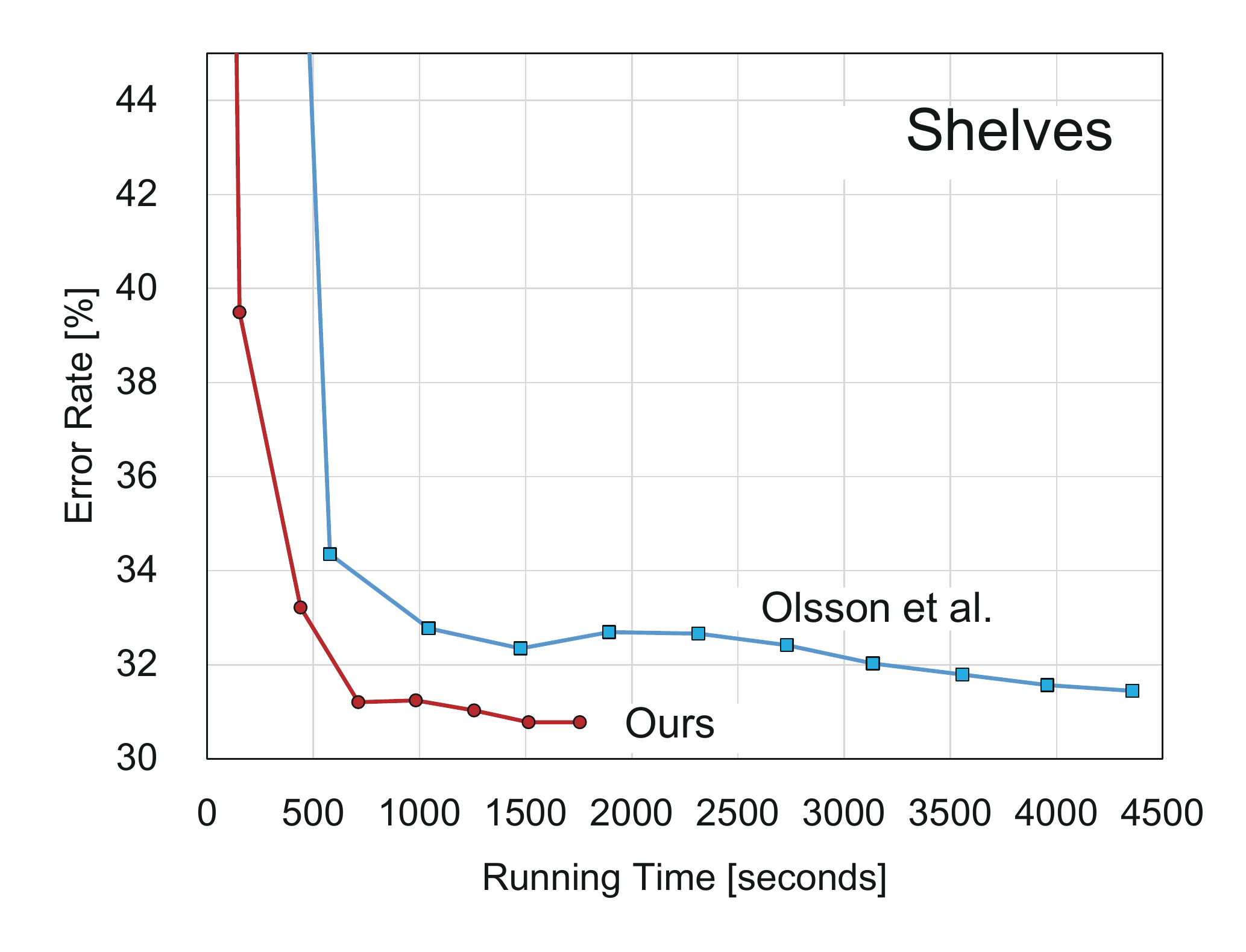}\hfil
		\includegraphics[width=0.3\textwidth]{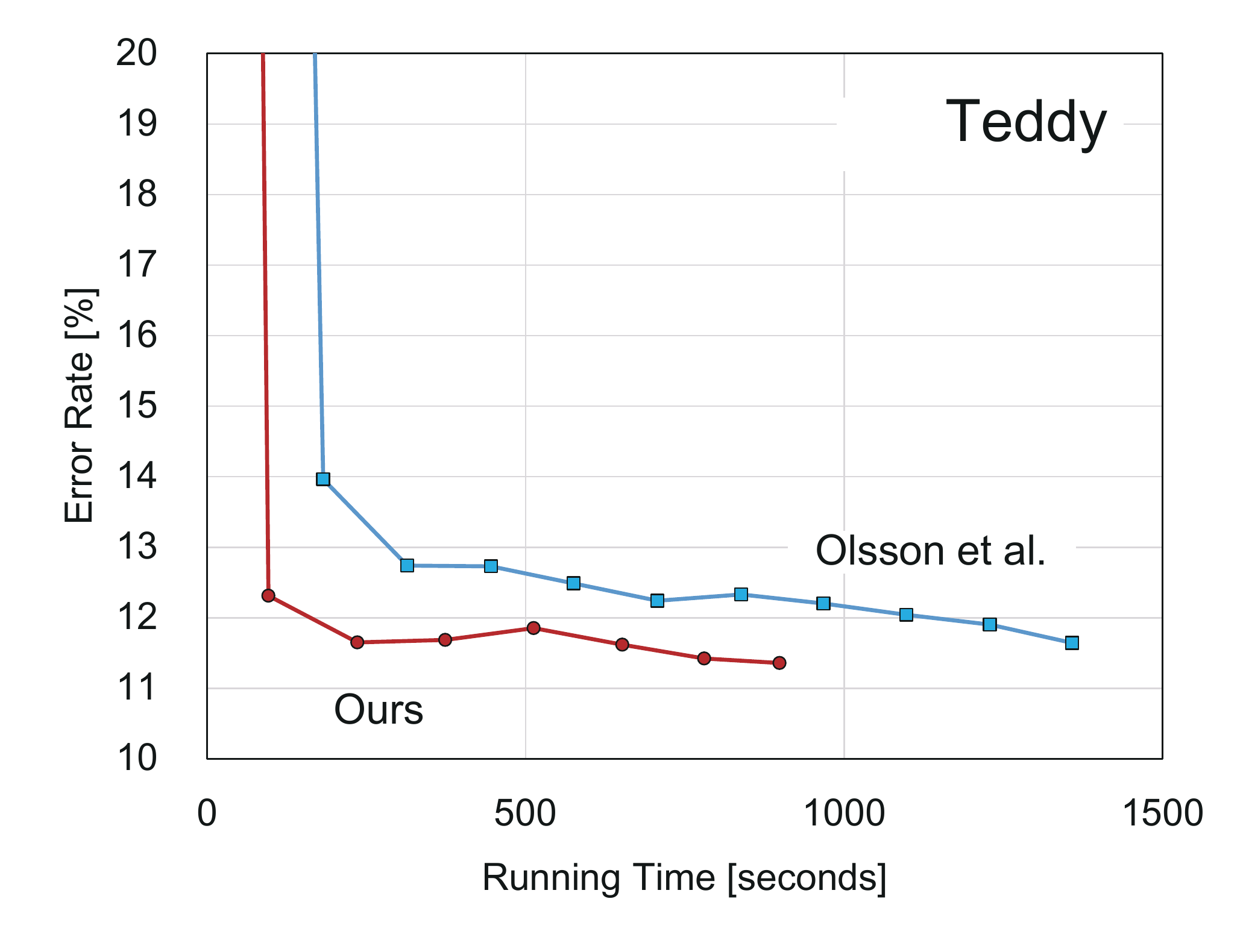}\hfil
		\includegraphics[width=0.3\textwidth]{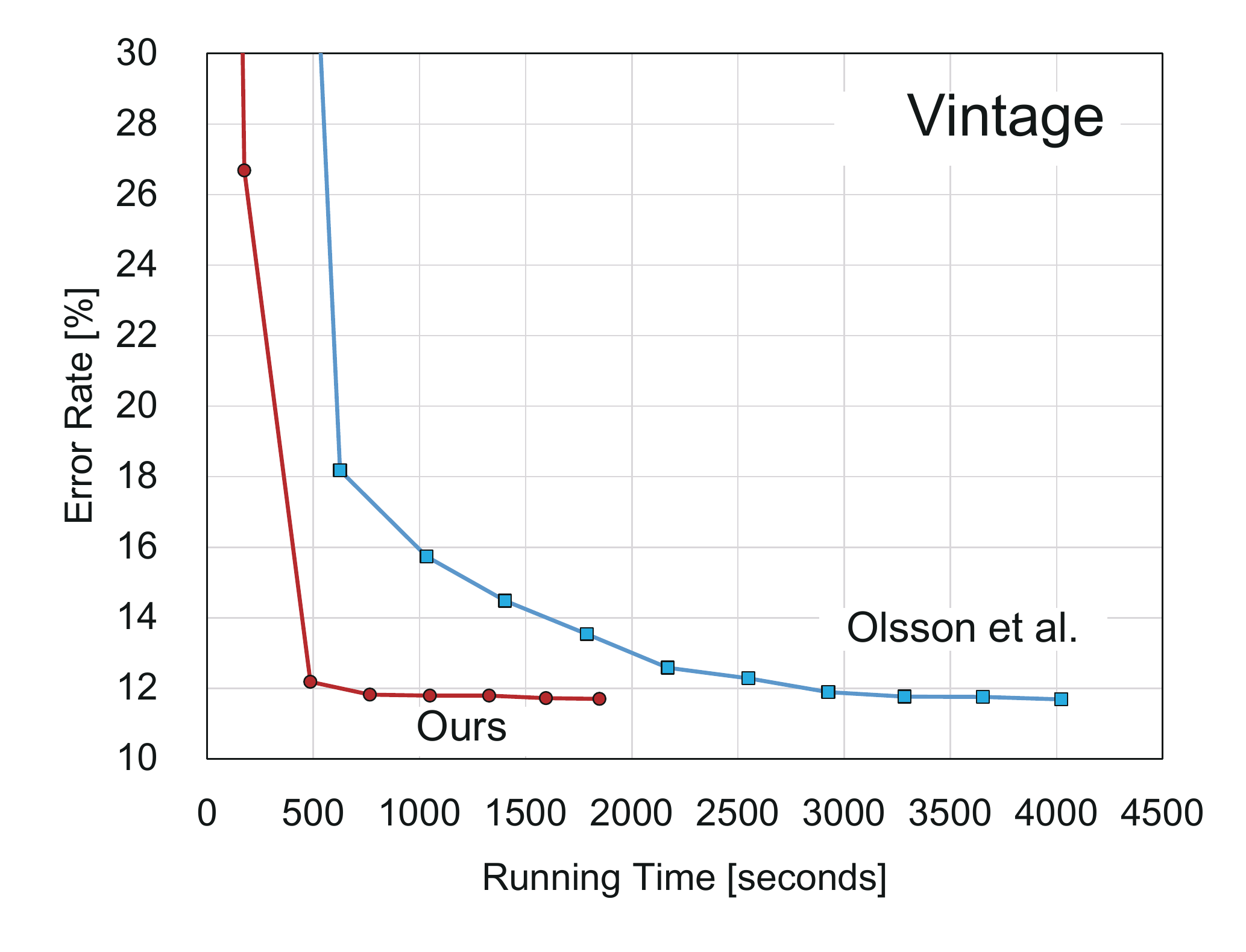}
	\end{center}
\caption{Convergence comparison with the optimization method used in Olsson~\etal~\cite{Olsson13} on 15 image pairs from the Middlebury V3 training dataset. Error rates are evaluated by the \emph{bad 2.0} metric for all regions. For most of the image pairs, our method is about 6x faster to reach comparable or better accuracies than Olsson~\etal~\cite{Olsson13}. Both methods are run on a single CPU core to optimize the same energy function used in Sec~4.2 without post-processing. Our method can be further accelerated by parallelization as demonstrated in Fig.~\ref{figa:efficiency}.}
	\label{figa:olsson}
\end{figure}

%{
%	\bibliographystyle{ieee}
%	\bibliography{egbib}
%}

\end{document}